\documentclass{article} 
\usepackage{iclr2027_conference,times}

\usepackage{amsmath,amsfonts,bm}

\def\eqref#1{equation~\ref{#1}}

\def\1{\bm{1}}

\DeclareMathAlphabet{\mathsfit}{\encodingdefault}{\sfdefault}{m}{sl}
\SetMathAlphabet{\mathsfit}{bold}{\encodingdefault}{\sfdefault}{bx}{n}

\usepackage{hyperref}
\usepackage{url}

\title{Contextual Distributionally Robust Optimization with Causal and Continuous Structure}

\author{
Fenglin Zhang \\
School of Artificial Intelligence \\
The Chinese University of Hong Kong, Shenzhen \\
\texttt{fenglinzhang1@link.cuhk.edu.cn} \\
\And
Jie Wang \thanks{Corresponding author.} \\
School of Artificial Intelligence, School of Data Science \\
The Chinese University of Hong Kong, Shenzhen \\
\texttt{jwang@cuhk.edu.cn}
}
\iclrfinalcopy

\begin{document}

\renewcommand{\eqref}[1]{(\ref{#1})}

\maketitle

\begin{abstract}
We propose a framework for contextual distributionally robust optimization (DRO) that considers the causal and continuous structure of the underlying distribution, and we develop interpretable and tractable decision rules. We first introduce the causal Sinkhorn discrepancy (CSD), an entropy-regularized causal Wasserstein distance that encourages continuous transport plans while preserving causal consistency. We then formulate a contextual DRO model with a CSD-based ambiguity set, termed Causal Sinkhorn DRO (Causal-SDRO), and derive its strong dual reformulation, where the worst-case distribution is characterized as a mixture of Gibbs distributions. To obtain an (infinite-dimensional) optimal policy, we propose a soft regression forest (SRF) decision rule: it preserves the interpretability of classical decision trees while being fully parametric, differentiable, and Lipschitz-smooth, enabling intrinsic interpretation from both global and local perspectives. To solve Causal-SDRO with parametric decision rules, we develop an efficient stochastic compositional gradient algorithm that converges to an $\varepsilon$-stationary point with $\mathcal{O}(\varepsilon^{-4})$ iteration complexity and $\mathcal{O}(\varepsilon^{-5})$ sample complexity. Finally, numerical results on synthetic and real-world datasets demonstrate the superior performance and interpretability of our method. 
\end{abstract}

\section{Introduction}

Contextual stochastic optimization (CSO) is a widely applied method in real-world engineering and business decision-making \citep{sadana2025survey}. 
It assumes that decision-makers have access to historical data, including uncertain parameters and associated contextual information~\citep[called covariates or features,][]{chenreddy2022data}.
Such contextual information enables a more precise characterization of uncertain parameters, leading to better decisions~\citep{ban2019big, bertsimas2020predictive}. 
The goal of CSO is to seek an optimal policy that maps covariates to decisions, thereby avoiding repeatedly solving optimization models for each new covariate~\citep{sadana2025survey}. 

To hedge against distributional shifts in practice, the contextual distributionally robust optimization (DRO) method has received much attention. 
This approach seeks optimal robust decision rules that minimize the worst-case risk over an \emph{ambiguity set} containing all plausible joint distributions of features and uncertain parameters. 
Extensive work constructs ambiguity sets based on Wasserstein distance~\citep{kannan2024residuals, nguyen2020distributionally, yang2026decision}. 
However, these typically yield discrete worst-case distributions, which may lead to overly conservative or unrealistic decisions since the underlying distribution is typically continuous. 
Another critical yet often overlooked aspect in contextual DRO is the \textit{causal information structure}: future covariates are conditionally independent of observed uncertain parameters given historical covariates. 
It thus remains an open question how to develop a contextual DRO model that simultaneously captures the causal relationship and the absolute continuity of the worst-case distribution. 
\citet{wang2025sinkhorn} develop a DRO framework based on the Sinkhorn discrepancy to characterize the continuity of underlying distributions. 
This framework incorporates entropic regularization into the Wasserstein distance, thereby excluding discrete distributions from the ambiguity set~\citep{blanchet2023unifying, zhang2026generative}. 

While the Sinkhorn DRO can provide continuous worst-case distributions, the infinite-dimensional nature of policy optimization poses a computational challenge. 
To address this, recent literature has developed both parametric and non-parametric decision rules to seek effective approximate policies. 
For parametric rules, \citet{ban2019big} consider computationally efficient but locally optimal affine decision rules. Later,~\citet{han2025deep} and~\citet{qi2023practical} propose deep learning-based decision rules for general problems, but these rules are difficult to interpret. 
Alternatively, non-parametric rules developed in~\citet{nguyen2025robustifying} and~\citet {zhang2024optimal} offer greater interpretability but apply only to specific problems. 
These lead to the following open question. 
\vspace{-0.8em}
\begin{quote}
\textit{
\centering
    How to develop computationally tractable and interpretable decision rules for general contextual DRO problems?
}   
\end{quote}
\vspace{-0.6em}
The tree-based family, including decision trees and random forests, has been developed for general CSO models to estimate conditional distributions of uncertain parameters and improve interpretability~\citep{bertsimas2020predictive, elmachtoub2020decision, kallus2023stochastic}. 
However, few studies employ these models for decision rule optimization. 
This is primarily because constructing an optimal tree is NP-Hard, and its non-differentiable structure precludes efficient end-to-end policy optimization~\citep{notz2024explainable, aghaei2025strong}. 
It is desirable to explore a computationally tractable, tree-based model for decision-rule optimization that preserves interpretability. 

In this paper, we propose a framework for contextual DRO that preserves the causal consistency and continuity of worst-case distributions and includes a tractable and interpretable decision rule. Our main contributions are summarized as follows. 
\begin{enumerate}
    \item First, we develop a Causal Sinkhorn DRO (Causal-SDRO) model with ambiguity sets based on causal and entropy-regularized Wasserstein distance, combining~\citet{yang2026decision} and~\citet{wang2025sinkhorn}. We further derive the strong dual reformulation and the worst-case distribution for the inner problem of Causal-SDRO under regularity assumptions (Assumption 1).   
    \item To optimize Causal-SDRO, we propose a parametric, differentiable, and Lipschitz-smooth Soft Regression Forest (SRF) decision rule that approximates optimal policies in a measurable, continuous, and differentiable function space while retaining the intrinsic interpretability of non-differentiable decision trees. 
    \item {We reformulate Causal-SDRO with parametric decision rules as a multilevel stochastic compositional optimization problem and develop a conditional mini-batch stochastic compositional gradient algorithm that, under regularity assumptions (Assumption 2), converges to an $\varepsilon$-stationary point within $\mathcal{O}(\varepsilon^{-4})$ iterations and uses $\mathcal{O}(\varepsilon^{-5})$ samples and $\mathcal{O}(\varepsilon^{-5})$ gradient computations. }
    Numerical results demonstrate the superior performance of our method. 
\end{enumerate}

The remainder of this paper is organized as follows. 
Section~\ref{sec-setup} formulates the Causal-SDRO model. 
Section~\ref{sec-CSDRO} derives the strong dual formulation of Causal-SDRO given a fixed decision rule. 
Section~\ref{sec-srf} proposes the SRF decision rule. 
Section~\ref{sec-algo} presents an algorithm for solving Causal-SDRO. 
Section~\ref{sec-results} reports representative numerical results.  
Section~\ref{sec-conclusion} concludes the paper.  
Proofs, analyses, additional discussions, and detailed numerical results are provided in the appendix to this paper. 

\textit{Notation. } For integer $K \in \mathbb{Z}_+$, define $\left [ K \right ] \triangleq \left \{ 1,\cdots,K \right \} $. 
We denote random vectors by bold uppercase letters (for example, $\boldsymbol{X}, \boldsymbol{Y}$) and their realizations by bold lowercase letters (for example, $\boldsymbol{x}, \boldsymbol{y}$). 
For a measurable set $\mathcal{Y}$, denote $\mathcal{M}(\mathcal{Y})$ as the set of measures on $\mathcal{Y}$, and $\mathcal{P}(\mathcal{Y})$ as the set of probability measures on $\mathcal{Y}$. 
Denote $\mathbb{P} \otimes \mathbb{Q}$ as the product measure of two probability measures $\mathbb{P}$ and $\mathbb{Q}$. 
Given a probability distribution $\mathbb{P}$ and a measure $\mu$, we denote by $\textnormal{supp}\, \mathbb{P}$ the support of $\mathbb{P}$, and write $\mathbb{P} \ll \mu$ if $\mathbb{P}$ is absolutely continuous with respect to $\mu$. 
Let the logarithm function $\log$ be taken with base $e$.
A function $f : \mathcal{X} \to \mathbb{R}^n$ is said to be $L$-Lipschitz continuous, if there exists a constant $L > 0$ such that $\|f(\boldsymbol{x}_1) - f(\boldsymbol{x}_2)\| \le L\|\boldsymbol{x}_1 - \boldsymbol{x}_2\|$ for any $\boldsymbol{x}_1, \boldsymbol{x}_2 \in \mathcal{X}$. 
A function $f: \mathcal{X} \to \mathbb{R}^n$ is said to be $S$-Lipschitz smooth if it is continuously differentiable and its Jacobian $J_{\boldsymbol{f}}(\boldsymbol{x})$ is $S$-Lipschitz continuous. 
For a vector $\boldsymbol{w} \in \mathbb{R}^d$, we denote $[ \boldsymbol{w} ]_k$ as its $k$-th element for any $k \in [d]$. We use $\left| \cdot \right|$ to represent the cardinality of a set. 

\section{The Causal-SDRO Model} \label{sec-setup}

In this section, we introduce the Causal-SDRO model. 
Consider a decision rule $f: \mathcal{X} \to \mathcal{Z}$ that maps a covariate vector $\boldsymbol{x}$ from $\mathcal{X} \subseteq \mathbb{R}^{d_x}$ to a decision $\boldsymbol{z}$ from $\mathcal{Z} \subseteq \mathbb{R}^{d_z}$.
{
Let $N$ be the finite number of observed covariate-uncertain-parameter pairs, and $M$ be the finite number of observed covariate values. 
For each covariate $\widehat{\boldsymbol{x}}_m, m\in[M]$, assume there are $n_m$ observations of the associated uncertain parameter $\widehat{\boldsymbol{y}}_{m, i}, i\in[n_m]$. It holds that $N = \sum_{m\in[M]}n_{m}$. 
Let $\widehat{p}_{m, i}$ be the probability mass of the data point $(\widehat{\boldsymbol{x}}_m, \widehat{\boldsymbol{y}}_{m, i})$. 
Define the empirical distribution 
\begin{equation} \label{Eq-empirical-distri}
    \widehat{\mathbb{P}}=\frac{1}{N}\sum_{m=1}^M \sum_{i=1}^{n_{m}}  \delta_{\left(\widehat{\boldsymbol{x}}_{m}, \widehat{\boldsymbol{y}}_{m, i}\right)}.  
\end{equation}} 
The contextual DRO problem assumes that the unknown true distribution $\mathbb{P}$ lies within an ambiguity set $\Re(\widehat{\mathbb{P}})$ constructed around $\widehat{\mathbb{P}}$, and seeks an optimal decision rule that minimizes the worst-case expected loss over all distributions in the ambiguity set: 
\begin{equation}\label{DRCO}\nonumber
 \inf_{f \in \mathcal{F}} \, \sup_{\mathbb{P} \in \, \Re\, ({\widehat{\mathbb{P}}})}\quad \mathbb{E}_{(\boldsymbol{x}, \boldsymbol{y}) \sim\mathbb{P}}\Big[\Psi(f \left ( \boldsymbol{x} \right ) , \boldsymbol{y})\Big]
\end{equation}
where $\mathcal{F}$ is a space of measurable functions, and $\Psi:\mathcal{Z}\times\mathcal{Y}\to\mathbb{R}\cup \{\infty\}$ is a measurable loss function.
To construct our ambiguity set, we first recall the causal transport distance~\citep{yang2026decision}, which incorporates the causal structure into the optimal transport framework. 
\begin{definition}\label{def-causal-trans}
    \textnormal{\textbf{\citep[Causal Transport Distance,][]{yang2026decision}.}} Let distributions $\mathbb{P}, \mathbb{Q} \in \mathcal{P}(\mathcal{X}\times\mathcal{Y})$. A joint distribution $\gamma \in \Gamma(\mathbb{P}, \mathbb{Q})$ is termed a \textit{causal transport plan} if, for $((\widehat{\boldsymbol{X}}, \widehat{\boldsymbol{Y}}), (\boldsymbol{X}, \boldsymbol{Y})) \sim \gamma$, the random variable $\boldsymbol{X}$ is conditionally independent of $\widehat{\boldsymbol{Y}}$ given $\widehat{\boldsymbol{X}}$, denoted as $\boldsymbol{X} \perp \widehat{\boldsymbol{Y}} \mid \widehat{\boldsymbol{X}}$. 
    
    Let $\Gamma_c(\mathbb{P}, \mathbb{Q})$ be the set of all causal transport plans within $\Gamma(\mathbb{P}, \mathbb{Q})$. For $p \in [1, \infty)$, the $p$-causal transport distance between $\mathbb{P}$ and $\mathbb{Q}$ is defined as
    \begin{equation}\nonumber
        C_p(\mathbb{P}, \mathbb{Q}) := \left( \inf_{\gamma \in \Gamma_c(\mathbb{P}, \mathbb{Q})} \mathbb{E}_{((\widehat{\boldsymbol{x}},\widehat{\boldsymbol{y}}), (\boldsymbol{x},\boldsymbol{y}))\sim\gamma} \Big[ c_p((\widehat{\boldsymbol{x}},\widehat{\boldsymbol{y}}), (\boldsymbol{x},\boldsymbol{y})) \Big] \right)^{1/p}, 
    \end{equation}
    where $c_p((\widehat{\boldsymbol{x}},\widehat{\boldsymbol{y}}), (\boldsymbol{x},\boldsymbol{y})) := \|\boldsymbol{x} - \widehat{\boldsymbol{x}}\|^p + \|\boldsymbol{y} - \widehat{\boldsymbol{y}}\|^p$is a transport cost function. $\hfill\Diamond $
\end{definition}

The causal transport distance-based ambiguity set includes both discrete and continuous distributions, whereas it typically hedges against a discrete one, which may not be realistic in practice. 
We introduce the \emph{Causal Sinkhorn Discrepancy (CSD)} below by incorporating entropy regularization. 
When the reference measures $\nu_{\mathcal{X}}$ and $\nu_{\mathcal{Y}}$ are chosen as continuous measures, $\mathbb Q\ll\nu_{\mathcal{X}}\otimes\nu_{\mathcal{Y}}$ restricts distributions in the CSD-based ambiguity set to be absolutely continuous.

\begin{definition}\label{def-causal-sh}
    \textnormal{\textbf{ (Causal Sinkhorn Discrepancy, CSD). }} 
    Consider distributions $\mathbb{P}, \mathbb{Q} \in \mathcal{P}(\mathcal{X}\times\mathcal{Y})$, and let measures $\nu_{\mathcal{X}} \in \mathcal{M}(\mathcal{X}), \nu_{\mathcal{Y}} \in \mathcal{M}(\mathcal{Y})$ be reference measures such that $\mathbb{Q} \ll \nu_{\mathcal{X}} \otimes \nu_{\mathcal{Y}}$.  
    For regularization parameter $\epsilon \ge 0$ and $p \in [1, \infty)$, the $p$-CSD between $\mathbb{P}$ and $\mathbb{Q}$ is defined as
\begin{equation}\nonumber
    R_{\epsilon}(\mathbb{P}, \mathbb{Q}) := \left( \inf_{\gamma \in \Gamma_c(\mathbb{P}, \mathbb{Q})} \mathbb{E}_{((\widehat{\boldsymbol{x}},\widehat{\boldsymbol{y}}), (\boldsymbol{x},\boldsymbol{y}))\sim\gamma} \Big[ c_p((\widehat{\boldsymbol{x}},\widehat{\boldsymbol{y}}), (\boldsymbol{x},\boldsymbol{y})) \Big]  + \epsilon \cdot \mathbb{D}_{\mathrm{KL}}\left( \gamma \|  \mathbb{P} \otimes \left ( \nu_{\mathcal{X}}\otimes \nu_{\mathcal{Y} } \right ) \right)\right)^{1/p}, 
\end{equation}
where $\Gamma_c(\mathbb{P}, \mathbb{Q})$ is the causal transport plan set, and the relative entropy (also known as Kullback–Leibler divergence) of $\gamma$ with respect to the measure $\mathbb{P}\otimes \left ( \nu_{\mathcal{X}}\otimes \nu_{\mathcal{Y}}\right )$ is given by
\begin{equation}\nonumber
    \mathbb{D}_{\mathrm{KL}}\left( \gamma \|  \mathbb{P} \otimes \left ( \nu_{\mathcal{X}}\otimes \nu_{\mathcal{Y} } \right ) \right) = \mathbb{E}_{((\widehat{\boldsymbol{x}},\widehat{\boldsymbol{y}}), (\boldsymbol{x},\boldsymbol{y}))\sim\gamma}\Big [ \log\, \left ( \frac{\mathrm{d}\gamma ( (\widehat{\boldsymbol{x}}, \widehat{\boldsymbol{y}}), (\boldsymbol{x},\boldsymbol{y})) }{\mathrm{d}\mathbb{P}(\widehat{\boldsymbol{x}}, \widehat{\boldsymbol{y}})\mathrm{d} \nu_{\mathcal{X}}(\boldsymbol{x})\mathrm{d} \nu_{\mathcal{Y}}(\boldsymbol{y})}  \right )  \Big ] , 
\end{equation}  where $\frac{\mathrm{d}\gamma ( (\widehat{\boldsymbol{x}}, \widehat{\boldsymbol{y}}), (\boldsymbol{x},\boldsymbol{y})) }{\mathrm{d}\mathbb{P}(\widehat{\boldsymbol{x}}, \widehat{\boldsymbol{y}})\mathrm{d} \nu_{\mathcal{X}}(\boldsymbol{x})\mathrm{d} \nu_{\mathcal{Y}}(\boldsymbol{y})}$ stands for the density ratio of $\gamma$ with respect to $\mathbb{P} \otimes \left ( \nu_{\mathcal{X}}\otimes \nu_{\mathcal{Y}}\right )$ evaluated at $\left ( \widehat{\boldsymbol{x}}, \widehat{\boldsymbol{y}}, \boldsymbol{x}, \boldsymbol{y} \right ) $. $\hfill\Diamond $
\end{definition} 

In what follows, we use the \textit{Lebesgue measure} as our reference measure. 
Based on CSD, we study the following contextual DRO model, where the outer minimization seeks optimal decision rules and the inner maximization seeks the worst-case distribution in a CSD-based ambiguity set: 
\begin{equation}\label{causal-sdro}
\inf _{f \in \mathcal{F}}\,\,  \sup_{\mathbb{P} \in \, \Re\, ({\widehat{\mathbb{P}}})} \quad \mathbb{E}_{(\boldsymbol{x}, \boldsymbol{y}) \sim \mathbb{P}}\Big[\Psi(f(\boldsymbol{x}), \boldsymbol{y})\Big], \tag{\text{Causal-SDRO}}
\end{equation}
where 
\begin{equation}\nonumber 
    \, \Re\, ({\widehat{\mathbb{P}}}) = \left\{ \mathbb{P} \in \mathcal{P}(\mathcal{X} \times \mathcal{Y}): R_{\epsilon}(\widehat{\mathbb{P}}, \mathbb{P})^p \leq \rho^p \right\}. 
\end{equation}
This Causal Sinkhorn Distributionally Robust Optimization (Causal-SDRO) model hedges against a continuous worst-case distribution (see Section~\ref{subsec-CSDRO-worstcase}) while preserving the causal information structure and thereby avoiding overly conservative and causally inconsistent decisions. 
We provide detailed discussions on CSD and the causal transport plans in Appendix~\ref{app-sec-CSD}. 

As an example, we introduce an application of Causal-SDRO. 
We also evaluate Causal-SDRO on a two-stage inventory substitution problem and a portfolio selection problem; see Appendix~\ref{app-sec-results}. 
\begin{example}\label{example-news}
    \textnormal{\textbf{(Feature-based Newsvendor Problem). }} Consider a newsvendor who sells $d_z$ kinds of products. 
    Let $\boldsymbol{h} \in \mathbb{R}^{d_z}$ and $\boldsymbol{b} \in \mathbb{R}^{d_z}$ to represent the holding and stock-out cost, respectively. 
    Let $(\cdot)^+$ denote the component-wise positive part operator.
    Consider features $\boldsymbol{x} \in \mathcal{X} \subseteq \mathbb{R}^{d_x}$ (e.g., weather) and uncertain demand $\boldsymbol{y} \in \mathcal{Y} \subseteq \mathbb{R}^{d_y} (d_y=d_z)$, and then the problem is given by 
    \begin{equation}\nonumber
        \inf _{f \in \mathcal{F}}\,\, \sup_{\mathbb{P} \in \, \Re\, ({\widehat{\mathbb{P}}})} \quad \mathbb{E}_{(\boldsymbol{x}, \boldsymbol{y}) \sim \mathbb{P}}\Big[\boldsymbol{h}^{\top}\Big(f(\boldsymbol{x})-\boldsymbol{y}\Big)^+ + \boldsymbol{b}^{\top}\Big(\boldsymbol{y}-f(\boldsymbol{x})\Big)^+\Big].
    \end{equation}  
    This problem will be revisited in Sections~\ref{subsec-CSDRO-worstcase} and~\ref{sec-results}.  $\hfill \clubsuit$ 
\end{example}

\section{Duality Reformulation for Causal-SDRO}\label{sec-CSDRO}

In this section, we establish strong duality and characterize the worst-case distribution for the inner maximization problem in~\eqref {causal-sdro}, assuming a fixed decision rule $f \in \mathcal{F}$ and a regularization parameter $\epsilon > 0$. 
First define the primal problem as
\begin{equation}\label{eq-primal}
    v_{\mathrm{P}}:= \sup_{\mathbb{P} \in \mathcal{P}(\mathcal{X} \times \mathcal{Y})} \Big \{ \mathbb{E}_{(\boldsymbol{x}, \boldsymbol{y}) \sim \mathbb{P}}\Big[\Psi(f(\boldsymbol{x}), \boldsymbol{y})\Big]: R_{\epsilon}(\widehat{\mathbb{P}}, \mathbb{P})^p \leq \rho^p \Big\}.
\end{equation} 

We derive the corresponding dual problem $v_{\mathrm{D}}$ as (see Appendices~\ref{app-sec-dual-causal-sdro} and~\ref{ecsec-strong-dual})
\begin{subequations}\label{Eq:MCO:whole}
\begin{equation}\label{MCO_1}
    v_{\mathrm{D}} = \inf_{\lambda \ge 0} \left\{ \lambda \bar{\rho} + \mathbb{E}_{\widehat{\boldsymbol{x}} \sim \widehat{\mathbb{P}}_{\widehat{\boldsymbol{X}}}} \left[ \lambda\epsilon \log\, \mathbb{E}_{ \boldsymbol{\xi}_1 \sim U_{\epsilon}}\left[ \exp\left( \frac{h (\widehat{\boldsymbol{x}},  \boldsymbol{\xi}_1, \lambda)}{\lambda\epsilon} \right) \right]\right] \right\}, 
\end{equation}
where 
\begin{equation}\label{MCO_2}
    h (\widehat{\boldsymbol{x}},  \boldsymbol{\xi}_1, \lambda) := \mathbb{E}_{\widehat{\boldsymbol{y}} \sim  \widehat{\mathbb{P}}_{\widehat{\boldsymbol{Y}}|\widehat{\boldsymbol{X}}=\widehat{\boldsymbol{x}}}}\left[ \lambda\epsilon \log\, \mathbb{E}_{ \boldsymbol{\xi}_2 \sim W_{\epsilon}}\left[ \exp\left( \frac{\Psi(f(\widehat{\boldsymbol{x}}+ \boldsymbol{\xi}_1),\widehat{\boldsymbol{y}}+ \boldsymbol{\xi}_2)}{\lambda\epsilon} \right) \right] \right], 
\end{equation}     
\end{subequations} 
\begin{equation}\label{QW-distri}
    \mathrm{d} U_{\epsilon} \left ( \boldsymbol{\xi}_1 \right ) := \frac{e^{ - \| \boldsymbol{\xi}_1\|^p / \epsilon} \cdot   \mathrm{d} \nu_{\mathcal{X} }\left (  \boldsymbol{\xi}_1 \right )}{\int_{\mathbb{R}^{d_x}} e^{ - \| \boldsymbol{u}\|^p / \epsilon} \mathrm{d} \nu_{\mathcal{X} }\left (  \boldsymbol{u} \right )}, \quad
     \mathrm{d} W_{\epsilon } \left (  \boldsymbol{\xi}_2 \right ) := \frac{e^{ - \| \boldsymbol{\xi}_2\|^p / \epsilon} \cdot  \mathrm{d} \nu_{\mathcal{Y} }\left (  \boldsymbol{\xi}_2 \right )}{\int_{\mathbb{R}^{d_y}} e^{ - \| \boldsymbol{u}\|^p / \epsilon} \mathrm{d} \nu_{\mathcal{Y} }\left (  \boldsymbol{u} \right ) } ,
\end{equation}
and constant
\begin{equation}
    \bar{\rho} := \rho^p + \epsilon \cdot \mathbb{E}_{\widehat{\boldsymbol{x}} \sim \widehat{\mathbb{P}}_{\widehat{\boldsymbol{X}}}} \left[ \log\, \int_{\mathbb{R}^{d_x}} e^{ - \| \boldsymbol{u}\|^p / \epsilon} \mathrm{d} \nu_{\mathcal{X} }\left (  \boldsymbol{u} \right ) \right] + \epsilon \cdot \mathbb{E}_{\widehat{\boldsymbol{y}} \sim \widehat{\mathbb{P}}_{\widehat{\boldsymbol{Y}}|\widehat{\boldsymbol{X}}}} \left[ \log\, \int_{\mathbb{R}^{d_y}} e^{ - \| \boldsymbol{u}\|^p / \epsilon} \mathrm{d} \nu_{\mathcal{Y} }\left (  \boldsymbol{u} \right ) \right]. 
\end{equation} 
In the following section, Section~\ref{subsec-CSDRO-dual} presents the main result on strong duality.  Section~\ref{subsec-CSDRO-worstcase} presents the worst-case distribution of the soft-constrained Causal-SDRO.

\subsection{Strong Duality}\label{subsec-CSDRO-dual}

In this part, we first present the strong duality theorem, $v_{\mathrm{P}}=v_{\mathrm{D}}$, and then provide related discussions.
We consider the following assumptions.  
\begin{assumption}\label{assum-1}
    \begin{enumerate}
    \item Both $\mathcal{X}$ and $\mathcal{Z}$ are measurable sets.  
    The loss function $\Psi: \mathcal{Z}\times\mathcal{Y} \to \mathbb{R}\cup \left\{\infty\right\}$ is measurable and bounded from below by a finite constant, and the decision rule $f: \mathcal{X} \to \mathcal{Z}$ is measurable. \label{assum-1-3}
    \item For every joint distribution $\gamma$ on $(\mathcal{X}\times \mathcal{Y}) \times (\mathcal{X}\times \mathcal{Y})$ with first marginal distribution $\widehat{\mathbb{P}}$, it has a regular conditional distribution $\gamma_{(\widehat{\boldsymbol{x}},\widehat{\boldsymbol{y}})}$ given the value of the first marginal equals $(\widehat{\boldsymbol{x}}, \widehat{\boldsymbol{y}})$. \label{assum-1-4}
        \item The transport cost function $c_p((\widehat{\boldsymbol{x}},\widehat{\boldsymbol{y}}), (\boldsymbol{x},\boldsymbol{y}))$ is measurable, and for $\widehat{\mathbb{P}} \otimes \nu_{\mathcal{X}}\otimes \nu_{\mathcal{Y}}$-almost every $(\widehat{\boldsymbol{x}},\widehat{\boldsymbol{y}}, \boldsymbol{x},\boldsymbol{y})$, it holds that $0 \le c_p((\widehat{\boldsymbol{x}},\widehat{\boldsymbol{y}}), (\boldsymbol{x},\boldsymbol{y})) < \infty$. \label{assum-1-1}     
        \item 
        For any $\delta>0$, it holds that $\int_{\mathbb{R}^{d_x}} e^{ - \delta\| \boldsymbol{u}\|^p } \mathrm{d} \nu_{\mathcal{X} }\left (  \boldsymbol{u} \right ) < \infty$ and $\int_{\mathbb{R}^{d_y}} e^{ - \delta\| \boldsymbol{u}\|^p  } \mathrm{d} \nu_{\mathcal{Y} }\left (  \boldsymbol{u} \right )<\infty$. \label{assum-1-2}
    \end{enumerate}
\end{assumption}
Assumption~\ref{assum-1}\ref{assum-1-3} ensures that the expectation over $\Psi(f(\boldsymbol{x}), \boldsymbol{y})$ is well-defined.
Assumption~\ref{assum-1}\ref{assum-1-4} ensures that each optimal transport plan can be decomposed into several conditional optimal transport plans.
Assumptions~\ref{assum-1}\ref{assum-1-1} and \ref{assum-1}\ref{assum-1-2} ensure that the optimal value of Causal-SDRO is well-defined. 
Based on Assumption~\ref{assum-1}\ref{assum-1-2}, we introduce the light-tail condition on function $\Psi$ in the following Condition~\ref{condit-1} to distinguish the cases $v_{\mathrm{D}} < \infty$ and $v_{\mathrm{D}} = \infty$. We provide sufficient conditions to easily verify whether Condition~\ref{condit-1} holds in Appendix~\ref{ecsec-condition1}. 
\begin{condition}\label{condit-1}
    There exists $\lambda>0$ such that  $\mathbb{E}_{ \boldsymbol{\xi}_2 \sim W_{\epsilon}}\left[ \exp\left( \frac{\Psi(f(\widehat{\boldsymbol{x}}+ \boldsymbol{\xi}_1),\widehat{\boldsymbol{y}}+ \boldsymbol{\xi}_2)}{\lambda\epsilon} \right) \right] < \infty$ for $\widehat{\mathbb{P}}\otimes\nu_{\mathcal{X}}$-almost every $(\widehat{\boldsymbol{x}}, \widehat{\boldsymbol{y}}, \boldsymbol{\xi}_1)$ and $\mathbb{E}_{ \boldsymbol{\xi}_1 \sim U_{\epsilon}}\left[ \exp\left( \frac{h (\widehat{\boldsymbol{x}},  \boldsymbol{\xi}_1, \lambda)}{\lambda\epsilon} \right) \right] < \infty$ for $\widehat{\mathbb{P}}_{\widehat{\boldsymbol{X}}}$-almost every $\widehat{\boldsymbol{x}}$. 
\end{condition}

We call the constraint $R_{\epsilon}(\widehat{\mathbb{P}}, \mathbb{P})^p \leq \rho^p$ in the primal problem~\eqref{eq-primal} the CSD constraint in the following. Based on Assumption~\ref{assum-1} and Condition~\ref{condit-1}, the following strong duality theorem holds. 
\begin{theorem}\label{theo-strong-duality}
    \textnormal{\textbf{(Strong Duality).}} Under Assumption~\ref{assum-1}, the following results hold. 
    \begin{enumerate}
        \item The primal problem $v_{\mathrm{P}}$ is feasible if and only if $\bar{\rho} \ge 0$. \label{theo-strong-duality-1}
        \item Additionally, assume that $\bar{\rho}\ge0$ is bounded above such that the CSD constraint is binding. Then, if Condition~\ref{condit-1} holds, then $v_{\mathrm{P}} = v_{\mathrm{D}} < \infty$; otherwise, $v_{\mathrm{P}} = v_{\mathrm{D}} = \infty$. \label{theo-strong-duality-2}
    \end{enumerate}
\end{theorem}

The proof of Theorem~\ref{theo-strong-duality} is provided in Appendix~\ref{ecsec-strong-dual}. 
We further provide discussions, including a comparison with causal Wasserstein DRO and the worst-case distribution, in Appendix~\ref{app-sec-disc-csdro}.

\begin{remark}\label{remark-soft-cons}
    \textnormal{\textbf{(Soft-constrained Causal-SDRO).  }} In~\eqref{causal-sdro}, the radius $\rho$ is a hyperparameter to be tuned while $\lambda$ is the dual variable. 
    In the following soft-constrained Causal-SDRO problem, it suffices to tune a hyperparameter $\lambda > 0$: 
    \begin{equation}\label{soft-causal-sdro}
    \inf_{f \in \mathcal{F}}\, \sup_{\mathbb{P} \in \mathcal{P}(\mathcal{X}\times\mathcal{Y})} \quad  \mathbb{E}_{(\boldsymbol{x}, \boldsymbol{y}) \sim \mathbb{P}}\Big[\Psi(f(\boldsymbol{x}), \boldsymbol{y})\Big] -\lambda \cdot R_{\epsilon}(\widehat{\mathbb{P}}, \mathbb{P})^p. 
        \tag{\text{Soft-Causal-SDRO}}
    \end{equation}
    By the Fenchel duality, the dual formulation of~\eqref{soft-causal-sdro} is given by 
    \begin{equation}\label{soft-causal-sdro-dual}
        \inf_{f\in \mathcal{F}}\quad \mathbb{E}_{\widehat{\boldsymbol{x}} \sim \widehat{\mathbb{P}}_{\widehat{\boldsymbol{X}}}} \left[ \lambda\epsilon \log\, \mathbb{E}_{ \boldsymbol{\xi}_1 \sim U_{\epsilon}} \Big[\exp\left( \frac{h (\widehat{\boldsymbol{x}}, \boldsymbol{\xi}_1, \lambda)}{\lambda\epsilon} \right)\Big] \right], 
    \end{equation}
    where the definition of function $h$ is the same as Eq.~\eqref{MCO_2}. 
    Under Assumption~\ref{assum-1}, the strong duality of the inner problem in~\eqref{soft-causal-sdro} holds (by a proof similar to that of Theorem~\ref{theo-strong-duality}). $\hfill \clubsuit$ 
\end{remark}

\subsection{Worst-Case Distribution} \label{subsec-CSDRO-worstcase}

Based on Theorem~\ref{theo-strong-duality}, Theorem~\ref{corollary-worst-case} derives the worst-case distribution expression for the inner problem of~\eqref{soft-causal-sdro}, see Appendices~\ref{app-sec-disc-worstcase} and~\ref{ecsec-theo-worst} for details. 

\begin{theorem}\label{corollary-worst-case}
    \textnormal{\textbf{(Worst-case Distribution of the Soft-constrained Causal-SDRO).}}
    Under Assumption~\ref{assum-1} and Condition~\ref{condit-1}, the density of worst-case distribution $\mathbb{P}_{\lambda}^*$ of the inner problem of \eqref{soft-causal-sdro} for any $\lambda>0$ is given by 
    \begin{equation}
        \frac{\mathrm{d}\mathbb{P}_{\lambda}^* (\boldsymbol{x},\boldsymbol{y}) }{\mathrm{d} \nu_{\mathcal{X}}(\boldsymbol{x})\mathrm{d} \nu_{\mathcal{Y}}(\boldsymbol{y})} = \mathbb{E}_{(\widehat{\boldsymbol{x}},\widehat{\boldsymbol{y}})\sim\widehat{\mathbb{P}}}\Big[\alpha_{\widehat{\boldsymbol{x}}}(\lambda) \cdot \beta_{\widehat{\boldsymbol{x}}, \widehat{\boldsymbol{y}}, \boldsymbol{x}}(\lambda) \cdot e^{r^{\prime}(\lambda, \widehat{\boldsymbol{x}}, \boldsymbol{x}) + s^{\prime}(\lambda, \widehat{\boldsymbol{x}},\widehat{\boldsymbol{y}}, \boldsymbol{x},\boldsymbol{y})} \Big],
    \end{equation}
    where $\alpha_{\widehat{\boldsymbol{x}}}(\lambda) = (\int_{\mathcal{X}} e^{r^{\prime}(\lambda, \widehat{\boldsymbol{x}}, \boldsymbol{x})}  \mathrm{d} \nu_{\mathcal{X}}(\boldsymbol{x}))^{-1}$,  $\beta_{\widehat{\boldsymbol{x}}, \widehat{\boldsymbol{y}}, \boldsymbol{x}}(\lambda) =(\int_{\mathcal{Y}} e^{s^{\prime}(\lambda, \widehat{\boldsymbol{x}},\widehat{\boldsymbol{y}}, \boldsymbol{x},\boldsymbol{y})}\mathrm{d} \nu_{\mathcal{Y}}(\boldsymbol{y}) )^{-1}$, 
    
    $s^{\prime}(\lambda, \widehat{\boldsymbol{x}},\widehat{\boldsymbol{y}}, \boldsymbol{x},\boldsymbol{y}) = \frac{\Psi(f(\boldsymbol{x}),\boldsymbol{y})-\lambda c_p((\widehat{\boldsymbol{x}},\widehat{\boldsymbol{y}}), (\boldsymbol{x},\boldsymbol{y}))}{\lambda\epsilon}$, and
    \begin{equation}\nonumber
        r^{\prime}(\lambda, \widehat{\boldsymbol{x}}, \boldsymbol{x}) = \mathbb{E}_{\widehat{\boldsymbol{y}} \sim  \widehat{\mathbb{P}}_{\widehat{\boldsymbol{Y}}|\widehat{\boldsymbol{X}}=\widehat{\boldsymbol{x}}}}\left[ \log\,  \int _{\mathcal{Y}}  e^{s^{\prime}(\lambda, \widehat{\boldsymbol{x}},\widehat{\boldsymbol{y}}, \boldsymbol{x},\boldsymbol{y})}  \mathrm{d} \nu_{\mathcal{Y}}(\boldsymbol{y}) \right].
    \end{equation}
\end{theorem}

\begin{figure}[!htb]
  \centering
  \begin{subfigure}{.24\linewidth}
    \centering
    \includegraphics[width=\linewidth]
      {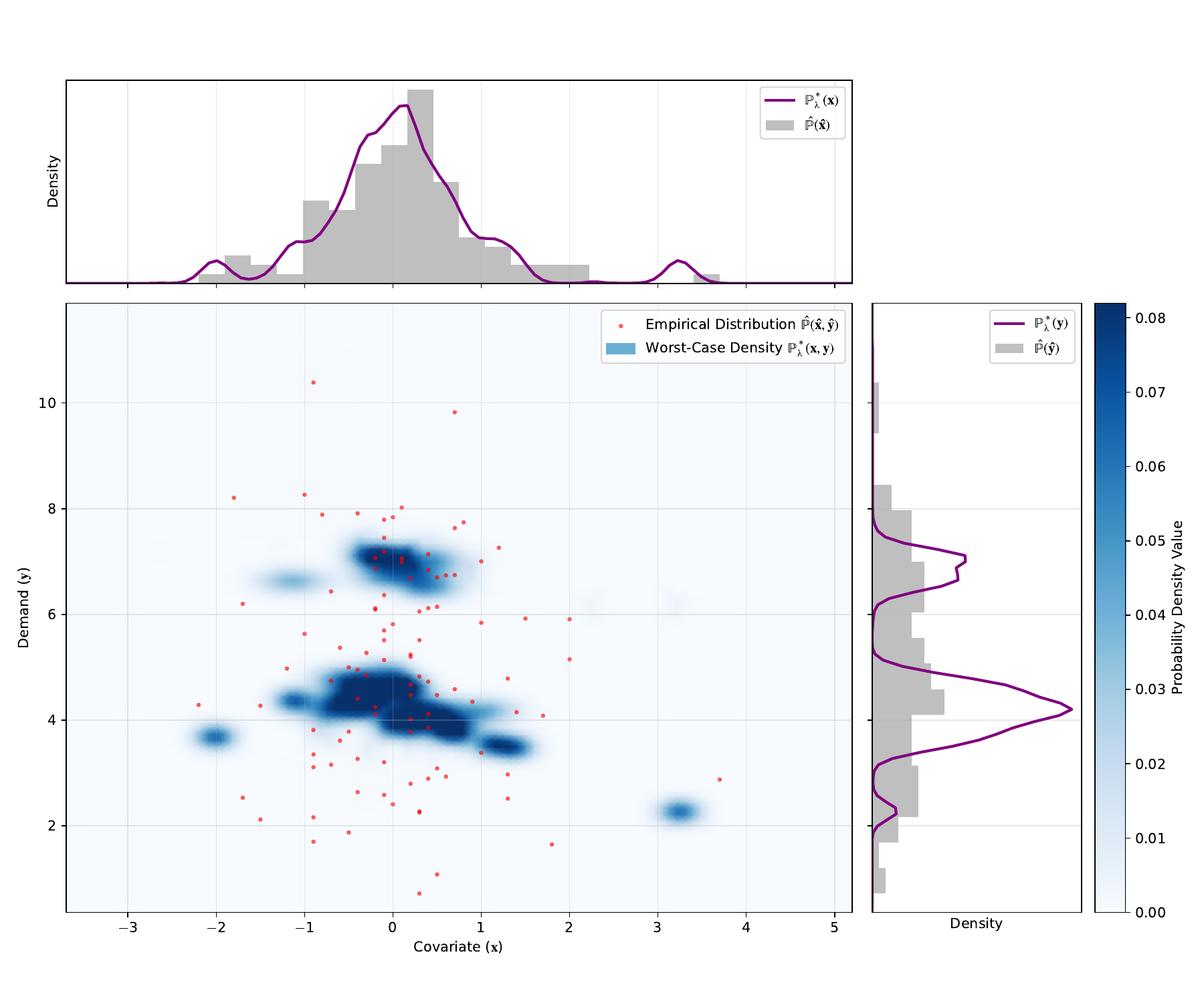}
    \caption{$\mathbb{P}_{\lambda}^*$ vs. $\widehat{\mathbb{P}}$}
    \label{fig:worst-case-causal-sdro}
  \end{subfigure}
  \hfill
  \begin{subfigure}{.24\linewidth}
    \centering
    \includegraphics[width=\linewidth]
      {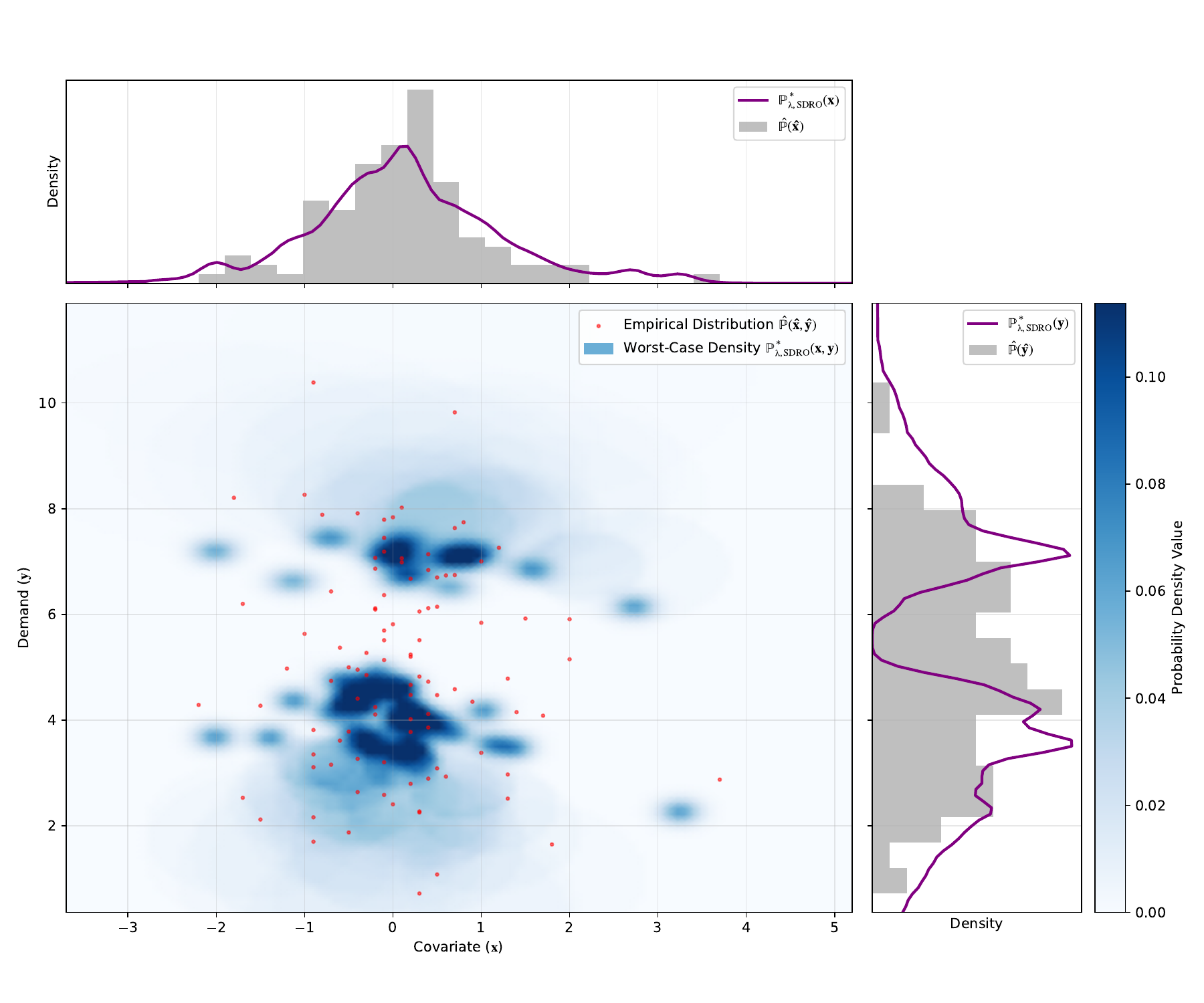}
    \caption{$\mathbb{P}_{\lambda, \textnormal{SDRO}}^*$ vs. $\widehat{\mathbb{P}}$}
    \label{fig:worst-case-sdro}
  \end{subfigure}
  \hfill
    \begin{subfigure}{.24\linewidth}
    \centering
    \includegraphics[width=\linewidth]
      {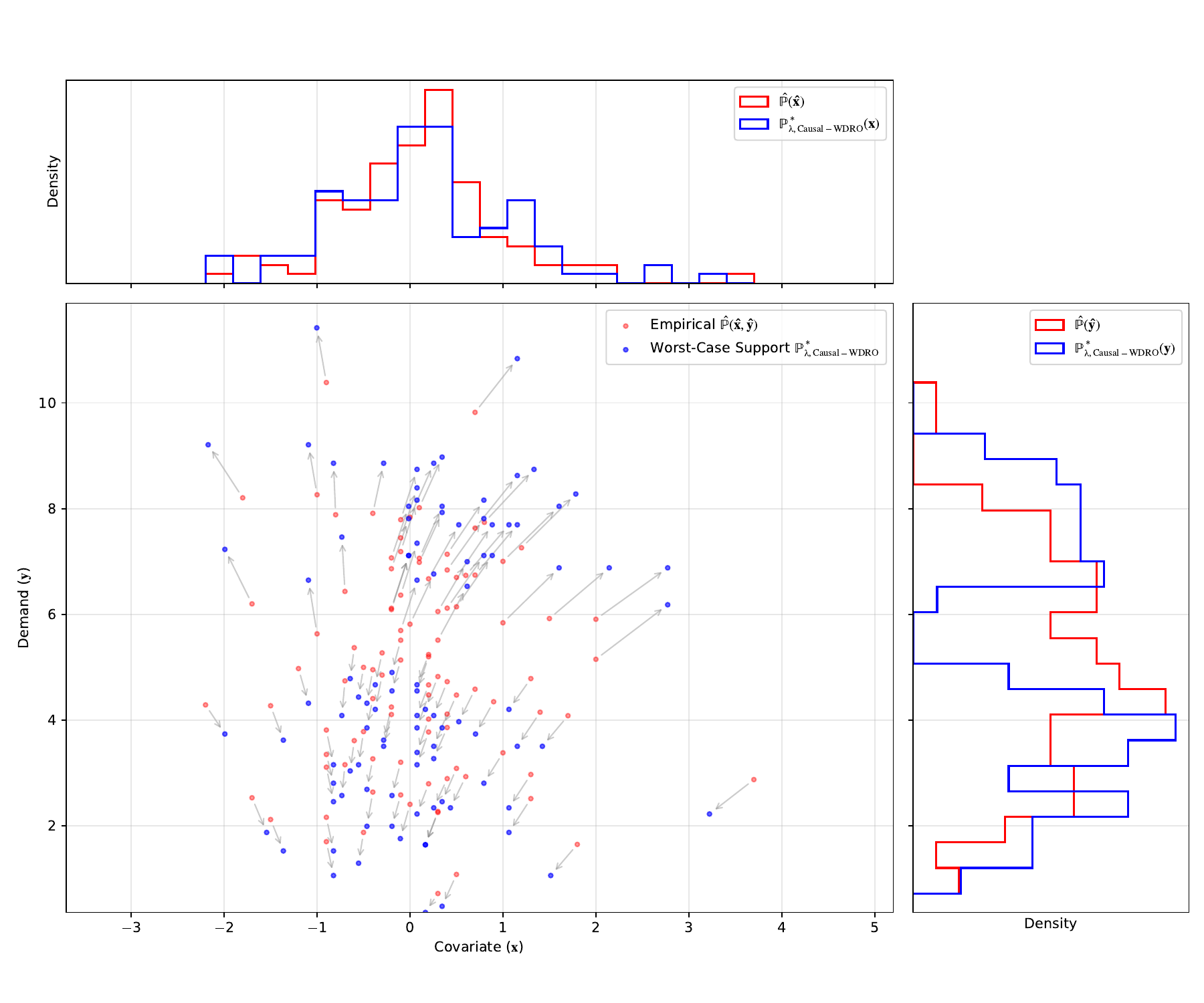}
    \caption{$\mathbb{P}_{\lambda, \textnormal{Causal-WDRO}}^*$ vs. $\widehat{\mathbb{P}}$}
    \label{fig:worst-case-causal-wdro}
  \end{subfigure}
  \hfill
  \begin{subfigure}{.24\linewidth}
    \centering
    \includegraphics[width=\linewidth]
      {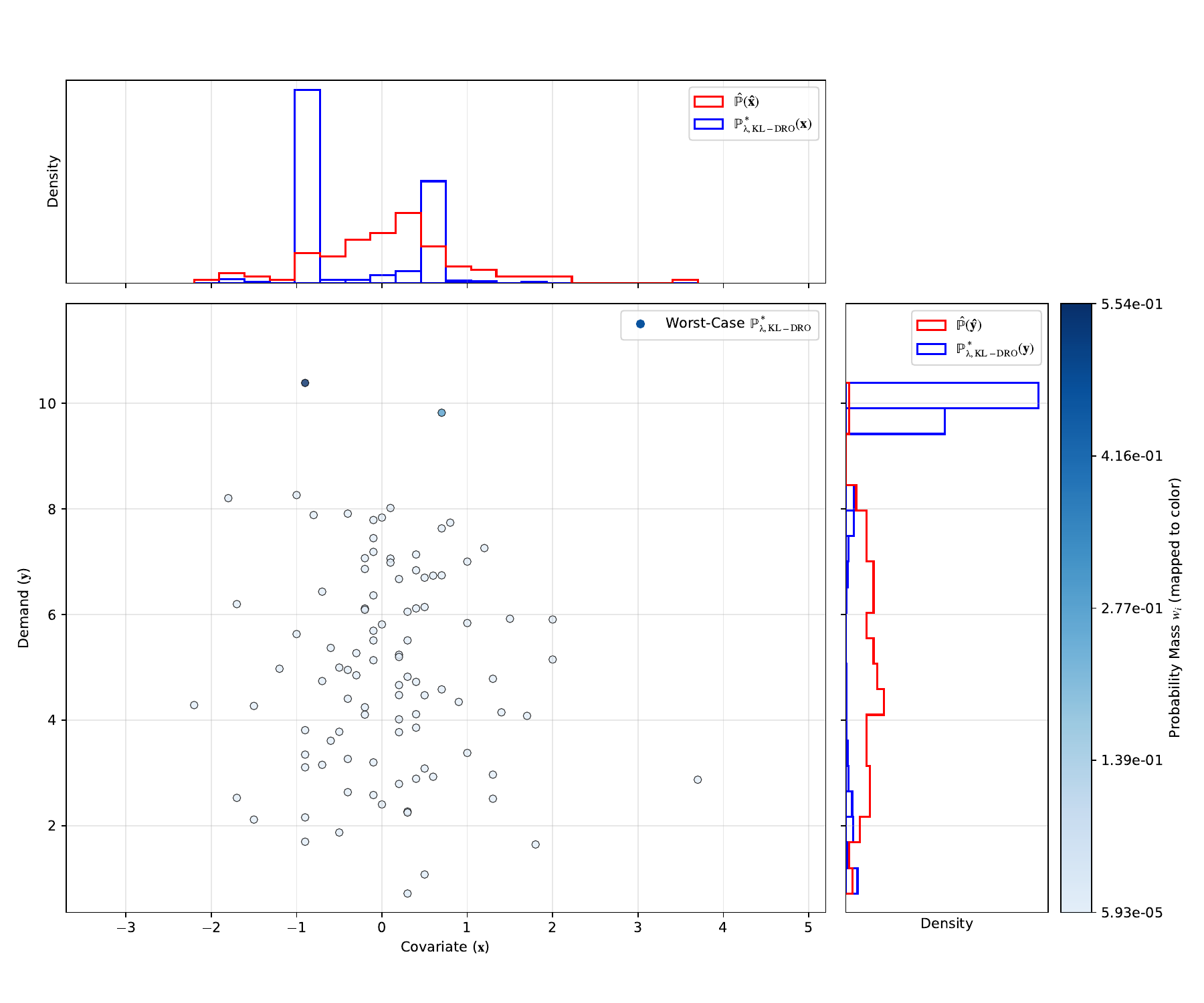}
    \caption{$\mathbb{P}_{\lambda, \textnormal{KL-DRO}}^*$ vs. $\widehat{\mathbb{P}}$}
    \label{fig:worst-case-kldro}
  \end{subfigure}
  \caption{Structure of distributions $\widehat{\mathbb{P}}$ (red points in 1a-1c), $\mathbb{P}_{\lambda}^*$, $\mathbb{P}_{\lambda, \textnormal{SDRO}}^*$, $\mathbb{P}_{\lambda, \textnormal{Causal-WDRO}}^*$, and $\mathbb{P}_{\lambda, \textnormal{KL-DRO}}^*$.($p=2$, $\lambda=0.5$, $\epsilon=0.05$, and $N=100$)}
  \label{fig:worst-case}
\end{figure}

Theorem~\ref{corollary-worst-case} reveals that the worst-case distribution $\mathbb{P}_{\lambda}^*$ is a mixture of Gibbs distributions. 
We next compare $\mathbb{P}_{\lambda}^*$ with the worst-case distribution formulation of soft-constrained Sinkhorn DRO (SDRO), Causal Wasserstein DRO (Causal-WDRO), and KL-divergence-based DRO (KL-DRO) models in contextual settings, denoted by $\mathbb{P}_{\lambda,\textnormal{SDRO}}^*$, $\mathbb{P}_{\lambda, \textnormal{Causal-WDRO}}^*$, and $\mathbb{P}^*_{\lambda,\textnormal{KL-DRO}}$. 
We summarize the formulations of these models and worst-case distributions in Appendix~\ref{ecsec-wc-distribution}. 

\setcounter{example}{0}
\begin{example}
    \textnormal{\textbf{(Revisited). }} Consider a single-product feature-based newsvendor problem with one covariate (see Example~\ref{example-news}, $d_x = d_y = d_z =1$). 
    We choose the Lebesgue measure as the reference measure. Consider a true decision rule $f=f_{\text{true}}$, in Figure~\ref{fig:worst-case}, we show the structure of distributions $\widehat{\mathbb{P}}$, $\mathbb{P}_{\lambda}^*$, $\mathbb{P}_{\lambda, \textnormal{SDRO}}^*$, $\mathbb{P}_{\lambda, \textnormal{Causal-WDRO}}^*$, and $\mathbb{P}_{\lambda, \textnormal{KL-DRO}}^*$, as well as their marginal probability density or mass. 
    Comparing Figure~\ref{fig:worst-case}(a-b) with Figure~\ref{fig:worst-case}(c-d), we see that entropic regularization yields continuous worst-case distributions. 
    Comparing Figure~\ref{fig:worst-case}(a) with Figure~\ref{fig:worst-case}(b), the probability density of $\mathbb{P}^*_{\lambda}$ is more concentrated than that of $\mathbb{P}^*_{\lambda, \mathrm{SDRO}}$, as Causal-SDRO strictly remains causally consistent, thereby avoiding causally implausible robustness scenarios. 
     $\hfill \clubsuit$ 
\end{example}

\vspace{-1em}
\section{ Soft Regression Forest Decision Rule } \label{sec-srf}
\vspace{-0.8em}

Optimizing policies in a general measurable function space is computationally challenging due to the space's infinite dimensionality. 
Instead, we consider a parametric decision rule $f:~\mathcal{X}\to\mathcal{Z}$ that approximates the optimal mapping from features to decisions.  
In this section, we propose a parametric and interpretable Soft Regression Forest (SRF) decision rule. 

Unlike uninterpretable deep-learning-based methods, we propose a soft regression forest (SRF) decision rule based on the principles of soft decision trees~\citep{frosst2017distilling} and ensemble learning to capture the `if-then' decision-making structure. Compared with the traditional decision-tree-based methods, SRF is parametric and differentiable while retaining intrinsic interpretability. 

\begin{figure}[!htb]
    \centering
    \includegraphics[width=0.95\linewidth]{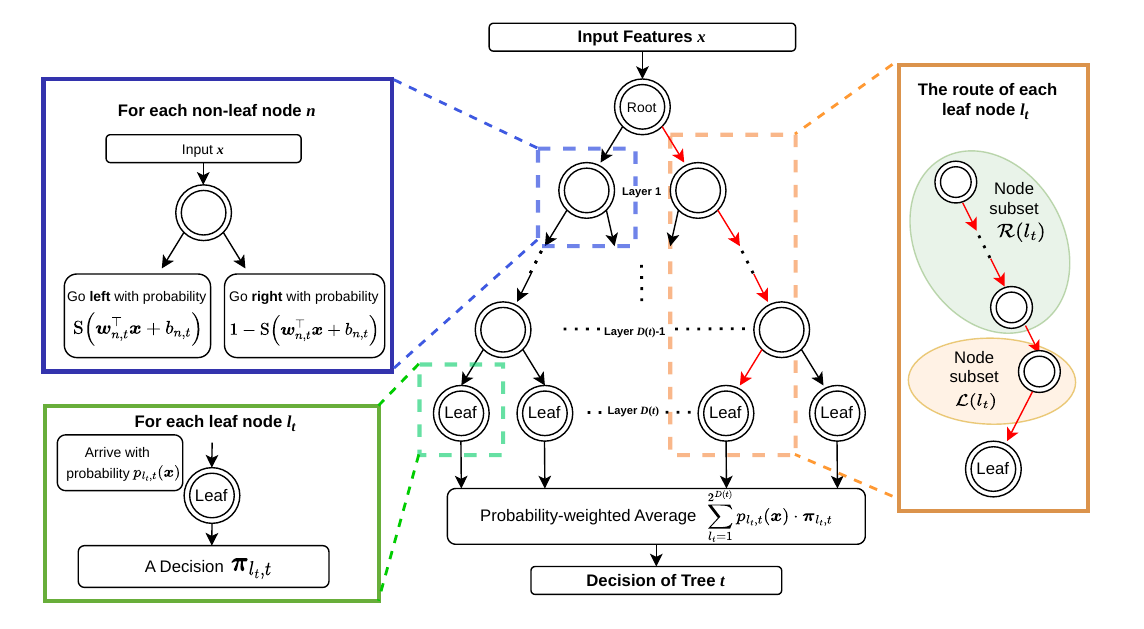}
    \caption{Structure of a soft regression tree $t$ with depth $D(t)$.}
    \label{fig:SRT}
\end{figure}
\vspace{-1em} 

The SRF consists of $T$ full binary Soft Regression Trees (SRTs). For the $t$-th tree ($t \in [T]$) with depth $D(t) \ge 2$, each leaf node $l_t\in \left[2^{D(t)}\right]$ corresponds to a decision $\boldsymbol{\pi}_{l_t,t} \in \mathcal{Z}$ and a unique route from the root node in the tree.
Unlike a hard decision tree, which selects only one child node at each branch, a Sigmoid function $\text{S}(\boldsymbol{w}_{j,t}^{\top}\boldsymbol{x} + b_{j,t})$ determines the probability of directing the input $\boldsymbol{x}$ to the left child at each internal node $j \in [2^{D(t)}-1]$ for $\boldsymbol{w}_{j,t} \in \mathbb{R}^{d_x} $ and $b_{j,t} \in \mathbb{R}$. 
Let $\mathcal{L}(l_t)$ and $\mathcal{R}(l_t)$ denote  the left- and right-hand-side node sets on the route of leaf node $l_t$ in tree $t$, respectively, and $\Lambda (l_t) := \mathcal{L}(l_t) \cup \mathcal{R}(l_t)$. 
Consequently, the probability that a given input covariate $\boldsymbol{x}$ reaches leaf node $l$ (corresponds to the decision $\boldsymbol{\pi}_{l_t, t}$) in tree $t$ is given by 
\begin{equation}\nonumber
    p_{l_t, t} (\boldsymbol{x}) = \prod_{i \in \mathcal{L}(l_t)} \text{S}(\boldsymbol{w}_{i,t}^{\top}\boldsymbol{x} + b_{i,t}) \cdot \prod_{j \in \mathcal{R}(l_t)} \Big(1 - \text{S}(\boldsymbol{w}_{j,t}^{\top}\boldsymbol{x} + b_{j,t}) \Big) . 
\end{equation}
Assume that $\mathcal{Z}$ is a convex set. 
The final output of the SRF is the ensemble average of the expected decisions from all trees, and thus the SRF decision rule $f_{\boldsymbol{\theta}}^{\text{SRF}}: \mathcal{X} \to \mathcal{Z}$ is explicitly defined as 
\begin{equation}
    \Big[f_{\boldsymbol{\theta}}^{\text{SRF}}(\boldsymbol{x})\Big]_k := \frac{1}{T}\sum_{t=1}^{T} \sum_{l_t=1}^{2^{D(t)}} p_{l_t, t} (\boldsymbol{x}) \cdot \Big[ \boldsymbol{\pi}_{l_t, t} \Big]_k , \quad \forall k \in \left [d_z\right],  \tag{\text{SRF}} 
\end{equation} 
where $[\boldsymbol{\pi}_{l_t, t}]_k$ represents the $k$-th decision for any $k \in \left[d_z\right]$, 
the vector $\boldsymbol{\theta} \in \Theta$ is the collection of all individual parameters $\{\boldsymbol{w}_{i,t}, \boldsymbol{b}_{i,t}, \boldsymbol{\pi}_{l_t, t}\}$ for each non-leaf node $i \in \Lambda (l_t)$, leaf node $l_t \in \left[2^{D(t)}\right]$ and tree $t\in \left[T\right]$. This decision rule has $(d_x+1)\cdot \sum_{t=1}^{T}  (2^{D(t)}-1) +d_z\cdot \sum_{t=1}^{T} 2^{D(t)}$ parameters. 

Recall that~\citet{bertsimas2020predictive} and~\citet{kallus2023stochastic} introduce tree-based models for solving contextual optimization. 
In their framework, they use hard split decision trees to estimate the conditional local weights in the weighted sample average approximation method. 
In comparison, the proposed parametric SRF decision rule makes end-to-end decisions based on covariates. 

Figure~\ref{fig:SRT} shows the structure of a single SRT $t$ with $D(t)$ layers (let the root node be in layer 0). As illustrated, on the left-hand side we show the input and output structure of each node, while on the right-hand side we show the route and the node subsets $\mathcal{L}(l_t)$ and $\mathcal{R}(l_t)$ for each leaf node. 
We further demonstrate the intrinsic interpretability of SRF, including its Lipschitz smoothness, consistency with traditional decision trees, and both global and local interpretation, in Appendix~\ref{subsec-srf-interpret}. 
We also examine the interpretability of SRF in one of our applications in Appendix~\ref{subsec-results-portfolio-interpret}. 

\vspace{-0.8em}
\section{Algorithm for Solving Causal-SDRO} \label{sec-algo}
\vspace{-0.8em}

In this section, we propose an algorithm to solve the Causal-SDRO problem. 
Section~\ref{subsec-algo-reform} reformulates~\eqref{soft-causal-sdro-dual} into a three-level stochastic compositional optimization problem. 
For this tractable formulation, Section~\ref{subsec-algo-sco} develops a gradient-based algorithm. 

\vspace{-0.8em}
\subsection{Tractable Reformulation for Causal-SDRO} \label{subsec-algo-reform}
\vspace{-0.8em}

{We parameterize the decision rule by parameter vector $\boldsymbol{\theta} \in \Theta \subseteq \mathbb{R}^{d_{\theta}}$. 
Since $\widehat{\boldsymbol{X}}$ takes $M$ distinct values with probabilities $\widehat{p}_m:=n_m/N$, we rewrite the dual of~\eqref{soft-causal-sdro} as: 
\begin{subequations}\label{whole-theta-soft-causal-sdro}
    \begin{equation}\label{theta-soft-causal-sdro}
    \min_{\boldsymbol{\theta} \in \Theta}~\left\{F \left(\boldsymbol{\theta}\right)= \sum_{m=1}^{M} \widehat{p}_m \cdot\lambda\epsilon \log\, \mathbb{E}_{ \boldsymbol{\xi}_1 \sim U_{\epsilon}}\left[ \exp\left( \frac{h(\boldsymbol{\theta};\widehat{\boldsymbol{x}}_m,  \boldsymbol{\xi}_1, \lambda)}{\lambda\epsilon} \right) \right]\right\},
\end{equation}
where 
\begin{equation}\label{theta-soft-causal-sdro-h}
    h(\boldsymbol{\theta};\widehat{\boldsymbol{x}}_m,  \boldsymbol{\xi}_1) = \mathbb{E}_{\widehat{\boldsymbol{y}} \sim  \widehat{\mathbb{P}}_{\widehat{\boldsymbol{Y}}|\widehat{\boldsymbol{X}}=\widehat{\boldsymbol{x}}_m}}\left[ \lambda\epsilon \log\,  \mathbb{E}_{ \boldsymbol{\xi}_2 \sim W_{\epsilon}}\left[ \exp\left( \frac{\Psi(f_{\boldsymbol{\theta}}(\widehat{\boldsymbol{x}}_m + \boldsymbol{\xi}_1),\widehat{\boldsymbol{y}}+ \boldsymbol{\xi}_2)}{\lambda\epsilon} \right) \right] \right].
\end{equation}
\end{subequations} }

{For any distinct covariate $\widehat{\boldsymbol{x}}_m$, we define the set of $n_m$ observations of the uncertain parameter corresponding to it as $\mathcal{Y}_m$. 
Then, the empirical conditional probability weights of $\widehat{\mathbb{P}}_{\widehat{\boldsymbol{Y}} \mid\widehat{\boldsymbol{X}}=\widehat{\boldsymbol{x}}_m}$ in Eq.~\eqref{theta-soft-causal-sdro-h} can be defined as $p\, (\widehat{\boldsymbol{y}}_j \mid \widehat{\boldsymbol{x}}_m) : = \mathbbm{1}_{\{ \boldsymbol{y}_j \in \mathcal{Y}_m\}} /n_m$ for given $\widehat{\boldsymbol{x}}_m$ and any $j \in [N]$, which satisfies $\sum_{j=1}^N p(\widehat{\boldsymbol{y}}_j \mid \widehat{\boldsymbol{x}}_m) = 1$, where $\mathbbm{1}_{\{ \boldsymbol{y}_j \in \mathcal{Y}_m\}}$ equals 1 when $\boldsymbol{y}_j \in \mathcal{Y}_m$ and 0 otherwise. }

Therefore, by rewriting the outer expectation in Eq.~\eqref{theta-soft-causal-sdro-h} as a finite weighted sum, we reformulate the problem~\eqref{whole-theta-soft-causal-sdro} as a {three-level Conditional Stochastic Compositional Optimization (CSCO) problem}, driven by three mutually independent random vectors $\widehat{\boldsymbol{x}}$, $\boldsymbol{\xi}_1$, and $\boldsymbol{\xi}_2$:  
{\begin{equation}\label{dual-sco}
    \min_{\boldsymbol{\theta} \in \Theta}~\lambda\epsilon \cdot \sum_{m=1}^{M}\widehat{p}_m \cdot t_1\Big( \mathbb{E}_{ \boldsymbol{\xi}_1 \sim U_{\epsilon}}\Big[ t_2\Big(   \mathbb{E}_{ \boldsymbol{\xi}_2 \sim W_{\epsilon}}\Big[t_3\Big( \boldsymbol{\theta}; \widehat{\boldsymbol{x}}_m, \boldsymbol{\xi}_1, \boldsymbol{\xi}_2  \Big)\Big]; \widehat{\boldsymbol{x}}_m, \boldsymbol{\xi}_1\Big)\Big]; \widehat{\boldsymbol{x}}_m \Big),
    \tag{CSCO}
\end{equation}
where
\begin{equation}\label{t-functions}
    \begin{aligned}
    t_1&: \mathbb{R}_+ \to \mathbb{R}, & t_1(z;\widehat{\boldsymbol{x}}_m) & = \log\, (z), \\
    t_2&: \mathbb{R}^{N} \to \mathbb{R}_+, & t_2 (\boldsymbol{v}; \widehat{\boldsymbol{x}}_m, \boldsymbol{\xi}_1) &=  \exp\Big(\sum_{j=1}^{N} p(\widehat{\boldsymbol{y}}_j \mid \widehat{\boldsymbol{x}}_m) \cdot \log\, \left( \boldsymbol{v}_j \right) \Big), \\
    t_3&: \Theta \to \mathbb{R}^{N}, & \Big[t_3 (\boldsymbol{\theta}; \widehat{\boldsymbol{x}}_m, \boldsymbol{\xi}_1, \boldsymbol{\xi}_2)\Big]_j &= \exp\left( \frac{\Psi(f_{\boldsymbol{\theta}}(\widehat{\boldsymbol{x}}_m + \boldsymbol{\xi}_1),\widehat{\boldsymbol{y}}_j + \boldsymbol{\xi}_2)}{\lambda\epsilon} \right), \forall j \in \left[N\right].
\end{aligned}
\end{equation}}In the above expression, $v_j$ denotes the $j$-th component of the vector $\boldsymbol{v} \in \mathbb{R}^N$.
We impose the following assumptions, following the conventions in the related literature (see, e.g., \citet{hu2020sample}). 

\begin{assumption}\label{assumption-2}
    \begin{enumerate}
    \item (Lipschitz Continuity and Smoothness). For any $\boldsymbol{x}$ and $\boldsymbol{y}$ and given parametric decision rule, the loss function $L(\boldsymbol{\theta}; \boldsymbol{x}, \boldsymbol{y}) := \Psi(f_{\boldsymbol{\theta}}(\boldsymbol{x}), \boldsymbol{y})$ is: (1) differentiable with respect to $\boldsymbol{\theta}$ on $\Theta$; (2) $L_{\boldsymbol{\theta}}$-Lipschitz continuous and $S_{\boldsymbol{\theta}}$-Lipschitz smooth with respect to $\boldsymbol{\theta}$. 
    \label{assumption-2-lip-con}
    \item (Bounded Cost). The loss function satisfies $0\le \Psi(\boldsymbol{z}, \boldsymbol{y}) \le B$ for any $\boldsymbol{z} \in \mathcal{Z}$ and $\boldsymbol{y} \in \mathcal{Y}$. \label{assumption-2-bound}
    \end{enumerate}
\end{assumption} 
Assumption~\ref{assumption-2}\ref{assumption-2-lip-con} is crucial for deriving the convergence result of the gradient-based algorithms. 
For this problem, we prove its properties under Assumption~\ref{assumption-2} and present the sample complexity of the SAA method in Appendix~\ref{app-sec-SCO}, thereby motivating the development of an efficient algorithm. 

\vspace{-0.8em}
\subsection{Stochastic Compositional Gradient-based Algorithm} \label{subsec-algo-sco} 
\vspace{-0.8em} 

Next, we develop a gradient-based algorithm for~\eqref{dual-sco}. 
{Since the compositional functions in~\eqref{dual-sco} are generally non-linear, the stochastic gradient descent (SGD) method may lead to a biased gradient estimator. Thus, SGD cannot be used directly (see Appendix~\ref{app-subsec-algo-scsc}).
\cite{wang2017stochastic} and \cite{chen2021solving} develop gradient-based methods for standard stochastic compositional optimization. However, these methods cannot be applied directly to~\eqref{dual-sco} because its inner expectation also depends on outer random variables. Their problems are special cases of CSCO. } 

{To address this conditional structure, we develop a Conditional Mini-batch variant of the Stochastically Corrected Stochastic Compositional gradient method, termed CM-SCSC, for~\eqref{dual-sco}. The method combines condition-specific mini-batch estimation with same-sample corrections of the auxiliary tracking variables. 
At iteration $k$, we draw $m_k \in [M]$, $\boldsymbol{\xi}_1^{k} \sim U_{\epsilon}$ and $Q$ i.i.d. samples $\boldsymbol{\xi}_2^{k, 1}, \cdots, \boldsymbol{\xi}_2^{k, Q} \sim W_{\epsilon}$, which are mutually independent and independent of all samples generated before iteration $k$. 
For brevity, for any $q \in [Q]$, $k=0,\cdots,K-1$, and each $\widehat{\boldsymbol{x}}_m$, we define 
\begin{equation}\nonumber
    t_{1,m}(z) := t_1(z; \widehat{\boldsymbol{x}}_m), \quad t_{2,m}^k(\boldsymbol{v}) := t_2( \boldsymbol{v}; \widehat{\boldsymbol{x}}_m, \boldsymbol{\xi}_1^{k}),  \text{ and  } t_{3, m}^{k,q}(\boldsymbol{\theta}) := t_3( \boldsymbol{\theta}; \widehat{\boldsymbol{x}}_m, \boldsymbol{\xi}_1^{k}, \boldsymbol{\xi}_2^{k,q}). 
\end{equation}} 
{In each iteration $k$, for each historical covariate value $\widehat{\boldsymbol{x}}_m$, the expectations of $t_3$ with respect to $\boldsymbol{\xi}_2$ and of $t_2$ with respect to $\boldsymbol{\xi}_1$ are estimated by $\boldsymbol{y}_{2,m}^{k+1}\in\mathbb{R}^N$ and $y_{1,m}^{k+1}\in\mathbb{R}$, respectively. 
Then, we update the parameters using the sampled covariate group $m_k$: 
\begin{equation}\label{theta-update}
    \boldsymbol{\theta}^{k+1} := \boldsymbol{\theta}^{k} - \alpha_k \lambda \epsilon \cdot \nabla t_{1,m_k}^k(y_{1,m_k}^{k+1}) \cdot \Big [ \frac{1}{Q} \sum_{q=1}^Q J_{\boldsymbol{\theta}}t_{3,m_k}^{k,q} (\boldsymbol{\theta}^{k}) \Big ]^\top \cdot \nabla t_{2, m_k}^k (\boldsymbol{y}_{2, m_k}^{k+1}) . 
\end{equation}}
{The pseudo-code of the CM-SCSC is provided in Algorithm~\ref{algo-SCSC-ML}. We show the update of $\boldsymbol{y}_{2,m}^{k}$ and $y_{1,m}^{k}$ in Steps 6 and 8, respectively. 
In Step 8, we apply a truncation on the value of $y_{1,m}^{k+1}$, since under Assumption~\ref{assumption-2}\ref{assumption-2-bound}, the auxiliary variable $y_{1,m}^{k+1} \in [1, \exp(B/\lambda\epsilon)]$ for any $k$ and $m$. }

\begin{algorithm}[!ht]
\caption{{ CM-SCSC for the problem~\eqref{dual-sco}} } 
\label{algo-SCSC-ML}
    1: Initialize $\boldsymbol{\theta}^{-1}=\boldsymbol{\theta}^{0}=\boldsymbol{\theta}^{\mathrm{prev}}$, $y_{1,m}^{0}=1$ for all $m\in[M]$, {$Q\ge1$}, and stepsizes $\{\alpha_k,\beta_k\}_{k=0}^{K-1}$; \\
    2: \textbf{for} $k = 0, \cdots, K-1$, \textbf{do} \\
    3: \quad Independently sample $m_k \in [M]$ (with probability $\widehat{p}_m$),  $\boldsymbol{\xi}_1^{k} \sim U_{\epsilon}$, and {$\boldsymbol{\xi}_2^{k,1}, \cdots \boldsymbol{\xi}_2^{k,Q} \sim W_{\epsilon}$}; \\
    4: \quad { \textbf{for} $m \in [M]$ \textbf{do} } \\
    5: \qquad {Compute $t_{3,m}^{k,q}(\boldsymbol{\theta}^{k})$ and $t_{3,m}^{k,q}(\boldsymbol{\theta}^{\mathrm{prev}})$ for all $q \in [Q]$}; \\ 
    6: \qquad {Update $\boldsymbol{y}_{2,m}^{k} = \frac{1}{Q}\sum_{q=1}^Q t_{3,m}^{k,q}(\boldsymbol{\theta}^{k})$ and $\boldsymbol{y}_{2,m}^{\mathrm{prev}} = \frac{1}{Q}\sum_{q=1}^Q t_{3,m}^{k,q}(\boldsymbol{\theta}^{\mathrm{prev}})$}; \\
    7: \qquad {Compute $H_m^{k} = t_{2,m}^k(\boldsymbol{y}_{2,m}^{k+1})$ and $H_m^{\mathrm{prev}} = t_{2,m}^k(\boldsymbol{y}_{2,m}^{\mathrm{prev}})$}; \\
    8: \qquad { Update $y_{1,m}^{k+1} \gets \Pi_{[1, \exp(B/\lambda\epsilon)]} \Big [ \left ( 1-\beta _k\right ) \cdot \Big ( y_{1,m}^{k} + H_m^{k} - H_m^{\mathrm{prev}} \Big )  + \beta_k \cdot H_m^{k} \Big ] $}; \\
    9: { \quad \textbf{end for} \\ }
    10: \ \ Update $\boldsymbol{\theta}^{k+1}$ by
    Eq.~\eqref{theta-update}, and $\boldsymbol{\theta}^{\text{prev}} \gets \boldsymbol{\theta}^{k}$; \\ 
    11: \textbf{end for}
\end{algorithm} 
        
{The convergence of CM-SCSC for~\eqref{dual-sco} is given by Theorem~\ref{theo-scsc-conver}. 
\begin{theorem}\label{theo-scsc-conver}
    \textnormal{\textbf{(Convergence of CM-SCSC for the Problem~\eqref{dual-sco}).}} Suppose that the iterates generated by Algorithm 1 remain in $\Theta$ almost surely. Under Assumption~\ref{assumption-2}, there exist constants $K_0, A_1, A_2>0$ such that, if we choose an integer $K\ge K_0$, step sizes $\alpha_k=A_1/\sqrt{K}$, $\beta_k=A_2 \alpha_k$, and batch size $Q=\lceil K^{1/4}\rceil$, then 
    \begin{enumerate}
        \item The iterates $\{\boldsymbol{\theta}^k\}$ of the Algorithm~\ref{algo-SCSC-ML} satisfy: 
        \begin{equation}\nonumber
            \frac{1}{K}\sum_{k=0}^{K-1} \mathbb{E}\Big[\|\nabla F(\boldsymbol{\theta}^k)\|^2\Big] \le \frac{C_{\textnormal{const}}}{\sqrt{K}}, 
        \end{equation}
        where $ K_0, A_1, A_2, C_{\textnormal{const}} > 0$ are finite constants independent of $K$ and $Q$; \label{theo-scsc-conver-1} 
        \item To obtain an $\varepsilon$-stationary point of function $F$, i.e., a point $\widehat{\boldsymbol{\theta}}$ satisfying $\mathbb{E}\Big[\|\nabla F(\widehat{\boldsymbol{\theta}})\|^2\Big] \le \varepsilon^2$, the number of iterations required, the sample complexity (i.e., the number of samples required), and the gradient complexity (i.e., the number of gradient computations required) are $\mathcal{O}(\varepsilon^{-4})$, $\mathcal{O}(\varepsilon^{-5})$, and $\mathcal{O}(\varepsilon^{-5})$, respectively. 
        \label{theo-scsc-conver-2}
    \end{enumerate}
\end{theorem}} 
The proof of Theorem~\ref{theo-scsc-conver} is provided in Appendix~\ref{ecsec-theo-scsc}. 
{Theorem~\ref{theo-scsc-conver} establishes an $\mathcal{O}(K^{-1/2})$ bound on the average expected squared gradient norm over $K$ iterations, which is of the same order as SGD's rate for stochastic non-compositional nonconvex problems. 
However, in our setting, CM-SCSC requires $\mathcal{O}(\varepsilon^{-5})$ samples, compared with the $\mathcal{O}(\varepsilon^{-4})$ sample complexity of standard SGD. 
This additional sampling cost arises from the nonlinear function $t_2$, which introduces a bias when its inner expectation is approximated by a few samples. 
If we set the batch size $Q=1$, this bias generally does not vanish as $K$ increases. 
Therefore, we choose $Q = \left \lceil K^{1/4} \right \rceil$ to control the error. 
This suboptimal sampling complexity may be improved by multi-level Monte Carlo sampling methods~\citep{hu2023contextual, hu2024multilevel}, which could be a future direction. } 

{
\begin{remark}[Memory and Computational Costs]
    For the CM-SCSC algorithm: (1) The per-iteration computational cost, i.e., the number of scalar function evaluations in each iteration, is $\mathcal{O}(MNQ)$. (2) The per-iteration gradient complexity, i.e., the number of gradient computations in each iteration, is $\mathcal{O}(NQ)$. (3) The memory cost is $\mathcal{O}(M+N)$, as it maintains $M$ scalars and two $N$-dimensional inner estimations. 
    For a target convergence threshold $\varepsilon$, the number of iteration $K = \mathcal{O}(\varepsilon^{-4})$ is independent of $M$ and $N$ and we have $Q = \mathcal{O}(\varepsilon^{-1})$, $M=N=\mathcal{O}(1)$. Then, the total computational cost and gradient complexity are both $\mathcal{O}(\varepsilon^{-5})$, and the memory cost is $\mathcal{O}(1)$.
\end{remark} } 

\section{Numerical Results}\label{sec-results}

In this section, we evaluate the performance of the proposed approach on the feature-based newsvendor problem (see Example~\ref{example-news}). 
We further apply our approach to the inventory substitution problem using synthetic data and the portfolio selection problem using real data, with detailed results reported in Appendices~\ref{subsec-results-inventory} and~\ref{subsec-results-portfolio}, respectively.  
Additionally, we demonstrate the interpretability of the SRF decision rule for the portfolio problem in Section~\ref{subsec-results-portfolio-interpret}. 
The source code and experimental results are provided in the supplementary material. 
Codes and results are available at \url{https://github.com/Arthas819/CausalSDRO}. 

{In Appendix~\ref{Appdex-sec-losses}, we provide the detailed training losses and verify the conditions required for the convergence analysis of CM-SCSC in all three applications. 
For newsvendor and inventory substitution problems that use synthetic data, we set the number of uncertain parameters associated with each covariate as $n_m = \{1,5,10\}$. 
For the portfolio problem using real-world data, we employ a k-nearest-neighbor (k-NN) approach to group similar covariates into $M$ groups, each containing $n_m \in \{1, 10\}$ observations; please see Appendix~\ref{subsec-results-portfolio} for details. } 

In all experiments, we train the decision rule by solving the soft-constrained Causal-SDRO model. Following Remark~\ref{remark-soft-cons}, we treat the penalty coefficient $\lambda$ as a tunable hyperparameter instead of the radius $\bar{\rho}$. 
We use the $\ell_p$-norm as the transport cost in the causal Sinkhorn discrepancy, with $p \in \{1, 2\}$. For brevity, we denote the resulting model as \textit{$p$-Causal-SDRO} and solve it using the proposed CM-SCSC algorithm.
We evaluate the out-of-sample performance using the coefficient of prescriptiveness \citep{bertsimas2020predictive}, which is defined as 
\begin{equation}\label{out-of-sample-performance-value}
    \text{Out-of-sample Performance}(\boldsymbol{\theta}) := \Big(1- \frac{\mathcal{H}\left ( \boldsymbol{\theta} \right ) - \mathcal{H}^*}{\mathcal{H}\left ( \boldsymbol{\theta}^{\text{ERM}} \right ) - \mathcal{H}^*} \Big ) \times 100\%, 
\end{equation}
where $\mathcal{H}\left ( \boldsymbol{\theta} \right )$ is the average out-of-sample loss of the trained decision rule $f_{\boldsymbol{\theta}}$, $\mathcal{H}^*$ represents the oracle loss under perfect information, and $\boldsymbol{\theta}^{\text{ERM}}$ denotes the parameter trained by the empirical risk minimization (ERM) model. 
A higher value indicates a better out-of-sample performance of decision rules. 
We place other necessary detailed settings in Appendix~\ref{app-sec-results}. 

\begin{figure}[!htb]

  \centering
    \includegraphics[width=\linewidth]{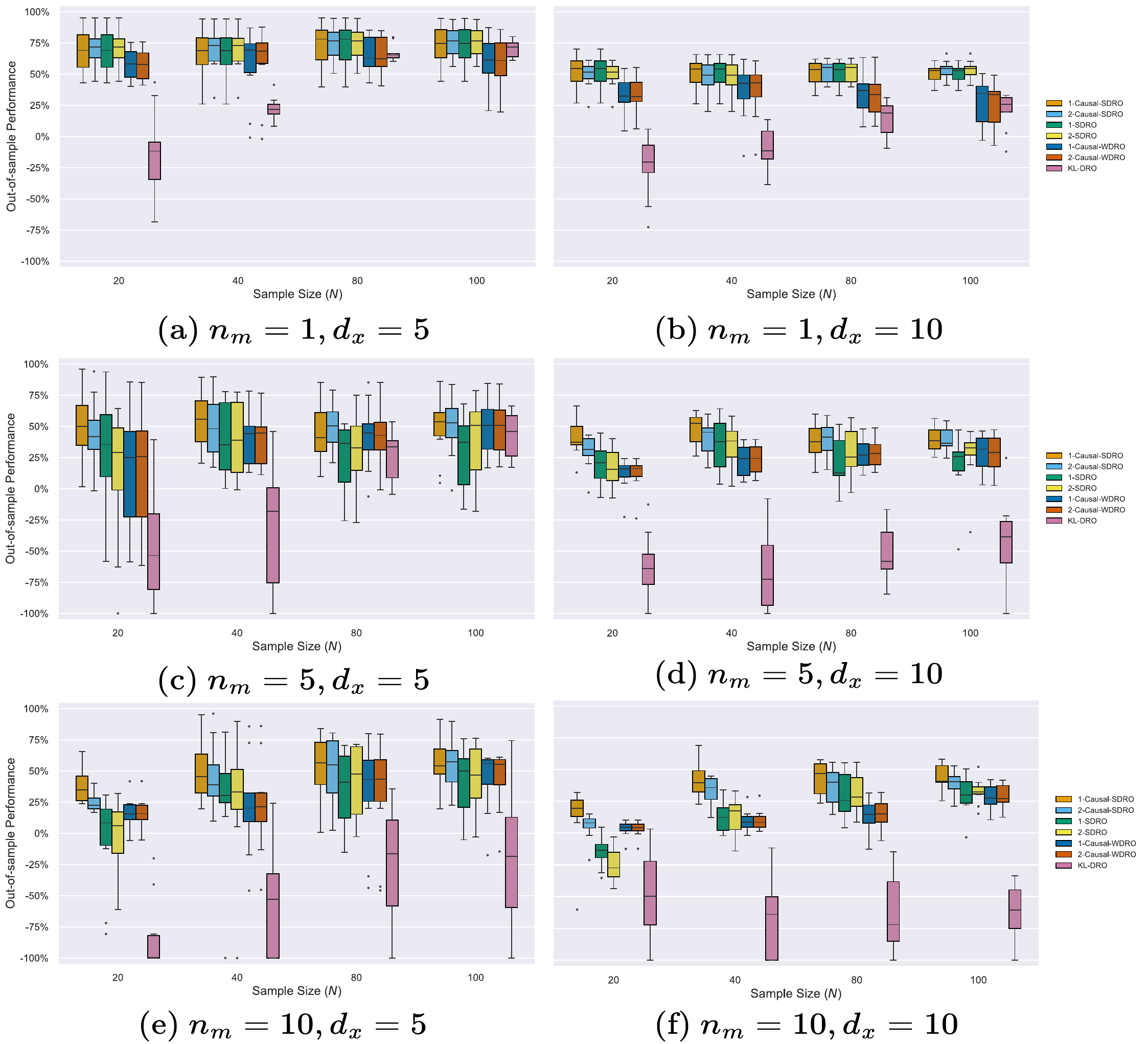}
  \caption{{Comparison of different DRO models on the newsvendor problem  ($n_m \in \{1,5,10\}, d_x \in \{5,10\}$)}}
  \label{fig:news-dro-box} 
\end{figure}

{We report selected numerical results for the feature-based newsvendor problem, with additional numerical results for all applications provided in Appendix~\ref{app-sec-results}. 
Figure~\ref{fig:news-dro-box} compares the out-of-sample performance of the proposed 1-Causal-SDRO and 2-Causal-SDRO models against three DRO benchmarks, namely SDRO, Causal-WDRO, and KL-DRO, all equipped with the same SRF decision rule (see Appendix~\ref{ecsec-wc-distribution}), across $n_m \in \{1,5,10\}$, sample sizes $N \in \{20,40,80,100\}$, and feature dimensions $d_x \in \{5, 10\}$. 
Note that when $n_m=1$, each covariate value occurs only once, so the causal constraint is automatically satisfied and Causal-SDRO is equivalent to SDRO. 
Therefore, we only solve Causal-SDRO and use the resulting solution for both methods in this case. 
The parameters of these DRO models are determined by cross-validation. 
As illustrated, both Causal-SDRO models achieve positive average out-of-sample performance and outperform the ERM baseline on nearly all instances.  
Specifically, for $n_m = 1, 5$, and $10$, respectively, the average out-of-sample performance is 60.6\%, 45.7\%, and 42.3\% for 1-Causal-SDRO, and 61.5\%, 42.7\%, and 36.2\% for 2-Causal-SDRO. Moreover, both variants achieve higher average performance than the DRO benchmarks in nearly all configurations. }

\section{Conclusion}\label{sec-conclusion}

In this paper, we consider the causal and continuous structure of the underlying distribution for contextual DRO. 
We develop a new contextual DRO framework, Causal-SDRO, that constructs the ambiguity set using the entropy-regularized causal Wasserstein distance, thereby excluding discrete and causally implausible distributions. 
We then develop a soft regression forest decision rule that approaches optimal policies over measurable, continuous, and differentiable function spaces while preserving interpretability and computational tractability. 
{To solve the resulting model, we develop a gradient-based algorithm that converges to an $\varepsilon$-stationary point with $\mathcal{O}(\varepsilon^{-4})$ iteration complexity and $\mathcal{O}(\varepsilon^{-5})$ sample complexity. } Future work may focus on designing more informative ambiguity sets and extending the framework to practical applications requiring robustness and interpretability. 

\bibliography{iclr2027_conference}

@article{sadana2025survey,
  title={A survey of contextual optimization methods for decision-making under uncertainty},
  author={Sadana, Utsav and Chenreddy, Abhilash and Delage, Erick and Forel, Alexandre and Frejinger, Emma and Vidal, Thibaut},
  journal={European Journal of Operational Research},
  volume={320},
  number={2},
  pages={271--289},
  year={2025},
  publisher={Elsevier}
}

@article{nguyen2020distributionally,
  title={Distributionally robust local non-parametric conditional estimation},
  author={Nguyen, Viet Anh and Zhang, Fan and Blanchet, Jose and Delage, Erick and Ye, Yinyu},
  journal={Advances in Neural Information Processing Systems},
  volume={33},
  pages={15232--15242},
  year={2020}
}

@article{chenreddy2022data,
  title={Data-driven conditional robust optimization},
  author={Chenreddy, Abhilash Reddy and Bandi, Nymisha and Delage, Erick},
  journal={Advances in Neural Information Processing Systems},
  volume={35},
  pages={9525--9537},
  year={2022}
}

@article{nguyen2025robustifying,
  title={Robustifying conditional portfolio decisions via optimal transport},
  author={Nguyen, Viet Anh and Zhang, Fan and Wang, Shanshan and Blanchet, Jose and Delage, Erick and Ye, Yinyu},
  journal={Operations Research},
  volume={73},
  number={5},
  pages={2801--2829},
  year={2025},
  publisher={INFORMS}
}

@article{hu2023contextual,
  title={Contextual stochastic bilevel optimization},
  author={Hu, Yifan and Wang, Jie and Xie, Yao and Krause, Andreas and Kuhn, Daniel},
  journal={Advances in Neural Information Processing Systems},
  volume={36},
  pages={78412--78434},
  year={2023}
}

@article{wang2025sinkhorn,
  title={Sinkhorn distributionally robust optimization},
  author={Wang, Jie and Gao, Rui and Xie, Yao},
  journal={Operations Research},
  year={2025},
  publisher={INFORMS}
}

@article{ma2018priori,
  title={A priori estimates of the population risk for two-layer neural networks},
  author={Ma, Chao and Wu, Lei and E, Weinan},
  journal={arXiv preprint arXiv:1810.06397},
  year={2018}
}

@article{hu2020sample,
  title={Sample complexity of sample average approximation for conditional stochastic optimization},
  author={Hu, Yifan and Chen, Xin and He, Niao},
  journal={SIAM Journal on Optimization},
  volume={30},
  number={3},
  pages={2103--2133},
  year={2020},
  publisher={SIAM}
}

@book{shapiro2021lectures,
  title={Lectures on stochastic programming: modeling and theory},
  author={Shapiro, Alexander and Dentcheva, Darinka and Ruszczynski, Andrzej},
  year={2021},
  publisher={SIAM}
}

@article{kleywegt2002sample,
  title={The sample average approximation method for stochastic discrete optimization},
  author={Kleywegt, Anton J and Shapiro, Alexander and Homem-de-Mello, Tito},
  journal={SIAM Journal on Optimization},
  volume={12},
  number={2},
  pages={479--502},
  year={2002},
  publisher={SIAM}
}

@article{chen2021solving,
  title={Solving stochastic compositional optimization is nearly as easy as solving stochastic optimization},
  author={Chen, Tianyi and Sun, Yuejiao and Yin, Wotao},
  journal={IEEE Transactions on Signal Processing},
  volume={69},
  pages={4937--4948},
  year={2021},
  publisher={IEEE}
}

@book{casella2024statistical,
  title={Statistical inference},
  author={Casella, George and Berger, Roger},
  year={2024},
  publisher={Chapman and Hall/CRC}
}

@book{cohn2013measure,
  title={Measure theory},
  author={Cohn, Donald L},
  volume={1},
  year={2013},
  publisher={Springer}
}

@article{wang2017stochastic,
  title={Stochastic compositional gradient descent: algorithms for minimizing compositions of expected-value functions},
  author={Wang, Mengdi and Fang, Ethan X and Liu, Han},
  journal={Mathematical Programming},
  volume={161},
  number={1},
  pages={419--449},
  year={2017},
  publisher={Springer}
}

@article{blackwell1963non,
  title={Non-existence of everywhere proper conditional distributions},
  author={Blackwell, David and Ryll-Nardzewski, Czes{\l}aw},
  journal={The Annals of Mathematical Statistics},
  volume={34},
  number={1},
  pages={223--225},
  year={1963},
  publisher={JSTOR}
}

@article{bertsimas2020predictive,
  title={From predictive to prescriptive analytics},
  author={Bertsimas, Dimitris and Kallus, Nathan},
  journal={Management Science},
  volume={66},
  number={3},
  pages={1025--1044},
  year={2020},
  publisher={INFORMS}
}

@article{ben2013robust,
  title={Robust solutions of optimization problems affected by uncertain probabilities},
  author={Ben-Tal, Aharon and Den Hertog, Dick and De Waegenaere, Anja and Melenberg, Bertrand and Rennen, Gijs},
  journal={Management Science},
  volume={59},
  number={2},
  pages={341--357},
  year={2013},
  publisher={INFORMS}
}

@article{blanchet2023unifying,
  title={Unifying distributionally robust optimization via optimal transport theory},
  author={Blanchet, Jose and Kuhn, Daniel and Li, Jiajin and Taskesen, Bahar},
  journal={arXiv preprint arXiv:2308.05414},
  year={2023}
}

@article{hu2013kullback,
  title={Kullback-{L}eibler divergence constrained distributionally robust optimization},
  author={Hu, Zhaolin and Hong, L Jeff},
  journal={Available at Optimization Online},
  volume={1},
  number={2},
  pages={9},
  year={2013}
}

@article{frosst2017distilling,
  title={Distilling a neural network into a soft decision tree},
  author={Frosst, Nicholas and Hinton, Geoffrey},
  journal={arXiv preprint arXiv:1711.09784},
  year={2017}
}

@article{kallus2023stochastic,
  title={Stochastic optimization forests},
  author={Kallus, Nathan and Mao, Xiaojie},
  journal={Management Science},
  volume={69},
  number={4},
  pages={1975--1994},
  year={2023},
  publisher={INFORMS}
}

@article{chen2019stochastic,
  title={Stochastic optimization with decisions truncated by positively dependent random variables},
  author={Chen, Xin and Gao, Xiangyu},
  journal={Operations Research},
  volume={67},
  number={5},
  pages={1321--1327},
  year={2019},
  publisher={INFORMS}
}

@article{ban2019big,
  title={The big data newsvendor: Practical insights from machine learning},
  author={Ban, Gah-Yi and Rudin, Cynthia},
  journal={Operations Research},
  volume={67},
  number={1},
  pages={90--108},
  year={2019},
  publisher={INFORMS}
}

@article{elmachtoub2022smart,
  title={Smart “predict, then optimize”},
  author={Elmachtoub, Adam N and Grigas, Paul},
  journal={Management Science},
  volume={68},
  number={1},
  pages={9--26},
  year={2022},
  publisher={INFORMS}
}

@article{qi2025integrated,
  title={Integrated conditional estimation-optimization},
  author={Qi, Meng and Grigas, Paul and Shen, Zuo-Jun},
  journal={Operations Research},
  year={2025},
  publisher={INFORMS}
}

@article{zhang2024optimal,
  title={Optimal robust policy for feature-based newsvendor},
  author={Zhang, Luhao and Yang, Jincheng and Gao, Rui},
  journal={Management Science},
  volume={70},
  number={4},
  pages={2315--2329},
  year={2024},
  publisher={INFORMS}
}

@article{kannan2024residuals,
  title={Residuals-based distributionally robust optimization with covariate information},
  author={Kannan, Rohit and Bayraksan, G{\"u}zin and Luedtke, James R},
  journal={Mathematical Programming},
  volume={207},
  number={1},
  pages={369--425},
  year={2024},
  publisher={Springer}
}

@article{bertsimas2022data,
  title={Data-driven optimization: A reproducing kernel {H}ilbert space approach},
  author={Bertsimas, Dimitris and Koduri, Nihal},
  journal={Operations Research},
  volume={70},
  number={1},
  pages={454--471},
  year={2022},
  publisher={INFORMS}
}

@article{qi2023practical,
  title={A practical end-to-end inventory management model with deep learning},
  author={Qi, Meng and Shi, Yuanyuan and Qi, Yongzhi and Ma, Chenxin and Yuan, Rong and Wu, Di and Shen, Zuo-Jun},
  journal={Management Science},
  volume={69},
  number={2},
  pages={759--773},
  year={2023},
  publisher={INFORMS}
}

@article{han2025deep,
  title={Deep neural newsvendor},
  author={Han, Jinhui and Hu, Ming and Shen, Guohao},
  journal={Management Science},
  year={2025},
  publisher={INFORMS}
}

@inproceedings{elmachtoub2020decision,
  title={Decision trees for decision-making under the predict-then-optimize framework},
  author={Elmachtoub, Adam N and Liang, Jason Cheuk Nam and McNellis, Ryan},
  booktitle={International Conference on Machine Learning},
  pages={2858--2867},
  year={2020},
  organization={PMLR}
}

@article{liyanage2005practical,
  title={A practical inventory control policy using operational statistics},
  author={Liyanage, Liwan H and Shanthikumar, J George},
  journal={Operations Research Letters},
  volume={33},
  number={4},
  pages={341--348},
  year={2005},
  publisher={Elsevier}
}

@article{esteban2022distributionally,
  title={Distributionally robust stochastic programs with side information based on trimmings},
  author={Esteban-P{\'e}rez, Adri{\'a}n and Morales, Juan M},
  journal={Mathematical Programming},
  volume={195},
  number={1},
  pages={1069--1105},
  year={2022},
  publisher={Springer}
}

@article{srivastava2021data,
  title={On data-driven prescriptive analytics with side information: A regularized {N}adaraya-{W}atson approach},
  author={Srivastava, Prateek R and Wang, Yijie and Hanasusanto, Grani A and Ho, Chin Pang},
  journal={arXiv preprint arXiv:2110.04855},
  year={2021}
}

@article{poursoltani2023robust,
  title={Robust data-driven prescriptiveness optimization},
  author={Poursoltani, Mehran and Delage, Erick and Georghiou, Angelos},
  journal={arXiv preprint arXiv:2306.05937},
  year={2023}
}

@article{perakis2023robust,
  title={Robust pricing and production with information partitioning and adaptation},
  author={Perakis, Georgia and Sim, Melvyn and Tang, Qinshen and Xiong, Peng},
  journal={Management Science},
  volume={69},
  number={3},
  pages={1398--1419},
  year={2023},
  publisher={INFORMS}
}

@article{qi2022distributionally,
  title={Distributionally robust conditional quantile prediction with fixed design},
  author={Qi, Meng and Cao, Ying and Shen, Zuo-Jun},
  journal={Management Science},
  volume={68},
  number={3},
  pages={1639--1658},
  year={2022},
  publisher={INFORMS}
}

@article{fu2024distributionally,
  title={Distributionally Robust Newsvendor Under Stochastic Dominance with a Feature-Based Application},
  author={Fu, Mingyang and Li, Xiaobo and Zhang, Lianmin},
  journal={Manufacturing $\&$ Service Operations Management},
  volume={26},
  number={5},
  pages={1962--1977},
  year={2024},
  publisher={INFORMS}
}

@article{doshi2017towards,
  title={Towards a rigorous science of interpretable machine learning},
  author={Doshi-Velez, Finale and Kim, Been},
  journal={arXiv preprint arXiv:1702.08608},
  year={2017}
}

@article{rudin2019stop,
  title={Stop explaining black box machine learning models for high stakes decisions and use interpretable models instead},
  author={Rudin, Cynthia},
  journal={Nature Machine Intelligence},
  volume={1},
  number={5},
  pages={206--215},
  year={2019},
  publisher={Nature Publishing Group UK London}
}

@article{bertsimas2017optimal,
  title={Optimal classification trees},
  author={Bertsimas, Dimitris and Dunn, Jack},
  journal={Machine Learning},
  volume={106},
  number={7},
  pages={1039--1082},
  year={2017},
  publisher={Springer}
}

@article{bertsimas2021voice,
  title={The voice of optimization},
  author={Bertsimas, Dimitris and Stellato, Bartolomeo},
  journal={Machine Learning},
  volume={110},
  number={2},
  pages={249--277},
  year={2021},
  publisher={Springer}
}

@article{aghaei2025strong,
  title={Strong optimal classification trees},
  author={Aghaei, Sina and G{\'o}mez, Andr{\'e}s and Vayanos, Phebe},
  journal={Operations Research},
  volume={73},
  number={4},
  pages={2223--2241},
  year={2025},
  publisher={INFORMS}
}

@article{notz2024explainable,
  title={Explainable subgradient tree boosting for prescriptive analytics in operations management},
  author={Notz, Pascal M and Pibernik, Richard},
  journal={European Journal of Operational Research},
  volume={312},
  number={3},
  pages={1119--1133},
  year={2024},
  publisher={Elsevier}
}

@inproceedings{forel2023explainable,
  title={Explainable data-driven optimization: from context to decision and back again},
  author={Forel, Alexandre and Parmentier, Axel and Vidal, Thibaut},
  booktitle={International Conference on Machine Learning},
  pages={10170--10187},
  year={2023},
  organization={PMLR}
}

@article{oroojlooyjadid2020applying,
  title={Applying deep learning to the newsvendor problem},
  author={Oroojlooyjadid, Afshin and Snyder, Lawrence V and Tak{\'a}{\v{c}}, Martin},
  journal={IISE Transactions},
  volume={52},
  number={4},
  pages={444--463},
  year={2020},
  publisher={Taylor \& Francis}
}

@article{bertsimas2019price,
  title={The price of interpretability},
  author={Bertsimas, Dimitris and Delarue, Arthur and Jaillet, Patrick and Martin, Sebastien},
  journal={arXiv preprint arXiv:1907.03419},
  year={2019}
}

@article{lundberg2020local,
  title={From local explanations to global understanding with explainable {AI} for trees},
  author={Lundberg, Scott M and Erion, Gabriel and Chen, Hugh and DeGrave, Alex and Prutkin, Jordan M and Nair, Bala and Katz, Ronit and Himmelfarb, Jonathan and Bansal, Nisha and Lee, Su-In},
  journal={Nature Machine Intelligence},
  volume={2},
  number={1},
  pages={56--67},
  year={2020},
  publisher={Nature Publishing Group}
}

@article{alvarez2018robustness,
  title={On the robustness of interpretability methods},
  author={Alvarez-Melis, David and Jaakkola, Tommi S},
  journal={arXiv preprint arXiv:1806.08049},
  year={2018}
}

@article{ancona2017towards,
  title={Towards better understanding of gradient-based attribution methods for deep neural networks},
  author={Ancona, Marco and Ceolini, Enea and {\"O}ztireli, Cengiz and Gross, Markus},
  journal={arXiv preprint arXiv:1711.06104},
  year={2017}
}

@article{dwivedi2023explainable,
  title={Explainable {AI (XAI)}: Core ideas, techniques, and solutions},
  author={Dwivedi, Rudresh and Dave, Devam and Naik, Het and Singhal, Smiti and Omer, Rana and Patel, Pankesh and Qian, Bin and Wen, Zhenyu and Shah, Tejal and Morgan, Graham and others},
  journal={ACM Computing Surveys},
  volume={55},
  number={9},
  pages={1--33},
  year={2023},
  publisher={ACM New York, NY}
}

@inproceedings{sundararajan2017axiomatic,
  title={Axiomatic attribution for deep networks},
  author={Sundararajan, Mukund and Taly, Ankur and Yan, Qiqi},
  booktitle={International Conference on Machine Learning},
  pages={3319--3328},
  year={2017},
  organization={PMLR}
}

@misc{hastie2009elements,
  title={The elements of statistical learning},
  author={Hastie, Trevor and Tibshirani, Robert and Friedman, Jerome and others},
  year={2009},
  publisher={Springer series in statistics New-York}
}

@article{zhang2026generative,
  title={Generative Neural Networks for {S}inkhorn Distributionally Robust Hypothesis Testing},
  author={Zhang, Fenglin and Liu, Teyan and Wang, Jie},
  journal={arXiv preprint arXiv:2608.22746},
  year={2026}
}

@book{peyre2019computational,
  title={Computational optimal transport: With applications to data science},
  author={Peyr{\'e}, Gabriel and Cuturi, Marco},
  year={2019},
  publisher={Now Foundations and Trends}
}

@article{yang2026decision,
  title={Decision Making with Side Information: A Causal Transport Robust Approach},
  author={Yang, Jincheng and Zhang, Luhao and Chen, Ningyuan and Gao, Rui and Hu, Ming},
  journal={Operations Research},
  year={2026},
  publisher={INFORMS}
}

@article{erion2021improving,
  title={Improving performance of deep learning models with axiomatic attribution priors and expected gradients},
  author={Erion, Gabriel and Janizek, Joseph D and Sturmfels, Pascal and Lundberg, Scott M and Lee, Su-In},
  journal={Nature Machine Intelligence},
  volume={3},
  number={7},
  pages={620--631},
  year={2021},
  publisher={Nature Publishing Group UK London}
}

@article{hu2024multilevel,
  title   = {Multi-level {Monte-Carlo} Gradient Methods for Stochastic Optimization with Biased Oracles},
  author  = {Hu, Yifan and Wang, Jie and Chen, Xin and He, Niao},
  journal = {arXiv preprint arXiv:2408.11084},
  year    = {2024}
}
\bibliographystyle{iclr2027_conference}

\newpage

\appendix
\setcounter{page}{1}
\renewcommand{\thepage}{Appendix Page \arabic{page}}

\section*{Appendix}

\section{Related Literature}\label{app-sec-literature}

We review existing literature related to our work below.

\textit{On Data-driven Prescriptive Analytics. } 
Contextual Stochastic Optimization~(CSO) is a powerful tool that uses covariates to improve decision-making performance under parameter uncertainty~\citep{sadana2025survey}.  
Existing studies have developed various frameworks for CSO.
\citet{bertsimas2020predictive} propose a data-driven framework based on weighted sample average approximation. It estimates the conditional distribution by generating data-driven weights using machine learning models and takes decisions by minimizing the reweighted empirical cost. 
Building on this, \citet{kallus2023stochastic} propose a stochastic optimization forest method that calculates weights by optimizing downstream decision quality rather than prediction accuracy. 
\citet{elmachtoub2022smart} develop a smart predict-then-optimize framework that integrates learning and optimization through a decision-oriented loss function. 
\citet{qi2025integrated} present an integrated conditional estimation-optimization framework based on the downstream objective. 
Unlike the aforementioned approaches, which still involve an intermediate estimation step, our work focuses on the decision-rule approach, which directly maps covariates to final decisions~\citep{liyanage2005practical}. 
Notable examples include the linear and kernel-based decision rules for newsvendor problems by~\citet{ban2019big} and the deep-learning-based rules for inventory management by~\citet{qi2023practical}. 

\textit{On Contextual Distributionally Robust Optimization (DRO). }
To hedge against the distributional shift issue in CSO, contextual DRO has received much attention with various types of ambiguity sets.
Early approaches focused on $\phi$-divergence~\citep{srivastava2021data, poursoltani2023robust} and moment-based~\citep{perakis2023robust} ambiguity sets, primarily due to their computational tractability. 
The optimal-transport-based ambiguity sets are widely used for contextual DRO modeling~\citep{qi2022distributionally, esteban2022distributionally, zhang2024optimal, kannan2024residuals, nguyen2025robustifying}, due to their data-driven nature and satisfactory out-of-sample guarantees.
\citet{yang2026decision} and \citet{wang2025sinkhorn} proposed variants of the Wasserstein distance to model worst-case distributions with causal or continuous structures. 
Our work accounts for the structure of the contextual DRO model and develops an efficient, interpretable decision-rule approach.

\textit{On Decision Rule Approach for Contextual DRO. } 
Decision rule approaches seek an optimal policy within a pre-specified function class that makes end-to-end decisions based on covariates.
Existing research has developed several parametric and non-parametric decision rule approaches for solving contextual DRO. 
\citet{yang2026decision} reformulate the causal optimal transport-based DRO as a conic program for affine decision rules, but it may not be tractable for other parametric decision rules.
\citet{hu2023contextual} reformulate the problem as a large-scale bilevel program, which remains computationally challenging.
For non-parametric decision rules, finite-dimensional convex reformulations of contextual DRO have been proposed~\citep{yang2026decision, esteban2022distributionally, fu2024distributionally, zhang2024optimal, nguyen2025robustifying} for specific problem structures or for particular choices of the Wasserstein distance and its variants.
In this paper, we propose an SRF decision rule to approximate an optimal policy in measurable, continuous, and differentiable function spaces, while balancing interpretability and computational tractability. 

\textit{On trustworthy and interpretable decision-making. } 
Interpretability generally refers to the ability to explain or to present in understandable terms to humans~\citep{doshi2017towards, bertsimas2019price}. 
Kernel- or deep-learning-based methods for contextual optimization perform well but lack trustworthiness and interpretability~\citep{oroojlooyjadid2020applying, bertsimas2022data}. 
In high-stakes applications, such as healthcare and finance, these methods may pose risks~\citep{rudin2019stop, forel2023explainable}. 
Instead of explaining deep learning models~\citep{lundberg2020local}, extensive literature is dedicated to applying or developing inherently trustworthy models for decision-making~\citep{forel2023explainable}, such as decision tree and forest methods, which have been considered interpretable and used in some decision-making processes due to their explicit ``if-then'' logical structures~\citep{bertsimas2017optimal, bertsimas2021voice, aghaei2025strong}. 
These methods have also been applied in CSO~\citep{elmachtoub2020decision, bertsimas2020predictive, kallus2023stochastic, notz2024explainable}. 
Unlike existing tree-based methods, which rely on greedy heuristics due to non-differentiability and NP-hardness, the proposed SRF decision rule is differentiable and can be efficiently trained with gradient-based methods while maintaining intrinsic interpretability. 

\section{Discussion on the Causal Sinkhorn DRO}\label{app-sec-CSD}

In this section, we provide detailed discussions on the causal Sinkhorn discrepancy (CSD), causal transport plans, the CSD-based ambiguity set, and the worst-case distributions of the Causal-SDRO. 

\subsection{Discussion on CSD and Causal Transport Plans}

\begin{remark}[Choice of reference measures in CSD]
    The entropy regularization term $\mathbb{D}_{\mathrm{KL}}\left(\gamma\|\mathbb{P}\otimes(\nu_{\mathcal{X}}\otimes\nu_{\mathcal{Y}})\right)$ appeared in Definition~\ref{def-causal-sh} is defined with respect to general reference measures $\nu_{\mathcal{X}}$ and $\nu_{\mathcal{Y}}$. The relative entropy in entropically regularized optimal transport~\citep{peyre2019computational} is defined as
    \[
        \mathbb{D}_{\mathrm{KL}}^{\prime}\left(\gamma\|\mathbb{P}\otimes\mathbb{Q}\right):=\mathbb{E}_{((\widehat{\boldsymbol{x}},\widehat{\boldsymbol{y}}),(\boldsymbol{x},\boldsymbol{y}))\sim\gamma}\left[\log\frac{\mathrm{d}\gamma((\widehat{\boldsymbol{x}},\widehat{\boldsymbol{y}}),(\boldsymbol{x},\boldsymbol{y}))}{\mathrm{d}\mathbb{P}(\widehat{\boldsymbol{x}},\widehat{\boldsymbol{y}})\mathrm{d}\mathbb{Q}(\boldsymbol{x},\boldsymbol{y})}\right].
    \]
    If $\mathbb{Q}\ll\nu_{\mathcal{X}}\otimes\nu_{\mathcal{Y}}$, then $\mathrm{d}\mathbb{Q}(\boldsymbol{x},\boldsymbol{y})=\frac{\mathrm{d}\mathbb{Q}(\boldsymbol{x},\boldsymbol{y})}{\mathrm{d}\nu_{\mathcal{X}}(\boldsymbol{x})\mathrm{d}\nu_{\mathcal{Y}}(\boldsymbol{y})}\mathrm{d}\nu_{\mathcal{X}}(\boldsymbol{x})\mathrm{d}\nu_{\mathcal{Y}}(\boldsymbol{y})$, and therefore
    \[
        \mathbb{D}_{\mathrm{KL}}\left(\gamma\|\mathbb{P}\otimes(\nu_{\mathcal{X}}\otimes\nu_{\mathcal{Y}})\right) = \mathbb{D}_{\mathrm{KL}}^{\prime}\left(\gamma\|\mathbb{P}\otimes\mathbb{Q}\right) + \mathbb{E}_{(\boldsymbol{x},\boldsymbol{y})\sim\mathbb{Q}}\left[\log\frac{\mathrm{d}\mathbb{Q}(\boldsymbol{x},\boldsymbol{y})}{\mathrm{d}\nu_{\mathcal{X}}(\boldsymbol{x})\mathrm{d}\nu_{\mathcal{Y}}(\boldsymbol{y})}\right].
    \] 
    In particular, if $\nu_{\mathcal{X}}$ and $\nu_{\mathcal{Y}}$ are the Lebesgue measures and $\mathbb{Q}$ admits a joint density $q(\boldsymbol{x},\boldsymbol{y})$, then $\mathrm{d}\nu_{\mathcal{X}}(\boldsymbol{x})=\mathrm{d}\boldsymbol{x}$, $\mathrm{d}\nu_{\mathcal{Y}}(\boldsymbol{y})=\mathrm{d}\boldsymbol{y}$, and
    \begin{equation}\nonumber
    \tag*{$\clubsuit$} \mathbb{D}_{\mathrm{KL}}\left(\gamma\|\mathbb{P}\otimes(\nu_{\mathcal{X}}\otimes\nu_{\mathcal{Y}})\right) =\mathbb{D}_{\mathrm{KL}}^{\prime}\left(\gamma\|\mathbb{P}\otimes\mathbb{Q}\right) + \mathbb{E}_{(\boldsymbol{x},\boldsymbol{y})\sim\mathbb{Q}}\left[\log q(\boldsymbol{x},\boldsymbol{y})\right].
    \end{equation}
\end{remark}

Recall we use Lebesgue measure as the reference measure. 
Unlike deterministic transport plans derived from the Wasserstein distance, Sinkhorn transport plans are probabilistic.
Specifically, entropy regularization ($\epsilon > 0$) penalizes deterministic transport plans, yielding smooth plans that map each source point to a probability distribution over the target space rather than a single point. 

\begin{figure}[!htb]
    \centering
    \includegraphics[width=0.8\linewidth]{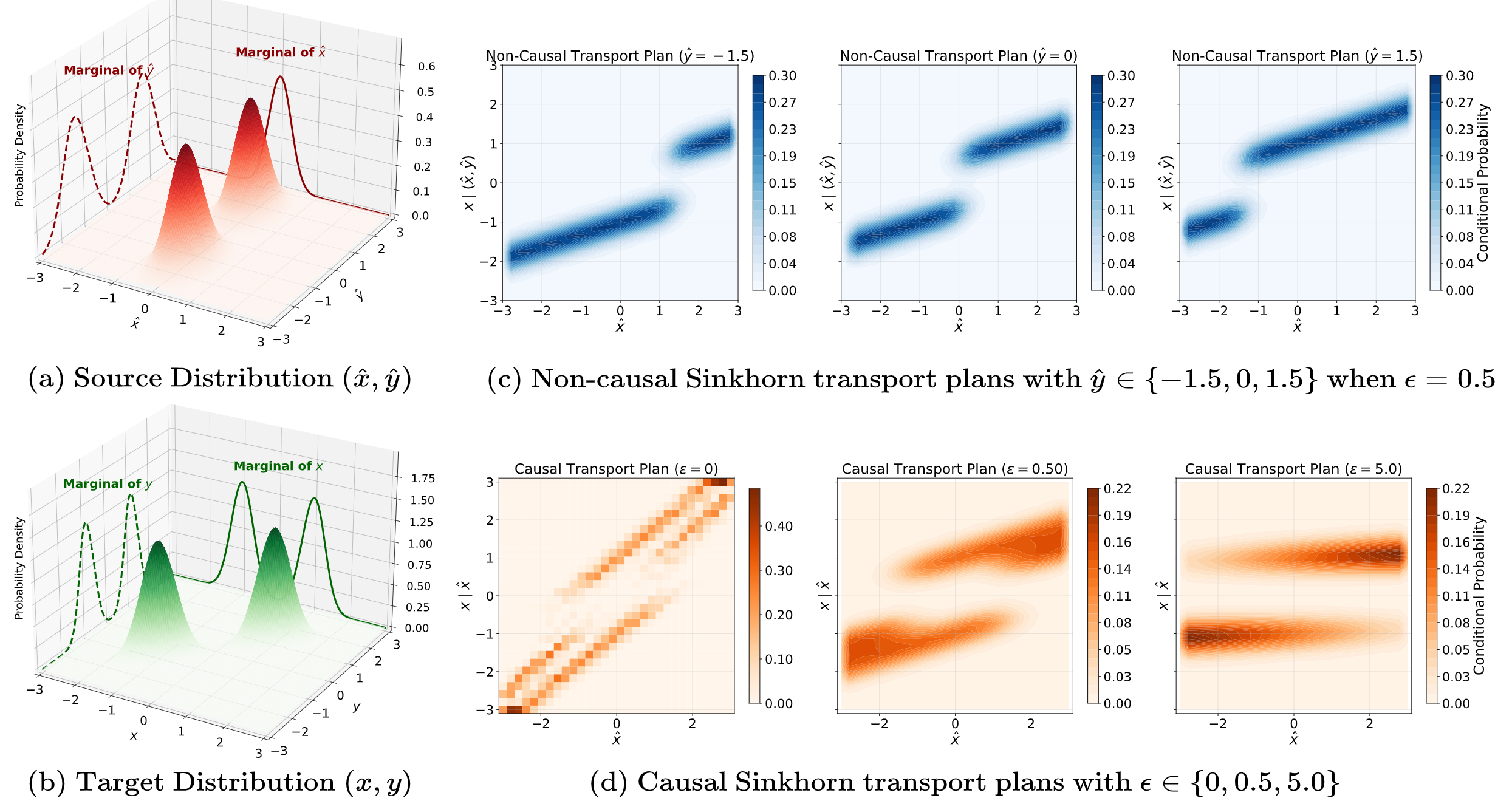}
  \caption{Visualization for causal and non-causal Sinkhorn transport plans. }
  \label{fig:transport_plan}
\end{figure}

Figure~\ref{fig:transport_plan} visualizes and compares the causal and non-causal transport plans for the illustrative source $(\widehat{x}, \widehat{y})$ and target $(x, y)$ distributions supported on $[-3, 3]^2$. 
The source shown in Figure~\ref{fig:transport_plan}(a) comprises $\widehat{x} \sim \mathcal{N}(0, 0.1)$ and $\widehat{y}$ from an equiprobable mixture of $\mathcal{N}(\pm 1.5, 0.3)$. The target shown in Figure~\ref{fig:transport_plan}(b) is a mixture of two bivariate Gaussians centered at $(-1, -1)$ and $(1, 1)$ with positive correlation.  
Figure~\ref{fig:transport_plan}(c) illustrates non-causal Sinkhorn transport plans, where the mapping from $\widehat{x}$ to $x$ explicitly depends on $\widehat{y}$. 
Specifically, when $\widehat{y} = -1.5$ (left panel), the source mode at $\widehat{x} = 0$ is transported primarily to the target's bottom-left peak ($x \approx -1$) to minimize transport cost. Similarly, when $\widehat{y} = 1.5$ (right panel), the mass is directed toward the closer upper-right peak ($x \approx 1$).
In contrast, their corresponding causal Sinkhorn transport plan (Figure~\ref{fig:transport_plan}(d), middle panel) enforces conditional independence and, consequently, appears as an aggregate of the non-causal plans. 
Figure~\ref{fig:transport_plan}(d) further demonstrates the difference between the causal Wasserstein ($\epsilon=0$) and the causal Sinkhorn transport plans ($\epsilon=0.5$ and $5$). 
As $\epsilon$ increases, the causal Sinkhorn transport plans converge towards the product of the marginal distributions of $\widehat{x}$ and $x$. 

\subsection{Reformulation for the Dual Causal-SDRO} \label{app-sec-dual-causal-sdro}

We derive the dual problem $v_{\mathrm{D}}$ of $v_{\mathrm{P}}$ in Eq.~\eqref{eq-primal} as (see Appendix~\ref{ecsec-strong-dual}): 
\begin{subequations}\label{causal-sdro-dual:whole}
\begin{equation}\label{causal-sdro-dual}
    v_{\mathrm{D}} := \inf_{\lambda \ge 0} \left\{ \lambda\rho^p + \mathbb{E}_{\widehat{\boldsymbol{x}} \sim \widehat{\mathbb{P}}_{\widehat{\boldsymbol{X}}}} \left[ \lambda\epsilon \log\, 
        \int_{\mathcal{X}} \exp\left( \frac{g(\widehat{\boldsymbol{x}}, \boldsymbol{x}, \lambda)}{\lambda\epsilon} \right) \mathrm{d} \nu_{\mathcal{X}}(\boldsymbol{x}) \right] \right\},
\end{equation}
where 
\begin{equation}\label{g-function}
    g(\widehat{\boldsymbol{x}}, \boldsymbol{x}, \lambda) = \mathbb{E}_{\widehat{\boldsymbol{y}} \sim  \widehat{\mathbb{P}}_{\widehat{\boldsymbol{Y}}|\widehat{\boldsymbol{X}}=\widehat{\boldsymbol{x}}}}\left[ \lambda\epsilon \log\,  \int _{\mathcal{Y}} \exp\left( \frac{\Psi(f(\boldsymbol{x}),\boldsymbol{y}) - \lambda c_p((\widehat{\boldsymbol{x}},\widehat{\boldsymbol{y}}), (\boldsymbol{x},\boldsymbol{y}))}{\lambda\epsilon} \right) \mathrm{d} \nu_{\mathcal{Y}}(\boldsymbol{y}) \right].
\end{equation}     
\end{subequations}

Next, we reformulate the dual problem $v_{\mathrm{D}}$ as a stochastic optimization with a nested expectation structure such that, except for $\widehat{\mathbb{P}}_{\widehat{\boldsymbol{Y}}|\widehat{\boldsymbol{X}}=\widehat{\boldsymbol{x}}}$, all random vectors are mutually independent.
Define the kernel probability distributions~(they are Laplace or Gaussian distributions when $p\in\{1,2\}$) for the random vectors $\boldsymbol{\xi}_1$ and $\boldsymbol{\xi}_2$ as 
    \begin{equation}
    \mathrm{d} U_{\epsilon} \left ( \boldsymbol{\xi}_1 \right ) := \frac{e^{ - \| \boldsymbol{\xi}_1\|^p / \epsilon} \cdot   \mathrm{d} \nu_{\mathcal{X} }\left (  \boldsymbol{\xi}_1 \right )}{\int_{\mathbb{R}^{d_x}} e^{ - \| \boldsymbol{u}\|^p / \epsilon} \mathrm{d} \nu_{\mathcal{X} }\left (  \boldsymbol{u} \right )}, \quad
     \mathrm{d} W_{\epsilon } \left (  \boldsymbol{\xi}_2 \right ) := \frac{e^{ - \| \boldsymbol{\xi}_2\|^p / \epsilon} \cdot  \mathrm{d} \nu_{\mathcal{Y} }\left (  \boldsymbol{\xi}_2 \right )}{\int_{\mathbb{R}^{d_y}} e^{ - \| \boldsymbol{u}\|^p / \epsilon} \mathrm{d} \nu_{\mathcal{Y} }\left (  \boldsymbol{u} \right ) } ,
\end{equation}
and a constant
\begin{equation}\nonumber
    \bar{\rho} := \rho^p + \epsilon \cdot \mathbb{E}_{\widehat{\boldsymbol{x}} \sim \widehat{\mathbb{P}}_{\widehat{\boldsymbol{X}}}} \left[ \log\, \int_{\mathbb{R}^{d_x}} e^{ - \| \boldsymbol{u}\|^p / \epsilon} \mathrm{d} \nu_{\mathcal{X} }\left (  \boldsymbol{u} \right ) \right] + \epsilon \cdot \mathbb{E}_{\widehat{\boldsymbol{y}} \sim \widehat{\mathbb{P}}_{\widehat{\boldsymbol{Y}}|\widehat{\boldsymbol{X}}}} \left[ \log\, \int_{\mathbb{R}^{d_y}} e^{ - \| \boldsymbol{u}\|^p / \epsilon} \mathrm{d} \nu_{\mathcal{Y} }\left (  \boldsymbol{u} \right ) \right]. 
\end{equation}
Note that when $\nu_{\mathcal{X}}$ and $\nu_{\mathcal{Y}}$ are taken as the Lebesgue measures, the kernel distributions $U_{\epsilon}$ and $W_\epsilon$ of $\boldsymbol{\xi}_1:=\boldsymbol{x}-\widehat{\boldsymbol{x}}$ and $\boldsymbol{\xi}_2:=\boldsymbol{y}-\widehat{\boldsymbol{y}}$, respectively, are independent of the historical observations due to the translation invariance of the Lebesgue measure.

As $(\widehat{\boldsymbol{x}}+ \boldsymbol{\xi}_1 ) \in \mathcal{X}$ and $ (\widehat{\boldsymbol{y}}+ \boldsymbol{\xi}_2)  \in \mathcal{Y}$, the problem~\eqref{causal-sdro-dual:whole} can be reformulated as 
\begin{subequations} \nonumber
\begin{equation} \nonumber
    v_{\mathrm{D}} = \inf_{\lambda \ge 0} \left\{ \lambda \bar{\rho} + \mathbb{E}_{\widehat{\boldsymbol{x}} \sim \widehat{\mathbb{P}}_{\widehat{\boldsymbol{X}}}} \left[ \lambda\epsilon \log\, \mathbb{E}_{ \boldsymbol{\xi}_1 \sim U_{\epsilon}}\left[ \exp\left( \frac{h (\widehat{\boldsymbol{x}},  \boldsymbol{\xi}_1, \lambda)}{\lambda\epsilon} \right) \right]\right] \right\}, 
\end{equation}
where 
\begin{equation} \nonumber
    h (\widehat{\boldsymbol{x}},  \boldsymbol{\xi}_1, \lambda) = \mathbb{E}_{\widehat{\boldsymbol{y}} \sim  \widehat{\mathbb{P}}_{\widehat{\boldsymbol{Y}}|\widehat{\boldsymbol{X}}=\widehat{\boldsymbol{x}}}}\left[ \lambda\epsilon \log\, \mathbb{E}_{ \boldsymbol{\xi}_2 \sim W_{\epsilon}}\left[ \exp\left( \frac{\Psi(f(\widehat{\boldsymbol{x}}+ \boldsymbol{\xi}_1),\widehat{\boldsymbol{y}}+ \boldsymbol{\xi}_2)}{\lambda\epsilon} \right) \right] \right], 
\end{equation}     
\end{subequations} 
which is the same as~\eqref{Eq:MCO:whole}. 

\subsection{Discussion on Causal Sinkhorn DRO}\label{app-sec-disc-csdro}

In this subsection, we discuss the CSD-based ambiguity set and Causal-SDRO. 

\begin{remark} \label{remark-KL-connect}
    \textnormal{\textbf{(Connection with KL-Divergence-based DRO)}}. 
    When $\epsilon > 0$, the constraint $R_{\epsilon}(\widehat{\mathbb{P}}, \mathbb{P})^p \le \rho^p$ in~\eqref{causal-sdro} can be rewritten as (see Appendix~\ref{ecsec-strong-dual}, Proof of Theorem~\ref{theo-strong-duality} for details)
    \begin{equation}\label{kl-dro-cons}
        \mathbb{E}_{(\widehat{\boldsymbol{x}},\widehat{\boldsymbol{y}})\sim\widehat{\mathbb{P}}} \Big[ \mathbb{D}_{\mathrm{KL}}\Big(\gamma_{(\widehat{\boldsymbol{x}},\widehat{\boldsymbol{y}})} || \mathcal{K}_{(\widehat{\boldsymbol{x}},\widehat{\boldsymbol{y}}),\epsilon}\Big) \Big] \le \frac{\bar{\rho}}{\epsilon}, 
    \end{equation}
    where $\gamma_{(\widehat{\boldsymbol{x}},\widehat{\boldsymbol{y}})}$ is the conditional distribution of $\gamma$ given the value of the marginal $(\widehat{\boldsymbol{x}},\widehat{\boldsymbol{y}})$, and $\mathcal{K}_{(\widehat{\boldsymbol{x}},\widehat{\boldsymbol{y}}),\epsilon}$ is a kernel probability distribution defined by kernel distributions $U_{\epsilon}$ and $W_{\epsilon}$ in Eq.~\eqref{QW-distri}: 
    \begin{equation}\nonumber
        \mathrm{d}\mathcal{K}_{(\widehat{\boldsymbol{x}},\widehat{\boldsymbol{y}}),\epsilon}(\boldsymbol{x},\boldsymbol{y}) := \mathrm{d}U_{\epsilon}(\boldsymbol{x}-\widehat{\boldsymbol{x}}) \cdot \mathrm{d} W_{\epsilon}(\boldsymbol{y}-\widehat{\boldsymbol{y}}). 
    \end{equation}
    Compared with the KL-divergence-based DRO problem~\citep{ben2013robust, blanchet2023unifying} with constraint $\mathbb{D}_{\mathrm{KL}} (\mathbb{P} || \widehat{\mathbb{P}} ) \le \rho$, in constraint~\eqref{kl-dro-cons}, the conditional distribution $\gamma_{(\widehat{\boldsymbol{x}},\widehat{\boldsymbol{y}})}$ is remained due to the causal consideration in~\eqref{causal-sdro}, and the distribution $\mathcal{K}_{(\widehat{\boldsymbol{x}},\widehat{\boldsymbol{y}}),\epsilon}$ can be viewed as a non-parametric kernel estimation constructed from $\widehat{\mathbb{P}}$. $\hfill \clubsuit$ 
\end{remark}

Remark~\ref{remark-KL-connect} also implies that when $\bar{\rho} = 0$, the Causal-SDRO is equivalent to SDRO, and both reduce to a kernel-smoothed sample average approximation formulation, as shown in Appendix~\ref{App-subsec-kernal-smooth-SAA}. 
The following remark shows the relationship between the Causal-SDRO and causal Wasserstein DRO proposed in~\citet{yang2026decision}. 

\begin{remark}
    \textnormal{\textbf{(Comparison with Causal Wasserstein DRO). }} \label{remark-limit} 
    Suppose additionally that, for each $(\widehat{\boldsymbol{x}},\widehat{\boldsymbol{y}})$, the function $\Psi(f(\boldsymbol{x}),\boldsymbol{y}) -\lambda c_p((\widehat{\boldsymbol{x}},\widehat{\boldsymbol{y}}), (\boldsymbol{x},\boldsymbol{y}))$ is lower semicontinuous on $\operatorname{supp}\nu_{\mathcal{X}}\times \operatorname{supp}\nu_{\mathcal{Y}}$. 
    If $\epsilon \to 0$, then the dual objective of the problem~\eqref{causal-sdro-dual:whole} converges to (see Appendix~\ref{ecsec-remark-limit} for detailed proof)
    \begin{equation}\label{Causal-WDRO-dual}
        \lambda \rho^{p}+\mathbb{E}_{\widehat{\mathbb{P}}_{\widehat{\boldsymbol{X}}}}\left[\sup_{\boldsymbol{x} \in \text{supp}\,\nu_{\mathcal{X}}}\left\{\mathbb{E}_{\widehat{\mathbb{P}}_{\widehat{\boldsymbol{Y}} \mid \widehat{\boldsymbol{X}}}}\left[\sup_{\boldsymbol{y} \in \text{supp}\,\nu_{\mathcal{Y}}}\Big\{\Psi(f(\boldsymbol{x}), \boldsymbol{y})-\lambda c_p((\widehat{\boldsymbol{x}},\widehat{\boldsymbol{y}}), (\boldsymbol{x},\boldsymbol{y}))\Big\} \mid \widehat{\boldsymbol{X}}\right]\right\}\right]. 
    \end{equation}
    When $\mathcal{X}=\mathbb{R}^{d_x}$ and $\mathcal{Y}=\mathbb{R}^{d_y}$, and $\nu_{\mathcal{X}}$ and $\nu_{\mathcal{Y}}$ are chosen as the corresponding Lebesgue measures, we have $\operatorname{supp}\nu_{\mathcal{X}}=\mathcal{X}$ and $\operatorname{supp}\nu_{\mathcal{Y}}=\mathcal{Y}$. Therefore, Eq.~\eqref{Causal-WDRO-dual} coincides with the dual formulation of Causal Wasserstein DRO (Causal-WDRO) in~\citet{yang2026decision}. 
    The optimization for Causal-WDRO is computationally challenging due to the nested inner supremums within expectations in \eqref{Causal-WDRO-dual}, which leads to a non-smooth stochastic min-max structure.  
    Existing algorithms for solving Causal-WDRO typically focus on special cases. For example, \citet{yang2026decision} assume a specific structure for the loss function, and \citet{hu2023contextual} assume that $\lambda$ is sufficiently large, thereby converting the problem into a contextual stochastic bilevel optimization problem with a strongly convex lower-level problem.    $\hfill \clubsuit$ 
\end{remark}

\subsection{Discussion on the Worst-case Distribution of Causal-SDRO} \label{app-sec-disc-worstcase}

As demonstrated in the proof of Theorem~\ref{theo-strong-duality} (see Appendix~\ref{ecsec-strong-dual} for details), Theorem~\ref{theo-worst-case-distri} characterizes the worst-case distribution of the primal problem $v_{\mathrm{P}}$. 
\begin{theorem}\label{theo-worst-case-distri}
\textnormal{\textbf{(Worst-case Distribution of Causal-SDRO).}} 
    Under Assumption~\ref{assum-1} and Condition~\ref{condit-1}, suppose the dual problem $v_{\mathrm{D}}$ has an optimal solution $\lambda^* > 0$ and $\Psi (f(\cdot), \cdot)$ is not a constant function. Then the dual optimal solution $\lambda^*$ is unique, and the density of the worst-case distribution $\mathbb{P}^*$ of the primal problem $v_{\mathrm{P}}$ is given by 
    \begin{equation}\label{wcd-csdro}
        \frac{\mathrm{d}\mathbb{P}^* (\boldsymbol{x},\boldsymbol{y}) }{\mathrm{d} \nu_{\mathcal{X}}(\boldsymbol{x})\mathrm{d} \nu_{\mathcal{Y}}(\boldsymbol{y})} = \mathbb{E}_{(\widehat{\boldsymbol{x}},\widehat{\boldsymbol{y}})\sim\widehat{\mathbb{P}}}\Big[\alpha_{\widehat{\boldsymbol{x}}}\cdot \beta_{\widehat{\boldsymbol{x}}, \widehat{\boldsymbol{y}}, \boldsymbol{x}}\cdot e^{r(\widehat{\boldsymbol{x}}, \boldsymbol{x}) + s(\widehat{\boldsymbol{x}},\widehat{\boldsymbol{y}}, \boldsymbol{x},\boldsymbol{y})} \Big],
    \end{equation}
    where $\alpha_{\widehat{\boldsymbol{x}}} = \Big(\int_{\mathcal{X}} e^{r(\widehat{\boldsymbol{x}}, \boldsymbol{x})}  \mathrm{d} \nu_{\mathcal{X}}(\boldsymbol{x})\Big)^{-1}$,  $\beta_{\widehat{\boldsymbol{x}}, \widehat{\boldsymbol{y}}, \boldsymbol{x}}=\left(\int_{\mathcal{Y}} e^{s(\widehat{\boldsymbol{x}},\widehat{\boldsymbol{y}}, \boldsymbol{x},\boldsymbol{y})}\mathrm{d} \nu_{\mathcal{Y}}(\boldsymbol{y}) \right)^{-1}$,
    \begin{equation}\nonumber
        s(\widehat{\boldsymbol{x}},\widehat{\boldsymbol{y}}, \boldsymbol{x},\boldsymbol{y}) = \frac{\Psi(f(\boldsymbol{x}),\boldsymbol{y})-\lambda^* c_p((\widehat{\boldsymbol{x}},\widehat{\boldsymbol{y}}), (\boldsymbol{x},\boldsymbol{y}))}{\lambda^*\epsilon},
    \end{equation}
    and
    \begin{equation}\nonumber
        r(\widehat{\boldsymbol{x}}, \boldsymbol{x}) = \mathbb{E}_{\widehat{\boldsymbol{y}} \sim  \widehat{\mathbb{P}}_{\widehat{\boldsymbol{Y}}|\widehat{\boldsymbol{X}}=\widehat{\boldsymbol{x}}}}\Big[ \log\,  \int _{\mathcal{Y}}  e^{s(\widehat{\boldsymbol{x}},\widehat{\boldsymbol{y}}, \boldsymbol{x},\boldsymbol{y})}  \mathrm{d} \nu_{\mathcal{Y}}(\boldsymbol{y}) \Big].
    \end{equation}
\end{theorem}
We provide the proof of Theorem~\ref{theo-worst-case-distri} in Appendix~\ref{ecsec-theo-worst}. 
Based on this, Theorem~\ref{corollary-worst-case} shows the worst-case distribution of soft-constrained Causal-SDRO and can be seen as a corollary of Theorem~\ref{theo-worst-case-distri}. 

\section{ Compared DRO Models } \label{ecsec-wc-distribution}

In this section, we show tractable formulations and worst-case distributions for the DRO models we compare, including Sinkhorn DRO (SDRO), causal Wasserstein DRO (Causal-WDRO), and KL-divergence-based DRO (KL-DRO), in soft-constrained and contextual settings. 

\subsection{SDRO}

Based on~\citet{wang2025sinkhorn}, the soft-constrained SDRO without causal consideration for contextual DRO is defined as 
\begin{equation}\label{non-causal-sdro}
\begin{aligned}
\inf _{f \in \mathcal{F}} \sup_{\mathbb{P} \in \mathcal{P}(\mathcal{X} \times \mathcal{Y})} \quad \mathbb{E}_{(\boldsymbol{x}, \boldsymbol{y}) \sim \mathbb{P}}\Big[\Psi(f(\boldsymbol{x}), \boldsymbol{y})-\lambda \cdot \mathcal{W}_{p}(\widehat{\mathbb{P}}, \mathbb{P})^p \Big], 
\end{aligned}
\tag{\text{SDRO}}
\end{equation}
where
\begin{equation}\nonumber
    \mathcal{W}_p(\mathbb{P}, \mathbb{Q}) := \left( \inf_{\gamma \in \Gamma(\mathbb{P}, \mathbb{Q})} \mathbb{E}_{((\widehat{\boldsymbol{x}},\widehat{\boldsymbol{y}}), (\boldsymbol{x},\boldsymbol{y}))\sim\gamma} \Big[ c_p((\widehat{\boldsymbol{x}},\widehat{\boldsymbol{y}}), (\boldsymbol{x},\boldsymbol{y})) \Big]  + \epsilon \cdot \mathbb{D}_{\mathrm{KL}}\left( \gamma \| \mathbb{P} \otimes \left ( \nu_{\mathcal{X}}\otimes \nu_{\mathcal{Y} } \right ) \right)\right)^{1/p}.
\end{equation}

We present the worst-case distribution for problem~\eqref{non-causal-sdro} in the following Lemma \ref{lem-worst-non-causal}, extending the results of~\citet {wang2025sinkhorn}. 
\begin{lemma}\label{lem-worst-non-causal}
    \textnormal{\textbf{\citep[Worst-case Distribution of the SDRO Problem,][]{wang2025sinkhorn}.}} 
    Under Assumption~\ref{assum-1}, the density of worst-case distribution $\mathbb{P}_{\lambda, \textnormal{SDRO}}^*$ of the inner problem of~\eqref{non-causal-sdro} for any $\lambda$ is given by 
    \begin{equation}
         \frac{\mathrm{d}\mathbb{P}_{\lambda,\textnormal{SDRO}}^* (\boldsymbol{x},\boldsymbol{y}) }{\mathrm{d} \nu_{\mathcal{X}}(\boldsymbol{x})\mathrm{d} \nu_{\mathcal{Y}}(\boldsymbol{y})} = \mathbb{E}_{(\widehat{\boldsymbol{x}},\widehat{\boldsymbol{y}})\sim\widehat{\mathbb{P}}}\Big[\tilde{\alpha}_{\widehat{\boldsymbol{x}}, \widehat{\boldsymbol{y}}}(\lambda) \cdot e^{s^{\prime}(\lambda,\widehat{\boldsymbol{x}},\widehat{\boldsymbol{y}}, \boldsymbol{x},\boldsymbol{y})} \Big],
    \end{equation}
    where $\tilde{\alpha}_{\widehat{\boldsymbol{x}}, \widehat{\boldsymbol{y}}} (\lambda) = \Big(\int_{\mathcal{X}\times \mathcal{Y}} e^{s^{\prime}(\lambda,\widehat{\boldsymbol{x}},\widehat{\boldsymbol{y}}, \boldsymbol{x},\boldsymbol{y})} \cdot \mathrm{d} \nu_{\mathcal{X}}\otimes \nu_{\mathcal{Y} } (\boldsymbol{x}, \boldsymbol{y}) \Big)^{-1}$. 
\end{lemma}  
According to Lemma~\ref{lem-worst-non-causal}, $\mathbb{P}_{\lambda,\textnormal{SDRO}}^*$ is a mixture of Gibbs distributions. 
Compared with the worst-case distribution of the problem~\eqref{causal-sdro} in Theorem~\ref{theo-worst-case-distri}, $\mathbb{P}_{\lambda,\textnormal{SDRO}}^*$ has a simpler density function structure. 

According to~\citet{wang2025sinkhorn}, the strong dual formulation of the problem~\eqref{non-causal-sdro} can be reformulated as follows when a parametric decision rule $f_{\boldsymbol{\theta}}$ is considered: 
\begin{equation}\label{sdro-dual}
    \min_{\boldsymbol{\theta} \in \Theta}\quad  \mathbb{E}_{(\widehat{\boldsymbol{x}}, \widehat{\boldsymbol{y}}) \sim \widehat{\mathbb{P}}} \left[ \lambda\epsilon \log\, \int_{\mathcal{X}\times \mathcal{Y}} \exp\left( \frac{\Psi(f_{\boldsymbol{\theta}}(\boldsymbol{x}), \boldsymbol{y}) - \lambda c_p((\widehat{\boldsymbol{x}},\widehat{\boldsymbol{y}}), (\boldsymbol{x},\boldsymbol{y}))}{\lambda\epsilon} \right)   \mathrm{d} \nu_{\mathcal{X}} \otimes \nu_{\mathcal{Y}} (\boldsymbol{x}, \boldsymbol{y})  \right]. 
\end{equation}
We define 
\begin{equation} \nonumber
    \mathrm{d} V_{\epsilon}(\boldsymbol{\xi}_3, \boldsymbol{\xi}_4) := \frac{e^{-\frac{\|\boldsymbol{\xi}_3\|^p + \|\boldsymbol{\xi}_4\|^p}{\epsilon} }}{\mathbb{E}_{(\boldsymbol{\xi}_3, \boldsymbol{\xi}_4)\sim \nu_{\mathcal{X}} \otimes \nu_{\mathcal{Y}}}\Big[e^{-\frac{\|\boldsymbol{\xi}_3\|^p + \|\boldsymbol{\xi}_4\|^p}{\epsilon}} \Big]}\cdot \mathrm{d} \nu_{\mathcal{X}} \otimes \nu_{\mathcal{Y}} (\boldsymbol{\xi}_3, \boldsymbol{\xi}_4) , 
\end{equation}
and then Eq.~\eqref{sdro-dual} is equivalent to 
\begin{equation} \label{sdro-dual-sco}
    \min_{\boldsymbol{\theta} \in \Theta} \quad C \cdot \mathbb{E}_{(\widehat{\boldsymbol{x}}, \widehat{\boldsymbol{y}}) \sim \widehat{\mathbb{P}}} \Big[ \lambda\epsilon \cdot \log\, \mathbb{E}_{(\boldsymbol{\xi}_3, \boldsymbol{\xi}_4)\sim V_{\epsilon}} \Big[ \exp\left( \frac{\Psi(f_{\boldsymbol{\theta}}(\widehat{\boldsymbol{x}}+\boldsymbol{\xi}_3), \widehat{\boldsymbol{y}}+\boldsymbol{\xi}_4)}{\lambda\epsilon} \right)  \Big] \Big]
\end{equation}
where $C =  \mathbb{E}_{(\boldsymbol{\xi}_3, \boldsymbol{\xi}_4)\sim \nu_{\mathcal{X}} \otimes \nu_{\mathcal{Y}}} \Big[e^{- (\|\boldsymbol{\xi}_3\|^p + \|\boldsymbol{\xi}_4\|^p)/{\epsilon}} \Big]$. 
The problem~\eqref{sdro-dual-sco} is a stochastic compositional optimization problem and can be efficiently solved by the proposed CM-SCSC algorithm.  

\subsection{Causal-WDRO}

Without considering the entropy regularization, the Causal-WDRO model is defined as~\citet{yang2026decision}:
\begin{equation}\label{Causal-WDRO}
    \inf _{f \in \mathcal{F}} \sup_{\mathbb{P} \in \mathcal{P}(\mathcal{X} \times \mathcal{Y})} \quad \mathbb{E}_{(\boldsymbol{x}, \boldsymbol{y}) \sim \mathbb{P}}\Big[\Psi(f(\boldsymbol{x}), \boldsymbol{y}) -\lambda \cdot C_{p}(\widehat{\mathbb{P}}, \mathbb{P})^p \Big],
\tag{\text{Causal-WDRO}}
\end{equation}
where the causal transport distance $C_{p}(\widehat{\mathbb{P}}, \mathbb{P})$ is defined in Definition~\ref{def-causal-trans}. 
\citet{yang2026decision} characterize the worst-case distribution of the Causal-WDRO problem. 
Suppose that the empirical distribution $\widehat{\mathbb{P}}$ is grouped into $M$ distinct covariates $\widehat{\boldsymbol{x}}_m$ for any $m \in [M]$. 
For each covariate, there are $n_m$ observations of the uncertain parameter, denoted by $\widehat{\boldsymbol{y}}_{m, i}$ for any $i \in [n_m]$. 
Let $\widehat{p}_{m, i}$ be the probability mass of the data point $(\widehat{\boldsymbol{x}}_m, \widehat{\boldsymbol{y}}_{m, i})$. 
Then, according to~\citet{yang2026decision}, the following lemma holds. 
\begin{lemma}\label{lem-worst-causal}
      \textnormal{\textbf{\citep[Worst-case Distribution of the Causal-WDRO Problem,][]{yang2026decision}.}} If the worst-case distribution of Causal-WDRO problem exists for given $\lambda > 0$, then it has the following form 
      \begin{equation}\label{wcd-cdro}
      \begin{aligned}
            \mathbb{P}_{\lambda, \textnormal{Causal-WDRO}}^* = & \sum_{m \ne m_0} \sum_{i=1}^{n_m} \widehat{p}_{m, i} \widehat{\mathbb{P}}_{(\boldsymbol{x}_m^*(\lambda), \boldsymbol{y}_{m, i}^*(\lambda))} \\ & + \sum_{i=1}^{n_{m_0}} \widehat{p}_{m_0 i} \left( q \widehat{\mathbb{P}}_{(\bar{\boldsymbol{x}}_{m_0}(\lambda), \bar{\boldsymbol{y}}_{m_0 i}(\lambda))} + (1-q) \widehat{\mathbb{P}}_{(\underline{\boldsymbol{x}}_{m_0}(\lambda), \underline{\boldsymbol{y}}_{m_0 i}(\lambda))} \right),
      \end{aligned}
      \end{equation}
      where $1 \le m_0 \le M$, $0 \le q \le 1, (\boldsymbol{x}_m^*(\lambda), \boldsymbol{y}_{m, i}^*(\lambda)) = (\bar{\boldsymbol{x}}_m(\lambda), \bar{\boldsymbol{y}}_{m, i}(\lambda))$, and for every $m$ and $i$, 
      \begin{equation}\nonumber
          \bar{\boldsymbol{x}}_m(\lambda), \underline{\boldsymbol{x}}_m(\lambda) \in \arg\max_{\boldsymbol{x} \in \mathcal{X}} \left\{ \mathbb{E}_{\widehat{\mathbb{P}}_{\widehat{\boldsymbol{Y}}|\widehat{\boldsymbol{X}}}} \left[ \sup_{\boldsymbol{y} \in \mathcal{Y}} \Big \{ \Psi(f(\boldsymbol{x}), \boldsymbol{y}) - \lambda ||\boldsymbol{y}-\widehat{\boldsymbol{y}}||^p \Big \} \mid \widehat{\boldsymbol{X}}=\widehat{\boldsymbol{x}}_m \right] - \lambda ||\boldsymbol{x}-\widehat{\boldsymbol{x}}_m||^p \right\},
      \end{equation} and 
      \begin{equation}\nonumber
          \bar{\boldsymbol{y}}_{m, i} (\lambda)\in \arg\max_{\boldsymbol{y} \in \mathcal{Y}} \Big\{ \Psi(f(\bar{\boldsymbol{x}}_k), \boldsymbol{y}) - \lambda||\boldsymbol{y}-\widehat{\boldsymbol{y}}_{m, i}||^p \Big\}, \underline{\boldsymbol{y}}_{m, i}(\lambda) \in \arg\max_{\boldsymbol{y} \in \mathcal{Y}} \Big\{ \Psi(f(\underline{\boldsymbol{x}}_m), \boldsymbol{y}) - \lambda ||\boldsymbol{y}-\widehat{\boldsymbol{y}}_{m, i}||^p \Big \}. 
      \end{equation}
\end{lemma}
The worst-case distribution of the problem~\eqref{Causal-WDRO} in Eq.~\eqref{wcd-cdro} is discrete, while that of Causal-SDRO in Eq.~\eqref{wcd-csdro} is continuous. This shows that the introduction of CSD enables a more realistic and smoother representation of the underlying distribution.  

We solve the problem~\eqref{Causal-WDRO} as a contextual stochastic bilevel optimization problem by the Random Truncated Multilevel Monte Carlo (RT-MLMC) approach proposed by~\citet{hu2023contextual}. 

\subsection{KL-DRO}

The contextual KL-divergence-based DRO (KL-DRO) model is defined as 
\begin{equation}\label{kl-dro}
    \inf _{f \in \mathcal{F}} \sup_{\mathbb{P} \in \mathcal{P}(\mathcal{X} \times \mathcal{Y})} \quad \mathbb{E}_{(\boldsymbol{x}, \boldsymbol{y}) \sim \mathbb{P}}\Big[\Psi(f(\boldsymbol{x}), \boldsymbol{y})-\lambda \cdot \mathbb{D}_{\mathrm{KL}}\Big(\mathbb{P} || \widehat{\mathbb{P}} \Big) \Big]. 
\tag{\text{KL-DRO}}    
\end{equation}
For problem~\eqref{kl-dro}, the worst-case distribution is given by Lemma~\ref{lem-worst-kl}. 
\begin{lemma}\label{lem-worst-kl}
      \textnormal{\textbf{\citep[Worst-case Distribution of the KL-DRO Problem,][]{hu2013kullback}.}} If the worst-case distribution of KL-DRO problem exists for given $\lambda > 0$, then it has the following form 
      \begin{equation}
          \mathbb{P}^*_{\lambda,\textnormal{KL-DRO}} = \sum_{i=1}^N \frac{\exp\left( \frac{\Psi(f(\boldsymbol{\widehat{x}}_i), \boldsymbol{\widehat{y}}_i)}{\lambda} \right)}{\sum_{j=1}^N \exp\left( \frac{\Psi(f(\boldsymbol{\widehat{x}}_j), \boldsymbol{\widehat{y}}_j)}{\lambda} \right)} \cdot \widehat{\mathbb{P}}_{(\boldsymbol{\widehat{x}}_i, \boldsymbol{\widehat{y}}_i)},
      \end{equation}
      where $N$ is the number of historical observations. 
\end{lemma}
In fact, KL-DRO finds the worst-case distribution by changing the likelihood ratios of the empirical distribution rather than changing its support. Therefore, the worst-case distribution in~\eqref {kl-dro} remains discrete. 

When a parametric decision rule $f_{\boldsymbol{\theta}}$ is considered, according to Fenchel duality, the problem \eqref{kl-dro} is equivalent to 
\begin{equation}
    \min_{\boldsymbol{\theta} \in \Theta} \quad \lambda \cdot \log\, \mathbb{E}_{(\boldsymbol{x}, \boldsymbol{y})\sim \widehat{\mathbb{P}}}\Big[ \exp\Big( \frac{\Psi(f_{\boldsymbol{\theta}}(\boldsymbol{x}), \boldsymbol{y})}{\lambda} \Big)\Big], 
\end{equation}
which can be solved by the standard SGD method. 

\section{Interpretability of the Soft Regression Forest}\label{subsec-srf-interpret}

In this section, we demonstrate the intrinsic interpretability of SRF through its transparent structure and stability for decision-making (i.e., Lipschitz continuity and smoothness) in Appendix~\ref{subsec-srf-interpret-1}, and we introduce both global and local intrinsic interpretation measures for SRF that depend only on its structure and avoid post hoc explanation analyses in Appendix~\ref{ecsec-interpretability}. 

\subsection{Interpretability as a Tree-based Model} \label{subsec-srf-interpret-1}

As a tree-based model, SRF inherits the transparent structure from the traditional methods. 
\begin{remark} 
\textnormal{\textbf{(Asymptotic Consistency to Hard Regression Forest). }}
    Adding a scaling parameter $\tau$ to the linear transformation part at each internal node in all SRTs, that is, $\text{S}\Big( (\boldsymbol{w}^{\top}\boldsymbol{x} + b)/{\tau}\Big)$ for all internal nodes $j_t \in [2^{D(t)}-1]$, we obtain a variant of SRF decision rule termed $f_{\theta, \tau}^{SRF}(\boldsymbol{x})$ where each leaf node $l_t \in[2^{D(t)}]$ in tree $t$ can be reached with probability $p_{l_t, t, \tau}$. 
    The proposed SRF $f_{\theta}^{SRF}(\boldsymbol{x})$ shown in Section~\ref{sec-srf} is a special case with $\tau = 1$.
    For any input covariate $\boldsymbol{x}$ not lying on any decision boundary, that is, $\{\boldsymbol{x} \in \mathcal{X} \mid \boldsymbol{w}_{j_t,t}^{\top}\boldsymbol{x} + \boldsymbol{b}_{j_t,t} \neq 0, \forall j_t, t\}$, when $\tau \to 0$, the structure of SRF converges to a hard regression forest, which implies that the sigmoid function takes value only in $\{0,1\}$ and thus only one deterministic leaf node can be selected as the final decision, that is, $ \sum_{i=1}^{2^{D(t)}} p_{i, t, \tau} = 1$ and $\lim_{\tau \to 0}\, p_{i, t, \tau} \in \{0,1\}$. 
    Unlike traditional univariate decision trees, all trees in the resulting hard regression forest provide multivariate splits at every node, thereby improving accuracy and interpretability by reducing tree depth~\citep{bertsimas2017optimal, bertsimas2021voice}.  $\hfill \clubsuit$ 
\end{remark}

Although the SRT theoretically aggregates outputs across all routes, probabilities for weakly correlated routes effectively vanish as they are calculated as products of several Sigmoid functions. 
Consequently, the final decision is typically dominated by a few high-probability routes. 
This inherent sparsity enhances interpretability, enabling decision-makers to easily trace the primary routes driving the final prescription.
We provide empirical evidence for this in Appendix~\ref{subsec-results-portfolio}. 

As traditional decision trees allow for tracing the decision process via routes and identifying the impact of each feature, the SRT can also explicitly trace the influence of features and their interaction effects along each route. 
\begin{proposition} \label{prop-srt-derivation}
    \textnormal{ \textbf{(Traceability of Decisions in SRF). } }
    In SRF, given an input covariate $\boldsymbol{x} \in \mathbb{R}^{d_x}$, for each route selected with probability $p_{l_t, t} (\boldsymbol{x})$, the marginal contribution of each feature and the interaction effects among features along that route are explicitly characterized by 
    \begin{align*}\nonumber
        \frac{\partial p_{l_t, t} (\boldsymbol{x})}{\partial x_j} & = p_{l_t, t} (\boldsymbol{x}) \cdot \sum_{i \in \Lambda (l_t)} \psi_{i,t} \Big[ \boldsymbol{w}_{i, t} \Big]_{j},  \\
        \frac{\partial^2 p_{l_t, t}(\boldsymbol{x})}{\partial x_j\partial x_k} & =
            p_{l_t, t} (\boldsymbol{x}) \cdot \Bigg [ \Big( \sum_{i \in \Lambda (l_t)}\psi_{i,t} \Big[ \boldsymbol{w}_{i, t} \Big]_{j} \Big)\Big( \sum_{i \in \Lambda (l_t)}\psi_{i,t} \Big[ \boldsymbol{w}_{i, t} \Big]_{k} \Big) \, \, - \\
            & \quad \quad \quad \quad \quad 
            \sum_{i \in \Lambda (l_t)} \text{S}\Big( \boldsymbol{w}_{i,t}^{\top}\boldsymbol{x} + b_{i,t}\Big) \Big( 1- \text{S}\Big(\boldsymbol{w}_{i,t}^{\top}\boldsymbol{x} + b_{i,t}\Big)  \Big) \Big[ \boldsymbol{w}_{i, t} \Big]_{j} \Big[ \boldsymbol{w}_{i, t} \Big]_{k} \Bigg ], \\
        \forall & j, k \in [d_x], \, l_t \in [2^{D(t)}], \, t \in T,
    \end{align*}
    where for each $i \in \Lambda (l_t), l_t \in [2^{D(t)}], t \in T$, 
    \begin{equation}\nonumber
        \psi_{i,t} := \begin{cases} 
        1 - \text{S}\Big( \boldsymbol{w}_{i,t}^{\top}\boldsymbol{x} + b_{i,t}\Big), & \text{if route $ l_t$ goes left at node }i; \\ 
        -\text{S}\Big( \boldsymbol{w}_{i,t}^{\top}\boldsymbol{x} + b_{i,t}\Big),  & \text{if route $ l_t$ goes right at node }i. 
        \end{cases}
    \end{equation}
\end{proposition}
We provide the proof of Proposition~\ref{prop-srt-derivation} in Appendix~\ref{subsec-srf-interpret-1}. 
Although many uninterpretable deep learning models are also differentiable, their gradients are typically aggregated through opaque dense layers, obscuring the internal decision mechanism. In contrast, the SRF derivatives explicitly decompose the feature influence into specific decision nodes along the route, allowing us to exactly trace \textit{where} (at which node) and \textit{how} (direction and magnitude) a feature contributes to the decision process. 

Beyond the transparency and traceability, we next show that the mathematical smoothness of the SRF structure also contributes to interpretability. 
Define $W_{\max} := \max_{i,t} \| \boldsymbol{w}_{i,t} \|$ as the maximum norm of the internal node weights, $\Pi_{\max} := \max_{l_t, t} \|\boldsymbol{\pi}_{l_t, t}\|$ as the maximum norm of leaf vectors, and $D_{\max}$ as the maximum tree depth ($D_{\max} \ge 2$). 
Also, assume that both the gating weights and the leaf decisions are uniformly bounded over $\Theta$, i.e.,
\[ 
\sup_{\boldsymbol\theta\in\Theta}\max_{i, t}\|\boldsymbol w_{i,t}\|<\infty, \qquad \sup_{\boldsymbol\theta\in\Theta}\max_{t,l}\|\boldsymbol\pi_{l,t}\|<\infty. 
\] 
Let $R_x<\infty$ satisfy $\sup_{\boldsymbol{x}\in\mathcal{X}}\|\boldsymbol{x}\|\le R_x$. 
Then, the following proposition holds.
\begin{proposition}\label{prop-srf-lip}
    \textnormal{ \textbf{(Lipschitz Continuity and Smoothness of SRF). } }
    Assume that the leaf decisions of SRF are uniformly bounded over $\Theta$, and then the SRF decision rule $f_{\boldsymbol{\theta}}^{\textnormal{SRF}}: \mathcal{X} \to \mathcal{Z}$ is: 
    \begin{enumerate}
        \item Lipschitz continuous and Lipschitz smooth with respect to $\boldsymbol{x} \in \mathcal{X}$ for each fixed $\boldsymbol{\theta} \in \Theta$; \label{prop-srf-lip-1}
        \item Lipschitz continuous and Lipschitz smooth with respect to $\boldsymbol{\theta} \in \Theta$ with constants uniform over $\boldsymbol{x} \in \mathcal{X}$. 
    \end{enumerate}
\end{proposition}
We provide the proof of Proposition~\ref{prop-srf-lip} in Appendix~\ref{ecsec-prop-srt-lip}. 
These Lipschitz properties confirm the decision stability and robustness (interpretation stability) of SRF, distinguishing it from existing uninterpretable deep learning models and post hoc explanation methods, which are typically not Lipschitz, as small input perturbations may lead to abrupt changes in decisions and explanations~\citep{alvarez2018robustness}. 

All analyses above demonstrate the structural transparency, stability, and robustness of the SRF decision rule. 
In Appendix~\ref{ecsec-interpretability}, we further introduce both global and local intrinsic interpretation measures for SRF, which depend only on the structure of SRF and avoid post hoc explanation analyses. 
In Appendix~\ref{subsec-results-portfolio}, we confirm the practical interpretability of SRT, based on its structure and the proposed intrinsic interpretation measures for the portfolio problem, using real data. 

\subsection{Interpretability Measures} \label{ecsec-interpretability}

In this subsection, we analyze the intrinsic interpretability of the SRT by introducing global and local interpretation measures, as detailed in Appendices~\ref{ecsec-interpretability-global} and~\ref{ecsec-interpretability-local}, respectively. 

\subsubsection{ Global Interpretation Measure } \label{ecsec-interpretability-global}

Global interpretability provides a holistic view of the model by quantifying each feature's contribution to the overall decision-making process~\citep{dwivedi2023explainable}.
For tree-based methods, standard techniques include impurity-based and permutation-based importance measures~\citep{hastie2009elements, kallus2023stochastic}. 
However, impurity-based measures are designed for trees with hard splits and are inapplicable to SRTs, while permutation-based methods are often computationally expensive. 
Therefore, leveraging the differentiability of the SRT, we define feature importance based on the average marginal sensitivity of each decision with respect to each feature over the training set: 
\begin{equation}\nonumber
    \mathcal{C}_{j, k} := \frac{1}{N} \sum_{i=1}^{N} \Big \| \frac{\partial \Big [ f_{\boldsymbol{\theta}}^{\text{SRF}}(\boldsymbol{x}^i) \Big]_k}{\partial x^i_j}\Big \|_1 = \frac{1}{N\cdot T}  \sum_{i=1}^{N} \sum_{t=1}^{T} \sum_{l_t=1}^{2^{D(t)}} \Big \| \frac{\partial p_{l_t, t} (\boldsymbol{x}^i)}{\partial x^i_j} [\boldsymbol{\pi}_{l_t, t}]_k\Big \|_1, \quad \forall j \in [d_x], k \in [d_z], 
\end{equation}
where $\boldsymbol{x}^i$ is the $i$-th training sample and $x^i_j$ is its $j$-th element, and the partial derivatives of $p_{l_t, t} (\boldsymbol{x}^i)$ are computed following Proposition~\ref{prop-srt-derivation}. 
To show the relative importance of each feature to the $k$-th decision component, we normalize the importance scores $\mathcal{C}_{j,k}$ across all features $j\in[d_x]$ for each $k\in[d_z]$ such that they sum to one. 
The relative feature importance is then defined as 
\begin{equation}\nonumber
    \bar{C}_{j,k} := \frac{\mathcal{C}_j}{\sum_{i=1}^{d_x}\mathcal{C}_i}, \quad \forall j \in [d_x], k \in [d_z]. 
\end{equation}

\subsubsection{ Local Interpretation Measure }\label{ecsec-interpretability-local}

Local interpretability demonstrates how an individual decision is derived, clarifying the contribution of specific features and their interactions~\citep{dwivedi2023explainable, notz2024explainable}. 
\citet{lundberg2020local} propose SHAP (SHapley Additive exPlanations) as a post hoc local explainer to characterize feature contributions.  
While SHAP enhances the transparency of inherently uninterpretable models, it can only be used after the decision has been made and does not exploit the model's structure. 

In contrast, given that the SRF is intrinsically interpretable and differentiable, we {measure the contribution of each feature by the Empirical Integrated Gradient (EIG), which is based on the Integrated Gradients (IG) method proposed by~\citet{sundararajan2017axiomatic} and derives feature contributions directly from the model's structure. }
We validate the consistency of our intrinsic interpretability by comparing EIG with SHAP values. 

According to~\citet{ancona2017towards} and~\citet{notz2024explainable}, given input covariate $\boldsymbol{x}$, the $k$-th decision of the SRF can be decomposed into the sum of feature contributions relative to a baseline: 
\begin{equation}\nonumber
    \Big [ f_{\boldsymbol{\theta}}^{\text{SRF}}(\boldsymbol{x}) \Big ]_{k}- \text{Baseline}_k =  \sum_{j=1}^{d_x} \varphi_{j, k}(\boldsymbol{\theta}, \boldsymbol{x}). 
\end{equation}
Distinct from~\citet{sundararajan2017axiomatic} and \citet{ancona2017towards}, which typically use a zero baseline, we employ the average decision over the training set as a robust baseline.
Then, similar to~\citet{erion2021improving}, we define the EIG for feature $j$ and decision $k$ as
\begin{equation}\nonumber
    \varphi_{j, k}^{\text{EIG}}(\boldsymbol{\theta}, \boldsymbol{x}) := \frac{1}{N} \cdot \sum_{i=1}^{N} (x_j-x_j^i) \cdot \int_{\alpha= 0}^{1} \frac{\partial \Big [ f_{\boldsymbol{\theta}}^{\text{SRF}}\Big(\boldsymbol{x}^i+\alpha(\boldsymbol{x}-\boldsymbol{x}^i)\Big) \Big ]_k}{\partial x_j } \mathrm{d}\alpha, \quad \forall j \in [d_x], k \in [d_z],
\end{equation}
where $\boldsymbol{x}^i$ denotes the $i$-th training sample. 
Proposition~\ref{prop-EIG} shows that the $k$-th output of SRF can be decomposed into the sum of EIG contributions from each feature, with the baseline defined as the average value of the $k$-th output over the data set, that is, 
\begin{equation}\nonumber
    \text{Baseline}_k = \frac{1}{N} \sum_{i=1}^{N} \Big [ f_{\boldsymbol{\theta}}^{\textnormal{SRF}}(\boldsymbol{x}^{i}) \Big ]_{k}. 
\end{equation}

\begin{proposition}\label{prop-EIG}
    Since the decision rule $f_{\boldsymbol{\theta}}^{\textnormal{SRF}}: \mathbb{R}^{d_x} \to \mathbb{R}^{d_z}$ is differentiable, given an input covariate $\boldsymbol{x}$, the EIG value $\varphi_{j, k}^{\textnormal{EIG}}(\boldsymbol{\theta}, \boldsymbol{x})$ exactly quantifies the contribution of feature $j$ to the deviation of the decision from the average training decision baseline:  
    \begin{equation}\nonumber
        \Big [ f_{\boldsymbol{\theta}}^{\textnormal{SRF}}(\boldsymbol{x}) \Big ]_{k} - \frac{1}{N}\sum_{i=1}^{N} \Big [ f_{\boldsymbol{\theta}}^{\textnormal{SRF}}(\boldsymbol{x}^{i}) \Big ]_{k} = \sum_{j=1}^{d_x} \varphi_{j, k}^{\textnormal{EIG}}(\boldsymbol{\theta}, \boldsymbol{x}), \quad \forall k \in [d_z].
    \end{equation}
\end{proposition}

\begin{proof}[Proof of Proposition~\ref{prop-EIG}] 
    According to the Proposition 1 in~\citet{sundararajan2017axiomatic}, for any $\boldsymbol{x}^{\prime} \in \mathbb{R}^{d_x}$, we have 
    \begin{equation}\nonumber
        \begin{aligned}
            \Big [ f_{\boldsymbol{\theta}}^{\text{SRF}}(\boldsymbol{x}) \Big ]_{k} & - \Big [ f_{\boldsymbol{\theta}}^{\text{SRF}}(\boldsymbol{x}^{\prime}) \Big ]_{k} \\ & = \sum_{j=1}^{d_x} (x_j-x_j^{\prime}) \cdot \int_{\alpha= 0}^{1} \frac{\partial \Big [ f_{\boldsymbol{\theta}}^{\text{SRF}}\Big(\boldsymbol{x}^{\prime}+\alpha(\boldsymbol{x}-\boldsymbol{x}^{\prime})\Big) \Big ]_k}{\partial x_j } \mathrm{d}\alpha, \quad \forall k \in [d_z]. 
        \end{aligned}
    \end{equation}
    By setting the baseline as the average decision over the training set, we obtain 
    \begin{equation}\nonumber
    \begin{aligned}
        \Big [ f_{\boldsymbol{\theta}}^{\text{SRF}}(\boldsymbol{x}) \Big ]_{k} & - \frac{1}{N}\sum_{i=1}^{N} \Big [ f_{\boldsymbol{\theta}}^{\text{SRF}}(\boldsymbol{x}^{i}) \Big ]_{k}  \\ 
        & = \frac{1}{N}\sum_{i=1}^{N} \sum_{j=1}^{d_x} (x_j-x_j^{i}) \cdot \int_{\alpha= 0}^{1} \frac{\partial \Big [ f_{\boldsymbol{\theta}}^{\text{SRF}}\Big(\boldsymbol{x}^{i}+\alpha(\boldsymbol{x}-\boldsymbol{x}^{i})\Big) \Big ]_k}{\partial x_j } \mathrm{d}\alpha \\
        & = \sum_{j=1}^{d_x} \varphi_{j, k}^{\text{EIG}}(\boldsymbol{\theta}, \boldsymbol{x}), \quad \quad \forall k \in [d_z]. 
    \end{aligned}
    \end{equation}
\end{proof}

Regarding local feature interactions, the SRT's differentiability allows us to explicitly characterize interaction effects by computing the Hessian matrix, as detailed in Proposition~\ref{prop-srt-derivation}. 

\section{Analyses for the Stochastic Compositional Optimization} \label{app-sec-SCO}

In this section, we analyze the properties of~\eqref{dual-sco} in Appendix~\ref{app-subsec-algo-property}, establish the sample complexity of the sample average approximation method in Appendix~\ref{app-subsec-algo-saa}, and provide the motivation for developing a new gradient-based algorithm to solve~\eqref{dual-sco} in Appendix~\ref{app-subsec-algo-scsc}. 
 
\subsection{Properties of CSCO} \label{app-subsec-algo-property}

For brevity, for given $\widehat{\boldsymbol{x}}, \boldsymbol{\xi}_1$, and $\boldsymbol{\xi}_2$, we denote functions
\begin{equation}\label{t-function-brief}
    t_1 (z):=t_1 (z; \widehat{\boldsymbol{x}}), \quad t_2(\boldsymbol{v}):=t_2(\boldsymbol{v}; \widehat{\boldsymbol{x}}, \boldsymbol{\xi}_1), \quad t_3 (\boldsymbol{\theta}):=t_3 (\boldsymbol{\theta}; \widehat{\boldsymbol{x}}, \boldsymbol{\xi}_1, \boldsymbol{\xi}_2), 
\end{equation} and define 
\begin{equation}\label{phi-functions}
    \phi^{(0)}(\boldsymbol{\theta}) := \mathbb{E}_{\widehat{\boldsymbol{x}}}\Big[t_1 (\phi^{(1)}(\boldsymbol{\theta}))\Big], \quad \phi^{(1)}(\boldsymbol{\theta}) := \mathbb{E}_{\boldsymbol{\xi}_1}\Big[t_2 (\phi^{(2)}(\boldsymbol{\theta}))\Big], \quad \phi^{(2)}(\boldsymbol{\theta}) := \mathbb{E}_{\boldsymbol{\xi}_2}\Big[t_3 ( \boldsymbol{\theta})\Big]. 
\end{equation}

Let $\mathbb{V}_{\boldsymbol{\xi}}(t(\cdot; \boldsymbol{\xi}))$ denote the variance of the random variable (or random vector) $t(\cdot; \boldsymbol{\xi})$. Under Assumption~\ref{assumption-2}, for the problem~\eqref{dual-sco}, we have the following Proposition~\ref{propos-lip-con}. 
\begin{proposition} \label{propos-lip-con} 
\textnormal{\textbf{(Properties of the Problem~\ref{dual-sco})}.}
    Under Assumption~\ref{assumption-2}, functions $t_1$, $t_2$, and $t_3$ in Eq.~\eqref{t-functions} are restricted to $[1, \exp(B/\lambda\epsilon)]$, $[1, \exp(B/\lambda\epsilon)]^N$, and $\Theta$, respectively. On these domains, the following properties hold.  
    \begin{enumerate}
        \item  $t_1$, $t_2$, and $t_3$ are respectively $L_1$-, $L_2$-, and $L_3$-Lipschitz continuous;
        \item  $t_1$, $t_2$, and $t_3$ are respectively $S_1$-, $S_2$-, and $S_3$-Lipschitz smooth;
        \item  Their first-order derivatives have uniformly bounded second moments, i.e., 
        \begin{equation}\nonumber
            \mathbb{E}\Big[|\nabla t_1(z) |^2\Big] \le C_1^2, \, \mathbb{E}\Big[\|\nabla t_2 (\boldsymbol{v}) \|^2\Big] \le C_2^2, \, \mathbb{E}\Big[\|J_{\boldsymbol{\theta}} t_3(\boldsymbol{\theta}) \|_{\text{F}}^2\Big] \le C_3^2;
        \end{equation}
        \item They have finite variances, i.e., $\sigma_1^2 = \sup_{\boldsymbol{\theta}} \mathbb{V}_{\widehat{\boldsymbol{x}}} \Big( t_1\Big(\phi^{(1)}(\boldsymbol{\theta})\Big) \Big)$, $\sigma_2^2 = \sup_{\boldsymbol{\theta}, \widehat{\boldsymbol{x}}} \mathbb{V}_{\boldsymbol{\xi}_1} \Big( t_2\Big(\phi^{(2)}(\boldsymbol{\theta})\Big)\Big)$, and $\sigma_3^2 = \sup_{\boldsymbol{\theta}, \widehat{\boldsymbol{x}}, \boldsymbol{\xi}_1} \mathbb{V}_{\boldsymbol{\xi}_2} \Big( t_3\Big( \boldsymbol{\theta} \Big)\Big)$ are all finite. 
    \end{enumerate}
    Here, $J_{\boldsymbol{\theta}} t_3(\boldsymbol{\theta})$ denotes the Jacobian of function $t_3$ with respect to $\boldsymbol{\theta}$, and
    \begin{equation}\nonumber
    \begin{aligned}
            L_1 & = S_1 = C_1 =1; \quad \quad \quad
            L_2  = C_2 = \exp\Big(\frac{B}{\lambda\epsilon}\Big), \quad S_2 = \sqrt{2} L_2; \\ 
            L_3 & =C_3 = \frac{\sqrt{N} L_{\boldsymbol{\theta}}}{\lambda\epsilon} \exp\Big(\frac{B}{\lambda\epsilon}\Big),\quad  S_3 = \sqrt{N} e^{\frac{B}{\lambda\epsilon}} \Big( \frac{S_{\boldsymbol{\theta}}}{\lambda\epsilon} + \frac{L_{\boldsymbol{\theta}}^2}{\lambda^2 \epsilon^2} \Big). 
    \end{aligned}
    \end{equation}
\end{proposition}
We provide the proof of Proposition~\ref{propos-lip-con} in Appendix~\ref{ecsec-prop-lip}. 
These properties are used for the analyses in Appendices~\ref{app-subsec-algo-saa} and ~\ref{app-subsec-algo-scsc}

\subsection{Complexity of Sample Average Approximation} \label{app-subsec-algo-saa}

A standard method for solving~\eqref{dual-sco} is the Sample Average Approximation (SAA), which replaces each nested expectation with its empirical average based on finite samples generated via Monte Carlo sampling. In this section, we analyze the sample and computational complexity of SAA on our problem. 

{ We first introduce the following assumption in this subsection regarding the diameter of the decision space.
\begin{assumption}\label{assumption-2-theta}
\textnormal{\textbf{(Bounded Diameter)}.} The decision set $\Theta \subseteq \mathbb{R}^{d_\theta}$ is convex and has a positive finite diameter $D_{\Theta}>0$, that is, for any $\boldsymbol{\theta}_1, \boldsymbol{\theta}_2 \in \Theta$, $\|\boldsymbol{\theta}_1 - \boldsymbol{\theta}_2\| \le D_{\Theta}$. 
\end{assumption}
Note that this assumption is used only in the sample complexity analysis of SAA. The convergence analysis of the proposed CM-SCSC algorithm does not require this assumption. }

We draw $N_1$ independent and identically distributed (i.i.d.) samples $\{(\widehat{\boldsymbol{x}}^i, \widehat{\boldsymbol{y}}^i)\}_{i=1}^{N_1}$ from the nominal distribution $\widehat{\mathbb{P}}$, $N_2$ i.i.d. samples $\{\boldsymbol{\xi}_1^j\}_{j=1}^{N_2}$ from the kernel distribution $U_{\epsilon}$, and $N_3$ i.i.d. samples $\{\boldsymbol{\xi}_2^k\}_{k=1}^{N_3}$ from the kernel distribution $W_{\epsilon}$. This leads to the following SAA formulation for~\eqref{dual-sco}:
\begin{equation}\label{f-saa}
    \min_{\boldsymbol{\theta} \in \Theta}\quad \widehat{F}_{N_1,N_2,N_3}\left ( \boldsymbol{\theta} \right ) = \frac{\lambda\epsilon}{N_1}  \sum_{i=1}^{N_1} t_1\Big( \frac{1}{N_2} \sum_{j=1}^{N_2} t_2\Big( \frac{1}{N_3} \sum_{k=1}^{N_3}t_3\Big( \boldsymbol{\theta}; \widehat{\boldsymbol{x}}^i, \boldsymbol{\xi}_1^j, \boldsymbol{\xi}_2^k  \Big); \widehat{\boldsymbol{x}}^i, \boldsymbol{\xi}_1^j \Big) ;\widehat{\boldsymbol{x}} \Big). 
    \tag{SAA}
\end{equation}

Let $\boldsymbol{\theta}^*$ and $\widehat{\boldsymbol{\theta}}_{N_1,N_2,N_3}$ be the optimal solutions of the problems~\eqref{dual-sco} and~\eqref{f-saa}, respectively. 
We next analyze the number of samples required for the solution to the problem~\eqref{f-saa} to be $\delta$-optimal of the problem~\eqref{dual-sco} with high probability, that is, $\mathrm{Pr} \left ( F\left ( \widehat{\boldsymbol{\theta}}_{N_1, N_2, N_3}  \right ) -F \left ( \boldsymbol{\theta}^* \right ) \le \delta \right )\ge 1-\alpha$ for any $\delta > 0$ and $\alpha \in \left(0,1\right)$. 

Based on Assumption~\ref{assumption-2}, we derive the sample complexity of the SAA method for this problem in Theorem~\ref{theo-saa-sample}. 
\begin{theorem}\label{theo-saa-sample}
\textnormal{\textbf{(Sample Complexity for the Problem~\ref{f-saa})}.}
    Under Assumption~\ref{assumption-2}, the following results hold. 
    \begin{enumerate}
    \item For any $\kappa > 0 $, there exists an $\delta_1>0$ such that for any $\delta \in \left(0, \delta_1 \right) $, it holds that  \begin{equation}\nonumber
        \begin{aligned}
            & \mathrm{Pr} \left ( F\left ( \widehat{\boldsymbol{\theta}}_{N_1,N_2,N_3}  \right ) -F \left ( \boldsymbol{\theta}^* \right ) > \delta \right ) \\
            \le & \mathcal{O}\left ( 1 \right )  \left ( \frac{8L_1L_2L_3D_{\Theta}}{\delta} \right )^{d_{\theta}} \Big( N_1^2 N_2 \exp\Big( -\frac{N_3 \delta^2}{144(2+\kappa)\lambda^2\epsilon ^2L_1^2L_2^2\sigma_3^2} \Big) \\
        & \quad \quad \quad \quad + N_1\exp\Big( -\frac{N_2 \delta^2}{144(2+\kappa)\lambda^2\epsilon ^2L_1^2\sigma_2^2}\Big) + \exp\Big(-\frac{N_1 \delta^2}{144(2+\kappa)\lambda^2\epsilon ^2\sigma_1^2}\Big) \Big). 
        \end{aligned}
    \end{equation}  \label{theorem-saa-2}
    \item With probability at least $1-\alpha$, the solution to the problem~\eqref{f-saa} is $\delta$-optimal to the original problem~\eqref{dual-sco} if the sample sizes $N_1,N_2$, and $N_3$ satisfy that
    \begin{equation}\nonumber
        \begin{aligned}
            N_1 &> \frac{\mathcal{O}\left ( 1 \right ) \sigma_1^2}{\delta^2} \Big[ d_{\theta}\log\, \left ( \frac{8L_1L_2L_3D_{\Theta}}{\delta} \right )  +\log\, \left ( \frac{1}{\alpha}  \right )  \Big ], \\
            N_2 &> \frac{\mathcal{O}\left ( 1 \right ) L_1^2\sigma_2^2}{\delta^2} \Big[ d_{\theta}\log\, \left ( \frac{8L_1L_2L_3D_{\Theta}}{\delta}  \right ) + \log\, \left ( \frac{1}{\alpha}   \right ) +\log\,  \left ( N_1 \right )  \Big ] , \\
            \text{and } \\
            N_3 &> \frac{\mathcal{O}\left ( 1 \right ) L_1^2L_2^2\sigma_3^2}{\delta^2} \Big[ d_{\theta}\log\, \left ( \frac{8L_1L_2L_3D_{\Theta}}{\delta}  \right ) + \log\, \left ( \frac{1}{\alpha}   \right ) +\log\,  \left ( N_1^2 N_2 \right )  \Big ] .
        \end{aligned}
    \end{equation}
    Ignoring the log factors, the total sample complexity of the problem~\eqref{f-saa} for achieving a $\delta$-optimal solution is $ T = N_1 + N_2 + N_3 = \mathcal{O}\Big( d_{\theta} / \delta^2 \Big)$. \label{theorem-saa-3} 
    \end{enumerate} 
\end{theorem} 
The proof of Theorem~\ref{theo-saa-sample} is provided in Appendix~\ref{ecsec-theo-saa}. 
For problem~\eqref{f-saa}, the computational complexity of gradient descent (GD) is at least $\mathcal{O}\Big( d_{\theta}^3/\delta^6 \Big)$, since a single iteration requires $N_1 \times N_2 \times N_3$ gradient updates. Similarly, the computational complexity of using the unbiased stochastic gradient descent (SGD) algorithm is at least $\mathcal{O}\Big( d_{\theta}^2/\delta^4 \Big)$, which is still computationally challenging. 

\subsection{The Motivation of CM-SCSC} \label{app-subsec-algo-scsc}

In this subsection, we show another challenge of applying the standard SGD algorithm to problem~\eqref{dual-sco}. 
In the $k$-th iteration, the true gradient of $F$ at $\boldsymbol{\theta}^k$ is given by 
\begin{equation}\nonumber 
\begin{aligned} 
    \nabla F(\boldsymbol{\theta}^{k}) = &  \nabla  t_1\Big(\phi^{(1)}(\boldsymbol{\theta}^{k})\Big) \cdot \Big ( J_{\boldsymbol{\theta}} t_3 (\boldsymbol{\theta}^{k}) \Big)^\top \nabla t_2\Big( \phi^{(2)}(\boldsymbol{\theta}^{k})\Big), 
\end{aligned} 
\end{equation}  
where functions $t_1$, $t_2$, $t_3$, $\phi^{(1)}$, and $\phi^{(2)}$ are defined in Eqs.~\eqref{t-function-brief}and~\eqref{phi-functions}. 
Given the samples $(\widehat{\boldsymbol{x}}^{k},\widehat{\boldsymbol{y}}^k)$, $\boldsymbol{\xi}_1^{k}$, and $\boldsymbol{\xi}_2^{k}$ drawn at iteration $k$, for brevity, we define 
\begin{equation}\nonumber
    t_1^k(z) := t_1(z; \widehat{\boldsymbol{x}}^{k}), \quad t_2^k(\boldsymbol{v}) := t_2( \boldsymbol{v}; \widehat{\boldsymbol{x}}^{k}, \boldsymbol{\xi}_1^{k}),  \text{ and  } t_3^k(\boldsymbol{\theta}) := t_3( \boldsymbol{\theta}; \widehat{\boldsymbol{x}}^{k}, \boldsymbol{\xi}_1^{k}, \boldsymbol{\xi}_2^{k}). 
\end{equation}
The SGD method replaces $t_2(\phi^{(2)}(\boldsymbol{\theta}^{k}))$ by $t_2^k( t_3^k ( \boldsymbol{\theta}^{k} ) )$, and $t_1 ( \phi^{(1)}(\boldsymbol{\theta}^{k}) )$ by $t_1^k ( t_2^k ( t_3^k ( \boldsymbol{\theta}^{k} ) ) )$ in each iteration. That is, it simplifies the computation of the true gradient $\nabla F(\boldsymbol{\theta})$ by replacing the expected values with stochastic estimates computed from single random samples. 
However, as functions $t_1$, $t_2$, and $t_3$ are generally non-linear, the stochastic gradient of the SGD method, denoted as $\Big(\nabla F(\boldsymbol{\theta}^{k})\Big)_{\text{SGD}}$, is biased, that is, 
\begin{equation}\nonumber
    \begin{aligned}
        \mathbb{E}\Big[ \Big(\nabla F(\boldsymbol{\theta}^{k})\Big)_{\text{SGD}}\Big] & =   \mathbb{E}_{\widehat{\boldsymbol{x}}^{k}, \boldsymbol{\xi}_1^{k},\boldsymbol{\xi}_2^{k}} \Big[
        \nabla t_1^k\Big( t_2^k\Big(   t_3^k\Big( \boldsymbol{\theta}^{k}  \Big)\Big)\Big)  \cdot \Big ( J_{\boldsymbol{\theta}} \nabla  t_3^k (\boldsymbol{\theta}^{k})\Big )^\top \nabla  t_2^k\Big(  t_3^k\Big( \boldsymbol{\theta}^{k} \Big)\Big) \Big] \neq \mathbb{E}\Big[\nabla F(\boldsymbol{\theta}^{k})\Big].
    \end{aligned}
\end{equation}
Thus, it cannot be used to solve problem~\eqref{dual-sco} directly. 
We present a new gradient-based method for solving~\eqref{dual-sco} in Section~\ref{subsec-algo-sco} that controls bias using a momentum-based gradient update. 

\section{Numerical Examples and Convergence Verification} \label{Appdex-sec-losses}

{Before presenting the detailed numerical results, we introduce our three applications of Causal-SDRO (i.e., the feature-based newsvendor problem, the feature-based inventory substitution problem, and the data-driven portfolio problem) in this section and show how to modify their loss function with the SRF decision rule to satisfy the convergence results of CM-SCSC. } 

{ We first define a softplus function with constant $a > 0$: 
\[
    \mathrm{s}_{a} (t) := a \cdot \log\, (1+e^{t/a}), 
\]
and a smooth truncation function that restricts the nonnegative input $t$ to $[0,B)$: 
\begin{equation}\label{Eq-TB}
     \mathrm{T}_{B}(t) := B \Big( 1 - \frac{\mathrm{s}_{a} (B-t)}{\mathrm{s}_{a} (B)} \Big), \qquad \forall t\ge0.    
\end{equation}
In practice, a small $a$ compared to $B$, e.g., $a = B/50$, may lead to a small approximation error between $\mathrm{T}_{B}(t)$ and the hard truncation $\min\{t,B\}$, which, specifically:
\[
    0 \le \min\{t,B\} - \mathrm{T}_{B}(t) \le a \, \log\, 2. 
\]}
{During training, the SRF leaf decisions are constrained to lie in the decision space $\mathcal{Z}$ and are uniformly bounded over $\Theta$. For bounded inputs to the SRF, Proposition~\ref{prop-srf-lip} establishes Lipschitz continuity and Lipschitz smoothness with respect to $\boldsymbol\theta$, uniformly over $\boldsymbol{x}\in\mathcal{X}$. 
Therefore, in the following, for each training loss that is Lipschitz continuous and Lipschitz smooth with respect to $\boldsymbol{z}$, uniformly over its effective $\boldsymbol{y}$ inputs, the corresponding properties with respect to $\boldsymbol\theta$ also hold by the chain rule. In the numerical experiments, perturbed covariates are clipped to fixed bounded ranges before being input to the SRF, ensuring that Proposition~\ref{prop-srf-lip} holds. } 

\subsection{Feature-based Newsvendor Problem} \label{appsubsec-intro-news}

{The decision space of the feature-based newsvendor problem is convex. To ensure Assumption~\ref{assumption-2} is satisfied, we use the following newsvendor training loss in the numerical study:
\[
   \widetilde{\Psi}_{\text{Newsvendor}}(\boldsymbol{z},\boldsymbol{y}) = \mathrm{T}_{B}\Big (\Psi_{\text{Newsvendor}}(\boldsymbol{z},\boldsymbol{y})\Big), 
\]
where $\mathrm{T}_{B}$ is defined in Eq.~\eqref{Eq-TB} and 
\begin{equation}\nonumber
    \Psi_{\text{Newsvendor}}(\boldsymbol{z},\boldsymbol{y}) := 
    \sum_{i=1}^{d_y} \Big( h_i \cdot \mathrm{s}_{a} (z_i -y_i) + b_i \cdot \mathrm{s}_{a} (y_i - z_i) \Big).
\end{equation}}  
{The loss function $\widetilde{\Psi}_{\text{Newsvendor}}$ is differentiable, bounded in $[0,B]$, and Lipschitz continuous and smooth with respect to $\boldsymbol{z}$. 
Then, the Lipschitz properties of the SRF decision rule ensure that $L_{\text{Newsvendor}}(\boldsymbol{\theta}; \boldsymbol{x}, \boldsymbol{y}):= \widetilde{\Psi}_{\text{Newsvendor}}(f_{\boldsymbol{\theta}}(\boldsymbol{x}), \boldsymbol{y})$ is differentiable, Lipschitz continuous, and smooth with respect to $\boldsymbol{\theta}$ on $\Theta$ for any covariate input $\boldsymbol{x}$. Hence, Assumptions~\ref{assumption-2}\ref{assumption-2-lip-con} and~\ref{assumption-2}\ref{assumption-2-bound} are satisfied. } 

\subsection{ Feature-based Inventory Substitution Problem } \label{subsec-intro-inventory} 

\noindent{\textbf{Model Setup. }}
This problem, as a variant of the supply chain substitution problem in~\citet{chen2019stochastic}, is a two-stage optimization problem. Consider a firm selling $d_z$ product types to satisfy customer demand, indexed by $i \in [d_z]$, where a lower index denotes a higher quantity. There is a demand class corresponding to each product, indexed by $j \in [d_y]$ ($d_z=d_y$). If any demand class $j$ cannot be satisfied, products with higher quantity $i < j$ can substitute for demand $j$ at an extra cost $s_{i,j} \ge 0$. Before the real demand is observed, the ``here-and-now" decision for the firm is to decide the prepared inventory level $\boldsymbol{z} \in \mathcal{Z} \subseteq \mathbb{R}_+^{d_z}$ (suppose that the initial inventory level is zero) with cost $\boldsymbol{c} \in \mathbb{R}_+^{d_z}$. Given the demand $\boldsymbol{y} \in \mathcal{Y} \subseteq \mathbb{R}_+^{d_y}$, the ``wait-and-see" decision is to allocate inventory to each demand class to minimize total cost. Let the decision variable $w_{i,j} \ge 0$ denote the quantity of product $i$ as a substitute for $j$, while $h_i \ge 0$ and $b_j \ge 0$ denote the unit holding cost of product $i$ and the shortage cost for demand $j$, respectively. Taking features $\boldsymbol{x} \in \mathcal{X} \subseteq \mathbb{R}^{d_x}$ (for example, demographic data) into account, this two-stage contextual DRO problem is given by 
    \begin{equation}\nonumber
        \inf _{f \in \mathcal{F}}\,\,  \sup_{\mathbb{P} \in \, \Re\, ({\widehat{\mathbb{P}}})} \quad  \mathbb{E}_{(\boldsymbol{x}, \boldsymbol{y}) \sim \mathbb{P}}\Big[\boldsymbol{c}^{\top}f(\boldsymbol{x}) +\Psi_{\text{Inventory}}(f(\boldsymbol{x}), \boldsymbol{y})\Big]
    \end{equation}  
    where the inventory substitution loss function $\Psi_{\text{Inventory}}:~\mathcal{Z}\times\mathcal{Y}\to\mathbb{R}$ is defined as
    \begin{equation}\label{psi-scs-problem}
        \begin{aligned}
           \Psi_{\text{Inventory}}(\boldsymbol{z}, \boldsymbol{y}) =  \min & \quad \sum_{j=1}^{d_y} \sum_{i=1}^{j} s_{i,j}w_{i,j} + \sum_{i=1}^{d_z} h_i u_i + \sum_{j=1}^{d_y} b_j u_j^{\prime} \\ 
           \text{s.t. } & \quad  \sum_{j=i}^{d_y}w_{i,j} +u_i = z_{i}, & \forall i \in \left[d_z\right], \\
            & \quad  \sum_{i=1}^{j}w_{i,j} +u_j^{\prime} = y_j, & \forall j \in \left[d_y\right], \\
            & \quad u_i , u_j^{\prime}, w_{i,j} \ge 0, & \forall i \in \left[d_z\right], j \in \left[d_y\right]. 
        \end{aligned}
    \end{equation}
    where $u_i$ represents the leftover inventory of product $i$, $u_j^{\prime}$ represents the shortage of demand $j$, and the three terms in $\Psi_{\text{Inventory}} (\boldsymbol{z}, \boldsymbol{y})$ represent the total substitution cost, total holding cost, and total shortage cost, respectively. 

    We next focus on the feature-based inventory substitution problem with a soft CSD constraint, which is given by
    \begin{equation}\label{scs-problem}
        \inf _{f \in \mathcal{F}}\,\,  \sup_{\mathbb{P} \in \, \mathcal P(\mathcal{X} \times \mathcal{Y})} \quad \mathbb{E}_{(\boldsymbol{x}, \boldsymbol{y}) \sim \mathbb{P}}\Big[\boldsymbol{c}^{\top}f(\boldsymbol{x}) + \Psi_{\text{Inventory}}(f(\boldsymbol{x}), \boldsymbol{y})\Big]-\lambda \cdot R_{\epsilon}(\widehat{\mathbb{P}}, \mathbb{P})^p. 
    \end{equation}  
    The dual problem of the linear programming problem $\Psi_{\text{Inventory}}(f(\boldsymbol{x}), \boldsymbol{y})$ is given by 
    \begin{equation}
        \begin{aligned}
           \max_{\boldsymbol{\eta} \in \mathbb{R}^{d_z}, \boldsymbol{\upsilon} \in \mathbb{R}^{d_y}} & \quad \sum_{i=1}^{d_z} z_{i} \eta_i + \sum_{j=1}^{d_y} y_j \upsilon_j \\ 
           \text{s.t. } \quad & \quad  \eta_i \le h_i, & \forall i \in \left[d_z\right], \\
            & \quad  \upsilon_j \le b_j, & \forall j \in \left[d_y\right],  \\
            & \quad \eta_i+\upsilon_j \le s_{i,j}, & \forall  j \in \left\{i, i+1, \cdots, d_y\right\}, i \in \left[d_z\right]. 
        \end{aligned}
    \end{equation}
    where $\eta_i \in \mathbb{R}$ for each $i \in \left[d_z\right]$ represents the dual variable of the $i$-th constraint in the first constraint set of the problem~\eqref{psi-scs-problem}, while $\upsilon_j \in \mathbb{R}$ for each $j \in \left[d_y\right]$ represents the dual variable of the $j$-th constraint in the second constraint set of the problem~\eqref{psi-scs-problem}. According to the duality theorem, strong duality holds. 
    Define 
    \begin{equation}
    \begin{aligned}
    \Psi_{\text{Inventory}}^{\text{Dual}}(f(\boldsymbol{x}), \boldsymbol{y}) := \max_{\boldsymbol{\eta} \in \mathbb{R}^{d_z}, \boldsymbol{\upsilon} \in \mathbb{R}^{d_y}}&\quad 
    \sum_{i=1}^{d_z}[f(\boldsymbol{x})]_i (\eta_i + c_i)   + \sum_{j=1}^{d_y} y_j \upsilon_j \, \quad \\
    \mbox{s.t.}&\quad  \eta_i \le h_i, \forall i \in [d_z],\\
    &\quad \upsilon_j \le b_j,\forall j \in [d_y], \\
    &\quad \eta_i + \upsilon_j \le s_{i,j}, \forall  j \in \left\{i, i+1, \cdots, d_y\right\}, i \in \left[d_z\right]. 
   \end{aligned}
    \end{equation}
    Problem~\eqref{scs-problem} can be rewritten as 
    \begin{equation}\nonumber
    \inf_{f \in \mathcal{F}}\, \sup_{\mathbb{P} \in \, \Re\, ({\widehat{\mathbb{P}}})} \quad  \mathbb{E}_{(\boldsymbol{x}, \boldsymbol{y}) \sim \mathbb{P}}\Big[\Psi_{\text{Inventory}}^{\text{Dual}}(f(\boldsymbol{x}), \boldsymbol{y})\Big] -\lambda \cdot R_{\epsilon}(\widehat{\mathbb{P}}, \mathbb{P})^p,
    \end{equation}
    which is the same as~\eqref{soft-causal-sdro}. 

{
\noindent{\textbf{Verification of Assumptions. }}
The decision space of the feature-based inventory substitution problem is convex. 
As $\Psi_{\text{Inventory}}^{\text{Dual}}(\boldsymbol{z}, \boldsymbol{y})$ is piecewise affine in $\boldsymbol{z}$ and typically not differentiable, we define its regularized version:
\[
    \begin{aligned}
    \Psi_{\text{Inventory}}^*(\boldsymbol{z}, \boldsymbol{y}) := \max_{\boldsymbol{\eta} \in \mathbb{R}^{d_z}, \boldsymbol{\upsilon} \in \mathbb{R}^{d_y}} &\quad 
    \sum_{i=1}^{d_z}z_i (\eta_i + c_i)   + \sum_{j=1}^{d_y} y_j \upsilon_j - \mu\|\boldsymbol{\eta}\|^2 \, \quad \\
    \mbox{s.t.}&\quad  \eta_i \le h_i, \forall i \in [d_z],\\
    &\quad \upsilon_j \le b_j,\forall j \in [d_y], \\
    &\quad \eta_i + \upsilon_j \le s_{i,j}, \forall  j \in \left\{i, i+1, \cdots, d_y\right\}, i \in \left[d_z\right],
   \end{aligned}
\]
where $\mu > 0$ is a small constant. 
Since $(\boldsymbol{\eta}, \boldsymbol{\upsilon}) = (\boldsymbol{0},\boldsymbol{0})$ is a feasible solution of $\Psi_{\text{Inventory}}^*(\boldsymbol{z}, \boldsymbol{y})$ and the first-stage decision, demand, and all coefficients are all non-negative, the value of $\Psi_{\text{Inventory}}^*(\boldsymbol{z}, \boldsymbol{y})$ is non-negative, i.e., 
\[
    \Psi_{\text{Inventory}}^*(\boldsymbol{z}, \boldsymbol{y}) \ge \boldsymbol{c}^\top \boldsymbol{z} \ge 0. 
\]
The quadratic regularization ensures that the optimal $\boldsymbol\eta^(\boldsymbol{z},\boldsymbol{y})$ is unique and $1/2\mu$-Lipschitz continuous with respect to $\boldsymbol{z}$
To ensure the loss is upper-bounded, we use the following training loss: 
\[
    \widetilde{\Psi}_{\text{Inventory}}^*(f(\boldsymbol{x}), \boldsymbol{y}) = \mathrm{T}_{B}\Big ( \Psi_{\text{Inventory}}^*(f(\boldsymbol{x}), \boldsymbol{y}) \Big), 
\]
where $\mathrm{T}_{B}$ is defined in Eq.~\eqref{Eq-TB}. 
Thus, $\widetilde{\Psi}_{\mathrm{Inventory}}^*$ is bounded in $[0, B)$ and is differentiable, Lipschitz continuous, and Lipschitz smooth with respect to $\boldsymbol{z}$. 
Then, by Proposition~\ref{prop-srf-lip}, Assumptions~\ref{assumption-2}\ref{assumption-2-lip-con} and~\ref{assumption-2}\ref{assumption-2-bound} are satisfied. }
{We solve the regularized dual problem $\Psi_{\text{Inventory}}^*(f(\boldsymbol{x}), \boldsymbol{y})$ and compute its gradient with respect to $f(\boldsymbol{x})$ using the commercial solver COPT (version 8.0.7). }

\subsection{ Data-driven Portfolio Selection Problem } \label{subsec-intro-portfolio}

\noindent{\textbf{Model Setup. }}
Conditioned on a covariate $\boldsymbol{x} \in \mathcal{X} \subseteq \mathbb{R}^{d_x}$ (for example, macroeconomic indicators), the portfolio manager determines a portfolio allocation strategy across assets that minimizes the worst-case conditional risk-return trade-off~\citep {nguyen2025robustifying}. 
In this data-driven portfolio selection problem, the random vector $\boldsymbol{y} \in \mathcal{Y} \subseteq \mathbb{R}^{d_y}$ denotes the assets’ future returns, and risk can be characterized by variance, conditional value-at-risk, and related measures. We choose the variance of return as the measure of risk. In this example, this leads to the following model 
\begin{equation}\nonumber
    \inf _{f \in \mathcal{F}}\,\,  \sup_{\mathbb{P} \in \, \Re\, ({\widehat{\mathbb{P}}})} \quad \mathbb{E}_{(\boldsymbol{x}, \boldsymbol{y}) \sim \mathbb{P}}\Big[\Psi_{\text{Portfolio}}(f(\boldsymbol{x}), \boldsymbol{y})\Big], 
\end{equation}
where the portfolio loss function $\Psi_{\text{Portfolio}}:~\mathcal{Z}\times\mathcal{Y}\to\mathbb{R}$ is defined as
\begin{equation}\label{portfolio-loss-function}
    \Psi_{\text{Portfolio}}(\boldsymbol{z}, \boldsymbol{y}) := \left\{ - \omega \cdot  \sum_{i=1}^{d_y} y_i z_i + \left ( \sum_{i=1}^{d_y} y_i z_i - z_0 \right )^2 \, \Bigg | \, \begin{array}{lr}
        \sum_{i=1}^{d_y} z_i = 1,   \\ 
        z_i \ge 0, \quad \quad \forall i \in [d_y] 
    \end{array}\right\} , 
\end{equation}
where the parameter $\omega$ balances the trade-off between the portfolio return (the first term) and the associated risk (the second term), and for decision variables $\boldsymbol{z} = \left ( z_0, z_1, \cdots, z_{d_y} \right )^{\top} \in \mathcal{Z} \subseteq  \mathbb{R}^{d_y+1}$, $z_0$ represents the expected portfolio return while $z_i$ represents the portfolio on the asset $i$ for each $i \in [d_y]$. 

{
\noindent\textbf{Verification of Assumptions. }
The objective of $\Psi_{\text{Portfolio}}(\boldsymbol{z}, \boldsymbol{y})$ is equivalent to 
\[
    \Psi_{\text{Portfolio}}(\boldsymbol{z}, \boldsymbol{y}) = \left( \sum_{i=1}^{d_y} y_i z_i-z_0-\frac{\omega}{2}\right)^2 -\omega z_0-\frac{\omega^2}{4}. 
\]
For $\omega > 0$, we impose a fixed upper bound $z_0 \le \bar{z}_0 < \infty$, and define the lower bound of $\Psi_{\text{Portfolio}}(\boldsymbol{z}, \boldsymbol{y})$ as $\underline{B} := -\omega \bar{z}_0 -\frac{\omega^2}{4}$. 
Then, $\Psi_{\text{Portfolio}}(\boldsymbol{z}, \boldsymbol{y}) \ge \underline{B}$ for every feasible decision and return realization. 
For a fixed threshold $B > 0$, we use the following training loss in the numerical study:
\[
   \widetilde{\Psi}_{\text{Portfolio}}(\boldsymbol{z},\boldsymbol{y}) = \mathrm{T}_{B}\Big (\Psi_{\text{Portfolio}}(\boldsymbol{z},\boldsymbol{y}) - \underline{B}\Big) \in [0,B],  
\]
where $\mathrm{T}_{B}$ is defined in Eq.~\eqref{Eq-TB}. 
Thus, for any feasible decision and uncertain return, we have 
\[
    1 \le \mathbb{E}_{ \boldsymbol{\xi}_2 \sim W_{\epsilon}}\left[ \exp\left( \frac{\widetilde{\Psi}_{\text{Portfolio}}}{\lambda\epsilon} \right) \right] \le e^{B/\lambda\epsilon} < \infty, 
\]
which ensures a bounded inner expectation for experiments on both $p=1$ and $p=2$, therefore avoiding $\mathbb{E}_{ \boldsymbol{\xi}_2 \sim W_{\epsilon}}\left[ \exp\left( {\Psi_{\text{Portfolio}}}/{\lambda\epsilon} \right) \right] = \infty$ when $p=1$ for our quadratic loss. Therefore, Assumption~\ref{assumption-2}\ref{assumption-2-bound} is satisfied. 
}

{For the portfolio problem, Assumption~\ref{assumption-2}\ref{assumption-2-lip-con} may not be satisfied when the perturbed returns have unbounded support. 
To address this issue without truncating the return perturbations, we introduce the following alternative Condition~\ref{condition-portfolio}, which is easy to verify, for the convergence analysis of CM-SCSC. 
\begin{condition} \label{condition-portfolio}
    For any $\boldsymbol{x} \in\mathcal{X}$ and $\boldsymbol{y} \in \mathbb{R}^{d_y}$ and given parametric decision rule, define the loss function $\widetilde L_{\textnormal{Portfolio}}(\boldsymbol\theta;\boldsymbol x,\boldsymbol{y}) :=\widetilde\Psi_{\textnormal{Portfolio}} (f_{\boldsymbol\theta}(\boldsymbol x),\boldsymbol{y})$. Suppose that $\widetilde L_{\mathrm{Portfolio}}$ is twice continuously differentiable with respect to $\boldsymbol{\theta}$, and there exists finite constant $l_1, l_2 > 0$ such that for any $\boldsymbol{\theta} \in \Theta, \boldsymbol{x} \in\mathcal{X}$, and $\boldsymbol{y} \in \mathbb{R}^{d_y}$: 
    \begin{enumerate}
        \item $\|\nabla_\theta \widetilde L_{\textnormal{Portfolio}}(\boldsymbol{\theta}; \boldsymbol{x}, \boldsymbol{y}) \| \le l_1(1+\|\boldsymbol{y}\|^2)$ \textnormal{;} 
        \item $\|\nabla_\theta^2 \widetilde L_{\textnormal{Portfolio}}(\boldsymbol{\theta}; \boldsymbol{x}, \boldsymbol{y})\|_{\textnormal{F}} \le l_2(1+\|\boldsymbol{y}\|^4)$. 
    \end{enumerate}
\end{condition} 
To verify Condition~\ref{condition-portfolio}, note that the first- and second-order derivatives of the original quadratic portfolio loss with respect to $\boldsymbol{z}$ grow at most in $\|\boldsymbol{y}\|^2$, the first- and second-order derivatives of the SRF are uniformly bounded, and the first- and second-order derivatives of $T_B$ are bounded. Therefore, Condition~\ref{condition-portfolio} holds for the portfolio loss with both $p=1$ and $p=2$. }

{
We next briefly show that under Assumption~\ref{assumption-2}\ref{assumption-2-bound} and Condition~\ref{condition-portfolio}, the convergence of CM-SCSC (Theorem~\ref{theo-scsc-conver}) also holds. 
Since both the Laplace and Gaussian distributions have finite fourth-order moments, there exist finite constants $l_1^\prime, l_2^\prime > 0$ such that for any $\boldsymbol{\theta}$ 
\[
    \mathbb{E}_{\boldsymbol{\xi}_2} \|\nabla_\theta \widetilde L_{\textnormal{Portfolio}} \|^2 \le l_1^\prime < \infty, \quad \mathbb{E}_{\boldsymbol{\xi}_2}  \|\nabla_\theta^2 \widetilde L_{\textnormal{Portfolio}}\|_{\textnormal{F}}  \le l_2^\prime < \infty. 
\]
Define the current function $t_3: \Theta \to \mathbb{R}^N$ in Eq.~\eqref{t-functions} as 
\[
    \Big [ t_{3, \textnormal{Portfolio}} (\boldsymbol{\theta}; \widehat{\boldsymbol{x}}, \boldsymbol{\xi}_1, \boldsymbol{\xi}_2)\Big]_j = \exp\left( \frac{\widetilde L_{\textnormal{Portfolio}} (\boldsymbol{\theta};\widehat{\boldsymbol{x}} + \boldsymbol{\xi}_1,\widehat{\boldsymbol{y}}_j + \boldsymbol{\xi}_2)}{\lambda\epsilon} \right), \quad \forall j \in [N]. 
\]
Then, the following moment bounds required for the convergence analysis hold, i.e., 
\[
    \mathbb{E}_{\boldsymbol{\xi}_2} \|t_{3, \textnormal{Portfolio}}(\boldsymbol{\theta})-t_{3, \textnormal{Portfolio}}(\boldsymbol{\theta}^\prime)\|^2 \le C_3^2\|\boldsymbol{\theta}-\boldsymbol{\theta}^\prime\|^2,
\] 
and 
\[
    \mathbb{E}_{\boldsymbol{\xi}_2} \| J_{\boldsymbol{\theta}}t_{3, \textnormal{Portfolio}}(\boldsymbol{\theta}) - J_{\boldsymbol{\theta}}t_{3, \textnormal{Portfolio}}(\boldsymbol{\theta}^\prime) \|_{\text{F}} \le S_3\|\boldsymbol{\theta}-\boldsymbol{\theta}^\prime\|. 
\]
Here, we can prove that 
\[
    C_3^2 = \frac{N l_1^\prime}{\lambda^2 \epsilon^2} e^{\frac{2B}{\lambda \epsilon}}, \quad S_3 = N e^{\frac{B}{\lambda \epsilon}} \Big ( \frac{l_2^\prime}{\lambda \epsilon} + \frac{l_1^\prime}{\lambda^2 \epsilon^2} \Big). 
\]
Hence, under Assumption~\ref{assumption-2}\ref{assumption-2-bound}, Condition~\ref{condition-portfolio} ensures that the moment bounds required for the convergence analysis of CM-SCSC and the conclusions of Lemmas~\ref{lemma-smoothness-F} and~\ref{lemma-batch-bias-bound} hold. The convergence result of CM-SCSC therefore remains valid under the remaining conditions of Theorem~\ref{theo-scsc-conver}.}

\section{Detailed Numerical Results}\label{app-sec-results}

As a supplement to Section~\ref{sec-results}, in this section, we first provide details of our experimental settings and present detailed numerical results for three feature-based applications: the newsvendor problem in Appendix~\ref{appsubsec-results-news}, the inventory substitution problem in Appendix~\ref{subsec-results-inventory}, and the portfolio selection problem with real data in Appendix~\ref{subsec-results-portfolio}. 
Additionally, we demonstrate the interpretability of the SRF decision rule for the portfolio problem in Appendix~\ref{subsec-results-portfolio-interpret}. 

The experiments are coded in Python 3.8 and conducted on a personal computer equipped with an Intel Core i9-13900HX CPU, 32 GB of RAM, and an Nvidia GeForce RTX 4060 GPU. All GPU computations are performed using PyTorch 2.0.1 (utilizing CUDA 11.8).
We take $T=20, D(t) = \left \lceil \log_{2} \, d_x \right \rceil + 1$ for each $t\in \left[ T \right]$ for the SRF decision rule. 

For benchmark comparison, we also examine the performance of a two-layer neural network (2NN) decision rule for the feature-based newsvendor, which is a learning-based parametric decision rule but lacks interpretability. 
Let $m$ be the dimension of the hidden layer in the 2NN. Then, the 2NN decision rule $f_{\boldsymbol{\theta}}^{\text{2NN}}: \mathcal{X} \to \mathcal{Z}$, parametrized by $\boldsymbol{\theta} \in \Theta$ that collects all parameters $\Big\{\boldsymbol{a}^k \in \mathbb{R}^m, \boldsymbol{b}^k \in \mathbb{R}^m, \boldsymbol{w}_i^k \in \mathbb{R}^{d_{x}} \Big\}$ for all $i \in \left [m\right]$ and $k \in \left [d_z\right]$, is given by 
\begin{equation}
\Big[ f_{\boldsymbol{\theta}}^{\text{2NN}}(\boldsymbol{x})\Big]_{k} := \frac{1}{m} \sum_{i=1}^{m} a_i^k \cdot \text{ReLU} \Big ( (\boldsymbol{w}_i^k)^{\top}\boldsymbol{x} +b_i^k\Big ), \quad \forall k \in \left [d_z\right], \tag{\text{2NN}}
\end{equation}
where $\text{ReLU}: \mathbb{R} \to \mathbb{R}_+$ represents the ReLU activation function, while $a_i^k$ and $b_i^k$ represent the $i$-th element in vectors $\boldsymbol{a}^k$ and $\boldsymbol{b}^k$, respectively, for any $i \in \left[m\right]$. This decision rule requires training a total of $m\cdot d_z(d_x+2)$ parameters. 
According to~\citet{ma2018priori}, over-parameterized 2NNs can effectively approximate optimal policies within the Barron space with dimension-independent convergence rates.
We take $m=64\times d_x$ for the 2NN decision rule to ensure its number of trainable parameters is comparable to that of SRF. 
To ensure a fair comparison, we select the ERM solution with the lower validation loss between the SRF and 2NN decision rules as the common ERM benchmark.

\subsection{Feature-based Newsvendor Problem} \label{appsubsec-results-news}

In this subsection, we provide more details and results for the feature-based newsvendor problem, which adopts an experimental setup similar to that of~\citet{yang2026decision}, where the demand $y \in \mathbb{R}_+$ depends on the covariate $\boldsymbol{x}$ in a nonlinear way: 
\begin{equation} \label{data-newsvendor} \nonumber
    y = f_{\text{true}}(\boldsymbol{\beta}^\top \boldsymbol{x}) + \varsigma,  \quad \text{where } f_{\text{true}}(\lambda) = c\Big[\sin(2\lambda) + 2\exp(-16\lambda^2) + 1\Big]. 
\end{equation}
Here, $\varsigma \sim \mathcal{N}(0, 1)$ represents an independent Gaussian noise, the coefficient vector $\boldsymbol{\beta} \in \mathbb{R}^{d_x}$ is generated by taking each component sampled from the uniform distribution $\mathcal{U}([-0.1, 0.1])$, the covariate  $\boldsymbol{x} \in \mathbb{R}^{d_x}$ is generated from a multivariate normal distribution with zero mean and covariance matrix $\Sigma$ with $\Sigma_{ij} = 0.5^{|i-j|}$ for each $i,j\in[d_x]$, the constant $c=1.7$.
To ensure non-negativity, the demand $y$ is simulated using the acceptance-rejection method. 
{We conduct experiments across observed historical sample size $N \in \left\{ 20, 40, 80, 100 \right\}$, feature dimension $d_x \in \left\{3, 5, 10\right\}$, and the number of uncertain parameters associated with each covariate $n_m \in \{1, 5, 10\}$.} The unit holding and stock-out costs are set to $h = 0.6$ and $b = 1.0$, respectively. 

\begin{figure}[!htb]
  \centering
    \begin{subfigure}{.48\linewidth}
    \centering 
    \includegraphics[width=\linewidth]{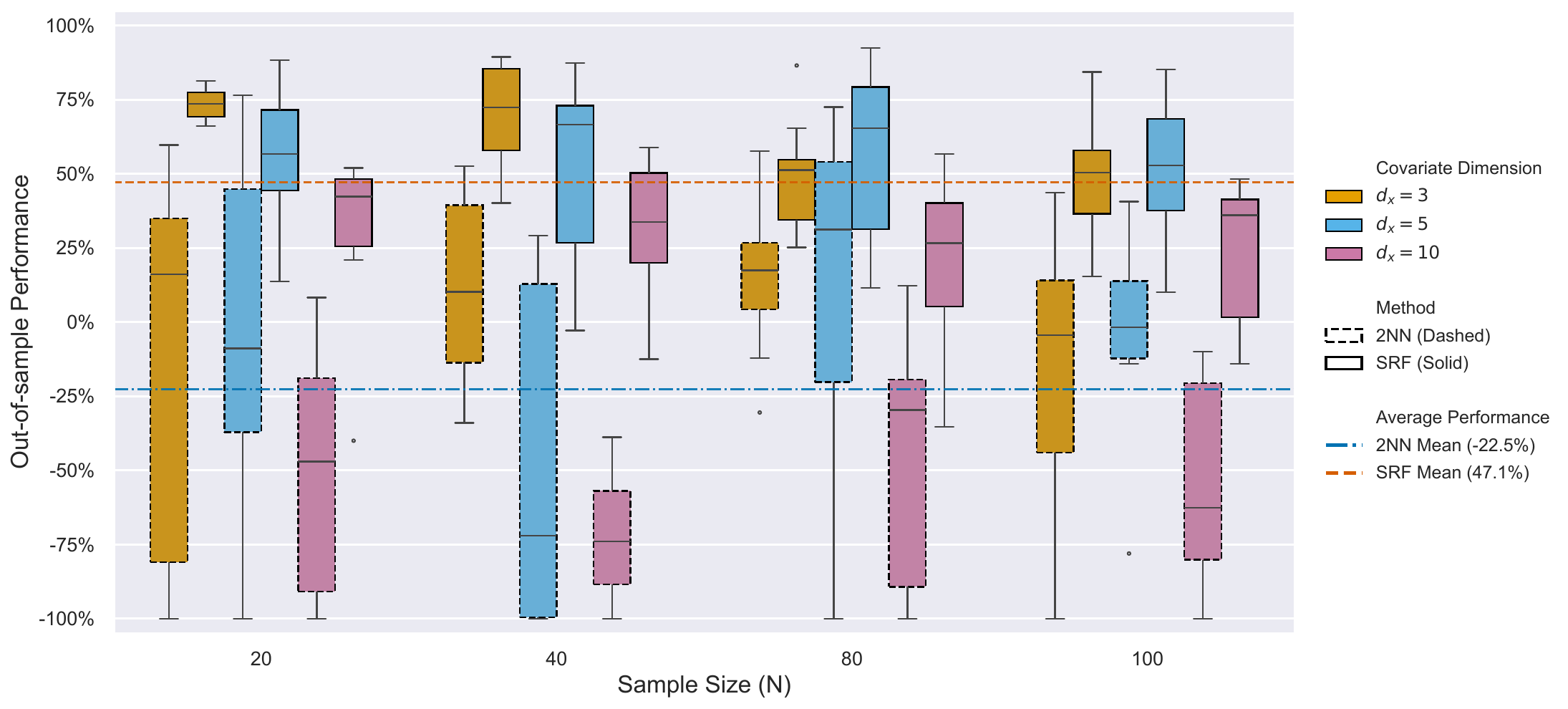}
    \caption{1-Causal-SDRO, $n_m=1$}
    \label{fig:news-box-1-1}
  \end{subfigure}
  \hfill
    \begin{subfigure}{.48\linewidth}
    \centering
    \includegraphics[width=\linewidth]{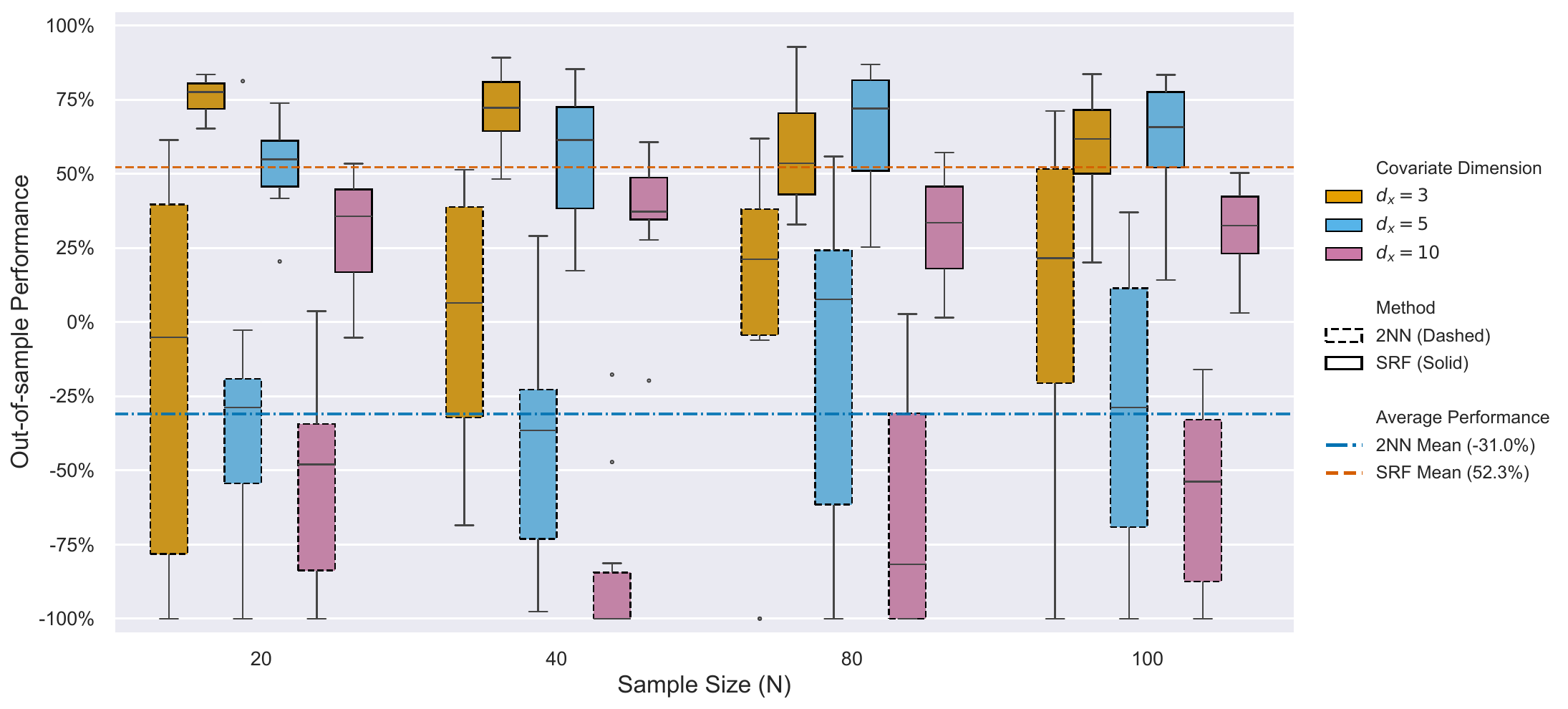}
    \caption{2-Causal-SDRO, $n_m=1$}
    \label{fig:news-box-2-1}
  \end{subfigure}
   \par\vspace{0.5em}
       \begin{subfigure}{.48\linewidth}
    \centering 
    \includegraphics[width=\linewidth]{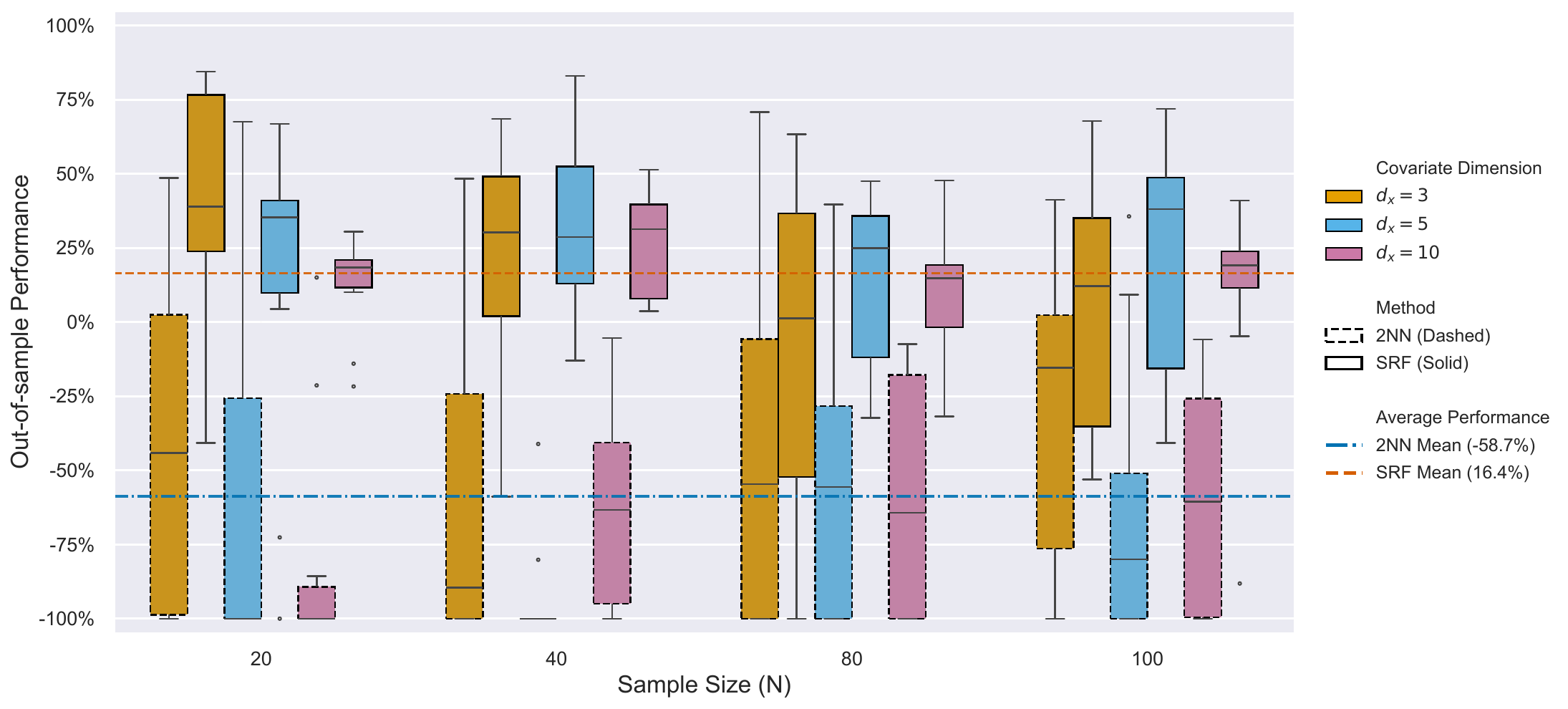}
    \caption{1-Causal-SDRO, $n_m=5$}
    \label{fig:news-box-1-5}
  \end{subfigure}
  \hfill
    \begin{subfigure}{.48\linewidth}
    \centering
    \includegraphics[width=\linewidth]{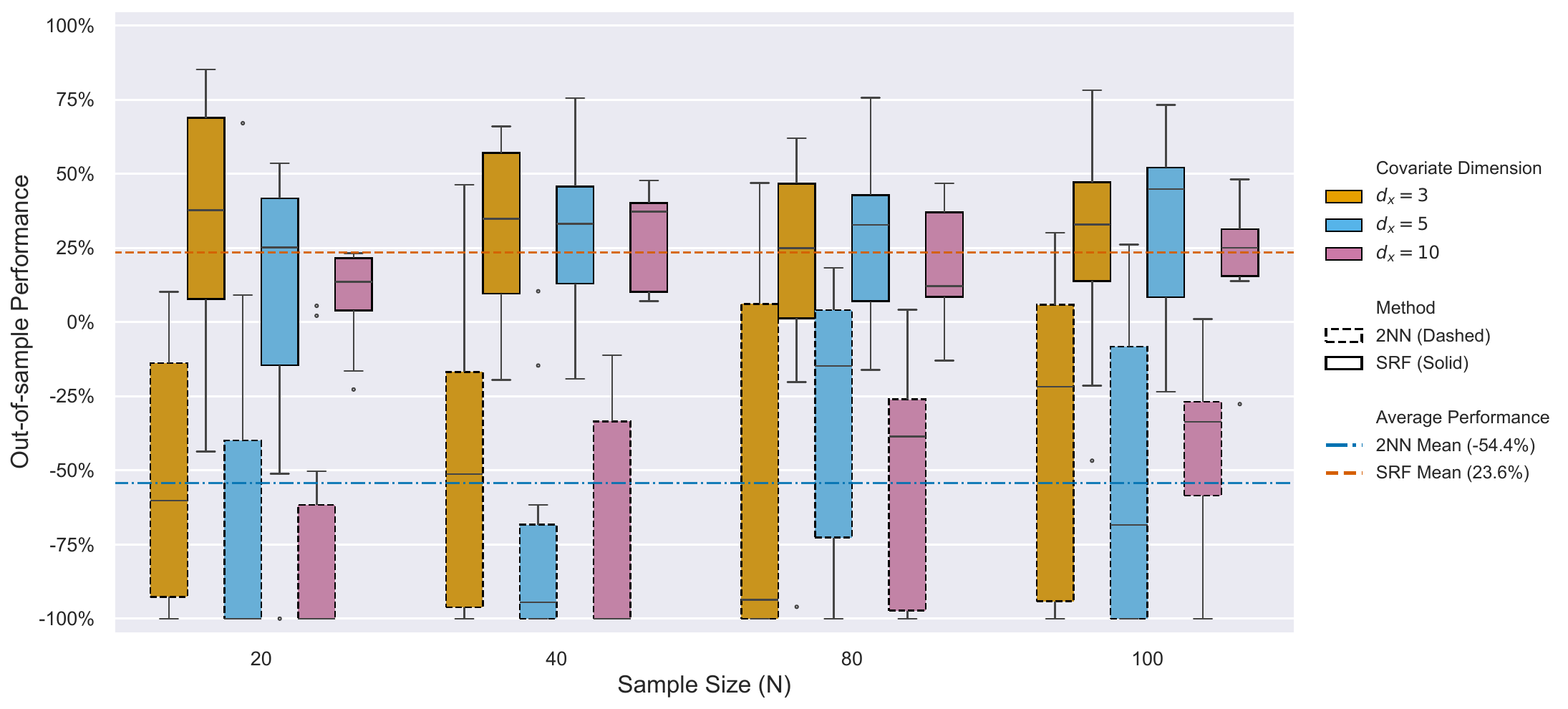}
    \caption{2-Causal-SDRO, $n_m=5$}
    \label{fig:news-box-2-5}
  \end{subfigure}
     \par\vspace{0.5em}
       \begin{subfigure}{.48\linewidth}
    \centering 
    \includegraphics[width=\linewidth]{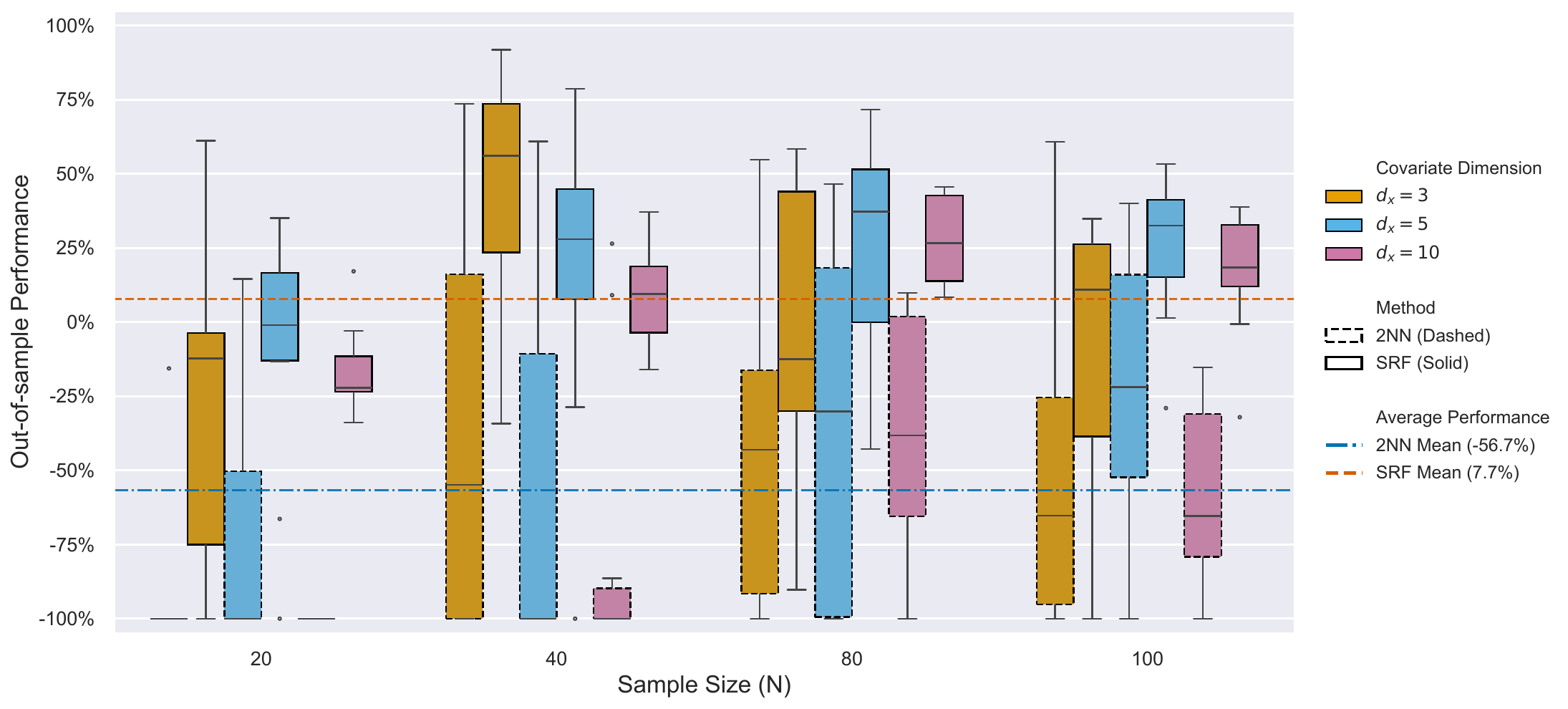}
    \caption{1-Causal-SDRO, $n_m=10$}
    \label{fig:news-box-1-10}
  \end{subfigure}
  \hfill
    \begin{subfigure}{.48\linewidth}
    \centering
    \includegraphics[width=\linewidth]{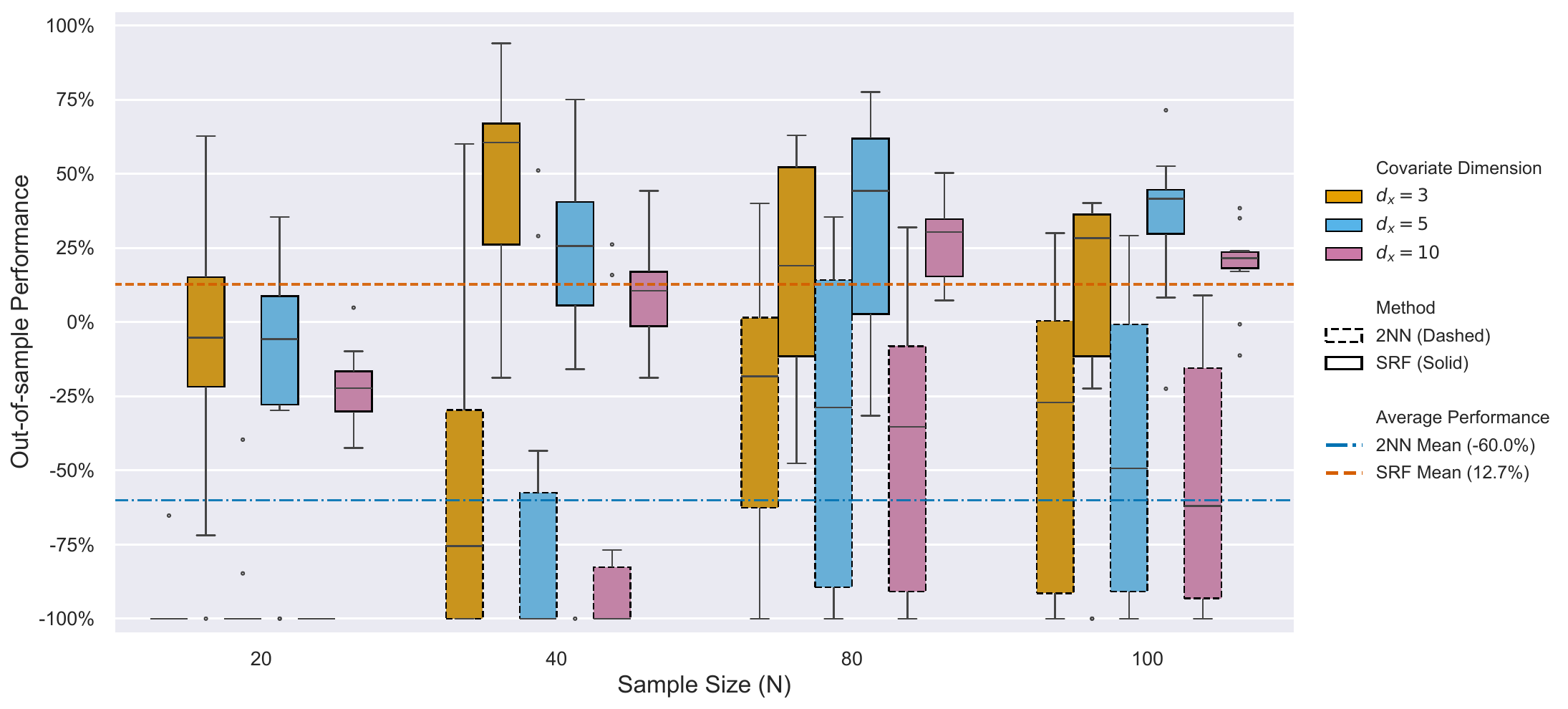}
    \caption{2-Causal-SDRO, $n_m=10$}
    \label{fig:news-box-2-10}
  \end{subfigure}
  \caption{{Out-of-sample performance of the decision rules on the newsvendor problem}}
  \label{fig:news-box}
\end{figure}

{Figure~\ref{fig:news-box} compares the out-of-sample performance of SRF and 2NN decision rules across different sample sizes and feature dimensions.  
In these plots, the results of SRF are distinguished by boxes with solid borders.
As illustrated, the proposed SRF decision rule outperforms the 2NN benchmark across all tested configurations and achieves positive average out-of-sample performance. 
These results show that the proposed SRF rule is effective for decision-making tasks, even when historical data is limited. }

\begin{figure}[!htb]
  \centering
    \begin{subfigure}{.49\linewidth}
    \centering 
    \includegraphics[width=\linewidth]
      {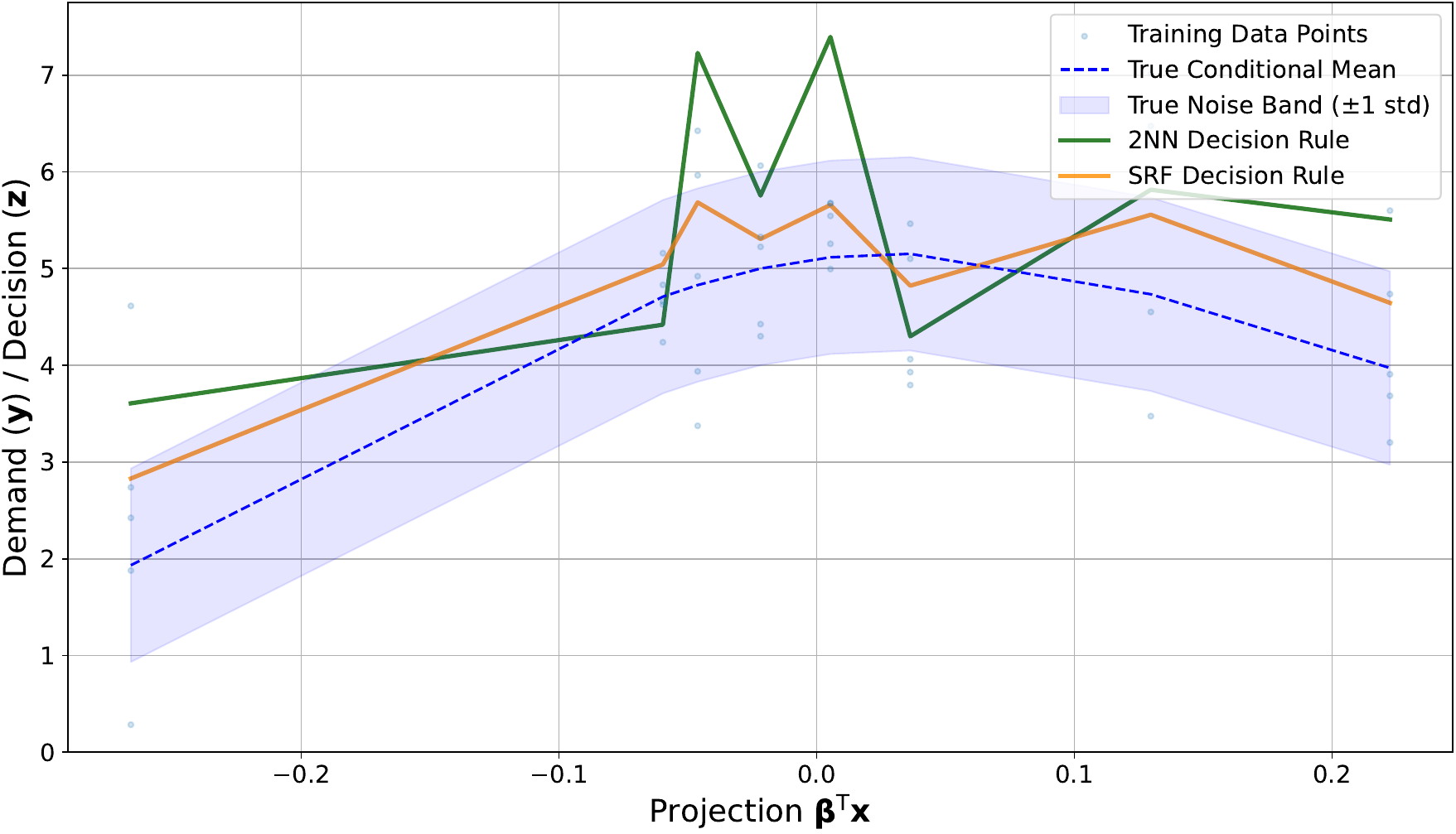}
    \caption{$N=40$}
    \label{fig:news-disfit-400}
  \end{subfigure}
  \hfill
    \begin{subfigure}{.49\linewidth}
    \centering
    \includegraphics[width=\linewidth]
      {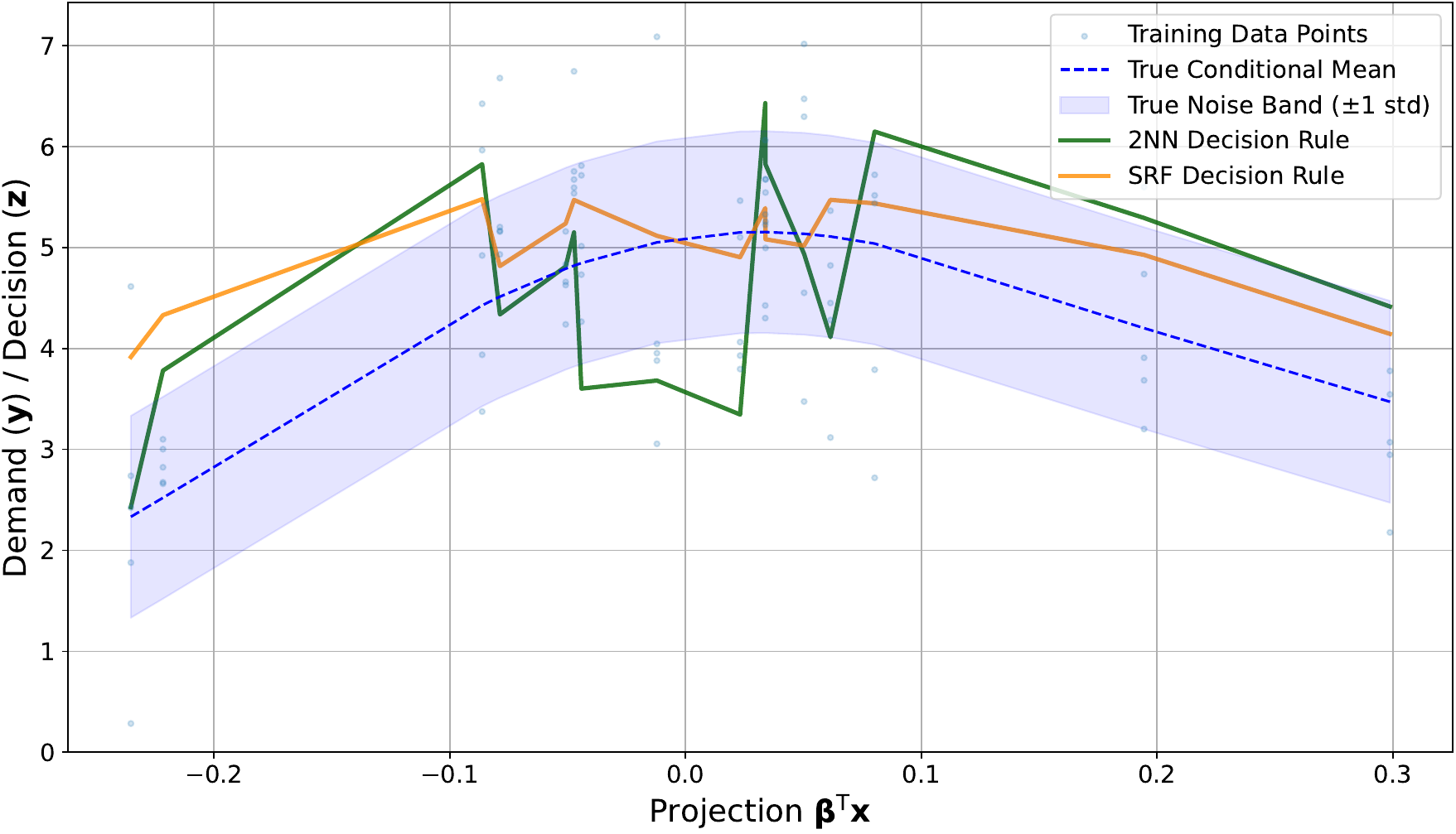}
    \caption{$N=80$}
    \label{fig:news-disfit-1000}
  \end{subfigure}
  \caption{{True distribution vs. Trained decision rules for 2-Causal-SDRO ($d_x=10, n_m=5, \lambda=2, \epsilon=0.8$)}}
  \label{fig:news-disfit}
\end{figure}

{Figure~\ref{fig:news-disfit} visualizes the fitted 2NN and SRF decision rules against the true conditional mean (blue dashed line) for sample sizes $N\in\{40, 80\}$. The 2NN decision rule (green line) exhibits high variance, resulting in overfitting to the observed noise. In contrast, the SRF decision rule (orange line) yields a stable fit that captures the underlying shape of the true function well. 
Note that the SRF curve lies consistently above the conditional mean. 
This alignment correctly reflects that the unit holding cost is lower than the unit stock-out cost ($h < b$), thereby decision-makers prefer maintaining higher inventory levels to mitigate stock-out risks. }

\begin{figure}[!htb]
  \centering
  \begin{subfigure}{.49\linewidth}
    \centering
    \includegraphics[width=\linewidth]
      {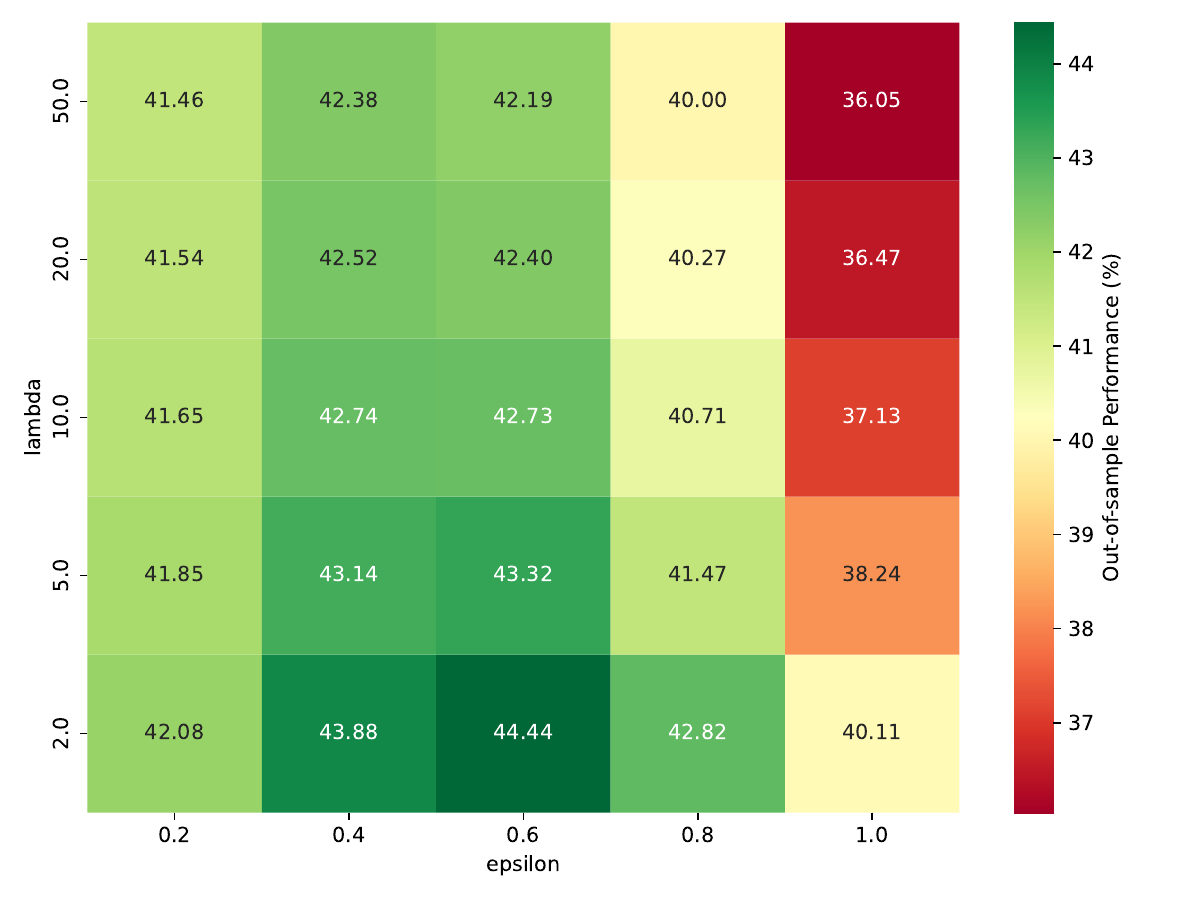}
    \caption{1-Causal-SDRO}
    \label{fig:news-cv-plot-p1}
  \end{subfigure}
  \hfill
    \begin{subfigure}{.49\linewidth}
    \centering
    \includegraphics[width=\linewidth]
      {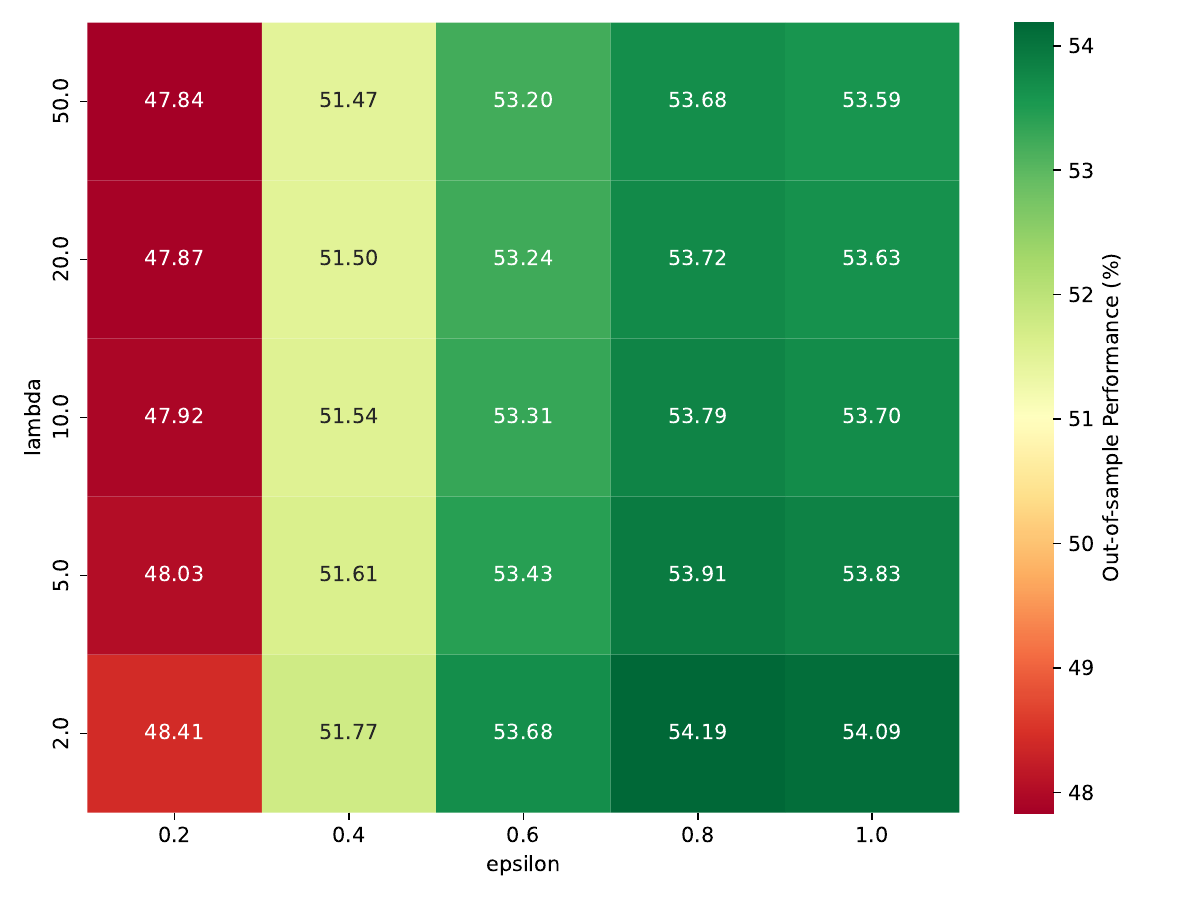}
    \caption{2-Causal-SDRO}
    \label{fig:news-cv-plot-p2}
  \end{subfigure}
  \caption{{Out-of-sample performance of the newsvendor problem with different parameters ($N = 40, d_x = 10, n_m=5$)}}
  \label{fig:news-cv}
\end{figure}

{Figure~\ref{fig:news-cv} reports the out-of-sample performance of the proposed method for the Causal-SDRO model across different parameter combinations, including the penalty parameter $\lambda$, the regularization parameter $\epsilon$, and the norm $p$, using instances with $N = 40, d_x = 10$, and $n_m=5$. 
As shown in these plots, although both models achieve positive out-of-sample performance on all instances, the 2-Causal-SDRO exhibits higher out-of-sample performance across most parameter combinations. 
These results indicate that moderate values of both $\lambda$ and $\epsilon$ tend to yield better performance. 
This is because a small $\lambda$ leads to excessive conservatism, while a large $\lambda$ reduces the model to ERM. Similarly, an insufficient $\epsilon$ fails to adequately characterize the continuity of the underlying distribution, while an excessive $\epsilon$ dilutes the correlation between covariates and uncertain parameters. }

\subsection{ Feature-based Inventory Substitution Problem } \label{subsec-results-inventory}

In this subsection, we conduct a numerical study on the feature-based inventory substitution problem. This application serves as a representative two-stage contextual DRO problem. 

{We examine a scenario with $d_z = d_y = 3$ products, varying feature dimensions $d_x \in \{3, 5, 10\}$, sample sizes $N \in \{20, 40, 80, 100\}$, and $n_m \in \{1, 5, 10\}$. }
Let the conditional demand distributions of the products be exponential and Gamma distributions parametrized by covariates:
\begin{equation}\nonumber 
    \begin{aligned}
        y_1 \mid \boldsymbol{x} \sim \text{Exp} ( e^{\boldsymbol{\beta}^{\top}\boldsymbol{x}}), \quad  y_2 \mid \boldsymbol{x} \sim \text{Gamma}(2,e^{\boldsymbol{\beta}^{\top}\boldsymbol{x}}), \quad y_3 \mid \boldsymbol{x} \sim \text{Gamma}(4, e^{\boldsymbol{\beta}^{\top}\boldsymbol{x}}), 
    \end{aligned}
\end{equation}
where $\boldsymbol{\beta} \in \mathbb{R}^{d_x}$ is sampled from $\mathcal{U}([-0.1, 0.1])$, and $\boldsymbol{x} \in \mathbb{R}^{d_x}$ is constructed by the procedure in Appendix~\ref{appsubsec-results-news}. 
We specify the cost parameters as 
\begin{equation}\nonumber
    \boldsymbol{h} = \begin{bmatrix}
 1 \\
 0.7 \\
 0.6
\end{bmatrix}, \quad 
\boldsymbol{b} = \begin{bmatrix}
 1.8 \\
 1.6 \\
 1.2
\end{bmatrix}, \quad 
\boldsymbol{S} = 
\begin{pmatrix}
 0 & 1.7 & 2 \\
 \infty & 0 & 1.5 \\
 \infty & \infty & 0
\end{pmatrix},
\end{equation}
and $\boldsymbol{c} = \boldsymbol{0}$. Note that in the substitution cost matrix $\boldsymbol{S}$, $S_{i,j} = \infty$ when $i > j$, since lower-quality products cannot substitute for higher-quality ones. 

\begin{figure}[!htb]
  \centering
  \begin{subfigure}{.49\linewidth}
    \centering
    \includegraphics[width=\linewidth]{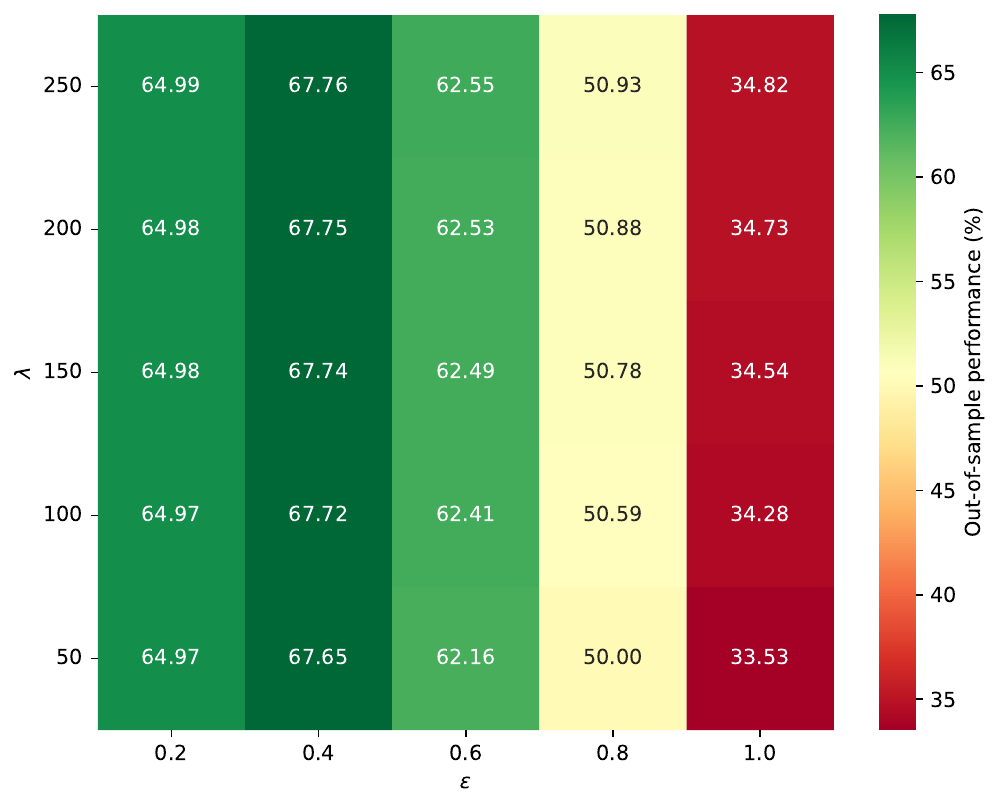}
    \caption{1-Causal-SDRO}
    \label{fig:is-cv-plot-p1}
  \end{subfigure}
  \hfill
    \begin{subfigure}{.49\linewidth}
    \centering
    \includegraphics[width=\linewidth]{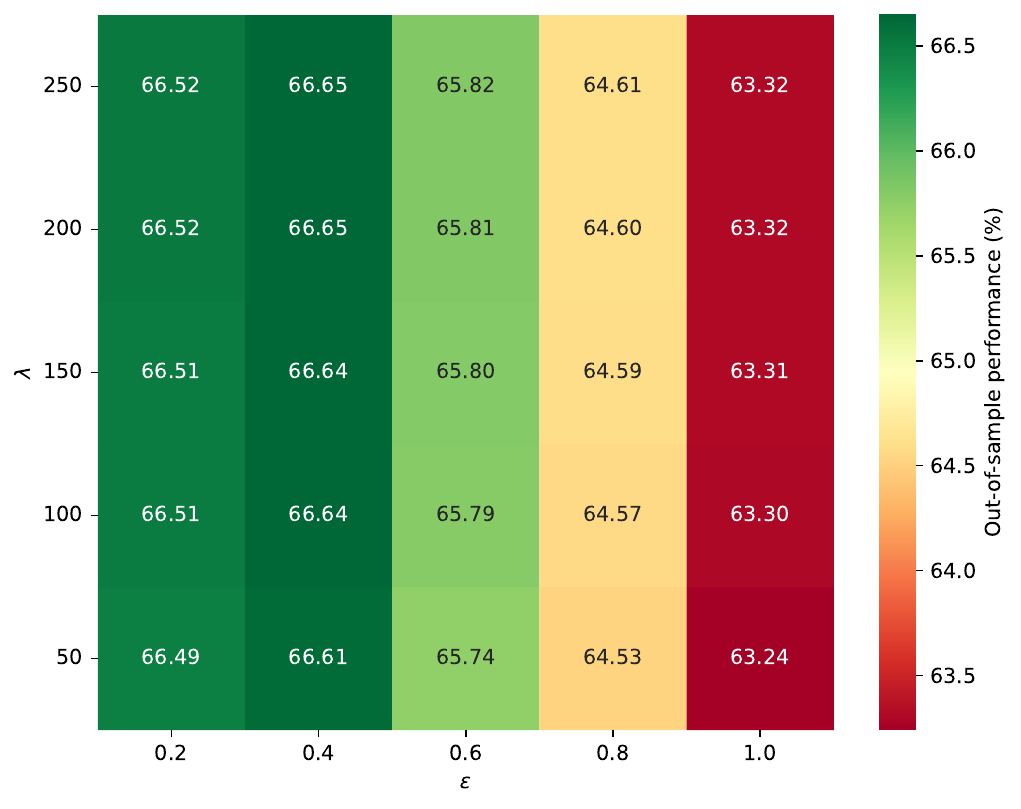}
    \caption{2-Causal-SDRO}
    \label{fig:is-cv-plot-p2}
  \end{subfigure}
  \caption{{Out-of-sample Performance of the inventory substitution problem with different parameters ($N = 200, d_x = 10$)}}
  \label{fig:is-cv}
\end{figure}

\begin{figure}[!htb]
    \centering
    \includegraphics[width=\linewidth]{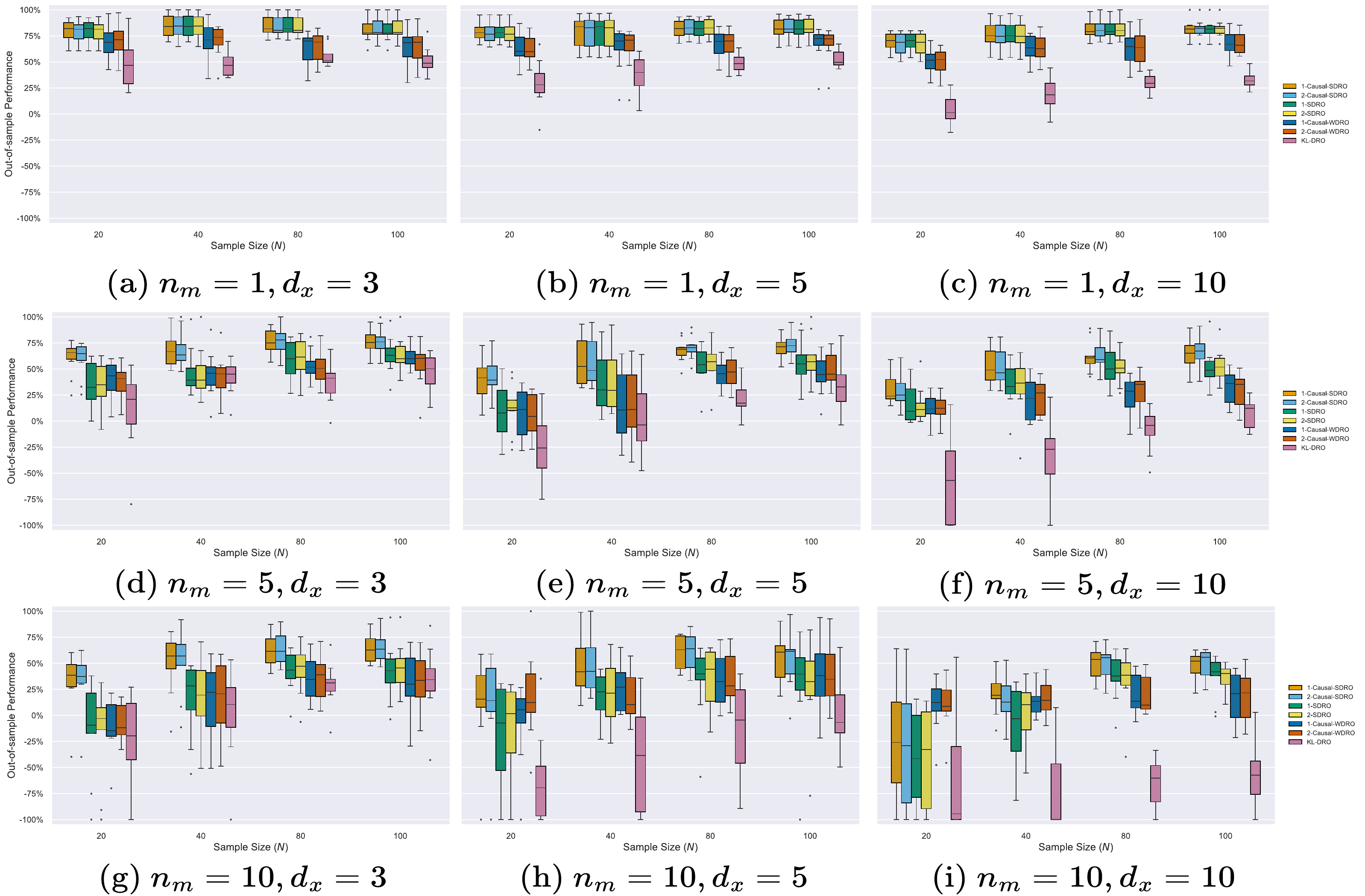}
  \caption{{Comparison of different DRO models on the inventory substitution problem }}
  \label{fig:inv-dro-box} 
\end{figure} 

{Figure~\ref{fig:is-cv} illustrates the out-of-sample performance of the proposed Causal-SDRO method across different combinations of regularization hyperparameters, meaning that Causal-SDRO models perform well across a relatively wide range of hyperparameter choices. 
Figure~\ref{fig:inv-dro-box} compares Causal-SDRO against other DRO benchmarks, similar to those obtained in experiments on the feature-based newsvendor problem. 
These results demonstrate the applicability of the proposed approach to two-stage contextual DRO problems. } 

\subsection{ Data-driven Portfolio Selection Problem } \label{subsec-results-portfolio}

This section applies the proposed method to the data-driven portfolio selection problem. 
Appendix~\ref{subsec-results-portfolio-performance} presents the performance of our approach on the portfolio problem, and Appendix~\ref{subsec-results-portfolio-interpret} empirically demonstrates its intrinsic interpretability. 

\subsubsection{Performance of the Proposed Approach} \label{subsec-results-portfolio-performance}

In this subsection, we consider a data-driven portfolio selection problem using real-world market data adapted from \citet{nguyen2025robustifying}. This dataset contains the historical asset returns of S\&P500 constituents from January 1, 2017, to March 31, 2023. All 399 selected assets have been in the S\&P 500 since 2010. The features include five publicly available market indices: (i) Volatility Index (VIX), (ii) 10-year Treasury Yield Index (TNX), (iii) Crude Oil Index (CL=F), (iv) S\&P 500 (GSPC), and (v) Dow Jones Index (DJI) to construct the covariate $\boldsymbol{x} \in \mathbb{R}^{5}$. 
These features capture the macro market environment and economic conditions, and it is reasonable to assume they are exogenous to the historical returns of individual assets. 

We implement a rolling-horizon experiment to validate the performance of the proposed approach. For the first trade day of each month between January 2021 and December 2022, we randomly sample $d_y = 50$ assets from the universe to form the stock pool. 
We make portfolio decisions based on an empirical distribution formed by the prior two-year window data on the covariates (i.e., 730 historical samples) and asset returns, targeting the best return over the subsequent 60-day holding period. 
{We employ a greedy k-nearest-neighbor (k-NN) approach to group similar covariates into $M$ groups, each containing $n_m \in \{1, 10\}$ observations, i.e., $M \in \{730, 73\}$, respectively. 
Specifically, starting from the first unassigned observation, we form a group by selecting its $n_m-1$ nearest unassigned observations, measured by Euclidean distance. 
We then repeat this procedure until all samples are grouped. 
For each group, the representative covariate is defined as the average of the group's covariates. }

Let $r_{i,j}$ denote the return of asset $i$ on day $j$ within the testing horizon ($j \in [60]$). We compare the following methods: 
\begin{enumerate}
    \item The post hoc testing (PT) model. This benchmark provides a theoretically optimal objective value (that is, $\mathcal{H}^*$) when the information of the future 60 days is completely known: 
    \begin{equation}\nonumber
        \min_{\boldsymbol{z} \in \mathcal{Z}}\quad \frac{1}{60}\sum_{j=1}^{60} \Big[-\omega \cdot  \sum_{i=1}^{d_y} r_{i,j} z_i + \left ( \sum_{i=1}^{d_y} r_{i,j} z_i - z_0 \right )^2 \Big]. 
    \end{equation}
    \item The equal-weighted (EW) model. This model assigns equal weights to each selected asset; that is, the investment amount for each asset is $1/d_y$. 
    \item The unconditional mean-variance (MV) model. This model is a traditional portfolio model that makes a decision based on the empirical distribution $\widehat{\mathbb{P}}_{\widehat{\boldsymbol{Y}}}$:
    \begin{equation}\nonumber
        \min_{\boldsymbol{z} \in \mathcal{Z}}\quad \mathbb{E}_{\boldsymbol{y} \sim \widehat{\mathbb{P}}_{\widehat{\boldsymbol{Y}}}}\Big[ \Psi_{\text{Portfolio}}(\boldsymbol{z}, \boldsymbol{y}) \Big]. 
    \end{equation}
    \item The conditional mean-variance (CMV) model. This is a contextual stochastic optimization (CSO) approach that trains a decision rule $f \in \mathcal{F}$ to minimize the empirical conditional risk: 
    \begin{equation}\nonumber
        \inf_{f \in \mathcal{F}} \quad \mathbb{E}_{(\boldsymbol{x}, \boldsymbol{y}) \sim \widehat{\mathbb{P}}}\Big[\Psi_{\text{Portfolio}}(f(\boldsymbol{x}), \boldsymbol{y})\Big]. 
    \end{equation}
    \item Our conditional $p$-Causal-SDRO mean-variance ($p$-Causal-SDRO) model. 
\end{enumerate} 

Let $\tilde{r}_t$ denote the realized daily return of the portfolio on day $t$ within the 60-day holding period ($t \in [60]$). We compute and report the following performance metrics:
(1) The mean return $\text{Mean}(\{\tilde{r}_t\}_{t \in [60]})$, and the standard deviation of return $\text{stdDev}(\{\tilde{r}_t\}_{t\in [60]})$. (2) The portfolio loss value. (3) The annualized Sharpe ratio $\sqrt{252} \times \text{Mean}(\{\tilde{r}_t\}_{t\in [60]}) / \text{stdDev}(\{\tilde{r}_t\}_{t\in [60]})$. (4) The Conditional Value-at-Risk (CVaR) at the $5\%$ level, which quantifies the expected loss in the worst-case scenarios. (5) The out-of-sample performance for all decision rules following Eq.~\eqref{out-of-sample-performance-value}.  

\begin{table}[!htb]
  \centering
  \caption{{Average performance for mean-variance methods}}
    \begin{tabular}{c|cccccc}
    \toprule
    \textbf{$\omega$} & \textbf{Methods} & \textbf{Average Loss}  $\downarrow$ & \textbf{Sharpe} $\uparrow$ & \textbf{Mean} $\uparrow$ & \textbf{stdDev} $\downarrow$ & \textbf{CVaR} $\downarrow$ \\
    \midrule
    \multirow{6}[2]{*}{\textbf{0.25}} & PT    & 0.579 & 2.861 & 0.113 & 0.756 & 1.452 \\
          & EW    & 1.476 & 1.163 & 0.068 & 1.195 & 2.450 \\
          & MV    & 1.048 & \textbf{1.601} & 0.074 & 1.000 & 1.955 \\
          & CMV   & 0.926 & 1.457 & 0.070 & 0.948 & 1.958 \\
          & 1-Causal-SDRO & \textbf{0.906} & 1.558 & \textbf{0.076} & \textbf{0.940} & \textbf{1.944} \\
          & 2-Causal-SDRO & 0.918 & 1.463 & 0.071 & 0.944 & 1.944 \\
    \midrule
    \multirow{6}[2]{*}{\textbf{0.5}} & PT    & 0.547 & 3.403 & 0.138 & 0.763 & 1.438 \\
          & EW    & 1.459 & 1.163 & 0.068 & 1.195 & 2.450 \\
          & MV    & 1.035 & \textbf{1.597} & 0.074 & 1.003 & 1.961 \\
          & CMV   & 0.911 & 1.448 & 0.070 & 0.949 & 1.960 \\
          & 1-Causal-SDRO & \textbf{0.885} & 1.548 & \textbf{0.075} & \textbf{0.939} & \textbf{1.940} \\
          & 2-Causal-SDRO & 0.901 & 1.468 & 0.071 & 0.945 & 1.944 \\
    \midrule
    \multirow{6}[2]{*}{\textbf{0.75}} & PT    & 0.510 & 3.831 & 0.160 & 0.773 & 1.438 \\
          & EW    & 1.442 & 1.163 & 0.068 & 1.195 & 2.450 \\
          & MV    & 1.023 & \textbf{1.592} & 0.073 & 1.006 & 1.967 \\
          & CMV   & 0.897 & 1.444 & 0.069 & 0.951 & 1.965 \\
          & 1-Causal-SDRO & \textbf{0.868} & 1.530 & \textbf{0.074} & \textbf{0.939} & \textbf{1.940} \\
          & 2-Causal-SDRO & 0.885 & 1.451 & 0.070 & 0.945 & 1.948 \\
    \midrule
    \multirow{6}[2]{*}{\textbf{1}} & PT    & 0.468 & 4.144 & 0.178 & 0.785 & 1.446 \\
          & EW    & 1.425 & 1.163 & 0.068 & 1.195 & 2.450 \\
          & MV    & 1.011 & \textbf{1.585} & \textbf{0.073} & 1.009 & 1.973 \\
          & CMV   & 0.884 & 1.429 & 0.069 & 0.953 & 1.969 \\
          & 1-Causal-SDRO & \textbf{0.848} & 1.517 & \textbf{0.073} & \textbf{0.938} & \textbf{1.940} \\
          & 2-Causal-SDRO & 0.872 & 1.445 & 0.069 & 0.947 & 1.949 \\
    \bottomrule
    \end{tabular}%
  \label{tab:portfolio-results}
    \begin{minipage}{\textwidth}\centering
    \footnotesize 
    \vspace{0.4em}
    \textit{Note}. Bold values indicate the best performance for each value of $\omega$, except for the PT model.
  \end{minipage}
\end{table}

{We employ the proposed SRF as the parametric decision rule for both the CMV and Causal-SDRO models. Table~\ref{tab:portfolio-results} presents experimental results for risk aversion levels $\omega \in \{0.25, 0.5, 0.75, 1\}$, corresponding to different trade-offs between portfolio mean return and variance. 
As illustrated, the Causal-SDRO models achieve better average loss, mean, standard deviation, and CVaR value across all values of $\omega$. 
In practice, although decision-makers cannot observe all covariates, our results demonstrate that our approach maintains robust performance even with imperfect observed features. } 

{The parameter sensitivity heatmaps for the data-driven portfolio selection problem are shown in Figure~\ref{fig:port-cv}. 
As illustrated, Causal-SDRO models demonstrate robust performance across a relatively wide range of hyperparameter choices. 
In Figure~\ref{fig:portfolio-dro-box}, the Causal-SDRO methods achieve better average out-of-sample performance than existing DRO baselines across nearly all configurations. 
These experimental results demonstrate the superior performance of our method on real data. }

\begin{figure}[!htb]
  \centering
  \begin{subfigure}{.49\linewidth}
    \centering
    \includegraphics[width=\linewidth]{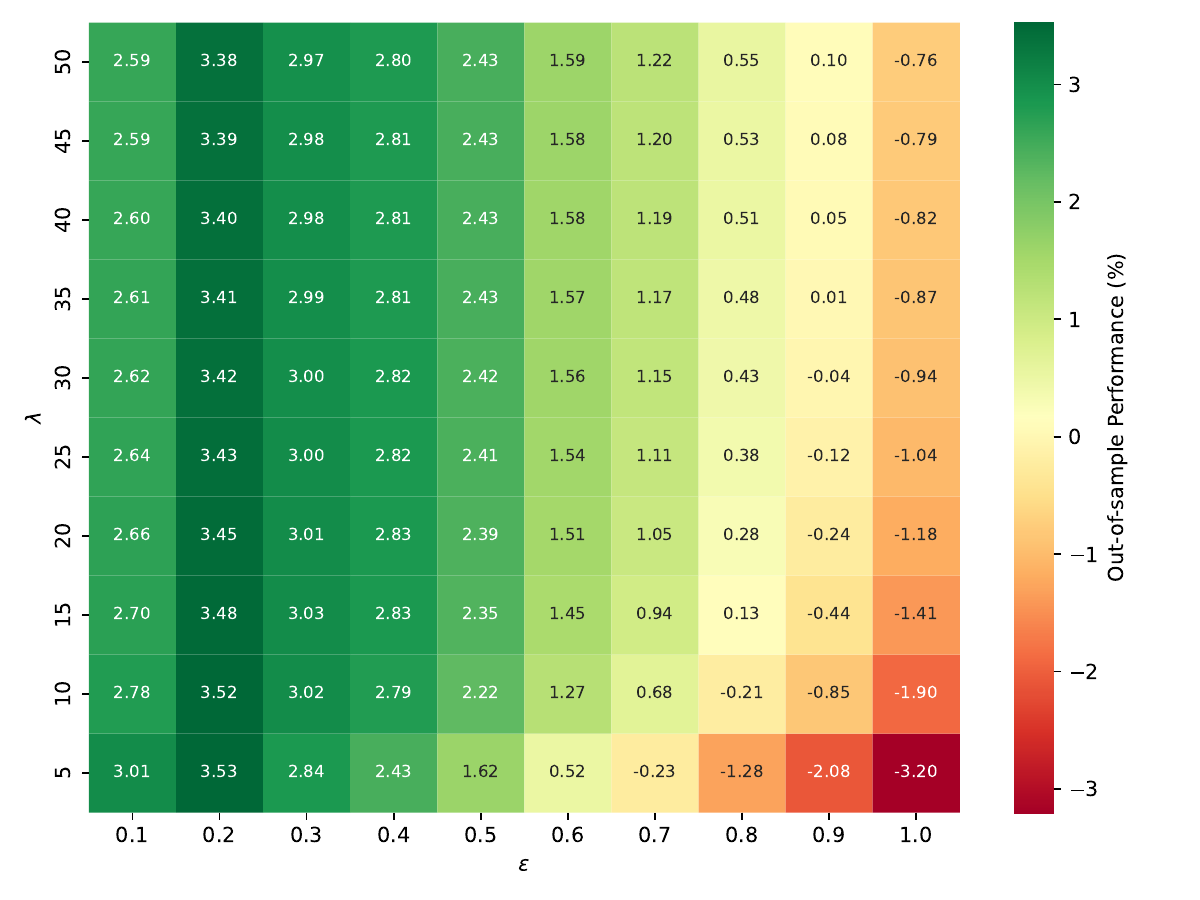}
    \caption{1-Causal-SDRO}
    \label{fig:port-cv-plot-p1}
  \end{subfigure}
  \hfill
    \begin{subfigure}{.49\linewidth}
    \centering
    \includegraphics[width=\linewidth]{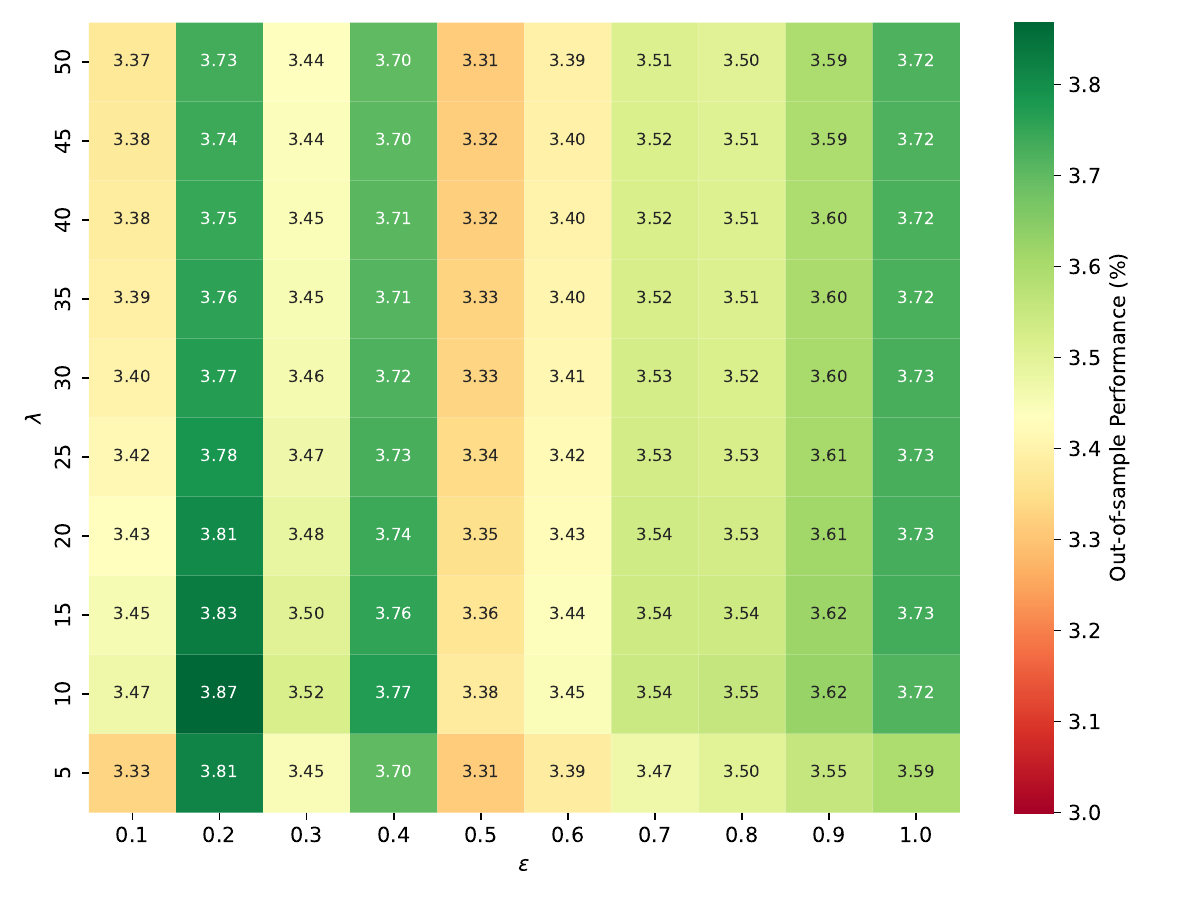}
    \caption{2-Causal-SDRO}
    \label{fig:port-cv-plot-p2}
  \end{subfigure}  
  \caption{{Out-of-sample Performance of the portfolio problem with different parameters ($\omega = 1, n_m=10$)}}
  \label{fig:port-cv}
\end{figure}

\begin{figure}[!htb]
  \centering
  \begin{subfigure}{.49\linewidth}
    \centering
    \includegraphics[width=\linewidth]{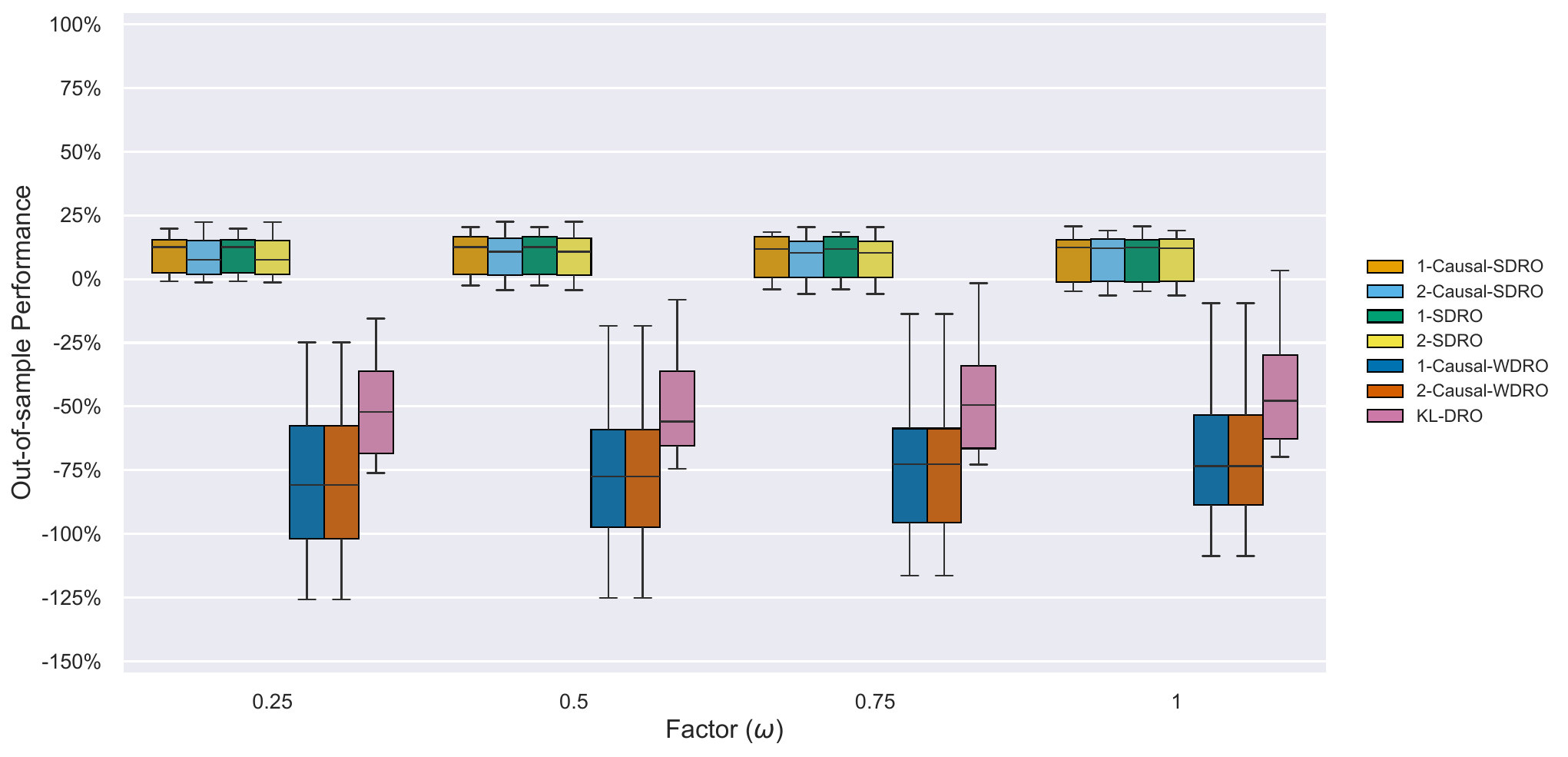}
    \caption{$n_m=1$}
    \label{fig:portfolio-dro-box-1}
  \end{subfigure}
  \hfill
    \begin{subfigure}{.49\linewidth}
    \centering
    \includegraphics[width=\linewidth]{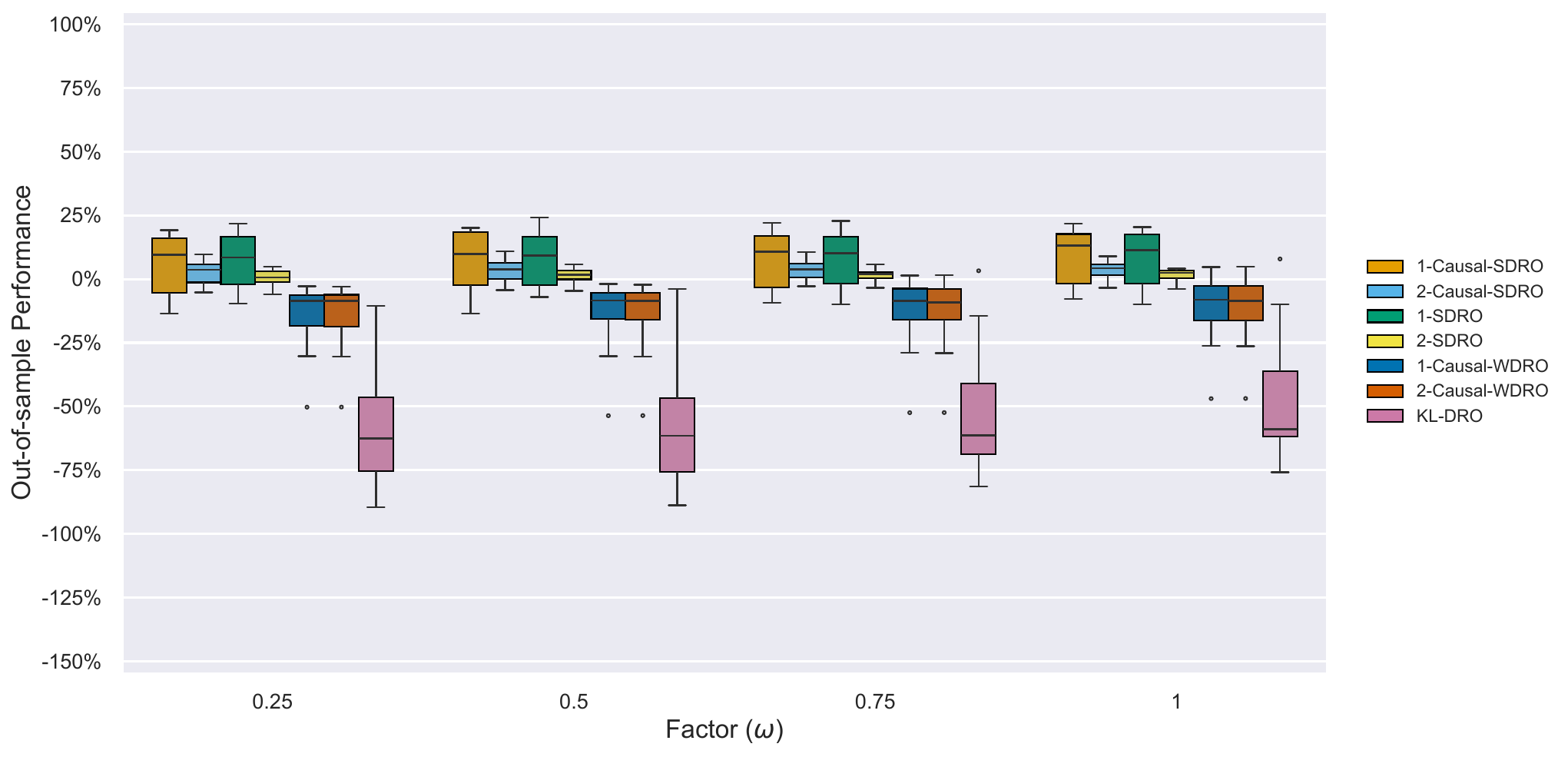}
    \caption{$n_m=10$}
    \label{fig:portfolio-dro-box-2}
  \end{subfigure}  
  \caption{{Comparison of different DRO models on the portfolio problem}}
  \label{fig:portfolio-dro-box} 
\end{figure}

\subsubsection{Interpretability of the Proposed SRF} \label{subsec-results-portfolio-interpret}

To further understand the intrinsic interpretability, we examine a simple Soft Regression Tree (SRT) with three layers, trained on the mean-variance portfolio problem with $\omega=5$. 
This SRT is trained using covariates and asset returns from the preceding two-year rolling window. 

Figure~\ref{fig:interpre-a-tree} illustrates the structure of the trained SRT, with normalized feature weights shown at each internal node. Consider two input covariates $\boldsymbol{x}_1, \boldsymbol{x}_2 \in \mathbb{R}^{5}$. The decision of $\boldsymbol{x}_1$ becomes $\boldsymbol{\pi}_8$ with weight probability (approximately) $1$. 
The model maps $\boldsymbol{x}_2$ to decisions $\boldsymbol{\pi}_7$ and $\boldsymbol{\pi}_4$ with weight probabilities $0.996$ and $0.004$, respectively.
These results confirm that the final decision is dominated by a few high-probability routes. 
\begin{figure}[!htb]
  \centering
    \includegraphics[width=0.9\linewidth]{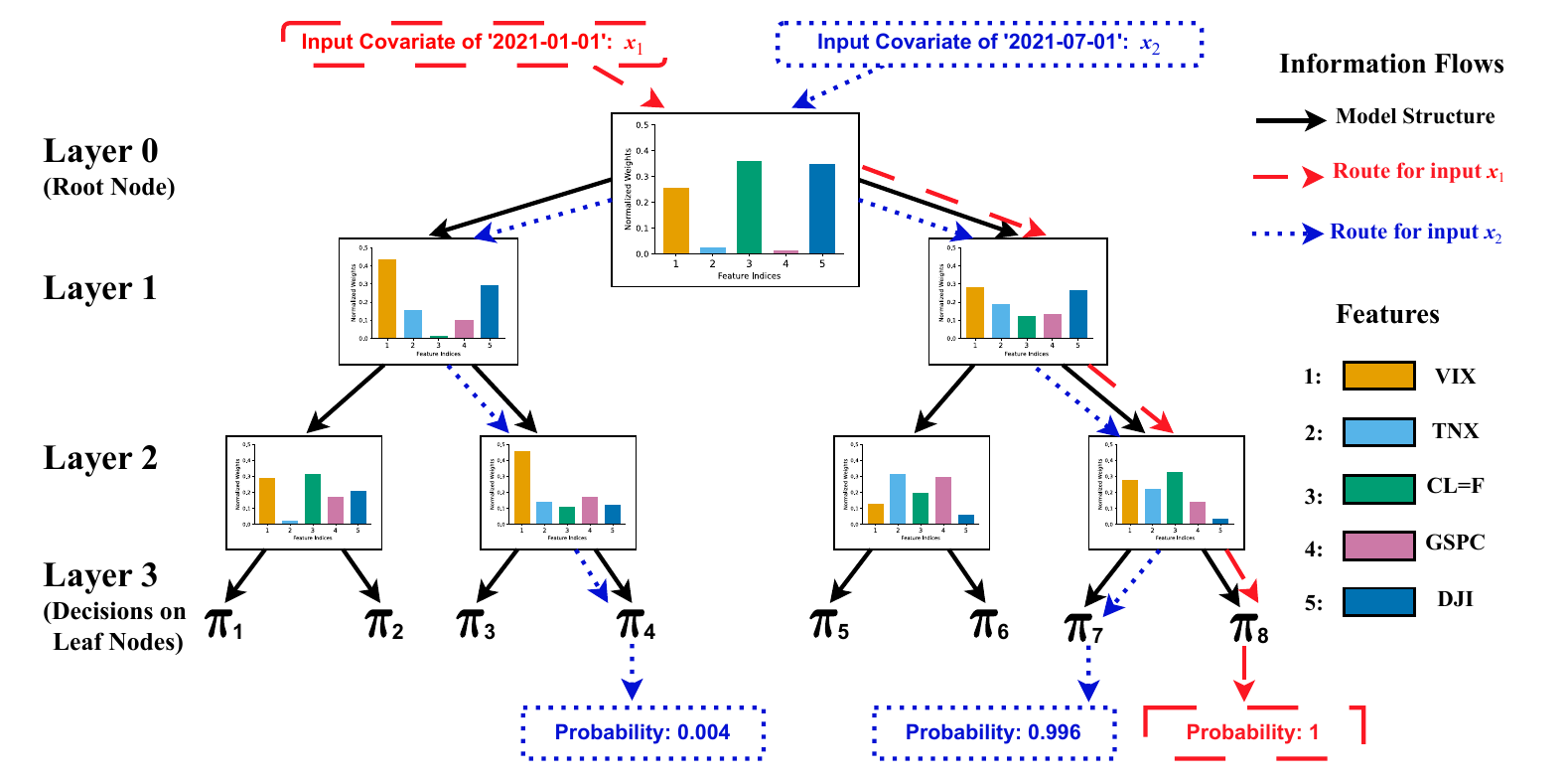}
  \caption{The structure of the trained SRT with three layers. Here, the solid black lines denote the model structure, while the red and blue dashed lines illustrate the decision routes and their corresponding selection probabilities for two input covariates $\boldsymbol{x}_1$ and $\boldsymbol{x}_2$, respectively.}
  \label{fig:interpre-a-tree} 
\end{figure}

We next demonstrate the intrinsic interpretability of this SRT from both global and local perspectives. 
The global interpretation provides a holistic view of the model by quantifying each feature's contribution to the overall decision-making process, while the local interpretation shows how an individual decision is derived for a specific set of covariates~\citep {dwivedi2023explainable}. 
In Appendix~\ref{ecsec-interpretability}, we introduce a global feature-importance measure and a local feature-attribution measure for SRF, both based on differentiability. 
In the following, we analyze the SRT using these interpretation measures and compare them with traditional post hoc explanation methods. 

\begin{figure} [!htb]
    \centering
    \includegraphics[width=0.70\linewidth]{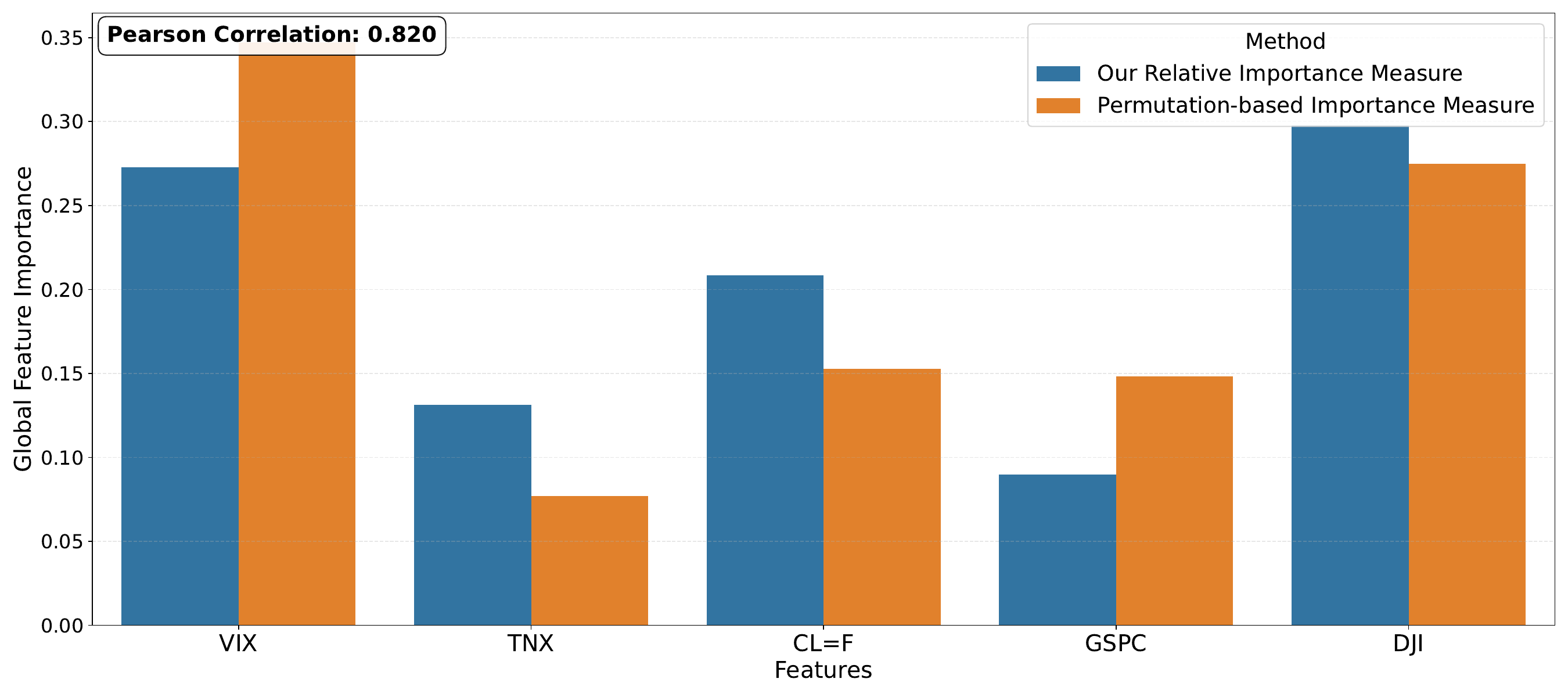}
    \caption{Global feature importance comparison for the trained SRT} 
    \label{fig: SRT-global}
\end{figure}

\begin{figure}[!htb]
    \centering
    \includegraphics[width=\linewidth]{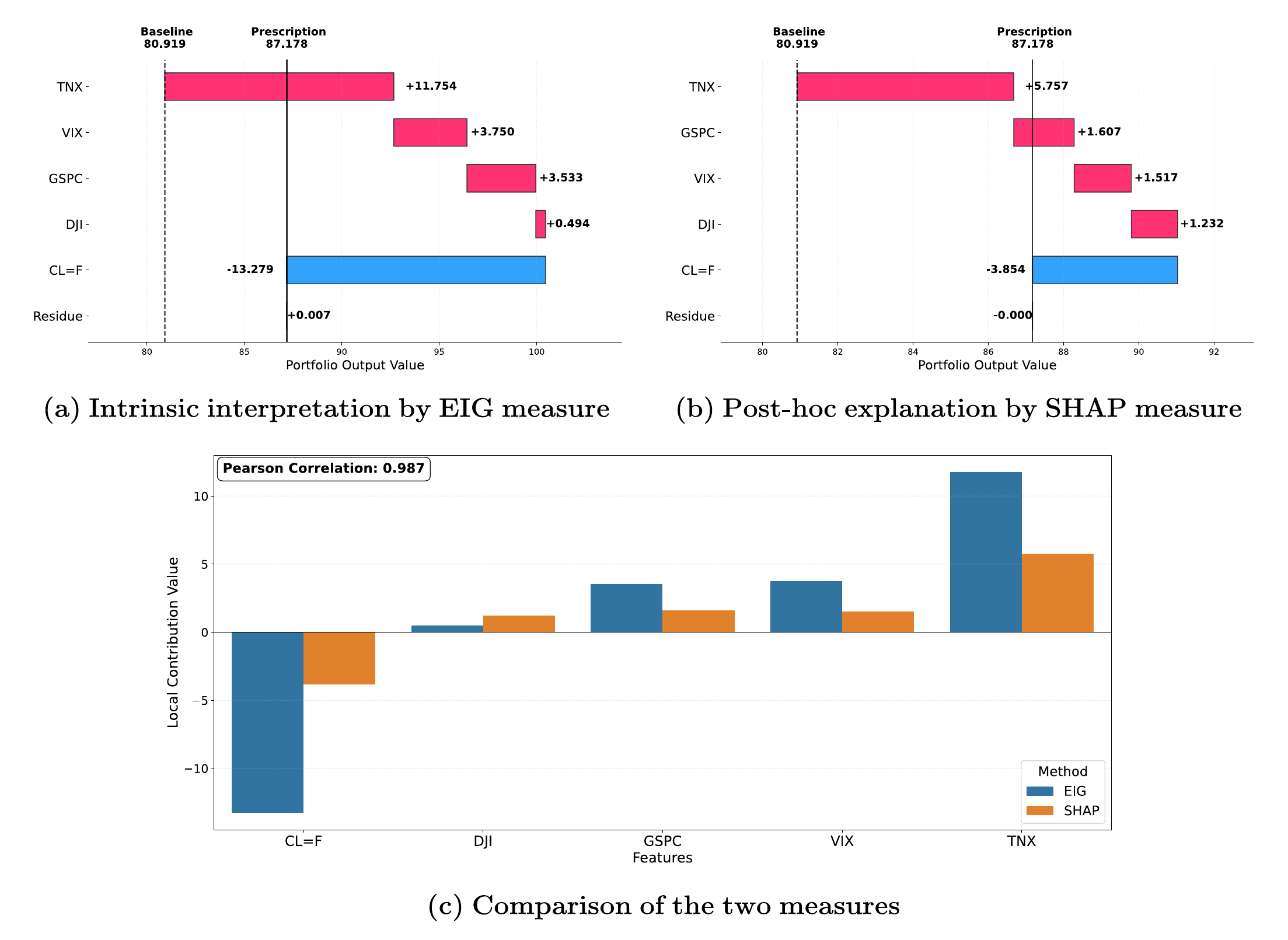}
  \caption{ Local feature attribution comparison for the trained SRT }
  \label{fig:local-interpret}
\end{figure}

Globally, our feature importance measure is defined as the average marginal sensitivity of the decision to each feature over the training set. 
Figure~\ref{fig: SRT-global} visualizes the relative feature importance of this SRT using the proposed measure in Appendix~\ref{ecsec-interpretability} and a perturbation-based measure ~\citep{hastie2009elements}. 
The difference between the two measures is that our measure shows feature importance intrinsically, based solely on the SRF derivatives, and avoids post hoc perturbation analyses. 
As illustrated, these measures are highly correlated with a Pearson correlation coefficient of 0.820. This result confirms that the intrinsic interpretation of SRF is consistent with established global post hoc explanations. 

Locally, our intrinsic interpretation measure, the empirical integrated gradient (EIG), decomposes the decision into the sum of feature contributions defined by integrated gradients~\citep{sundararajan2017axiomatic}. 
Figure~\ref{fig:local-interpret} visualizes the feature attributions derived from the proposed EIG measure and a traditional post hoc explanation \citep[SHAP measure,][]{lundberg2020local} for this SRT given a specific covariate. 
Figures~\ref{fig:local-interpret}(a) and~\ref{fig:local-interpret}(b) display the waterfall plots for EIG and SHAP measures, respectively, where red and blue bars indicate positive and negative contributions, respectively.  
The prescription and baseline value are calculated by $\sum_{k=1}^{d_z}[ f_{\boldsymbol{\theta}}^{\text{SRF}}(\boldsymbol{x}) ]_{k}$ and $\sum_{i=1}^{N} \sum_{k=1}^{d_z} [ f_{\boldsymbol{\theta}}^{\text{SRF}}(\boldsymbol{x}^{i}) ]_{k} / N$, respectively. 
Figure~\ref{fig:local-interpret}(c) directly compares the EIG and SHAP values across all features. 
As illustrated, the high correlation coefficient of 0.987 between these measures reflects the consistency between our local intrinsic interpretation and the local post hoc explanation. 

\subsection{Additional Experiments on Smoothing Baselines} \label{App-subsec-kernal-smooth-SAA}

When $\bar{\rho}=0$, our formulation reduces to the sample average approximation formulation with a kernel-smoothed probability distribution, which we refer to as the smoothing baseline:
\[
\inf_{f\in\mathcal{F}}~\mathbb{E}_{(\widehat{\boldsymbol{x}}, \widehat{\boldsymbol{y}})\sim\widehat{\mathbb{P}}}\mathbb{E}_{(\boldsymbol{x}, \boldsymbol{y})\sim\mathcal{K}_{(\widehat{\boldsymbol{x}}, \widehat{\boldsymbol{y}}), \epsilon}}[\Psi(f(\boldsymbol{x}), \boldsymbol{y}))],
\]
where the kernel probability distribution $\mathcal{K}_{(\boldsymbol{x}, \boldsymbol{y}), \epsilon}$ is defined in Remark~\ref{remark-KL-connect}. 
We denote the smoothing baselines by 1-Smooth and 2-Smooth for $p=1$ and $p=2$, respectively. 

{
Figures~\ref{fig:news-smooth-box} and~\ref{fig:inv-smooth-box} compare the performance of Causal-SDRO models with the smoothing baselines across different sample sizes, feature dimensions, and numerical instances. 
For a controlled comparison, we use the same regularization parameter $\epsilon$ for each smoothing baseline as in its corresponding Causal-SDRO model. 
As illustrated, the smoothing baselines achieve competitive out-of-sample performance, whereas Causal-SDRO models achieve a moderately higher average performance across several configurations. 
These results suggest that smoothing baselines alone are strong, while additional distributional robustness may yield further improvements beyond kernel smoothing. 
}

\begin{figure}[!htb]
  \centering
    \includegraphics[width=\linewidth]{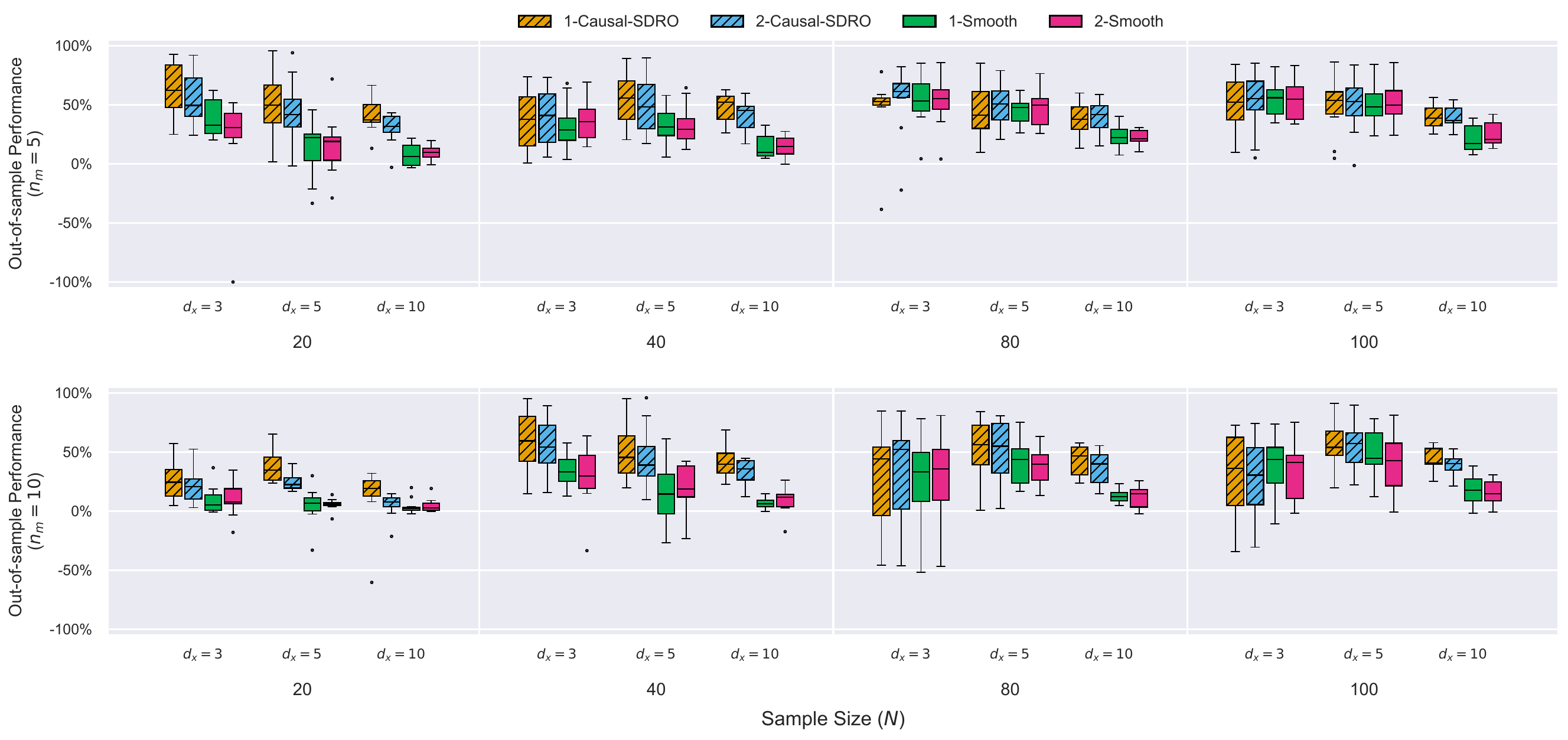}
  \caption{{Comparison between Causal-SDRO and smoothing baselines on the newsvendor problem}}
  \label{fig:news-smooth-box} 
\end{figure}

\begin{figure}[!htb]
  \centering
    \includegraphics[width=\linewidth]{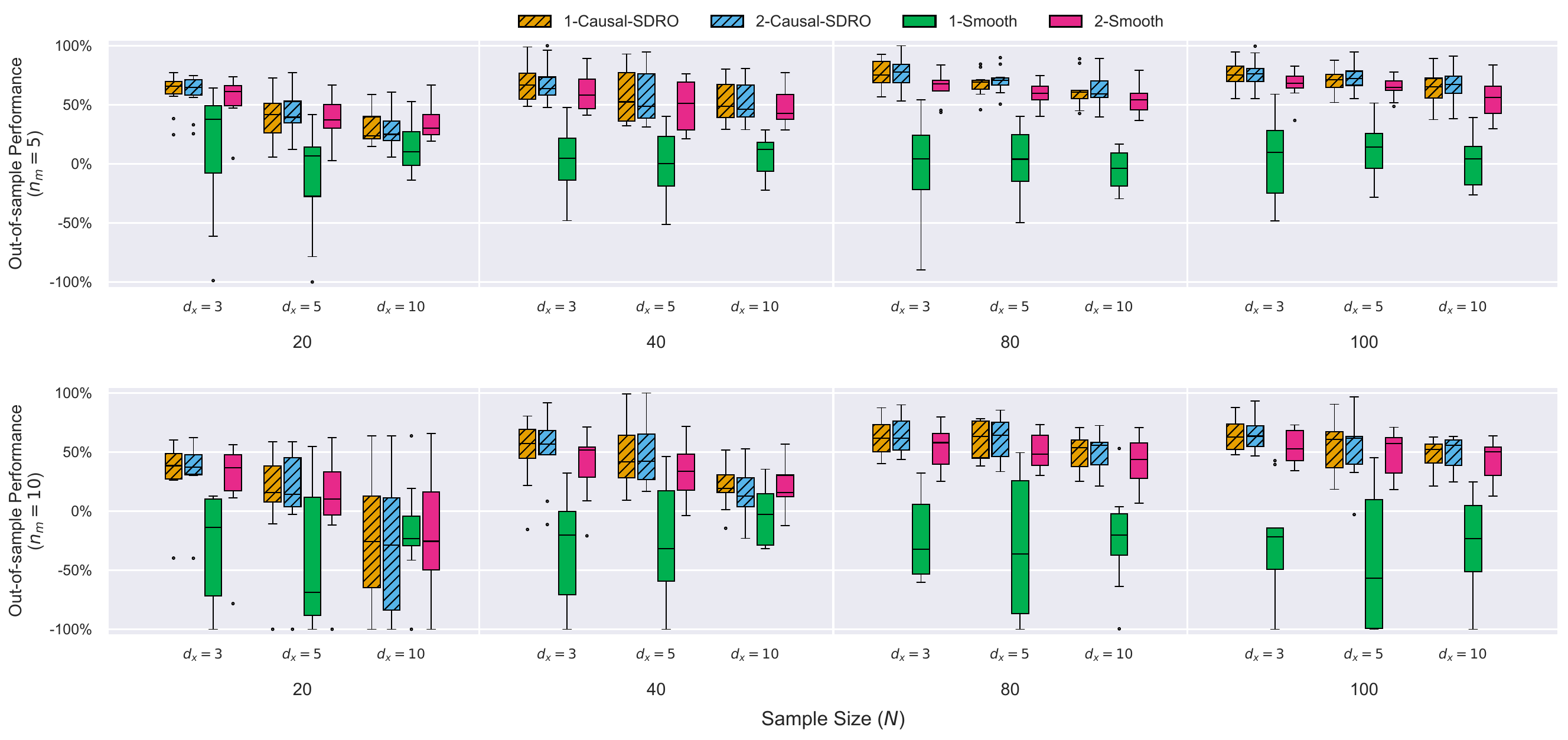}
  \caption{{Comparison between Causal-SDRO and smoothing baselines on the inventory substitution problem}}
  \label{fig:inv-smooth-box} 
\end{figure}

\newpage

\section{Omitted Analyses and Proof}

\subsection{Analysis for Condition~\ref{condit-1} in Section~\ref{subsec-CSDRO-dual}}\label{ecsec-condition1}  

We present the following sufficient conditions to verify whether Condition~\ref{condit-1} holds. 

\begin{proposition}\label{ec-prop-condit-1}
    Condition~\ref{condit-1} holds if there exist $\lambda > 0$ and a constant $\alpha \in \left[0, 1\right)$ so that for $\widehat{\mathbb{P}}\otimes\nu_{\mathcal{X}}\otimes\nu_{\mathcal{Y}}$-almost every $(\widehat{\boldsymbol{x}}, \widehat{\boldsymbol{y}}, \boldsymbol{x}, \boldsymbol{y})$, it holds that 
    \begin{equation} \label{suff-condition-1}
        \Psi(f(\boldsymbol{x}), \boldsymbol{y}) \le \alpha \cdot \lambda c_p((\widehat{\boldsymbol{x}}, \widehat{\boldsymbol{y}}), (\boldsymbol{x}, \boldsymbol{y})) + M(\widehat{\boldsymbol{x}}, \widehat{\boldsymbol{y}}) , 
    \end{equation}
    where $M(\widehat{\boldsymbol{x}}, \widehat{\boldsymbol{y}})$ is a measurable function satisfying $\mathbb{E}_{\widehat{\boldsymbol{y}}\sim\widehat{P}_{\widehat{\boldsymbol{Y}}|\widehat{\boldsymbol{X}}=\widehat{\boldsymbol{x}}}}[e^{M(\widehat{\boldsymbol{x}}, \widehat{\boldsymbol{y}})/(\lambda\epsilon)}] < \infty$ for $\widehat{\mathbb{P}}_{\widehat{\boldsymbol{X}}}$-almost every $\widehat{\boldsymbol{x}}$. 
\end{proposition}

\begin{proof}[Proof of Proposition~\ref{ec-prop-condit-1}] 
In Eq.~\eqref{QW-distri}, we set $\boldsymbol{\xi}_1 = \boldsymbol{x} - \widehat{\boldsymbol{x}}$ and $\boldsymbol{\xi}_2 = \boldsymbol{y} - \widehat{\boldsymbol{y}}$. When relation~\eqref{suff-condition-1} holds, for $\widehat{\mathbb{P}}\otimes\nu_{\mathcal{X}}$-almost every $(\widehat{\boldsymbol{x}}, \widehat{\boldsymbol{y}}, \boldsymbol{x})$, we have 
    \begin{equation}\nonumber
        \begin{aligned}
            \mathbb{E}& _{ \boldsymbol{\xi}_2 \sim W_{\epsilon}} \left[ \exp\left( \frac{\Psi(f(\widehat{\boldsymbol{x}}+ \boldsymbol{\xi}_1),\widehat{\boldsymbol{y}}+ \boldsymbol{\xi}_2)}{\lambda\epsilon} \right) \right] \\ 
            & = \int_{ \boldsymbol{y} \sim \nu_{\mathcal{Y}}} \frac{e^{ - \| \boldsymbol{y} - \widehat{\boldsymbol{y}}\|^p / \epsilon} }{\int_{\mathbb{R}^{d_y}} e^{ - \| \boldsymbol{u} - \widehat{\boldsymbol{y}} \|^p / \epsilon} \mathrm{d} \nu_{\mathcal{Y} }\left (  \boldsymbol{u} \right ) } \cdot  \exp\left( \frac{\Psi(f(\widehat{\boldsymbol{x}}+ \boldsymbol{\xi}_1),\boldsymbol{y})}{\lambda\epsilon} \right) \mathrm{d} \nu_{\mathcal{Y} }\left (  \boldsymbol{y} \right ) \\
            & \le \int_{ \boldsymbol{y} \sim \nu_{\mathcal{Y}}}  \frac{e^{ - \| \boldsymbol{y} - \widehat{\boldsymbol{y}}\|^p / \epsilon} }{\int_{\mathbb{R}^{d_y}} e^{ - \| \boldsymbol{u} - \widehat{\boldsymbol{y}} \|^p / \epsilon} \mathrm{d} \nu_{\mathcal{Y} }\left (  \boldsymbol{u} \right ) } \cdot  \exp\left( \frac{\alpha \cdot \lambda c_p((\widehat{\boldsymbol{x}}, \widehat{\boldsymbol{y}}), (\boldsymbol{x}, \boldsymbol{y})) + M(\widehat{\boldsymbol{x}}, \widehat{\boldsymbol{y}})}{\lambda\epsilon} \right) \mathrm{d} \nu_{\mathcal{Y} }\left (  \boldsymbol{y} \right ) \\
            & = \frac{e^{ \alpha \| \boldsymbol{x} - \widehat{\boldsymbol{x}}\|^p / \epsilon + M(\widehat{\boldsymbol{x}}, \widehat{\boldsymbol{y}}) / \lambda\epsilon} }{\int_{\mathbb{R}^{d_y}} e^{ - \| \boldsymbol{u} - \widehat{\boldsymbol{y}} \|^p / \epsilon} \mathrm{d} \nu_{\mathcal{Y} }\left (  \boldsymbol{u} \right ) } \cdot \int_{ \boldsymbol{y} \sim \nu_{\mathcal{Y}}} \exp\left( -\frac{(1-\alpha) \| \boldsymbol{y} - \widehat{\boldsymbol{y}}\|^p }{\epsilon} \right) \mathrm{d} \nu_{\mathcal{Y} }\left (  \boldsymbol{y} \right ) \\
            & < \infty,
        \end{aligned}
    \end{equation}
    where the first equality follows from the change of variables $\boldsymbol{y}=\widehat{\boldsymbol{y}}+\boldsymbol{\xi}_2$ and the translation invariance of the Lebesgue measure, the first inequality is by~\eqref{suff-condition-1}, and the finiteness is by Assumptions~\ref{assum-1}\ref{assum-1-1} and~\ref{assum-1-2}, and the fact that $e^{ M(\widehat{\boldsymbol{x}}, \widehat{\boldsymbol{y}}) / \lambda\epsilon} < \infty$.    

    Similarly, Eq.~\eqref{suff-condition-1} and Jensen's inequality imply that
    \[
    \mathbb{E}_{ \boldsymbol{\xi}_1 \sim U_{\epsilon}}\left[ \exp\left( \frac{h (\widehat{\boldsymbol{x}},  \boldsymbol{\xi}_1, \lambda)}{\lambda\epsilon} \right) \right]<\infty
    \]
    for $\widehat{\mathbb{P}}_{\widehat{\boldsymbol{X}}}$-almost every $\widehat{\boldsymbol{x}}$, where the finiteness follows from Assumptions~\ref{assum-1}\ref{assum-1-1}--\ref{assum-1-2}. 
\end{proof}

\subsection{Proof of Theorem~\ref{theo-strong-duality} in Section~\ref{subsec-CSDRO-dual}} \label{ecsec-strong-dual}

To prove the strong duality in Theorem~\ref{theo-strong-duality}, we first establish Lemma~\ref{lem-meas} and the weak duality result in Lemma~\ref{lem-weak-duality}. 

\begin{lemma}\label{lem-meas}
    \textnormal{\textbf{(Several Measurable Functions). }} Assume that Assumption~\ref{assum-1} holds, and then the following results hold. 
    \begin{enumerate}
        \item Define the function $w_1(\widehat{\boldsymbol{x}}, \widehat{\boldsymbol{y}}, \boldsymbol{x}, \lambda):  \mathcal{X} \times \mathcal{Y} \times \mathcal{X} \times \mathbb{R}_+\to \mathbb{R}$ as
    \begin{equation}\nonumber
\begin{aligned}
 w_1(\widehat{\boldsymbol{x}}, \widehat{\boldsymbol{y}}, \boldsymbol{x}, \lambda) := \sup_{\gamma_{\boldsymbol{Y} \mid \widehat{\boldsymbol{X}}, \widehat{\boldsymbol{Y}}, \boldsymbol{X}}} \, \mathbb{E}_{\gamma_{\boldsymbol{Y} \mid \widehat{\boldsymbol{X}}, \widehat{\boldsymbol{Y}}, \boldsymbol{X}}}
 \Bigg[ & \Psi(f(\boldsymbol{x}),\boldsymbol{y})-  \lambda  c_p((\widehat{\boldsymbol{x}},\widehat{\boldsymbol{y}}), (\boldsymbol{x},\boldsymbol{y})) \\
 & -  \lambda \epsilon\log\, \left ( \frac{\mathrm{d}\gamma_{\boldsymbol{Y} \mid  \widehat{\boldsymbol{X}}, \widehat{\boldsymbol{Y}}, \boldsymbol{X}}(\boldsymbol{y}\mid \widehat{\boldsymbol{x}}, \widehat{\boldsymbol{y}}, \boldsymbol{x})}{\mathrm{d} \nu_{\mathcal{Y}}(\boldsymbol{y})}  \right ) \Bigg].
\end{aligned}
    \end{equation}
    This function is jointly measurable with respect to $(\widehat{\boldsymbol{x}}, \widehat{\boldsymbol{y}}, \boldsymbol{x})$ regardless of the choice of $\lambda \ge 0$. 
    \item Define the function $g(\widehat{\boldsymbol{x}}, \boldsymbol{x}, \lambda):  \mathcal{X} \times \mathcal{X} \times \mathbb{R}_+\to \mathbb{R}$ as
    \begin{equation}\nonumber
    \begin{aligned}
        g(\widehat{\boldsymbol{x}}, \boldsymbol{x}, \lambda) &:= \mathbb{E}_{\widehat{\boldsymbol{y}} \sim \widehat{\mathbb{P}}_{\widehat{Y}|\widehat{X}=\widehat{\boldsymbol{x}}}} \Big[ w_1(\widehat{\boldsymbol{x}}, \widehat{\boldsymbol{y}}, \boldsymbol{x},\lambda) \Big] \\
        &  = \mathbb{E}_{\widehat{\boldsymbol{y}} \sim  \widehat{\mathbb{P}}_{\widehat{\boldsymbol{Y}}|\widehat{\boldsymbol{X}}=\widehat{\boldsymbol{x}}}}\left[ \lambda\epsilon \log\, \int _{\mathcal{Y}} \exp\left( \frac{\Psi(f(\boldsymbol{x}),\boldsymbol{y}) - \lambda c_p((\widehat{\boldsymbol{x}},\widehat{\boldsymbol{y}}), (\boldsymbol{x},\boldsymbol{y}))}{\lambda\epsilon} \right) \mathrm{d} \nu_{\mathcal{Y}}(\boldsymbol{y}) \right].
    \end{aligned}
    \end{equation}
    This function is jointly measurable with respect to $(\widehat{\boldsymbol{x}}, \boldsymbol{x})$ for any $\lambda \ge 0$. 
    \item Define the function $w_2(\widehat{\boldsymbol{x}}, \lambda):  \mathcal{X} \times \mathbb{R}_+\to \mathbb{R}$ as 
    \begin{equation}\nonumber
        w_2(\widehat{\boldsymbol{x}}, \lambda) := \sup_{\gamma \in \Gamma_c(\mathbb{P}, \mathbb{Q})} \Bigg\{  \mathbb{E}_{\gamma_{\boldsymbol{X} \mid  \widehat{\boldsymbol{X}}}}\Bigg[g(\widehat{\boldsymbol{x}}, \boldsymbol{x}, \lambda) - \lambda \epsilon \log\, \left ( \frac{\mathrm{d}\gamma_{\boldsymbol{X} \mid \widehat{\boldsymbol{X}}}(\boldsymbol{x} \mid \widehat{\boldsymbol{x}})}{\mathrm{d} \nu_{\mathcal{X}}(\boldsymbol{x})}  \right )\bigg] \Bigg\}.
    \end{equation}
    This function is measurable with respect to $\widehat{\boldsymbol{x}}\sim \widehat{\mathbb{P}}_{\widehat{\boldsymbol{X}}}$ regardless of the choice of $\lambda \ge 0$.
    \end{enumerate}
\end{lemma}
\begin{proof}[Proof of Lemma~\ref{lem-meas}] 
We prove the measurability of the three functions in sequence.
\begin{enumerate}
    \item \textbf{For function $w_1(\widehat{\boldsymbol{x}}, \widehat{\boldsymbol{y}}, \boldsymbol{x}, \lambda)$}, we consider the following two cases. 
    \begin{itemize}
        \item When $\lambda = 0$, according to Lemma EC.2 in the E-companion of~\citet{wang2025sinkhorn}, it holds that 
        \begin{equation}\nonumber
            w_1(\widehat{\boldsymbol{x}}, \widehat{\boldsymbol{y}}, \boldsymbol{x}, 0) = \operatorname*{ess sup}_{\boldsymbol{y} \sim \nu_{\mathcal{Y}}} \quad \Psi(f(\boldsymbol{x}),\boldsymbol{y}).
        \end{equation}
        By Assumption~\ref{assum-1}, the loss function $\Psi(f(\boldsymbol{x}),\boldsymbol{y})$ is measurable. Since the essential supremum of a measurable function with respect to one variable is a measurable function of the remaining variables~\citep{blackwell1963non}, the function $w_1(\widehat{\boldsymbol{x}}, \widehat{\boldsymbol{y}}, \boldsymbol{x}, 0)$ is jointly measurable with respect to $(\widehat{\boldsymbol{x}}, \widehat{\boldsymbol{y}}, \boldsymbol{x})$. 
        \item When $\lambda > 0$, using the Fenchel duality, we have 
        \begin{equation}\nonumber
            w_1(\widehat{\boldsymbol{x}}, \widehat{\boldsymbol{y}}, \boldsymbol{x}, \lambda) = \lambda\epsilon \log\, \int _{\mathcal{Y}} \exp\left( \frac{\Psi(f(\boldsymbol{x}),\boldsymbol{y}) - \lambda c_p((\widehat{\boldsymbol{x}},\widehat{\boldsymbol{y}}), (\boldsymbol{x},\boldsymbol{y}))}{\lambda\epsilon} \right) \mathrm{d} \nu_{\mathcal{Y}}(\boldsymbol{y}). 
        \end{equation}
        According to Assumption~\ref{assum-1}, function $c_p$ is measurable. Since the difference of two measurable functions is measurable, the function $\Psi(f(\boldsymbol{x}),\boldsymbol{y}) -  \lambda c_p((\widehat{\boldsymbol{x}},\widehat{\boldsymbol{y}}), (\boldsymbol{x},\boldsymbol{y}))$ is measurable with respect to $(\widehat{\boldsymbol{x}},\widehat{\boldsymbol{y}}, \boldsymbol{x}, \boldsymbol{y})$. 
        As the composition of a measurable function with a continuous function is measurable, the function $ \exp\left( \frac{\Psi(f(\boldsymbol{x}),\boldsymbol{y}) - \lambda c_p((\widehat{\boldsymbol{x}},\widehat{\boldsymbol{y}}), (\boldsymbol{x},\boldsymbol{y}))}{\lambda\epsilon} \right)$ is measurable. According to Tonelli's Theorem~\citep[Proposition 5.2.1][]{cohn2013measure}, integrating a non-negative, jointly measurable function with respect to one variable (here, $\boldsymbol{y}$) yields a function that is measurable with respect to the remaining variables. Therefore, the function $\int _{\mathcal{Y}} \exp\left( \frac{\Psi(f(\boldsymbol{x}),\boldsymbol{y}) - \lambda c_p((\widehat{\boldsymbol{x}},\widehat{\boldsymbol{y}}), (\boldsymbol{x},\boldsymbol{y}))}{\lambda\epsilon} \right) \mathrm{d} \nu_{\mathcal{Y}}(\boldsymbol{y})$ is measurable, and its compositional with the continuous function $\log(\cdot)$ is also measurable, that is, the function $ w_1(\widehat{\boldsymbol{x}}, \widehat{\boldsymbol{y}}, \boldsymbol{x}, \lambda)$ is jointly measurable with respect to $(\widehat{\boldsymbol{x}}, \widehat{\boldsymbol{y}}, \boldsymbol{x})$ when $\lambda > 0$. 
    \end{itemize}
    Hence, function $ w_1(\widehat{\boldsymbol{x}}, \widehat{\boldsymbol{y}}, \boldsymbol{x}, \lambda)$ is jointly measurable with respect to $(\widehat{\boldsymbol{x}}, \widehat{\boldsymbol{y}}, \boldsymbol{x})$ regardless of the choice of $\lambda \ge 0$. 

    \item \textbf{For function $g(\widehat{\boldsymbol{x}}, \boldsymbol{x}, \lambda)$}, we also consider the following two cases. 
    \begin{itemize}
        \item When $\lambda>0$, Condition~\ref{condit-1} ensures that the log-partition function defining $w_1(\widehat{\boldsymbol{x}},\widehat{\boldsymbol{y}}, \boldsymbol{x},\lambda)$ is positive and finite. Hence, $w_1(\widehat{\boldsymbol{x}},\widehat{\boldsymbol{y}}, \boldsymbol{x},\lambda)$ is a finite-valued measurable function. 

        Since $\widehat{\mathbb P}$ is an empirical distribution, $\widehat{\mathbb P}_{\widehat{\boldsymbol{y}}\mid \widehat{\boldsymbol{x}}=\widehat{\boldsymbol{x}}}$ has finite support for every $\widehat{\boldsymbol{x}}\in \operatorname{supp}(\widehat{\mathbb P}_{\widehat{\boldsymbol{x}}})$. 
        Therefore,
        \[
            g(\widehat{\boldsymbol{x}},\boldsymbol{x},\lambda) = \mathbb E_{\widehat{\boldsymbol{y}}\sim \widehat{\mathbb P}_{\widehat{\boldsymbol{y}}\mid \widehat{\boldsymbol{x}}=\widehat{\boldsymbol{x}}}} \left[ w_1(\widehat{\boldsymbol{x}},\widehat{\boldsymbol{y}}, \boldsymbol{x},\lambda) \right], 
        \]
        is a finite weighted sum of measurable functions and is thus jointly measurable with respect to $(\widehat{\boldsymbol{x}},\boldsymbol{x})$. 

        \item When $\lambda = 0$, for the optimization problem to be well-posed, the loss function $\Psi$ must be bounded below, which implies that $w_1(\widehat{\boldsymbol{x}}, \widehat{\boldsymbol{y}}, \boldsymbol{x}, 0)$ is also bounded below. Let the lower bound of function $w_1(\widehat{\boldsymbol{x}}, \widehat{\boldsymbol{y}}, \boldsymbol{x}, 0)$ be $M \in \mathbb{R}$. We decompose the function $w_1(\widehat{\boldsymbol{x}}, \widehat{\boldsymbol{y}}, \boldsymbol{x}, 0)$ as two non-negative measurable functions
        \begin{equation}\nonumber
            w_1^+(\widehat{\boldsymbol{x}}, \widehat{\boldsymbol{y}}, \boldsymbol{x}, 0) = \max \left\{0, w_1(\widehat{\boldsymbol{x}}, \widehat{\boldsymbol{y}}, \boldsymbol{x}, 0)\right\} \in \left[0,+\infty\right],
        \end{equation}and 
        \begin{equation}\nonumber
            w_1^-(\widehat{\boldsymbol{x}}, \widehat{\boldsymbol{y}}, \boldsymbol{x}, 0) = - \min \left\{0, w_1(\widehat{\boldsymbol{x}}, \widehat{\boldsymbol{y}}, \boldsymbol{x}, 0)\right\} \in \left[0,  -\min\{0, M\}\right],
        \end{equation} such that 
        \begin{equation}\nonumber
        \begin{aligned}
            g(\widehat{\boldsymbol{x}}, \boldsymbol{x}, 0) & = \mathbb{E}_{\widehat{\boldsymbol{y}} \sim \widehat{\mathbb{P}}_{\widehat{Y}|\widehat{X}=\widehat{\boldsymbol{x}}}} \Big[ w_1(\widehat{\boldsymbol{x}}, \widehat{\boldsymbol{y}}, \boldsymbol{x}, 0) \Big] \\ & = \mathbb{E}_{\widehat{\boldsymbol{y}} \sim \widehat{\mathbb{P}}_{\widehat{Y}|\widehat{X}=\widehat{\boldsymbol{x}}}} \Big[ w_1^+(\widehat{\boldsymbol{x}}, \widehat{\boldsymbol{y}}, \boldsymbol{x}, 0) \Big]  - \mathbb{E}_{\widehat{\boldsymbol{y}} \sim \widehat{\mathbb{P}}_{\widehat{Y}|\widehat{X}=\widehat{\boldsymbol{x}}}} \Big[ w_1^-(\widehat{\boldsymbol{x}}, \widehat{\boldsymbol{y}}, \boldsymbol{x}, 0) \Big].
        \end{aligned}
        \end{equation}
        Based on Tonelli's Theorem, the expectations of both $w_1^+$ and $w_1^-$ with respect to $\widehat{\boldsymbol{y}}$ are well-defined measurable functions of $(\widehat{\boldsymbol{x}}, \boldsymbol{x})$. 
        Since $\mathbb{E}_{\widehat{\boldsymbol{y}} \sim \widehat{\mathbb{P}}_{\widehat{Y}|\widehat{X}=\widehat{\boldsymbol{x}}}} \Big[ w_1^-(\widehat{\boldsymbol{x}}, \widehat{\boldsymbol{y}}, \boldsymbol{x}, 0) \Big]$ is a finite measurable function, this subtraction is well-defined (it avoids the $\infty - \infty$ form). The difference of two measurable functions is measurable. Thus, function $g(\widehat{\boldsymbol{x}}, \boldsymbol{x}, 0)$ is also jointly measurable with respect to $(\widehat{\boldsymbol{x}}, \boldsymbol{x})$. 
    \end{itemize}
     Hence, function $g(\widehat{\boldsymbol{x}}, \boldsymbol{x}, \lambda)$ is jointly measurable with respect to $(\widehat{\boldsymbol{x}}, \boldsymbol{x})$ regardless of the choice of $\lambda \ge 0$. 
    \item \textbf{For function $w_2(\widehat{\boldsymbol{x}}, \lambda)$},  according to Lemma EC.3 in the E-companion of \citet{wang2025sinkhorn}, it's measurable with respect to $\widehat{\boldsymbol{x}}\sim \widehat{\mathbb{P}}_{\widehat{\boldsymbol{X}}}$ regardless of the choice of $\lambda \ge 0$.
    \end{enumerate} This completes the proof. 
\end{proof} 

\begin{lemma}\label{lem-weak-duality}
    \textnormal{\textbf{(Weak Duality).}} Under Assumption~\ref{assum-1}, it holds that $v_{\mathrm{P}} \le v_{\mathrm{D}}$. 
\end{lemma}

\begin{proof}[Proof of Lemma~\ref{lem-weak-duality}] 
For the primal problem of \eqref{causal-sdro}, let $\lambda \ge 0$ be the Lagrangian multiplier of the causal Sinkhorn ball constraint, and define $v_{\mathrm{P}}(\lambda)$ as the optimal value of the soft-constrained DRO problem with penalty parameter $\lambda$:
\begin{equation}
    v_{\mathrm{P}}(\lambda) := \sup_{\mathbb{P} \in \mathcal{P}\left(\mathcal{X}\times\mathcal{Y}\right)} \Big[\mathbb{E}_{(\boldsymbol{x},\boldsymbol{y})\sim\mathbb{P}}[\Psi(f(\boldsymbol{x}),\boldsymbol{y})] - \lambda R_{\epsilon}(\widehat{\mathbb{P}}, \mathbb{P})^p\Big].
\end{equation}
Based on the Lagrangian weak duality, it holds that
\begin{equation}\label{vp-le-vd}
    \begin{aligned}
        v_{\mathrm{P}} &= \sup_{\mathbb{P} \in \mathcal{P}(\mathcal{X} \times \mathcal{Y})} \Big \{ \mathbb{E}_{(\boldsymbol{x}, \boldsymbol{y}) \sim \mathbb{P}}\Big[\Psi(f(\boldsymbol{x}), \boldsymbol{y})\Big]: R_{\epsilon}(\widehat{\mathbb{P}}, \mathbb{P})^p \leq \rho^p \Big\} \\
        & = \sup_{\mathbb{P} \in \mathcal{P}(\mathcal{X} \times \mathcal{Y})} \inf_{\lambda \ge 0} \Big \{ \mathbb{E}_{(\boldsymbol{x}, \boldsymbol{y}) \sim \mathbb{P}}\Big[\Psi(f(\boldsymbol{x}), \boldsymbol{y})\Big] -\lambda \Big( R_{\epsilon}(\widehat{\mathbb{P}}, \mathbb{P})^p - \rho^p \Big) \Big\} \\
        & \le \inf_{\lambda \ge 0} \sup_{\mathbb{P} \in \mathcal{P}(\mathcal{X} \times \mathcal{Y})} \Big \{ \mathbb{E}_{(\boldsymbol{x}, \boldsymbol{y}) \sim \mathbb{P}}\Big[\Psi(f(\boldsymbol{x}), \boldsymbol{y})\Big] -\lambda \Big( R_{\epsilon}(\widehat{\mathbb{P}}, \mathbb{P})^p - \rho^p \Big) \Big\} \\
        & = \inf_{\lambda \ge 0}\Big\{ \lambda\rho^p + v_{\mathrm{P}}(\lambda) \Big\},
    \end{aligned}
\end{equation}
where the inequality holds because of the min-max inequality. 

Next, we reformulate the term $v_{\mathrm{P}}(\lambda)$.
For the formulation of CSD 
\begin{equation}\label{ec-eq-Rp}
    \begin{aligned}
        R_{\epsilon}(\widehat{\mathbb{P}}, \mathbb{P})^p &= \inf_{\gamma \in \Gamma_c(\widehat{\mathbb{P}}, \mathbb{P})} \mathbb{E}_{((\widehat{\boldsymbol{x}},\widehat{\boldsymbol{y}}), (\boldsymbol{x},\boldsymbol{y}))\sim\gamma} \Big[ c_p((\widehat{\boldsymbol{x}},\widehat{\boldsymbol{y}}), (\boldsymbol{x},\boldsymbol{y})) + \epsilon \cdot \mathbb{D}_{\mathrm{KL}}\left( \gamma \|  \widehat{\mathbb{P}} \otimes \left ( \nu_{\mathcal{X}}\otimes \nu_{\mathcal{Y} } \right ) \right)\Big], 
     \end{aligned}
\end{equation}
where the second item can be reformulated as 
\begin{equation}\label{Eq:relative:entropy:expand}
    \begin{aligned}
        \mathbb{D}_{\mathrm{KL}} & \left( \gamma \| \widehat{\mathbb{P}} \otimes \left ( \nu_{\mathcal{X}}\otimes \nu_{\mathcal{Y} } \right ) \right) 
        \\ &= \log\, \left ( \frac{\mathrm{d}\gamma (  (\widehat{\boldsymbol{x}},\widehat{\boldsymbol{y}}), (\boldsymbol{x},\boldsymbol{y})) }{\mathrm{d}\widehat{\mathbb{P}}(\widehat{\boldsymbol{x}},\widehat{\boldsymbol{y}})\mathrm{d} \nu_{\mathcal{X}}(\boldsymbol{x})\mathrm{d} \nu_{\mathcal{Y}}(\boldsymbol{y})} \right ) \\ 
        &= \log\, \left ( \frac{\mathrm{d}\widehat{\mathbb{P}}_{\widehat{\boldsymbol{X}}}(\widehat{\boldsymbol{x}})\mathrm{d}\gamma_{\boldsymbol{X} \mid \widehat{\boldsymbol{X}}}(\boldsymbol{x} \mid \widehat{\boldsymbol{x}}) \mathrm{d}\gamma_{\widehat{\boldsymbol{Y}}\mid \widehat{\boldsymbol{X}}, \boldsymbol{X}}(\widehat{\boldsymbol{y}}\mid \widehat{\boldsymbol{x}}, \boldsymbol{x}) \mathrm{d}\gamma_{\boldsymbol{Y} \mid  \widehat{\boldsymbol{X}}, \widehat{\boldsymbol{Y}}, \boldsymbol{X}}(\boldsymbol{y}\mid \widehat{\boldsymbol{x}}, \widehat{\boldsymbol{y}}, \boldsymbol{x})}{\mathrm{d}\widehat{\mathbb{P}}_{\widehat{\boldsymbol{X}}}(\widehat{\boldsymbol{x}}) \mathrm{d}\widehat{\mathbb{P}}_{\widehat{\boldsymbol{Y}} \mid \widehat{\boldsymbol{X}}}(\widehat{\boldsymbol{y}}\mid\widehat{\boldsymbol{x}})\mathrm{d} \nu_{\mathcal{X}}(\boldsymbol{x})\mathrm{d} \nu_{\mathcal{Y}}(\boldsymbol{y})}  \right ) 
        \\ & = \log\, \left ( \frac{\mathrm{d}\widehat{\mathbb{P}}_{\widehat{\boldsymbol{X}}}(\widehat{\boldsymbol{x}})\mathrm{d}\gamma_{\boldsymbol{X} \mid \widehat{\boldsymbol{X}}}(\boldsymbol{x} \mid \widehat{\boldsymbol{x}}) \mathrm{d}\widehat{\mathbb{P}}_{\widehat{\boldsymbol{Y}} \mid \widehat{\boldsymbol{X}}}(\widehat{\boldsymbol{y}}\mid \widehat{\boldsymbol{x}}) \mathrm{d}\gamma_{\boldsymbol{Y} \mid  \widehat{\boldsymbol{X}}, \widehat{\boldsymbol{Y}}, \boldsymbol{X}}(\boldsymbol{y}\mid \widehat{\boldsymbol{x}}, \widehat{\boldsymbol{y}}, \boldsymbol{x})}{\mathrm{d}\widehat{\mathbb{P}}_{\widehat{\boldsymbol{X}}}(\widehat{\boldsymbol{x}}) \mathrm{d}\widehat{\mathbb{P}}_{\widehat{\boldsymbol{Y}} \mid \widehat{\boldsymbol{X}}}(\widehat{\boldsymbol{y}}\mid\widehat{\boldsymbol{x}})\mathrm{d} \nu_{\mathcal{X}}(\boldsymbol{x})\mathrm{d} \nu_{\mathcal{Y}}(\boldsymbol{y})}  \right ) 
        \\ & = \log\, \left ( \frac{\mathrm{d}\gamma_{\boldsymbol{X} \mid \widehat{\boldsymbol{X}}}(\boldsymbol{x} \mid \widehat{\boldsymbol{x}}) \mathrm{d}\gamma_{\boldsymbol{Y} \mid  \widehat{\boldsymbol{X}}, \widehat{\boldsymbol{Y}}, \boldsymbol{X}}(\boldsymbol{y}\mid \widehat{\boldsymbol{x}}, \widehat{\boldsymbol{y}}, \boldsymbol{x})}{\mathrm{d} \nu_{\mathcal{X}}(\boldsymbol{x})\mathrm{d} \nu_{\mathcal{Y}}(\boldsymbol{y})}  \right ) 
        \\ &= \log\, \left ( \frac{\mathrm{d}\gamma_{\boldsymbol{X} \mid \widehat{\boldsymbol{X}}}(\boldsymbol{x} \mid \widehat{\boldsymbol{x}})}{\mathrm{d} \nu_{\mathcal{X}}(\boldsymbol{x})}  \right )+\log\, \left ( \frac{\mathrm{d}\gamma_{\boldsymbol{Y} \mid  \widehat{\boldsymbol{X}}, \widehat{\boldsymbol{Y}}, \boldsymbol{X}}(\boldsymbol{y} \mid \widehat{\boldsymbol{x}}, \widehat{\boldsymbol{y}}, \boldsymbol{x})}{\mathrm{d} \nu_{\mathcal{Y}}(\boldsymbol{y})}  \right ), 
    \end{aligned}
\end{equation} where the second equality is due to the chain rule decomposition of the densities for both the joint measure $\gamma$ and the empirical distribution $\widehat{\mathbb{P}}$, and $\gamma_{\widehat{\boldsymbol{X}}}(\widehat{\boldsymbol{x}}) =   \widehat{\mathbb{P}}_{\widehat{\boldsymbol{X}}}(\widehat{\boldsymbol{x}})$. The third equality in relation~\eqref{Eq:relative:entropy:expand} holds since $\gamma_{\widehat{\boldsymbol{Y}}\mid \widehat{\boldsymbol{X}}, \boldsymbol{X}}(\widehat{\boldsymbol{y}}\mid \widehat{\boldsymbol{x}}, \boldsymbol{x}) = \gamma_{\widehat{\boldsymbol{Y}} \mid \widehat{\boldsymbol{X}}}(\widehat{\boldsymbol{y}}\mid \widehat{\boldsymbol{x}}) = \widehat{\mathbb{P}}_{\widehat{\boldsymbol{Y}} \mid \widehat{\boldsymbol{X}}}(\widehat{\boldsymbol{y}}\mid \widehat{\boldsymbol{x}})$ under the causal optimal transport setting. 
By the tower property, we have 
\begin{equation}\label{ec-eq-tower}
    \mathbb{E}_{\gamma}[\cdot] = \mathbb{E}_{\widehat{\mathbb{P}}_{\widehat{\boldsymbol{X}}}}\left[\mathbb{E}_{\gamma_{\boldsymbol{X} \mid  \widehat{\boldsymbol{X}}}}\left[\mathbb{E}_{\widehat{\mathbb{P}}_{\widehat{\boldsymbol{Y}} \mid \widehat{\boldsymbol{X}}}}\left[\mathbb{E}_{\gamma_{\boldsymbol{Y} \mid \widehat{\boldsymbol{X}}, \widehat{\boldsymbol{Y}}, \boldsymbol{X}}}[\cdot \mid \widehat{\boldsymbol{X}}, \widehat{\boldsymbol{Y}}, \boldsymbol{X}] \mid \widehat{\boldsymbol{X}}, \boldsymbol{X}\right] \mid \widehat{\boldsymbol{X}}\right]\right].
\end{equation}
From relations~\eqref{Eq:relative:entropy:expand} and~\eqref{ec-eq-tower}, the expectation term on the RHS of the Eq.~\eqref{ec-eq-Rp} can be rewritten as 
\begin{equation}\nonumber
    \begin{aligned}
         & \mathbb{E}_{((\widehat{\boldsymbol{x}},\widehat{\boldsymbol{y}}), (\boldsymbol{x},\boldsymbol{y}))\sim\gamma} \Big[ c_p((\widehat{\boldsymbol{x}},\widehat{\boldsymbol{y}}), (\boldsymbol{x},\boldsymbol{y})) + \epsilon \mathbb{D}_{\mathrm{KL}}\left( \gamma \|  \widehat{\mathbb{P}} \otimes \left ( \nu_{\mathcal{X}}\otimes \nu_{\mathcal{Y} } \right ) \right) \Big] \\ & = \mathbb{E}_{\widehat{\mathbb{P}}_{\widehat{\boldsymbol{X}}}}\Bigg[\mathbb{E}_{\gamma_{\boldsymbol{X} \mid  \widehat{\boldsymbol{X}}}}\Bigg[\mathbb{E}_{\widehat{\mathbb{P}}_{\widehat{\boldsymbol{Y}} \mid \widehat{\boldsymbol{X}}}}\Bigg[\mathbb{E}_{\gamma_{\boldsymbol{Y} \mid \widehat{\boldsymbol{X}}, \widehat{\boldsymbol{Y}}, \boldsymbol{X}}}[c_p((\widehat{\boldsymbol{x}},\widehat{\boldsymbol{y}}), (\boldsymbol{x},\boldsymbol{y})) \\ 
         & \quad \quad \quad \quad \quad \quad \quad \quad \quad \quad \quad + \epsilon \mathbb{D}_{\mathrm{KL}}\left( \gamma \|  \widehat{\mathbb{P}} \otimes \left ( \nu_{\mathcal{X}}\otimes \nu_{\mathcal{Y} } \right ) \right) \mid \widehat{\boldsymbol{X}}, \widehat{\boldsymbol{Y}}, \boldsymbol{X}] \mid \widehat{\boldsymbol{X}}, \boldsymbol{X}\Bigg] \mid \widehat{\boldsymbol{X}}\Bigg]\Bigg] 
         \\ & = \mathbb{E}_{\widehat{\mathbb{P}}_{\widehat{\boldsymbol{X}}}}\Bigg[\mathbb{E}_{\gamma_{\boldsymbol{X} \mid  \widehat{\boldsymbol{X}}}}\Bigg[\mathbb{E}_{\widehat{\mathbb{P}}_{\widehat{\boldsymbol{Y}} \mid \widehat{\boldsymbol{X}}}}\Big[\mathbb{E}_{\gamma_{\boldsymbol{Y} \mid \widehat{\boldsymbol{X}}, \widehat{\boldsymbol{Y}}, \boldsymbol{X}}}[  c_p((\widehat{\boldsymbol{x}},\widehat{\boldsymbol{y}}), (\boldsymbol{x},\boldsymbol{y})) 
         \\ &  + \epsilon\log\, \left ( \frac{\mathrm{d}\gamma_{\boldsymbol{Y} \mid  \widehat{\boldsymbol{X}}, \widehat{\boldsymbol{Y}}, \boldsymbol{X}}(\boldsymbol{y}\mid \widehat{\boldsymbol{x}}, \widehat{\boldsymbol{y}}, \boldsymbol{x})}{\mathrm{d} \nu_{\mathcal{Y}}(\boldsymbol{y})}  \right ) \mid \widehat{\boldsymbol{X}}, \widehat{\boldsymbol{Y}}, \boldsymbol{X}] \mid \widehat{\boldsymbol{X}}, \boldsymbol{X}\Big] + \epsilon \log\, \left ( \frac{\mathrm{d}\gamma_{\boldsymbol{X} \mid \widehat{\boldsymbol{X}}}(\boldsymbol{x} \mid \widehat{\boldsymbol{x}})}{\mathrm{d} \nu_{\mathcal{X}}(\boldsymbol{x})}  \right )\mid \widehat{\boldsymbol{X}}\bigg]\Bigg], 
    \end{aligned}
\end{equation}
Therefore, we have
\[
\begin{aligned}
v_{\mathrm{P}}(\lambda) &= 
\sup_{\gamma \in \Gamma_c(\mathbb{P}, \mathbb{Q})} \Bigg\{  \mathbb{E}_{\widehat{\mathbb{P}}_{\widehat{\boldsymbol{X}}}}\Bigg[\mathbb{E}_{\gamma_{\boldsymbol{X} \mid  \widehat{\boldsymbol{X}}}}\Bigg[\mathbb{E}_{\widehat{\mathbb{P}}_{\widehat{\boldsymbol{Y}} \mid \widehat{\boldsymbol{X}}}}\Bigg[\mathbb{E}_{\gamma_{\boldsymbol{Y} \mid \widehat{\boldsymbol{X}}, \widehat{\boldsymbol{Y}}, \boldsymbol{X}}}\Bigg[ \Psi(f(\boldsymbol{x}),\boldsymbol{y}) -  \lambda c_p((\widehat{\boldsymbol{x}},\widehat{\boldsymbol{y}}), (\boldsymbol{x},\boldsymbol{y})) \\ &  -  \lambda\epsilon\log\, \left ( \frac{\mathrm{d}\gamma_{\boldsymbol{Y} \mid  \widehat{\boldsymbol{X}}, \widehat{\boldsymbol{Y}}, \boldsymbol{X}}(\boldsymbol{y}\mid \widehat{\boldsymbol{x}}, \widehat{\boldsymbol{y}}, \boldsymbol{x})}{\mathrm{d} \nu_{\mathcal{Y}}(\boldsymbol{y})}  \right ) \mid \widehat{\boldsymbol{X}}, \widehat{\boldsymbol{Y}}, \boldsymbol{X} \Bigg] \mid \widehat{\boldsymbol{X}}, \boldsymbol{X}\Bigg]   - \lambda\epsilon \log\, \left ( \frac{\mathrm{d}\gamma_{\boldsymbol{X} \mid \widehat{\boldsymbol{X}}}(\boldsymbol{x} \mid \widehat{\boldsymbol{x}})}{\mathrm{d} \nu_{\mathcal{X}}(\boldsymbol{x})}  \right )\mid \widehat{\boldsymbol{X}}\bigg]\Bigg] \Bigg\}   \Bigg\}. 
\end{aligned}
\]
Similar to relation~\eqref{ec-eq-tower}, the optimization for $\gamma$ can be decomposed to optimize $\gamma_{\boldsymbol{X} \mid  \widehat{\boldsymbol{X}}}$ and $\gamma_{\boldsymbol{Y} \mid \widehat{\boldsymbol{X}}, \widehat{\boldsymbol{Y}}, \boldsymbol{X}}$ (distributions $\widehat{\mathbb{P}}_{\widehat{\boldsymbol{X}}}$ and $\widehat{\mathbb{P}}_{\widehat{\boldsymbol{Y}} \mid \widehat{\boldsymbol{X}}}$ are determined). 
Since each conditional kernel induced by a feasible causal transport plan is feasible for the corresponding pointwise supremum, the conditional objective is bounded above by its pointwise supremum. 
Thus, taking expectations and then the supremum over all feasible transport plans yields 
\[
\begin{aligned}
v_{\mathrm{P}}(\lambda) & \le 
\mathbb{E}_{\widehat{\mathbb{P}}_{\widehat{\boldsymbol{X}}}}\Bigg[ \sup_{\gamma_{\boldsymbol{X} \mid  \widehat{\boldsymbol{X}}}} \, \mathbb{E}_{\gamma_{\boldsymbol{X} \mid  \widehat{\boldsymbol{X}}}}\Bigg[\mathbb{E}_{\widehat{\mathbb{P}}_{\widehat{\boldsymbol{Y}} \mid \widehat{\boldsymbol{X}}}}\Bigg[\sup_{\gamma_{\boldsymbol{Y} \mid \widehat{\boldsymbol{X}}, \widehat{\boldsymbol{Y}}, \boldsymbol{X}}} \, \mathbb{E}_{\gamma_{\boldsymbol{Y} \mid \widehat{\boldsymbol{X}}, \widehat{\boldsymbol{Y}}, \boldsymbol{X}}}\Bigg[ \Psi(f(\boldsymbol{x}),\boldsymbol{y}) -  \lambda c_p((\widehat{\boldsymbol{x}},\widehat{\boldsymbol{y}}), (\boldsymbol{x},\boldsymbol{y})) \\ &  -  \lambda \epsilon\log\, \left ( \frac{\mathrm{d}\gamma_{\boldsymbol{Y} \mid  \widehat{\boldsymbol{X}}, \widehat{\boldsymbol{Y}}, \boldsymbol{X}}(\boldsymbol{y}\mid \widehat{\boldsymbol{x}}, \widehat{\boldsymbol{y}}, \boldsymbol{x})}{\mathrm{d} \nu_{\mathcal{Y}}(\boldsymbol{y})}  \right ) \mid \widehat{\boldsymbol{X}}, \widehat{\boldsymbol{Y}}, \boldsymbol{X} \Bigg] \mid \widehat{\boldsymbol{X}}\Bigg]   - \lambda \epsilon \log\, \left ( \frac{\mathrm{d}\gamma_{\boldsymbol{X} \mid \widehat{\boldsymbol{X}}}(\boldsymbol{x} \mid \widehat{\boldsymbol{x}})}{\mathrm{d} \nu_{\mathcal{X}}(\boldsymbol{x})}  \right )\mid \widehat{\boldsymbol{X}}\bigg]\Bigg] \\
& = \mathbb{E}_{\widehat{\mathbb{P}}_{\widehat{\boldsymbol{X}}}}\Bigg[ \underbrace{\sup_{\gamma_{\boldsymbol{X} \mid  \widehat{\boldsymbol{X}}}} \, \mathbb{E}_{\gamma_{\boldsymbol{X} \mid  \widehat{\boldsymbol{X}}}}\Bigg[\underbrace{\mathbb{E}_{\widehat{\mathbb{P}}_{\widehat{\boldsymbol{Y}} \mid \widehat{\boldsymbol{X}}}}\Bigg[w_1(\widehat{\boldsymbol{x}}, \widehat{\boldsymbol{y}}, \boldsymbol{x}, \lambda) \mid \widehat{\boldsymbol{X}}\Bigg]}_{g(\widehat{\boldsymbol{x}}, \boldsymbol{x}, \lambda)}   - \lambda \epsilon \log\, \left ( \frac{\mathrm{d}\gamma_{\boldsymbol{X} \mid \widehat{\boldsymbol{X}}}(\boldsymbol{x} \mid \widehat{\boldsymbol{x}})}{\mathrm{d} \nu_{\mathcal{X}}(\boldsymbol{x})} \right )\mid \widehat{\boldsymbol{X}}\bigg] \Bigg]}_{w_2(\widehat{\boldsymbol{x}}, \lambda)}, 
\end{aligned}
\]
where functions $w_1(\widehat{\boldsymbol{x}}, \widehat{\boldsymbol{y}}, \boldsymbol{x}, \lambda)$, $g(\widehat{\boldsymbol{x}}, \boldsymbol{x}, \lambda)$, and $w_2(\widehat{\boldsymbol{x}}, \lambda)$ are all measurable for any $\lambda \ge 0$ according to Lemma~\ref{lem-meas}. 
By the Fenchel duality, we have 
\begin{equation}\nonumber
    \sup_{\mathbb{P}} \Big\{ \mathbb{E}_{\boldsymbol{y}\sim \mathbb{P}}[f(\boldsymbol{y})] - \epsilon \mathbb{D}_{\mathrm{KL}}(\mathbb{P}\|\mathbb{Q}) \Big\} = \epsilon \log\,  \int  \exp\left( \frac{f(\boldsymbol{y})}{\epsilon} \right) \mathrm{d} \mathbb{Q}(\boldsymbol{y}) ,
\end{equation}
and it follows that
\begin{equation}\nonumber
    v_{\mathrm{P}}(\lambda) =  \mathbb{E}_{\widehat{\boldsymbol{x}} \sim \widehat{\mathbb{P}}_{\widehat{\boldsymbol{X}}}} \left[ \lambda\epsilon \log\,  \int_{\mathcal{X}} \exp\left( \frac{g(\widehat{\boldsymbol{x}}, \boldsymbol{x}, \lambda)}{\lambda\epsilon} \right) \mathrm{d} \nu_{\mathcal{X}}(\boldsymbol{x})\right]. 
\end{equation}
Hence, according to Eq.~\eqref{vp-le-vd}, we have
\begin{equation}\nonumber
    \begin{aligned}
         v_{\mathrm{P}} & \le \inf_{\lambda \ge 0} \Big \{
        \lambda\rho^p + v_{\mathrm{P}}(\lambda) \Big\} \\
        & = \inf_{\lambda \ge 0} \left\{ \lambda\rho^p + \mathbb{E}_{\widehat{\boldsymbol{x}} \sim \widehat{\mathbb{P}}_{\widehat{\boldsymbol{X}}}} \left[ \lambda\epsilon \log\,  \int_{\mathcal{X}} \exp\left( \frac{g(\widehat{\boldsymbol{x}}, \boldsymbol{x}, \lambda)}{\lambda\epsilon} \right) \mathrm{d} \nu_{\mathcal{X}}(\boldsymbol{x})\right] \right\}
        \\ & = v_{\mathrm{D}}.
    \end{aligned}
\end{equation}
This completes the proof of the weak duality.
\end{proof}

Next, we complete the proof of the strong duality theorem. 
\begin{proof}[Proof of Theorem~\ref{theo-strong-duality}]
    To prove Theorem~\ref{theo-strong-duality}\ref{theo-strong-duality-1}, we first rewrite the constraint of the primal problem~\eqref{causal-sdro} as 
    \begin{equation}\label{re-constraint}
        \mathbb{E}_{((\widehat{\boldsymbol{x}},\widehat{\boldsymbol{y}}), (\boldsymbol{x},\boldsymbol{y}))\sim\gamma} \Big[ c_p((\widehat{\boldsymbol{x}}, \widehat{\boldsymbol{y}}), (\boldsymbol{x},\boldsymbol{y}))  + \epsilon \log\, \left ( \frac{\mathrm{d}\gamma ( (\widehat{\boldsymbol{x}}, \widehat{\boldsymbol{y}}), (\boldsymbol{x},\boldsymbol{y})) }{\mathrm{d}\widehat{\mathbb{P}}(\widehat{\boldsymbol{x}}, \widehat{\boldsymbol{y}})\mathrm{d} \nu_{\mathcal{X}}(\boldsymbol{x})\mathrm{d} \nu_{\mathcal{Y}}(\boldsymbol{y})}  \right )\Big] \le \rho^p.
    \end{equation}
    Based on Assumption~\ref{assum-1}\ref{assum-1-4}, the relation~\eqref{re-constraint} can be reformulated as 
    \begin{equation}\label{re-constraint-2}
        \mathbb{E}_{(\widehat{\boldsymbol{x}},\widehat{\boldsymbol{y}})\sim\widehat{\mathbb{P}}}\mathbb{E}_{(\boldsymbol{x},\boldsymbol{y})\sim\gamma_{(\widehat{\boldsymbol{x}},\widehat{\boldsymbol{y}})}} \Big[ c_p((\widehat{\boldsymbol{x}}, \widehat{\boldsymbol{y}}), (\boldsymbol{x},\boldsymbol{y}))  + \epsilon \log\, \left ( \frac{\mathrm{d}\gamma ( (\widehat{\boldsymbol{x}}, \widehat{\boldsymbol{y}}), (\boldsymbol{x},\boldsymbol{y})) }{\mathrm{d}\widehat{\mathbb{P}}(\widehat{\boldsymbol{x}}, \widehat{\boldsymbol{y}})\mathrm{d} \nu_{\mathcal{X}}(\boldsymbol{x})\mathrm{d} \nu_{\mathcal{Y}}(\boldsymbol{y})}  \right )\Big] \le \rho^p.
    \end{equation}
    We define a kernel probability distribution $\mathcal{K}_{(\widehat{\boldsymbol{x}},\widehat{\boldsymbol{y}}),\epsilon}(\boldsymbol{x},\boldsymbol{y})$ using the kernel distributions $U_{\epsilon}$ and $W_{\epsilon}$ from Eq.~\eqref{QW-distri}:
    \begin{equation}\nonumber
        \mathrm{d}\mathcal{K}_{(\widehat{\boldsymbol{x}},\widehat{\boldsymbol{y}}),\epsilon}(\boldsymbol{x},\boldsymbol{y}) := \mathrm{d}U_{\epsilon}(\boldsymbol{x}-\widehat{\boldsymbol{x}}) \cdot \mathrm{d} W_{\epsilon}(\boldsymbol{y}-\widehat{\boldsymbol{y}}) = \frac{e^{ - c_p((\widehat{\boldsymbol{x}},\widehat{\boldsymbol{y}}), (\boldsymbol{x},\boldsymbol{y})) / \epsilon} }{Z_{\epsilon}(\widehat{\boldsymbol{x}},\widehat{\boldsymbol{y}}) } \cdot  \mathrm{d} \nu_{\mathcal{X} }\left (  \boldsymbol{x} \right ) \mathrm{d} \nu_{\mathcal{Y} }\left (  \boldsymbol{y} \right ),
    \end{equation}
    where $Z_{\epsilon}(\widehat{\boldsymbol{x}},\widehat{\boldsymbol{y}}) := \int_{\mathbb{R}^{d_x} \times \mathbb{R}^{d_y}} e^{ - c_p((\widehat{\boldsymbol{x}},\widehat{\boldsymbol{y}}), (\boldsymbol{x},\boldsymbol{y})) / \epsilon} \mathrm{d} \nu_{\mathcal{X} }\left (  \boldsymbol{x} \right ) \mathrm{d} \nu_{\mathcal{Y} }\left (  \boldsymbol{y} \right )$. Therefore, we have 
    \begin{equation}\nonumber
        \log\, \left(\frac{\mathrm{d}\mathcal{K}_{(\widehat{\boldsymbol{x}},\widehat{\boldsymbol{y}}),\epsilon}(\boldsymbol{x},\boldsymbol{y})}{\mathrm{d}\nu_{\mathcal{X}}(\boldsymbol{x})\cdot \mathrm{d}\nu_{\mathcal{Y}}(\boldsymbol{y})}\right) = -\frac{c_p((\widehat{\boldsymbol{x}},\widehat{\boldsymbol{y}}),(\boldsymbol{x},\boldsymbol{y}))}{\epsilon} - \log\, Z_{\epsilon}(\widehat{\boldsymbol{x}}, \widehat{\boldsymbol{y}}). 
    \end{equation}
    We decompose the logarithm term in the constraint
    \begin{equation}\nonumber
        \begin{aligned}
         \mathbb{E}_{(\boldsymbol{x},\boldsymbol{y})\sim\gamma_{(\widehat{\boldsymbol{x}},\widehat{\boldsymbol{y}})}} & \Big[ \log\, \left(\frac{\mathrm{d}\gamma_{(\widehat{\boldsymbol{x}},\widehat{\boldsymbol{y}})}(\boldsymbol{x},\boldsymbol{y})}{\mathrm{d}\nu_{\mathcal{X}}(\boldsymbol{x})\mathrm{d}\nu_{\mathcal{Y}}(\boldsymbol{y})}\right) \Big] \\
         &= \mathbb{E}_{(\boldsymbol{x},\boldsymbol{y})\sim\gamma_{(\widehat{\boldsymbol{x}},\widehat{\boldsymbol{y}})}} \Big[  \log\, \left(\frac{\mathrm{d}\gamma_{(\widehat{\boldsymbol{x}},\widehat{\boldsymbol{y}})}(\boldsymbol{x},\boldsymbol{y})}{\mathrm{d}\mathcal{K}_{(\widehat{\boldsymbol{x}},\widehat{\boldsymbol{y}}),\epsilon}(\boldsymbol{x},\boldsymbol{y})}\right) + \log\, \left(\frac{\mathrm{d}\mathcal{K}_{(\widehat{\boldsymbol{x}},\widehat{\boldsymbol{y}}),\epsilon}(\boldsymbol{x},\boldsymbol{y})}{\mathrm{d}\nu_{\mathcal{X}}(\boldsymbol{x})\mathrm{d}\nu_{\mathcal{Y}}(\boldsymbol{y})}\right) \Big ] \\
        & = \mathbb{D}_{\mathrm{KL}}\Big(\gamma_{(\widehat{\boldsymbol{x}},\widehat{\boldsymbol{y}})} || \mathcal{K}_{(\widehat{\boldsymbol{x}},\widehat{\boldsymbol{y}}),\epsilon}\Big) -\mathbb{E}_{(\boldsymbol{x},\boldsymbol{y})\sim\gamma_{(\widehat{\boldsymbol{x}},\widehat{\boldsymbol{y}})}} \Big[  \frac{c_p((\widehat{\boldsymbol{x}},\widehat{\boldsymbol{y}}),(\boldsymbol{x},\boldsymbol{y}))}{\epsilon} - \log\, Z_{\epsilon}(\widehat{\boldsymbol{x}}, \widehat{\boldsymbol{y}}) \Big ], 
        \end{aligned}
    \end{equation}
    where $\mathbb{D}_{\mathrm{KL}}\Big(\gamma_{(\widehat{\boldsymbol{x}},\widehat{\boldsymbol{y}})} || \mathcal{K}_{(\widehat{\boldsymbol{x}},\widehat{\boldsymbol{y}}),\epsilon}\Big)$ represents the KL-divergence from distribution $\gamma_{(\widehat{\boldsymbol{x}},\widehat{\boldsymbol{y}})}$ to $\mathcal{K}_{(\widehat{\boldsymbol{x}},\widehat{\boldsymbol{y}}),\epsilon}$. 
    Thus, we reformulate~\eqref{re-constraint-2} as the following equivalent constraint in terms of the KL-divergence
    \begin{equation}\label{re-constraint-kl}
            \epsilon \cdot \mathbb{E}_{(\widehat{\boldsymbol{x}},\widehat{\boldsymbol{y}})\sim\widehat{\mathbb{P}}} \Big[ \mathbb{D}_{\mathrm{KL}}\Big(\gamma_{(\widehat{\boldsymbol{x}},\widehat{\boldsymbol{y}})} || \mathcal{K}_{(\widehat{\boldsymbol{x}},\widehat{\boldsymbol{y}}),\epsilon}\Big) \Big] \le \rho^p + \epsilon \cdot \mathbb{E}_{(\widehat{\boldsymbol{x}},\widehat{\boldsymbol{y}})\sim\widehat{\mathbb{P}}}\Big[\log\, Z_{\epsilon}(\widehat{\boldsymbol{x}}, \widehat{\boldsymbol{y}})\Big] =\bar{\rho},
    \end{equation}
    which implies that the constraint of the problem~\eqref{causal-sdro} is equivalent to constraint~\eqref{re-constraint-kl},

    It is straightforward to prove Theorems~\ref {theo-strong-duality}\ref{theo-strong-duality-1} using the reformulated constraint~\eqref{re-constraint-kl} and the non-negativity of the KL-divergence.  

    For Theorem~\ref{theo-strong-duality}\ref{theo-strong-duality-2}, we first consider there exists $\lambda>0$ satisfying Condition~\ref{condit-1}.    
    According to Lemma~\ref{lem-weak-duality}, we already have $v_{\mathrm{P}} \le v_{\mathrm{D}}$. 
    Thus, we next prove $v_{\mathrm{P}} \ge v_{\mathrm{D}}$. Denote the optimal solution of $v_{\mathrm{D}}$ is $\lambda^*$, we next show the optimal distribution of $v_{\mathrm{P}}$, denoted by $\mathbb{P}^*$, exists. 
    Suppose $\bar{\rho}\ge 0$ is bounded above such that the CSD constraint is binding, i.e., $R_{\epsilon}\left(\widehat{\mathbb{P}},\mathbb{P}^*\right)=\rho$ and $\epsilon \cdot \mathbb{E}_{(\widehat{\boldsymbol{x}},\widehat{\boldsymbol{y}})\sim\widehat{\mathbb{P}}} \Big[ \mathbb{D}_{\mathrm{KL}}\Big(\gamma_{(\widehat{\boldsymbol{x}},\widehat{\boldsymbol{y}})} || \mathcal{K}_{(\widehat{\boldsymbol{x}},\widehat{\boldsymbol{y}}),\epsilon}\Big) \Big] = \bar{\rho}$. Then, by the theorem of complementary slackness, there exists $\lambda^* > 0$. 

    Since the dual problem is convex in $\lambda$, the $\lambda^*$ satisfies the following first-order optimality condition
    \begin{equation}\label{first-order-opt}
        \rho^p + \epsilon \mathbb{E}_{\widehat{\boldsymbol{x}} \sim \widehat{\mathbb{P}}_{\widehat{\boldsymbol{X}}}} \Big[ \log\,  \int_{\mathcal{X}} e^{r(\widehat{\boldsymbol{x}}, \boldsymbol{x})}  \mathrm{d} \nu_{\mathcal{X}}(\boldsymbol{x}) \Big] = \frac{1}{\lambda^*} \mathbb{E}_{\widehat{\boldsymbol{x}} \sim \widehat{\mathbb{P}}_{\widehat{\boldsymbol{X}}}}\Big[  \frac{\int_{\mathcal{X}} e^{r(\widehat{\boldsymbol{x}}, \boldsymbol{x})} t(\widehat{\boldsymbol{x}}, \boldsymbol{x})  \cdot \mathrm{d} \nu_{\mathcal{X}}(\boldsymbol{x})}{\int_{\mathcal{X}} e^{r(\widehat{\boldsymbol{x}}, \boldsymbol{x})}  \mathrm{d} \nu_{\mathcal{X}}(\boldsymbol{x})}   \Big],
    \end{equation}
    where 
    \begin{equation}\label{ec-func-rst}\nonumber
    \begin{aligned}
        & r(\widehat{\boldsymbol{x}}, \boldsymbol{x}) = \frac{g(\widehat{\boldsymbol{x}}, \boldsymbol{x}, \lambda^*)}{\lambda^* \epsilon} = \mathbb{E}_{\widehat{\boldsymbol{y}} \sim  \widehat{\mathbb{P}}_{\widehat{\boldsymbol{Y}}|\widehat{\boldsymbol{X}}=\widehat{\boldsymbol{x}}}}\left[ \log\,  \int _{\mathcal{Y}}  e^{s(\widehat{\boldsymbol{x}},\widehat{\boldsymbol{y}}, \boldsymbol{x},\boldsymbol{y})}  \mathrm{d} \nu_{\mathcal{Y}}(\boldsymbol{y}) \right], \\
        & t(\widehat{\boldsymbol{x}}, \boldsymbol{x}) =   \mathbb{E}_{\widehat{\boldsymbol{y}} \sim  \widehat{\mathbb{P}}_{\widehat{\boldsymbol{Y}}|\widehat{\boldsymbol{X}}=\widehat{\boldsymbol{x}}}}\Big[ \frac{\int _{\mathcal{Y}}  e^{ s(\widehat{\boldsymbol{x}},\widehat{\boldsymbol{y}}, \boldsymbol{x},\boldsymbol{y}) } \cdot \Psi(f(\boldsymbol{x}),\boldsymbol{y})  \cdot \mathrm{d} \nu_{\mathcal{Y}}(\boldsymbol{y})}{\int _{\mathcal{Y}}  e^{s(\widehat{\boldsymbol{x}},\widehat{\boldsymbol{y}}, \boldsymbol{x},\boldsymbol{y})}  \mathrm{d} \nu_{\mathcal{Y}}(\boldsymbol{y})}  \Big], \\
    \text{and} & \\
        & s(\widehat{\boldsymbol{x}},\widehat{\boldsymbol{y}}, \boldsymbol{x},\boldsymbol{y}) = \frac{\Psi(f(\boldsymbol{x}),\boldsymbol{y})-\lambda^* c_p((\widehat{\boldsymbol{x}},\widehat{\boldsymbol{y}}), (\boldsymbol{x},\boldsymbol{y}))}{\lambda^*\epsilon}.
    \end{aligned}
    \end{equation}

    We next construct a distribution $\mathbb{P}_*$, which can be proved to be feasible for the primal problem. We take the transport mapping $\gamma_*$ such that
    \begin{equation}\nonumber
        \frac{\mathrm{d}\gamma_* (  (\widehat{\boldsymbol{x}},\widehat{\boldsymbol{y}}), (\boldsymbol{x},\boldsymbol{y})) }{\mathrm{d}\widehat{\mathbb{P}}(\widehat{\boldsymbol{x}},\widehat{\boldsymbol{y}})\mathrm{d} \nu_{\mathcal{X}}(\boldsymbol{x})\mathrm{d} \nu_{\mathcal{Y}}(\boldsymbol{y})}\propto  e^{r(\widehat{\boldsymbol{x}}, \boldsymbol{x})+s(\widehat{\boldsymbol{x}},\widehat{\boldsymbol{y}}, \boldsymbol{x},\boldsymbol{y})},
    \end{equation}
    Let $\alpha_{\widehat{\boldsymbol{x}}} = \Big(\int_{\mathcal{X}} e^{r(\widehat{\boldsymbol{x}}, \boldsymbol{x})}  \mathrm{d} \nu_{\mathcal{X}}(\boldsymbol{x})\Big)^{-1}$ and $\beta_{\widehat{\boldsymbol{x}}, \widehat{\boldsymbol{y}}, \boldsymbol{x}}=\left(\int _{\mathcal{Y}}e^{s(\widehat{\boldsymbol{x}},\widehat{\boldsymbol{y}}, \boldsymbol{x},\boldsymbol{y})}\mathrm{d} \nu_{\mathcal{Y}}(\boldsymbol{y})\right)^{-1}$, we have
    \begin{equation}\nonumber
        \frac{\mathrm{d}\gamma_* (  (\widehat{\boldsymbol{x}},\widehat{\boldsymbol{y}}), (\boldsymbol{x},\boldsymbol{y})) }{\mathrm{d}\widehat{\mathbb{P}}(\widehat{\boldsymbol{x}},\widehat{\boldsymbol{y}})\mathrm{d} \nu_{\mathcal{X}}(\boldsymbol{x})\mathrm{d} \nu_{\mathcal{Y}}(\boldsymbol{y})} = \alpha_{\widehat{\boldsymbol{x}}} \cdot \beta_{\widehat{\boldsymbol{x}}, \widehat{\boldsymbol{y}}, \boldsymbol{x}} \cdot e^{r(\widehat{\boldsymbol{x}}, \boldsymbol{x}) + s(\widehat{\boldsymbol{x}},\widehat{\boldsymbol{y}}, \boldsymbol{x},\boldsymbol{y})}. 
    \end{equation}

    We verify the feasibility of distribution $\mathbb{P}_*$ by the definition of the CSD, that is
    \begin{equation}\nonumber
    \begin{aligned}
        R_{\epsilon}(\widehat{\mathbb{P}}, \mathbb{P}_*) &= \inf_{\gamma \in \Gamma_c(\widehat{\mathbb{P}}, \mathbb{P}_*)} \mathbb{E}_{((\widehat{\boldsymbol{x}},\widehat{\boldsymbol{y}}), (\boldsymbol{x},\boldsymbol{y}))\sim\gamma} \Big[ c_p((\widehat{\boldsymbol{x}},\widehat{\boldsymbol{y}}), (\boldsymbol{x},\boldsymbol{y})) + \epsilon \log\, \left ( \frac{\mathrm{d}\gamma (  (\widehat{\boldsymbol{x}},\widehat{\boldsymbol{y}}), (\boldsymbol{x},\boldsymbol{y})) }{\mathrm{d}\widehat{\mathbb{P}}(\widehat{\boldsymbol{x}},\widehat{\boldsymbol{y}})\mathrm{d} \nu_{\mathcal{X}}(\boldsymbol{x})\mathrm{d} \nu_{\mathcal{Y}}(\boldsymbol{y})}  \right ) \Big] \\
        & = \inf_{\gamma \in \Gamma_c(\widehat{\mathbb{P}}, \mathbb{P}_*)} \mathbb{E}_{((\widehat{\boldsymbol{x}},\widehat{\boldsymbol{y}}), (\boldsymbol{x},\boldsymbol{y}))\sim\gamma} \Big[ \epsilon \log\, \left ( \frac{e^{c_p((\widehat{\boldsymbol{x}},\widehat{\boldsymbol{y}}), (\boldsymbol{x},\boldsymbol{y}))/\epsilon}\cdot \mathrm{d}\gamma (  (\widehat{\boldsymbol{x}},\widehat{\boldsymbol{y}}), (\boldsymbol{x},\boldsymbol{y})) }{\mathrm{d}\widehat{\mathbb{P}}(\widehat{\boldsymbol{x}},\widehat{\boldsymbol{y}})\mathrm{d} \nu_{\mathcal{X}}(\boldsymbol{x})\mathrm{d} \nu_{\mathcal{Y}}(\boldsymbol{y})}  \right ) \Big] \\
        & \le \mathbb{E}_{((\widehat{\boldsymbol{x}},\widehat{\boldsymbol{y}}), (\boldsymbol{x},\boldsymbol{y}))\sim\gamma_*} \Big[ \epsilon \log\, \left ( \frac{e^{c_p((\widehat{\boldsymbol{x}},\widehat{\boldsymbol{y}}), (\boldsymbol{x},\boldsymbol{y}))/\epsilon}\cdot \mathrm{d}\gamma_* (  (\widehat{\boldsymbol{x}},\widehat{\boldsymbol{y}}), (\boldsymbol{x},\boldsymbol{y})) }{\mathrm{d}\widehat{\mathbb{P}}(\widehat{\boldsymbol{x}},\widehat{\boldsymbol{y}})\mathrm{d} \nu_{\mathcal{X}}(\boldsymbol{x})\mathrm{d} \nu_{\mathcal{Y}}(\boldsymbol{y})}  \right ) \Big] \\
        & =  \mathbb{E}_{((\widehat{\boldsymbol{x}},\widehat{\boldsymbol{y}}), (\boldsymbol{x},\boldsymbol{y}))\sim\gamma_*} \Big[ \epsilon \log\, \left ( e^{c_p((\widehat{\boldsymbol{x}},\widehat{\boldsymbol{y}}), (\boldsymbol{x},\boldsymbol{y}))/\epsilon} \cdot \alpha_{\widehat{\boldsymbol{x}}} \cdot e^{r(\widehat{\boldsymbol{x}}, \boldsymbol{x})} \cdot \beta_{\widehat{\boldsymbol{x}}, \widehat{\boldsymbol{y}}, \boldsymbol{x}} \cdot e^{s(\widehat{\boldsymbol{x}},\widehat{\boldsymbol{y}}, \boldsymbol{x},\boldsymbol{y})} \right ) \Big]  \\ 
        & =  \mathbb{E}_{((\widehat{\boldsymbol{x}},\widehat{\boldsymbol{y}}), (\boldsymbol{x},\boldsymbol{y}))\sim\gamma_*} \Big[ \frac{1}{\lambda^*}\Psi(f(\boldsymbol{x}), \boldsymbol{y}) +  \epsilon r(\widehat{\boldsymbol{x}}, \boldsymbol{x}) + \epsilon \log\, \left (\alpha_{\widehat{\boldsymbol{x}}}\right ) + \epsilon \log\, \left (\beta_{\widehat{\boldsymbol{x}}, \widehat{\boldsymbol{y}}, \boldsymbol{x}}\right )  \Big] \\ 
        & =  
        \frac{1}{\lambda^*} \mathbb{E}_{\widehat{\boldsymbol{x}} \sim \widehat{\mathbb{P}}_{\widehat{\boldsymbol{X}}}}\Big[  \frac{\int_{\mathcal{X}} e^{r(\widehat{\boldsymbol{x}}, \boldsymbol{x})} t(\widehat{\boldsymbol{x}}, \boldsymbol{x})  \cdot \mathrm{d} \nu_{\mathcal{X}}(\boldsymbol{x})}{\int_{\mathcal{X}} e^{r(\widehat{\boldsymbol{x}}, \boldsymbol{x})}  \mathrm{d} \nu_{\mathcal{X}}(\boldsymbol{x})}   \Big] - \epsilon \mathbb{E}_{\widehat{\boldsymbol{x}} \sim \widehat{\mathbb{P}}_{\widehat{\boldsymbol{X}}}} \Big[ \log\,  \int_{\mathcal{X}} e^{r(\widehat{\boldsymbol{x}}, \boldsymbol{x})}  \mathrm{d} \nu_{\mathcal{X}}(\boldsymbol{x}) \Big] \\
        & = \rho^p,
    \end{aligned}
    \end{equation}
    where the inequality relation holds because $\gamma_*$ is a feasible solution in $\Gamma_c(\widehat{\mathbb{P}}, \mathbb{P}_*)$, and the fourth and fifth equalities are obtained by substituting the expression of $\gamma_*$, and the last equality is due to the first-order optimality condition~\eqref{first-order-opt}. Therefore, under Assumption~\ref{assum-1},  the distribution $\mathbb{P}_*$ is feasible for the primal problem. We show that the primal optimal value is lower bounded by the dual optimal value 
    \begin{equation}\label{strong-duality}
        \begin{aligned}
            v_{\mathrm{P}} & \ge \mathbb{E}_{(\boldsymbol{x}, \boldsymbol{y})\sim\mathbb{P}_*}\Big[\Psi(f(\boldsymbol{x}), \boldsymbol{y})\Big] \\
            & = \mathbb{E}_{((\widehat{\boldsymbol{x}},\widehat{\boldsymbol{y}}), (\boldsymbol{x},\boldsymbol{y}))\sim\gamma_*}\Big[\Psi(f(\boldsymbol{x}), \boldsymbol{y})\Big] \\
            & = \int_{(\mathcal{X}\times\mathcal{Y}) \times (\mathcal{X}\times\mathcal{Y})} \Psi(f(\boldsymbol{x}), \boldsymbol{y}) \mathrm{d}\gamma (  (\widehat{\boldsymbol{x}},\widehat{\boldsymbol{y}}), (\boldsymbol{x},\boldsymbol{y})) \\
            & \qquad \qquad \qquad \cdot  \frac{\mathrm{d}\widehat{\mathbb{P}}(\widehat{\boldsymbol{x}},\widehat{\boldsymbol{y}})\mathrm{d} \nu_{\mathcal{X}}(\boldsymbol{x})\mathrm{d} \nu_{\mathcal{Y}}(\boldsymbol{y})}{\mathrm{d}\gamma (  (\widehat{\boldsymbol{x}},\widehat{\boldsymbol{y}}), (\boldsymbol{x},\boldsymbol{y})) } \cdot \frac{\mathrm{d}\gamma (  (\widehat{\boldsymbol{x}},\widehat{\boldsymbol{y}}), (\boldsymbol{x},\boldsymbol{y})) }{\mathrm{d}\widehat{\mathbb{P}}(\widehat{\boldsymbol{x}},\widehat{\boldsymbol{y}})\mathrm{d} \nu_{\mathcal{X}}(\boldsymbol{x})\mathrm{d} \nu_{\mathcal{Y}}(\boldsymbol{y})}\\
            & =  \mathbb{E}_{\widehat{\boldsymbol{x}} \sim \widehat{\mathbb{P}}_{\widehat{\boldsymbol{X}}}}\Big[  \frac{\int_{\mathcal{X}} e^{r(\widehat{\boldsymbol{x}}, \boldsymbol{x})} t(\widehat{\boldsymbol{x}}, \boldsymbol{x})  \cdot \mathrm{d} \nu_{\mathcal{X}}(\boldsymbol{x})}{\int_{\mathcal{X}} e^{r(\widehat{\boldsymbol{x}}, \boldsymbol{x})}  \mathrm{d} \nu_{\mathcal{X}}(\boldsymbol{x})}   \Big] \\
            & =  \lambda^*\rho^p + \lambda^* \epsilon \mathbb{E}_{\widehat{\boldsymbol{x}} \sim \widehat{\mathbb{P}}_{\widehat{\boldsymbol{X}}}} \left[ \log\, \int_{\mathcal{X}} \exp\left( \frac{g(\widehat{\boldsymbol{x}}, \boldsymbol{x}, \lambda^*)}{\lambda^* \epsilon} \right) \mathrm{d} \nu_{\mathcal{X}}(\boldsymbol{x})  \right]  \\
            & = v_{\mathrm{D}}.
        \end{aligned}
    \end{equation}
    According to Lemma~\ref{lem-weak-duality} and Eq.~\eqref{strong-duality}, we have $v_{\mathrm{P}}\ge v_{\mathrm{D}}$ and $v_{\mathrm{P}}\le v_{\mathrm{D}}$. Thus $v_{\mathrm{P}} = v_{\mathrm{D}}$. The proof for Theorem~\ref{theo-strong-duality}\ref{theo-strong-duality-2} is completed. 
    
    If Condition~\ref{condit-1} fails for every $\lambda>0$, then the corresponding exponential moment is infinite, and thus $v_{\mathrm D}=\infty$. Moreover, when $\bar{\rho}>0$, the equivalent conditional KL-divergence ambiguity set in Remark~\ref{remark-KL-connect} has a positive radius, and the lack of exponential integrability implies that the worst-case expected loss is unbounded, i.e., $v_{\mathrm P}=\infty$. Therefore, we have $v_{\mathrm P}=v_{\mathrm D}=\infty$. 
    \end{proof} 

\subsection{Proof of Theorem~\ref{theo-worst-case-distri} in Appendix~\ref{app-sec-disc-worstcase}} \label{ecsec-theo-worst}

\begin{proof}[Proof of Theorem~\ref{theo-worst-case-distri}]
    In the proof of Theorem~\ref{theo-strong-duality}, we have derived a worst-case distribution of $v_{\mathrm{P}}$, that is, 
    \begin{equation}
        \frac{\mathrm{d}\gamma (  (\widehat{\boldsymbol{x}},\widehat{\boldsymbol{y}}), (\boldsymbol{x},\boldsymbol{y})) }{\mathrm{d}\widehat{\mathbb{P}}(\widehat{\boldsymbol{x}},\widehat{\boldsymbol{y}})\mathrm{d} \nu_{\mathcal{X}}(\boldsymbol{x})\mathrm{d} \nu_{\mathcal{Y}}(\boldsymbol{y})} = \alpha_{\widehat{\boldsymbol{x}}} \cdot \beta_{\widehat{\boldsymbol{x}}, \widehat{\boldsymbol{y}}, \boldsymbol{x}} \cdot e^{r(\widehat{\boldsymbol{x}}, \boldsymbol{x}) + s(\widehat{\boldsymbol{x}},\widehat{\boldsymbol{y}}, \boldsymbol{x},\boldsymbol{y})},
    \end{equation}
    where 
    \begin{equation}\nonumber
        \alpha_{\widehat{\boldsymbol{x}}} = \Big(\int_{\mathcal{X}} e^{r(\widehat{\boldsymbol{x}}, \boldsymbol{x})}  \mathrm{d} \nu_{\mathcal{X}}(\boldsymbol{x})\Big)^{-1}, \quad \quad\beta_{\widehat{\boldsymbol{x}}, \widehat{\boldsymbol{y}}, \boldsymbol{x}}=\left(\int _{\mathcal{Y}}e^{s(\widehat{\boldsymbol{x}},\widehat{\boldsymbol{y}}, \boldsymbol{x},\boldsymbol{y})}\mathrm{d} \nu_{\mathcal{Y}}(\boldsymbol{y})\right)^{-1},
    \end{equation}
    and function $r(\widehat{\boldsymbol{x}}, \boldsymbol{x})$ and $s(\widehat{\boldsymbol{x}},\widehat{\boldsymbol{y}}, \boldsymbol{x},\boldsymbol{y})$ are defined in~\eqref{ec-func-rst}. 
    We now prove that $\lambda^*$ is the unique optimal solution of the dual problem, which implies that the worst-case distribution is also unique. 

    Recall that $v(\lambda)$ denotes the objective function for the dual problem in Eq.~\eqref{causal-sdro-dual:whole}, and then we have 
    \begin{equation}\nonumber
        \nabla_{\lambda} v(\lambda) = \rho^p + \epsilon \mathbb{E}_{\widehat{\boldsymbol{x}} \sim \widehat{\mathbb{P}}_{\widehat{\boldsymbol{X}}}} \Big[ \log\,  \int_{\mathcal{X}} e^{r(\lambda, \widehat{\boldsymbol{x}}, \boldsymbol{x})}  \mathrm{d} \nu_{\mathcal{X}}(\boldsymbol{x}) \Big]- \frac{1}{\lambda} \mathbb{E}_{\widehat{\boldsymbol{x}} \sim \widehat{\mathbb{P}}_{\widehat{\boldsymbol{X}}}}\Big[  \frac{\int_{\mathcal{X}} e^{r(\lambda, \widehat{\boldsymbol{x}}, \boldsymbol{x})} t(\lambda, \widehat{\boldsymbol{x}}, \boldsymbol{x})  \cdot \mathrm{d} \nu_{\mathcal{X}}(\boldsymbol{x})}{\int_{\mathcal{X}} e^{r(\lambda, \widehat{\boldsymbol{x}}, \boldsymbol{x})}  \mathrm{d} \nu_{\mathcal{X}}(\boldsymbol{x})}   \Big],
    \end{equation}
    where 
    \begin{equation}\nonumber
    \begin{aligned} 
        & r(\lambda, \widehat{\boldsymbol{x}}, \boldsymbol{x}) = \mathbb{E}_{\widehat{\boldsymbol{y}} \sim  \widehat{\mathbb{P}}_{\widehat{\boldsymbol{Y}}|\widehat{\boldsymbol{X}}=\widehat{\boldsymbol{x}}}}\Big[ \log\,  \int _{\mathcal{Y}}  e^{s(\lambda, \widehat{\boldsymbol{x}},\widehat{\boldsymbol{y}}, \boldsymbol{x},\boldsymbol{y})}  \mathrm{d} \nu_{\mathcal{Y}}(\boldsymbol{y}) \Big], \\
        & t(\lambda, \widehat{\boldsymbol{x}}, \boldsymbol{x}) = \mathbb{E}_{\widehat{\boldsymbol{y}} \sim  \widehat{\mathbb{P}}_{\widehat{\boldsymbol{Y}}|\widehat{\boldsymbol{X}}=\widehat{\boldsymbol{x}}}}\Big[ \frac{\int _{\mathcal{Y}}  e^{ s(\lambda, \widehat{\boldsymbol{x}},\widehat{\boldsymbol{y}}, \boldsymbol{x},\boldsymbol{y}) }  \Psi(f(\boldsymbol{x}),\boldsymbol{y}) \cdot  \mathrm{d} \nu_{\mathcal{Y}}(\boldsymbol{y})}{\int _{\mathcal{Y}} e^{ s(\lambda, \widehat{\boldsymbol{x}},\widehat{\boldsymbol{y}}, \boldsymbol{x},\boldsymbol{y}) }  \mathrm{d} \nu_{\mathcal{Y}}(\boldsymbol{y})}  \Big], \\
        \text{and} \quad  &    \\
       & s(\lambda, \widehat{\boldsymbol{x}},\widehat{\boldsymbol{y}}, \boldsymbol{x},\boldsymbol{y}) = \frac{\Psi(f(\boldsymbol{x}),\boldsymbol{y})-\lambda c_p((\widehat{\boldsymbol{x}},\widehat{\boldsymbol{y}}), (\boldsymbol{x},\boldsymbol{y}))}{\lambda\epsilon}.
    \end{aligned}
    \end{equation}
    For its second-order derivative function, we have
    \begin{equation}\nonumber
    \begin{aligned}
        \nabla_{\lambda}^2 v(\lambda) 
        & = \nabla_{\lambda} \Big[ \epsilon \mathbb{E}_{\widehat{\boldsymbol{x}} \sim \widehat{\mathbb{P}}_{\widehat{\boldsymbol{X}}}} \Big[ \log\, \int_{\mathcal{X}} e^{r(\lambda, \widehat{\boldsymbol{x}}, \boldsymbol{x})}  \mathrm{d} \nu_{\mathcal{X}}(\boldsymbol{x})\Big]\Big] - \nabla_{\lambda} \Big[  \frac{1}{\lambda} \cdot \mathbb{E}_{\widehat{\boldsymbol{x}} \sim \widehat{\mathbb{P}}_{\widehat{\boldsymbol{X}}}}\Big[  \frac{\int_{\mathcal{X}} e^{r\lambda, (\widehat{\boldsymbol{x}}, \boldsymbol{x})} t(\lambda, \widehat{\boldsymbol{x}}, \boldsymbol{x})  \cdot \mathrm{d} \nu_{\mathcal{X}}(\boldsymbol{x})}{\int_{\mathcal{X}} e^{r(\lambda, \widehat{\boldsymbol{x}}, \boldsymbol{x})}  \mathrm{d} \nu_{\mathcal{X}}(\boldsymbol{x})}   \Big]   \Big] \\
        & = - \frac{1}{\lambda^2} \cdot \mathbb{E}_{\widehat{\boldsymbol{x}} \sim \widehat{\mathbb{P}}_{\widehat{\boldsymbol{X}}}}\Big[  \frac{\int_{\mathcal{X}}  e^{r(\lambda, \widehat{\boldsymbol{x}}, \boldsymbol{x}) }t(\lambda, \widehat{\boldsymbol{x}}, \boldsymbol{x}) \cdot  \mathrm{d} \nu_{\mathcal{X}}(\boldsymbol{x})}{\int_{\mathcal{X}}  e^{r(\lambda, \widehat{\boldsymbol{x}}, \boldsymbol{x} ) }  \mathrm{d} \nu_{\mathcal{X}}(\boldsymbol{x})}   \Big] \\
        & \quad \quad \quad \quad \quad \quad - \nabla_{\lambda} \Big[  \frac{1}{\lambda} \cdot \mathbb{E}_{\widehat{\boldsymbol{x}} \sim \widehat{\mathbb{P}}_{\widehat{\boldsymbol{X}}}}\Big[  \frac{\int_{\mathcal{X}} e^{r(\lambda, \widehat{\boldsymbol{x}}, \boldsymbol{x})} t(\lambda, \widehat{\boldsymbol{x}}, \boldsymbol{x})  \cdot \mathrm{d} \nu_{\mathcal{X}}(\boldsymbol{x})}{\int_{\mathcal{X}} e^{r(\lambda, \widehat{\boldsymbol{x}}, \boldsymbol{x})}  \mathrm{d} \nu_{\mathcal{X}}(\boldsymbol{x})}   \Big]   \Big]  \\
        &= - \frac{1}{\lambda}\cdot  \nabla_{\lambda}  \Big[ \mathbb{E}_{\widehat{\boldsymbol{x}} \sim \widehat{\mathbb{P}}_{\widehat{\boldsymbol{X}}}}\Big[  \frac{\int_{\mathcal{X}} e^{r(\lambda,\widehat{\boldsymbol{x}}, \boldsymbol{x}) } t(\lambda, \widehat{\boldsymbol{x}}, \boldsymbol{x}) \cdot  \mathrm{d} \nu_{\mathcal{X}}(\boldsymbol{x})}{\int_{\mathcal{X}}  e^{r(\lambda, \widehat{\boldsymbol{x}}, \boldsymbol{x} ) }  \mathrm{d} \nu_{\mathcal{X}}(\boldsymbol{x})}   \Big] \Big]. 
    \end{aligned}
    \end{equation}
    Here, we have 
    \begin{equation}\nonumber
        \begin{aligned}
            \mathbb{E}_{\widehat{\boldsymbol{x}} \sim \widehat{\mathbb{P}}_{\widehat{\boldsymbol{X}}}}& \Big[  \frac{\int_{\mathcal{X}} e^{r(\lambda,\widehat{\boldsymbol{x}}, \boldsymbol{x}) } t(\lambda, \widehat{\boldsymbol{x}}, \boldsymbol{x}) \cdot  \mathrm{d} \nu_{\mathcal{X}}(\boldsymbol{x})}{\int_{\mathcal{X}}  e^{r(\lambda, \widehat{\boldsymbol{x}}, \boldsymbol{x} ) }  \mathrm{d} \nu_{\mathcal{X}}(\boldsymbol{x})}   \Big] \\
            & = - \frac{1}{\lambda^2 \epsilon}\cdot \mathbb{E}_{\widehat{\boldsymbol{x}} \sim \widehat{\mathbb{P}}_{\widehat{\boldsymbol{X}}}}\Big[\frac{ \int_{\mathcal{X}} e^{r(\lambda,\widehat{\boldsymbol{x}}, \boldsymbol{x}) }\left(t^2(\lambda, \widehat{\boldsymbol{x}}, \boldsymbol{x})+ u(\lambda, \widehat{\boldsymbol{x}}, \boldsymbol{x})\right)  \mathrm{d} \nu_{\mathcal{X}}(\boldsymbol{x}) \cdot \int_{\mathcal{X}} e^{r(\lambda,\widehat{\boldsymbol{x}}, \boldsymbol{x}) }  \mathrm{d} \nu_{\mathcal{X}}(\boldsymbol{x})}{\left ( \int_{\mathcal{X}}  e^{r(\lambda, \widehat{\boldsymbol{x}}, \boldsymbol{x} ) }  \mathrm{d} \nu_{\mathcal{X}}(\boldsymbol{x}) \right )^2 } \\
        & \quad \quad \quad \quad \quad \quad \quad -  \frac{\left ( \int_{\mathcal{X}} e^{r(\lambda,\widehat{\boldsymbol{x}}, \boldsymbol{x}) } t(\lambda, \widehat{\boldsymbol{x}}, \boldsymbol{x}) \cdot \mathrm{d} \nu_{\mathcal{X}}(\boldsymbol{x}) \right )^2 }{\left ( \int_{\mathcal{X}}  e^{r(\lambda, \widehat{\boldsymbol{x}}, \boldsymbol{x} ) }  \mathrm{d} \nu_{\mathcal{X}}(\boldsymbol{x}) \right )^2}\Big], 
        \end{aligned}
    \end{equation}
    where 
    \begin{equation}\nonumber
    \begin{aligned}
        u(\lambda, \widehat{\boldsymbol{x}}, \boldsymbol{x}) = \mathbb{E}_{\widehat{\boldsymbol{y}} \sim  \widehat{\mathbb{P}}_{\widehat{\boldsymbol{Y}}|\widehat{\boldsymbol{X}}=\widehat{\boldsymbol{x}}}}\Big [ & \frac{\int _{\mathcal{Y}} e^{ s(\lambda, \widehat{\boldsymbol{x}},\widehat{\boldsymbol{y}}, \boldsymbol{x},\boldsymbol{y}) } \Psi^2(f(\boldsymbol{x}),\boldsymbol{y}) \mathrm{d} \nu_{\mathcal{Y}}(\boldsymbol{y}) \cdot \int _{\mathcal{Y}} e^{ s(\lambda, \widehat{\boldsymbol{x}},\widehat{\boldsymbol{y}}, \boldsymbol{x},\boldsymbol{y}) }  \mathrm{d} \nu_{\mathcal{Y}}(\boldsymbol{y}) }{\left ( \int _{\mathcal{Y}} e^{ s(\lambda, \widehat{\boldsymbol{x}},\widehat{\boldsymbol{y}}, \boldsymbol{x},\boldsymbol{y}) }  \mathrm{d} \nu_{\mathcal{Y}}(\boldsymbol{y}) \right )^2 } \\
        & - \frac{ \left ( \int _{\mathcal{Y}}  e^{ s(\lambda, \widehat{\boldsymbol{x}},\widehat{\boldsymbol{y}}, \boldsymbol{x},\boldsymbol{y}) }  \Psi(f(\boldsymbol{x}),\boldsymbol{y}) \cdot  \mathrm{d} \nu_{\mathcal{Y}}(\boldsymbol{y}) \right )^2}{\left ( \int _{\mathcal{Y}} e^{ s(\lambda, \widehat{\boldsymbol{x}},\widehat{\boldsymbol{y}}, \boldsymbol{x},\boldsymbol{y}) }  \mathrm{d} \nu_{\mathcal{Y}}(\boldsymbol{y}) \right )^2}\Big ] .
    \end{aligned}
    \end{equation}
    According to the Cauchy-Schwarz inequality, we have $u(\lambda, \widehat{\boldsymbol{x}}, \boldsymbol{x}) \ge 0$ for any $\lambda \ge 0$, $\widehat{\boldsymbol{x}} \sim \widehat{\mathbb{P}}_{\widehat{\boldsymbol{X}}}$ and $\boldsymbol{x} \sim\nu_{\mathcal{X}}$. Thus, we have
    \begin{equation}\label{eceq-theo2-inequality}
        \mathbb{E}_{\widehat{\boldsymbol{x}} \sim \widehat{\mathbb{P}}_{\widehat{\boldsymbol{X}}}}\Big[\int_{\mathcal{X}} u(\lambda, \widehat{\boldsymbol{x}}, \boldsymbol{x})\mathrm{d} \nu_{\mathcal{X}}(\boldsymbol{x}) \cdot \int_{\mathcal{X}} e^{r(\lambda,\widehat{\boldsymbol{x}}, \boldsymbol{x}) }  \mathrm{d} \nu_{\mathcal{X}}(\boldsymbol{x})\Big] \ge 0,
    \end{equation}
    and it follows that 
\begin{equation}\nonumber
    \begin{aligned}
        \nabla_{\lambda}^2 v(\lambda) 
        & \ge \frac{1}{\lambda^3 \epsilon}\cdot \mathbb{E}_{\widehat{\boldsymbol{x}} \sim \widehat{\mathbb{P}}_{\widehat{\boldsymbol{X}}}}\Big[\frac{ \int_{\mathcal{X}} e^{r(\lambda,\widehat{\boldsymbol{x}}, \boldsymbol{x}) } t^2(\lambda, \widehat{\boldsymbol{x}}, \boldsymbol{x})  \cdot \mathrm{d} \nu_{\mathcal{X}}(\boldsymbol{x}) \int_{\mathcal{X}} e^{r(\lambda,\widehat{\boldsymbol{x}}, \boldsymbol{x}) }  \mathrm{d} \nu_{\mathcal{X}}(\boldsymbol{x})}{\left ( \int_{\mathcal{X}}  e^{r(\lambda, \widehat{\boldsymbol{x}}, \boldsymbol{x} ) }  \mathrm{d} \nu_{\mathcal{X}}(\boldsymbol{x}) \right )^2 } \\
        & \quad \quad \quad \quad \quad \quad \quad - \frac{\left ( \int_{\mathcal{X}} e^{r(\lambda,\widehat{\boldsymbol{x}}, \boldsymbol{x}) } t(\lambda, \widehat{\boldsymbol{x}}, \boldsymbol{x}) \cdot  \mathrm{d} \nu_{\mathcal{X}}(\boldsymbol{x}) \right )^2}{\left ( \int_{\mathcal{X}}  e^{r(\lambda, \widehat{\boldsymbol{x}}, \boldsymbol{x} ) }  \mathrm{d} \nu_{\mathcal{X}}(\boldsymbol{x}) \right )^2 } \Big]\ge 0, 
    \end{aligned}
    \end{equation}
    where the first inequality follows from relation~\eqref{eceq-theo2-inequality}, and the second follows from the Cauchy-Schwarz inequality. 

    Therefore, for any $\lambda>0$, we have $\nabla_{\lambda}^2 v(\lambda)\ge 0$, and equality holds if and only if the function $\Psi(f(\cdot), \cdot)$ is a constant function. Thus, the strict convexity of the dual problem~\eqref{causal-sdro-dual} holds for the dual objective, and it implies the uniqueness of $\lambda^*$. 
\end{proof}

\subsection{Proof of Proposition~\ref{propos-lip-con} in Appendix~\ref{app-subsec-algo-property}} \label{ecsec-prop-lip}

\begin{proof} [Proof of Proposition~\ref{propos-lip-con}]
    Under Assumption~\ref{assumption-2}\ref{assumption-2-bound}, the function $\Psi$ is bounded by compact set $\left[0,B\right]$. Next, we analyze the properties of functions $t_3$, $t_2$, and $t_1$ in sequence. 
\begin{itemize}
    \item 
    \textbf{For function $t_3$}, we first define a vector-valued function $\boldsymbol{u} (\boldsymbol{\theta}; \widehat{\boldsymbol{x}}, \boldsymbol{\xi}_1, \boldsymbol{\xi}_2 ): \Theta \to [0, B]^N$
    \begin{equation}\nonumber
        \boldsymbol{u} (\boldsymbol{\theta}; \widehat{\boldsymbol{x}}, \boldsymbol{\xi}_1, \boldsymbol{\xi}_2 )= \Big[ \Psi(f_{\boldsymbol{\theta}}(\widehat{\boldsymbol{x}}+ \boldsymbol{\xi}_1),\widehat{\boldsymbol{y}}_1(\widehat{\boldsymbol{x}})+ \boldsymbol{\xi}_2), \cdots , \Psi(f_{\boldsymbol{\theta}}(\widehat{\boldsymbol{x}}+ \boldsymbol{\xi}_1),\widehat{\boldsymbol{y}}_{N}(\widehat{\boldsymbol{x}}) + \boldsymbol{\xi}_2) \Big]^\top. 
    \end{equation}
    For any $\boldsymbol{\theta}, \boldsymbol{\theta}^{\prime} \in \Theta$ and samples $\widehat{\boldsymbol{x}}, \boldsymbol{\xi}_1, \boldsymbol{\xi}_2$, and $\{ \boldsymbol{y} \}_{i=1}^N$, we write $\boldsymbol{u} (\boldsymbol{\theta}) := \boldsymbol{u}(\boldsymbol{\theta}; \widehat{\boldsymbol{x}}, \boldsymbol{\xi}_1, \boldsymbol{\xi}_2 )$ and $\boldsymbol{u}(\boldsymbol{\theta}^{\prime}) := \boldsymbol{u}(\boldsymbol{\theta}^{\prime}; \widehat{\boldsymbol{x}}, \boldsymbol{\xi}_1, \boldsymbol{\xi}_2 )$ for brevity. 
    Then, under Assumption~\ref{assumption-2}\ref{assumption-2-lip-con}, we have 
    \begin{equation} \label{Eq-function-u-lip}
        \|\boldsymbol{u}(\boldsymbol{\theta}) - \boldsymbol{u}(\boldsymbol{\theta}^\prime)\|^2 = \sum_{i \in [N]} |u_i(\boldsymbol{\theta}) - u_i(\boldsymbol{\theta}^{ \prime})|^2 \le N L_{\boldsymbol{\theta}}^2 \cdot \| \boldsymbol{\theta} - \boldsymbol{\theta}^\prime \|^2. 
    \end{equation}
    
    For each $i \in \left[N\right]$, we have 
    \begin{equation}\nonumber
        \frac{\partial \Big[t_3(\boldsymbol{u}(\boldsymbol{\theta}))\Big]_i}{\partial \boldsymbol{u}(\boldsymbol{\theta})} = \frac{1}{\lambda\epsilon} \exp(u_i(\boldsymbol{\theta}) / \lambda\epsilon) \le  \frac{1}{\lambda\epsilon} \exp(B / \lambda\epsilon), 
    \end{equation}
    which implies that 
    \begin{equation}\nonumber
        \left| \Big[t_3(\boldsymbol{u}(\boldsymbol{\theta}))\Big]_i - \Big[t_3(\boldsymbol{u}(\boldsymbol{\theta}^{\prime}))\Big]_i \right| \le L_3^{\prime}   \left| u_i(\boldsymbol{\theta}) - u_i(\boldsymbol{\theta}^{\prime}) \right| ,
    \end{equation}
    where $L_3^{\prime} = \frac{1}{\lambda\epsilon} \exp(B / \lambda\epsilon)$ and $u_i(\boldsymbol{\theta}), u_i(\boldsymbol{\theta}^{\prime})$ are the $i$-th elements of $\boldsymbol{u}(\boldsymbol{\theta})$ and $\boldsymbol{u}(\boldsymbol{\theta}^{\prime})$, respectively. 
    Hence, we obtain
    \begin{equation}\label{Eq-function-t3-lip}
    \begin{aligned}
        \| t_3(\boldsymbol{u}(\boldsymbol{\theta})) - t_3(\boldsymbol{u}(\boldsymbol{\theta}^{\prime})) \|^2  & = \sum_{i\in \left[N\right]} \Big| \Big[t_3(\boldsymbol{u}(\boldsymbol{\theta}))\Big]_i - \Big[t_3(\boldsymbol{u}(\boldsymbol{\theta}^{\prime}))\Big]_i \Big| ^2 \\
        & \le (L_3^{\prime})^2  \sum_{i\in \left[N\right]} \left| u_i(\boldsymbol{\theta}) - u_i(\boldsymbol{\theta}^{\prime}) \right|^2 = (L_3^{\prime})^2 \| \boldsymbol{u}(\boldsymbol{\theta}) - \boldsymbol{u}(\boldsymbol{\theta}^{\prime}) \|^2 \\ 
        & \le L_3^2 \cdot \| \boldsymbol{\theta} - \boldsymbol{\theta}^\prime \|^2, 
    \end{aligned}
    \end{equation}
    where $L_3 = \sqrt{N} L_{\boldsymbol{\theta}} L_3^\prime = \frac{\sqrt{N} L_{\boldsymbol{\theta}}}{\lambda\epsilon} \exp(B / \lambda\epsilon)$, and the final inequality is due to Eq.~\eqref{Eq-function-u-lip}. This result shows that the function $t_3$ is $L_3$-Lipschitz continuous.  
    
    According to the second-order derivation of the function $t_3$, its Jacobian matrix is given by
    \begin{equation}\nonumber
        J_{\boldsymbol{\theta}} t_3(\boldsymbol{\theta}) = T(\boldsymbol{\theta}) J_{\boldsymbol{\theta}}\boldsymbol{u}(\boldsymbol{\theta}), 
    \end{equation}
    where $T(\boldsymbol{\theta}) := J_{\boldsymbol{u}}t_3(\boldsymbol{u}(\boldsymbol{\theta})) = \text{diag} \Big(\frac{1}{\lambda\epsilon} \exp(u_1(\boldsymbol{\theta}) / \lambda\epsilon), \cdot, \frac{1}{\lambda\epsilon} \exp(u_N(\boldsymbol{\theta}) / \lambda\epsilon)\Big)$ is an $N$-dimensional diagonal matrix.
    According to Eq.~\eqref{Eq-function-t3-lip}, we have $\| T(\boldsymbol{\theta})\|_{\text{op}} \le L_3^\prime$ and 
    \begin{equation}\nonumber
    \begin{aligned}
        \| T(\boldsymbol{\theta}) - T(\boldsymbol{\theta}^\prime)\|_{\text{op}} & = \max_{j \in [N]} \Big |\frac{\exp(u_j(\boldsymbol{\theta}) / \lambda\epsilon)}{\lambda\epsilon} - \frac{\exp(u_j(\boldsymbol{\theta}^{\prime}) / \lambda\epsilon)}{\lambda\epsilon} \Big | \\
        & \le S_3^\prime \cdot \max_{j \in [N]} \Big | u_j(\boldsymbol{\theta}) - u_j(\boldsymbol{\theta}^{\prime}) \Big | \\
        & \le S_3^\prime \cdot L_{\boldsymbol{\theta}} \| \boldsymbol{\theta} - \boldsymbol{\theta}^{\prime} \|,
    \end{aligned}
    \end{equation}
    where the first inequality is because the exponential function is Lipschitz continuous on the interval $[0, B/\lambda\epsilon]$ and $S_3^\prime = \frac{1}{(\lambda\epsilon)^2} \exp(B / \lambda\epsilon)$, and the second inequality follows from Eq.~\eqref{Eq-function-u-lip}, i.e., $|u_i(\boldsymbol{\theta}) - u_i(\boldsymbol{\theta}^{ \prime})| \le L_{\boldsymbol{\theta}} \cdot \| \boldsymbol{\theta} - \boldsymbol{\theta}^\prime \|$ for any $i \in [N]$. 
    
    In Assumption~\ref{assumption-2}\ref{assumption-2-lip-con}, for fixed $\boldsymbol{x}$ and $\boldsymbol{y}$, the Lipschitz smoothness implies that 
    \begin{equation} \nonumber
        \left\| \nabla_{{\boldsymbol{\theta}}} u_i(\boldsymbol{\theta})-\nabla_{{\boldsymbol{\theta}}} u_i\left(\boldsymbol{\theta}^{\prime}\right)\right\|_{2} \leq S_{\boldsymbol{\theta}}\left\|\boldsymbol{\theta}-\boldsymbol{\theta}^{\prime}\right\|_{2}, \quad \forall i \in [N]. 
    \end{equation}
    Then, for $J_{\boldsymbol{\theta}}\boldsymbol{u}(\boldsymbol{\theta})$, we have 
    \begin{equation}\nonumber
        \|J_{\boldsymbol{\theta}}\boldsymbol{u}(\boldsymbol{\theta})\|_{\text{F}}^2 = \sum_{i \in [N]} \| \nabla_{\theta} u_i(\boldsymbol{\theta}) \|_2^2 \le N \cdot L_{\boldsymbol{\theta}}^2, 
    \end{equation}
    \begin{equation}\nonumber
        \begin{aligned}
        \left\|J_{\theta} \boldsymbol{u}(\boldsymbol{\theta})-J_{\theta} \boldsymbol{u}\left(\boldsymbol{\theta}^{\prime}\right)\right\|_{\mathrm{F}}^{2} =\sum_{j=1}\left\|\nabla_{\theta} u_{j}(\boldsymbol{\theta})-\nabla_{\theta} u_{j}\left(\boldsymbol{\theta}^{\prime}\right)\right\|_{2}^{2} \leq N S_{\boldsymbol{\theta}}^{2} \cdot \left\|\boldsymbol{\theta}-\boldsymbol{\theta}^{\prime}\right\|_{2}^{2} .
        \end{aligned}
    \end{equation}
    Therefore, 
    \begin{equation} \nonumber
        \begin{aligned}
            \| J_{\boldsymbol{\theta}} t_3(\boldsymbol{\theta}) - J_{\boldsymbol{\theta}} t_3(\boldsymbol{\theta}^\prime)\|_{\text{F}} & = \|  T(\boldsymbol{\theta}) \Big( J_{\theta} \boldsymbol{u}(\boldsymbol{\theta})-J_{\theta} \boldsymbol{u}\left(\boldsymbol{\theta}^{\prime}\right)\Big) + \Big (  T(\boldsymbol{\theta}) - T(\boldsymbol{\theta}^\prime)\Big ) J_{\boldsymbol{\theta}}\boldsymbol{u}(\boldsymbol{\theta}^\prime)\|_{\text{F}} \\
            & \le \left\| T(\boldsymbol{\theta}) \Big( J_{\theta} \boldsymbol{u}(\boldsymbol{\theta})-J_{\theta} \boldsymbol{u}\left(\boldsymbol{\theta}^{\prime}\right)\Big) \right\|_{\mathrm{F}} + \left\| \Big (  T(\boldsymbol{\theta}) - T(\boldsymbol{\theta}^\prime)\Big ) J_{\boldsymbol{\theta}}\boldsymbol{u}(\boldsymbol{\theta}^\prime) \right\|_{\mathrm{F}} \\
            & \le \| T(\boldsymbol{\theta})\|_{\text{op}} \left\|J_{\theta} \boldsymbol{u}(\boldsymbol{\theta})-J_{\theta} \boldsymbol{u}\left(\boldsymbol{\theta}^{\prime}\right)\right\|_{\mathrm{F}} + \| T(\boldsymbol{\theta}) - T(\boldsymbol{\theta}^\prime)\|_{\text{op}} \|J_{\boldsymbol{\theta}}\boldsymbol{u}(\boldsymbol{\theta})\|_{\text{F}} \\
            & \le L_3^\prime \cdot \sqrt{N} S_{\boldsymbol{\theta}} \cdot \left\|\boldsymbol{\theta}-\boldsymbol{\theta}^{\prime}\right\|_{2} + \sqrt{N} \cdot L_{\boldsymbol{\theta}} S_3^\prime \cdot L_{\boldsymbol{\theta}} \| \boldsymbol{\theta}  - \boldsymbol{\theta}^{\prime} \|_2 \\
            & = S_3 \cdot \| \boldsymbol{\theta}  - \boldsymbol{\theta}^{\prime} \|_2 , 
        \end{aligned}
    \end{equation}
    where the first inequality is due to the triangle inequality, the second inequality is due to the property of matrix norm, i.e., $\| \boldsymbol{A} \boldsymbol{B}\|_{\text{F}} \le \| \boldsymbol{A}\|_{\text{op}} \|\boldsymbol{B}\|_{\text{F}}$, and
    \[
        S_3 := \sqrt{N} \Big(L_3^\prime S_{\boldsymbol{\theta}} +  L_{\boldsymbol{\theta}} S_3^\prime \Big) = \sqrt{N} e^{B /\lambda\epsilon} \Big( \frac{S_{\boldsymbol{\theta}}}{\lambda\epsilon} + \frac{L_{\boldsymbol{\theta}}^2}{\lambda^2 \epsilon^2} \Big).  
    \]
    For any $\boldsymbol{\theta} \in \Theta$, we have 
    \[
        \| J_{\boldsymbol{\theta}} t_3(\boldsymbol{\theta}) \|_{\text{F}} \le \| T(\boldsymbol{\theta})\|_{\text{op}} \left\|J_{\theta} \boldsymbol{u}(\boldsymbol{\theta})\right\|_{\mathrm{F}} \le \sqrt{N} L_{\boldsymbol{\theta}} L_3^\prime, 
    \]
    and thus
    \[
        \mathbb{E}\Big[ \| J_{\boldsymbol{\theta}} t_3(\boldsymbol{\theta})\|_{\text{F}}^2 \Big] \le N L_{\boldsymbol{\theta}}^2 (L_3^\prime)^2 = C_3^2, 
    \]
    where $C_3 = L_3$, which implies that the Jacobian of $t_3$ with respect to $\boldsymbol{\theta}$ has a uniformly bounded second moment. 
    
    According to a fundamental result in \citet{casella2024statistical}, the variance of a bounded random variable (or vector) is finite. Thus, to prove that the function $t_3$ has finite variance, it suffices to show that its norm is bounded. According to Assumption~\ref{assumption-2}\ref{assumption-2-bound}, given $\widehat{\boldsymbol{x}}$, we have 
    \begin{equation}\nonumber
        \| t_3(\boldsymbol{\theta}; \widehat{\boldsymbol{x}}, \boldsymbol{\xi}_1, \boldsymbol{\xi}_2 ) \|^2  = \sum_{i=1}^{N}\Big( \Big[t_3(\boldsymbol{u}(\boldsymbol{\theta}))\Big]_i\Big)^2 \le N\exp\Big(2B/\lambda\epsilon\Big),
    \end{equation}
    and it follows that $\sigma_3^2 = \sup_{\boldsymbol{\theta} \in \Theta, \widehat{\boldsymbol{x}}, \boldsymbol{\xi}_1} \mathbb{V}_{\boldsymbol{\xi}_2} \Big( t_3\Big( \boldsymbol{\theta}; \widehat{\boldsymbol{x}}, \boldsymbol{\xi}_1, \boldsymbol{\xi}_2  \Big)\Big) < \infty$. 
    
    \item \textbf{For function $t_2$}, according to the analyses for function $t_3$, its domain is also bounded by $\boldsymbol{v} \in \Big[1, \exp(B/\lambda\epsilon)\Big]^{N}$. For brevity, we denote $t_2(\boldsymbol{v}; \widehat{\boldsymbol{x}}, \boldsymbol{\xi}_1)$ as $t_2(\boldsymbol{v})$, and denote $p (\widehat{\boldsymbol{y}}_i \mid \widehat{\boldsymbol{x}})$ as $p_i$ for each $i \in \left[N\right]$. Since $\sum_{i=1}^{N} p_i = 1$, we can rewrite $t_2$ and obtain its range:  
    \begin{equation}\nonumber
        t_2 (\boldsymbol{v}) =  \exp\Big(\sum_{i=1}^{N} p_i \cdot \log\, \left( v_i \right) \Big) \in \Big[ 1, \exp(B/\lambda\epsilon) \Big].
    \end{equation}
    Therefore, we obtain 
    \begin{equation}\label{t2-ana}
    \begin{aligned}
        \| \nabla_{} t_2 (\boldsymbol{v}) \|^2 &= \sum_{i=1}^{N} \Big(\frac{\partial t_2 (\boldsymbol{v})}{\partial v_i} \Big)^2 =\sum_{i=1}^{N}\Big(t_2 (\boldsymbol{v}) \cdot \frac{p_i}{v_i} \Big)^2 \\
        & \le \exp\Big(2B/\lambda\epsilon\Big).
    \end{aligned}
    \end{equation}
    where the inequality arises from the domain and range of the function $t_2$ and $\sum_{j=1}^{N}p_j^2 \le 1$.
    Let $L_2 = C_2 = \exp (B/\lambda\epsilon )$. From Eq.~\eqref{t2-ana}, the function $t_2$ is $L_2$-Lipschitz continuous and has bounded stochastic gradients in expectation, that is, $\mathbb{E}\Big[\|\nabla t_2 (\boldsymbol{v}; \widehat{\boldsymbol{x}}, \boldsymbol{\xi}_1) \|^2\Big] \le C_2^2$, and its variance is also finite.

    We now show that the differentiable function $t_2$ is Lipschitz smooth. By Taylor's theorem, for any $\boldsymbol{v}_1, \boldsymbol{v}_2$ in the domain of $t_2$, there exists a point $\boldsymbol{c}$ on the line segment connecting $\boldsymbol{v}_1$ and $\boldsymbol{v}_2$ such that 
    \begin{equation}\nonumber
        t_2(\boldsymbol{v}_1) = t_2(\boldsymbol{v}_2) + \nabla t_2(\boldsymbol{v}_2)^{\top}(\boldsymbol{v}_1 - \boldsymbol{v}_2) + \frac{1}{2}(\boldsymbol{v}_1 - \boldsymbol{v}_2)^{\top} \nabla^2 t_2(\boldsymbol{c}) (\boldsymbol{v}_1 - \boldsymbol{v}_2).
    \end{equation}
    Thus, we have 
    \begin{equation}\label{t2-smooth-main}
        \begin{aligned}
            \Big|t_2(\boldsymbol{v}_1) - t_2(\boldsymbol{v}_2) - \nabla t_2(\boldsymbol{v}_2)^{\top}(\boldsymbol{v}_1-\boldsymbol{v}_2)\Big| & = \Big|\frac{1}{2} \cdot (\boldsymbol{v}_1-\boldsymbol{v}_2)^\top  \nabla^2 t_2(\boldsymbol{c}) (\boldsymbol{v}_1-\boldsymbol{v}_2) \Big| \\
            & \le \frac{1}{2} \| \nabla^2 t_2(\boldsymbol{c}) \|_{2} \cdot  \| (\boldsymbol{v}_1-\boldsymbol{v}_2) \|^2,
        \end{aligned}
    \end{equation}
    where the inequality is due to the Cauchy-Schwarz inequality. According to the definition of Lipschitz smoothness, if $\| \nabla^2 t_2(\boldsymbol{c}) \|_{2}$ is bounded by a constant $S_2$, that is, $\| \nabla^2 t_2(\boldsymbol{c}) \|_{2} \le S_2$, then the function $t_2$ is $S_2$-Lipschitz smooth. 
    Denote the $j$-th element in vector $\boldsymbol{c}$ as $c_j$ for any $j \in \left[N\right]$. For any $j,k \in \left[N\right]$, the $(j,k)$-th element of the Hessian Matrix $\nabla^2 t_2(\boldsymbol{c})$ is given by
    \begin{equation}\nonumber
        (\nabla^2 t_2(\boldsymbol{c}))_{jk} = \frac{\partial^2 t_2(\boldsymbol{c})}{\partial c_k \partial c_j} =
    \begin{cases}
    t_2(\boldsymbol{c}) \cdot \frac{p_j p_k}{c_j c_k}, & \text{if } j \neq k; \\
    t_2(\boldsymbol{c}) \cdot \left( \frac{p_j^2}{c_j^2} - \frac{p_j}{c_j^2} \right), & \text{if } j = k.
    \end{cases}
    \end{equation}
    Therefore, we have 
    \begin{equation}\label{t2-smooth-op}\nonumber
        \begin{aligned}
            \| \nabla^2 t_2(\boldsymbol{c}) \|_{2}^2 & \le \| \nabla^2 t_2(\boldsymbol{c}) \|_{\text{F}}^2 \\
            & =  \sum_{j=1}^{N} \Big( t_2(\boldsymbol{c}) \left( \frac{p_j^2}{c_j^2} - \frac{p_j}{c_j^2} \right)\Big)^2 + \sum_{j\neq k}\Big(t_2(\boldsymbol{c}) \frac{p_j p_k}{c_j c_k} \Big)^2 \\
            & \le \sum_{j=1}^{N} \Big( e^{\frac{B}{\lambda\epsilon}} \left( \frac{p_j-p_j^2}{1} \right)\Big)^2 + \sum_{j\neq k}\Big(e^{\frac{B}{\lambda\epsilon}}  \frac{p_j p_k}{1} \Big)^2 \\
            & \le e^{\frac{2B}{\lambda\epsilon}}  \Big[ \sum_{j=1}^{N}p_j^2 + \Big( \sum_{j=1}^{N}p_j^2\Big)^2  \Big] \\
            & \le 2\cdot e^{\frac{2B}{\lambda\epsilon}},
        \end{aligned}
    \end{equation}
    where the first inequality is because the Frobenius norm $\|\cdot \|_{\text{F}}$ is an upper bound of the spectral norm for a matrix, the second inequality is due to the domain and range of function $t_2$, the third inequality is due to $p_j \in \left[0,1\right]$ for any $j\in \left[N\right]$ and the Cauchy-Schwarz inequality, and the last inequality is due to $\sum_{j=1}^{N}p_j^2 \le 1$. Let $S_2 = \sqrt{2}\exp(B/\lambda\epsilon)$, based on relation~\eqref{t2-smooth-main}, we have
    \begin{equation}\nonumber
        \Big|t_2(\boldsymbol{v}_1) - t_2(\boldsymbol{v}_2) - \nabla t_2(\boldsymbol{v}_2)^{\top}(\boldsymbol{v}_1-\boldsymbol{v}_2)\Big|  \le \frac{1}{2} \| \nabla^2 t_2(\boldsymbol{c}) \|_{2} \cdot  \| (\boldsymbol{v}_1-\boldsymbol{v}_2) \|^2 \le \frac{S_2}{2}\cdot  \| (\boldsymbol{v}_1-\boldsymbol{v}_2) \|^2, 
    \end{equation}
    which implies that the function $t_2$ is $S_2$-Lipschitz smooth. 

    \item \textbf{For function $t_1$}, its range is $[0, B/\lambda\epsilon]$ because its domain is the same as the range of function $t_2$, that is, $[1, \exp(B/\lambda\epsilon)]$. 
    For brevity, we denote $t_1(z; \widehat{\boldsymbol{x}})$ as $t_1(z)$. The first and second derivatives of $t_1$ with respect to $z$ are 
    \begin{equation}\nonumber
        t_1^{\prime}(z) = \frac{1}{z} \quad \text{and} \quad t_1^{\prime\prime}(z) = -\frac{1}{z^2}.
    \end{equation}
    Over the domain $z \in \big[ 1, \exp(B/\lambda\epsilon) \big]$, we can bound the absolute values of these derivatives:
    \begin{equation}\nonumber
        |t_1^{\prime}(z)| = \frac{1}{z} \le 1,
    \end{equation}
    \begin{equation}\nonumber
        |t_1^{\prime\prime}(z)| = \frac{1}{z^2} \le 1.
    \end{equation}
    The first bound implies that $t_1$ is $L_1$-Lipschitz continuous with $L_1 = 1$. The second bound implies that $t_1$ is $S_1$-Lipschitz smooth with $S_1 = 1$. 
    Let $C_1 = 1$; the expected squared norm of the gradient is also bounded
    \begin{equation}\nonumber
        \mathbb{E}\Big[|\nabla t_1(z;\widehat{\boldsymbol{x}}) |^2\Big] =\mathbb{E}\Big[\Big(t_1^{\prime}(z)\Big)^2\Big] \le C_1^2.
    \end{equation}
    Since the value of the function $t_1$ is bounded, its variance is finite, that is, 
    \begin{equation}\nonumber
        \sigma_1^2 = \sup_{\boldsymbol{\theta} \in \Theta} \mathbb{V}_{\widehat{\boldsymbol{x}}} \Big( t_1\Big( \mathbb{E}_{ \boldsymbol{\xi}_1 \sim U_{\epsilon}}\Big[ t_2\Big(   \mathbb{E}_{ \boldsymbol{\xi}_2 \sim W_{\epsilon}}\Big[t_3\Big( \boldsymbol{\theta}; \widehat{\boldsymbol{x}}, \boldsymbol{\xi}_1, \boldsymbol{\xi}_2  \Big)\Big]; \widehat{\boldsymbol{x}}, \boldsymbol{\xi}_1\Big)\Big] ;\widehat{\boldsymbol{x}} \Big) \Big) < \infty. 
    \end{equation}
\end{itemize}
This completes the proof. 
\end{proof}

\subsection{Proof of Theorem~\ref{theo-scsc-conver} in Section~\ref{subsec-algo-sco}} \label{ecsec-theo-scsc}

{Let $\zeta = ({\boldsymbol{\xi}_1,\boldsymbol{\xi}_2^1,\ldots,\boldsymbol{\xi}_2^Q})$, and define 
\begin{equation}\nonumber
     \begin{aligned}
     \phi_{m,Q}^{(2)}(\boldsymbol{\theta};\boldsymbol{\xi}_1) & := \frac{1}{Q}\sum_{q=1}^{Q} t_3( \boldsymbol{\theta};\widehat{\boldsymbol{ x}}_m,\boldsymbol{\xi}_1,\boldsymbol{\xi}_2^q), \\
     H_{m,Q}(\boldsymbol{\theta}, \zeta) & := t_2 \left( \phi_{m,Q}^{(2)}(\boldsymbol{\theta};\boldsymbol{\xi}_1) ; \widehat{\boldsymbol{x}}_m,\boldsymbol{\xi}_1\right), \\
     \phi_{m,Q}^{(1)}(\boldsymbol{\theta})&:=\mathbb{E}_{\zeta} \Big [ H_{m,Q}(\boldsymbol{\theta}, \zeta) \Big ] ,  \label{eq:batch-target}\\
     F_Q(\boldsymbol{\theta})&:=\lambda\epsilon \sum_{m=1}^M \widehat{p}_m \cdot t_{1,m}(\phi_{m,Q}^{(1)}(\boldsymbol{\theta})), 
    \end{aligned}   
\end{equation}
and 
\begin{equation}\nonumber
    \phi_m^{(1)}(\boldsymbol{\theta}) := \mathbb{E}_{\boldsymbol{\xi}_1}\Big[t_{2,m} (\phi_m^{(2)}(\boldsymbol{\theta}))\Big], \quad \phi_m^{(2)}(\boldsymbol{\theta};\boldsymbol{\xi}_1) := \mathbb{E}_{\boldsymbol{\xi}_2} \Big[t_3( \boldsymbol{\theta};\widehat{\boldsymbol{ x}}_m,\boldsymbol{\xi}_1,\boldsymbol{\xi}_2)\Big]. 
\end{equation}
Follow the notation in Eqs.~\eqref{t-function-brief} and~\eqref{phi-functions}, we first prove the smoothness of the function $F_Q(\boldsymbol{\theta})$. }
\begin{lemma}[Smoothness of the Objective] \label{lemma-smoothness-F}
    Under Assumption~\ref{assumption-2}, function $F_Q(\boldsymbol{\theta})$ is $S_F$-Lipschitz smooth on $\Theta$, where
    \begin{equation}\nonumber
        S_F := \lambda\epsilon \left( S_1C_2^2C_3^2 + C_1S_2C_3^2 + C_1C_2S_3\right). 
    \end{equation}
\end{lemma}

\begin{proof}[Proof of Lemma~\ref{lemma-smoothness-F}]
The Jacobian of the inner mini-batch mean is
\[
    J_{\boldsymbol{\theta}}\phi_{m,Q}^{(2)}(\boldsymbol{\theta};\zeta)
    = \frac{1}{Q}\sum_{q=1}^Q J_{\boldsymbol{\theta}}t_3(\boldsymbol{\theta};\widehat{\boldsymbol{x}}_m,\boldsymbol{\xi}_1,\boldsymbol{\xi}_2^q).
\]
The uniform derivative bounds justify differentiation under the expectation defining $\phi_{m,Q}^{(1)}(\boldsymbol{\theta})$. By the chain rule, we have
\[
    \begin{aligned}
    \nabla_{\boldsymbol{\theta}}\phi_{m,Q}^{(1)}(\boldsymbol{\theta})
    &= \mathbb{E}_{\zeta}\Big[\big[J_{\boldsymbol{\theta}}\phi_{m,Q}^{(2)}(\boldsymbol{\theta};\zeta)\big]^\top
    \nabla t_2\left(\phi_{m,Q}^{(2)}(\boldsymbol{\theta};\zeta);\widehat{\boldsymbol{x}}_m,\boldsymbol{\xi}_1\right)\Big], \\
    \nabla_{\boldsymbol{\theta}}F_Q(\boldsymbol{\theta})
    &= \lambda\epsilon\sum_{m=1}^M\widehat{p}_m\cdot
    \nabla t_{1,m}\left(\phi_{m,Q}^{(1)}(\boldsymbol{\theta})\right)
    \nabla_{\boldsymbol{\theta}}\phi_{m,Q}^{(1)}(\boldsymbol{\theta}).
    \end{aligned}
\]
According to Proposition~\ref{propos-lip-con}, the uniform derivative bounds under Assumption~\ref{assumption-2} imply, for each $m\in[M]$ and batch $\zeta$,
\[
    \mathbb{E}_{\zeta}\Big[\|J_{\boldsymbol{\theta}}\phi_{m,Q}^{(2)}(\boldsymbol{\theta};\zeta)\|_{\mathrm{op}}^2\Big]\le C_3^2,
    \qquad
    \|\nabla_{\boldsymbol{\theta}}\phi_{m,Q}^{(1)}(\boldsymbol{\theta})\|_2\le C_2C_3.
\]
Since $\Theta$ is convex, $\phi_{m,Q}^{(1)}$ is $C_2C_3$-Lipschitz continuous, and
\[
    \mathbb{E}_{\zeta}\Big[\|\phi_{m,Q}^{(2)}(\boldsymbol{\theta};\zeta)-\phi_{m,Q}^{(2)}(\boldsymbol{\theta}^\prime;\zeta)\|_2^2\Big]
    \le C_3^2\|\boldsymbol{\theta}-\boldsymbol{\theta}^\prime\|_2^2.
\]
We also have
\[
    \mathbb{E}_{\zeta}\Big[\|J_{\boldsymbol{\theta}}\phi_{m,Q}^{(2)}(\boldsymbol{\theta};\zeta)-J_{\boldsymbol{\theta}}\phi_{m,Q}^{(2)}(\boldsymbol{\theta}^\prime;\zeta)\|_{\mathrm{op}}\Big]
    \le S_3\|\boldsymbol{\theta}-\boldsymbol{\theta}^\prime\|_2.
\]
Consequently,
\[
    \begin{aligned}
    &\left\|\nabla_{\boldsymbol{\theta}}\phi_{m,Q}^{(1)}(\boldsymbol{\theta})-\nabla_{\boldsymbol{\theta}}\phi_{m,Q}^{(1)}(\boldsymbol{\theta}^\prime)\right\|_2\\
    &=\Big\|\mathbb{E}_{\zeta}\Big[
    \Big(J_{\boldsymbol{\theta}}\phi_{m,Q}^{(2)}(\boldsymbol{\theta};\zeta)-J_{\boldsymbol{\theta}}\phi_{m,Q}^{(2)}(\boldsymbol{\theta}^\prime;\zeta)\Big)^\top
    \nabla t_2\left(\phi_{m,Q}^{(2)}(\boldsymbol{\theta};\zeta);\widehat{\boldsymbol{x}}_m,\boldsymbol{\xi}_1\right)\\
    &\qquad+\big[J_{\boldsymbol{\theta}}\phi_{m,Q}^{(2)}(\boldsymbol{\theta}^\prime;\zeta)\big]^\top
    \Big(\nabla t_2\left(\phi_{m,Q}^{(2)}(\boldsymbol{\theta};\zeta);\widehat{\boldsymbol{x}}_m,\boldsymbol{\xi}_1\right)
    -\nabla t_2\left(\phi_{m,Q}^{(2)}(\boldsymbol{\theta}^\prime;\zeta);\widehat{\boldsymbol{x}}_m,\boldsymbol{\xi}_1\right)\Big)
    \Big]\Big\|_2\\
    &\le C_2\mathbb{E}_{\zeta}\Big[
    \left\|J_{\boldsymbol{\theta}}\phi_{m,Q}^{(2)}(\boldsymbol{\theta};\zeta)-J_{\boldsymbol{\theta}}\phi_{m,Q}^{(2)}(\boldsymbol{\theta}^\prime;\zeta)\right\|_{\mathrm{op}}\Big]\\
    &\qquad+S_2\Big(\mathbb{E}_{\zeta}\Big[\|J_{\boldsymbol{\theta}}\phi_{m,Q}^{(2)}(\boldsymbol{\theta}^\prime;\zeta)\|_{\mathrm{op}}^2\Big]\Big)^{1/2}
    \Big(\mathbb{E}_{\zeta}\Big[\|\phi_{m,Q}^{(2)}(\boldsymbol{\theta};\zeta)-\phi_{m,Q}^{(2)}(\boldsymbol{\theta}^\prime;\zeta)\|_2^2\Big]\Big)^{1/2}\\
    &\le\left(C_2S_3+C_3S_2C_3\right)\left\|\boldsymbol{\theta}-\boldsymbol{\theta}^\prime\right\|_2, 
    \end{aligned}
\]
where the first inequality follows from the boundedness and Lipschitz continuity of $\nabla t_2$ and the Cauchy-Schwarz inequality. 
Then, using $\sum_{m=1}^M\widehat{p}_m=1$, function $F_Q(\boldsymbol{\theta})$ is $S_F$-Lipschitz smooth on $\Theta$:
\[
    \begin{aligned}
    &\left\|\nabla_{\boldsymbol{\theta}}F_Q(\boldsymbol{\theta})
    -\nabla_{\boldsymbol{\theta}}F_Q(\boldsymbol{\theta}^\prime)\right\|_2 \\
    &\le\lambda\epsilon\sum_{m=1}^M\widehat{p}_m\cdot\Big[
    \left|\nabla t_{1,m}\left(\phi_{m,Q}^{(1)}(\boldsymbol{\theta})\right)
    -\nabla t_{1,m}\left(\phi_{m,Q}^{(1)}(\boldsymbol{\theta}^\prime)\right)\right|
    \left\|\nabla_{\boldsymbol{\theta}}\phi_{m,Q}^{(1)}(\boldsymbol{\theta})\right\|_2 \\
    &\qquad+\left|\nabla t_{1,m}\left(\phi_{m,Q}^{(1)}(\boldsymbol{\theta}^\prime)\right)\right|
    \left\|\nabla_{\boldsymbol{\theta}}\phi_{m,Q}^{(1)}(\boldsymbol{\theta})
    -\nabla_{\boldsymbol{\theta}}\phi_{m,Q}^{(1)}(\boldsymbol{\theta}^\prime)\right\|_2
    \Big] \\
    &\le\lambda\epsilon\left[S_1(C_2C_3)^2+C_1(C_2S_3+S_2C_3^2)\right]
    \left\|\boldsymbol{\theta}-\boldsymbol{\theta}^\prime\right\|_2 \\
    &=S_F\cdot\left\|\boldsymbol{\theta}-\boldsymbol{\theta}^\prime\right\|_2,
    \end{aligned}
\]
where $S_F=\lambda\epsilon\left[S_1(C_2C_3)^2+C_1(C_2S_3+S_2C_3^2)\right]$.
\end{proof}

{We then prove a stochastic-gradient bias bound between $\nabla F_Q(\boldsymbol{\theta})$ and $\nabla F(\boldsymbol{\theta})$ arising from batches sampled from $W_{\epsilon}$. 
\begin{lemma} \label{lemma-batch-bias-bound}
    Under Assumption~\ref{assumption-2}, there is a finite constant $C_{\text{bias}}$ independent of the batch size $Q$ such that 
    \[
        \sup_{\boldsymbol{\theta}\in \Theta}\|\nabla_{\boldsymbol{\theta}} F_Q(\boldsymbol{\theta})-\nabla_{\boldsymbol{\theta}} F(\boldsymbol{\theta})\| \le \frac{C_{\mathrm{bias}}}{Q}.
    \]
\end{lemma}}

\begin{proof} [Proof of Lemma~\ref{lemma-batch-bias-bound}]
    { Under Assumption~\ref{assumption-2}, the range of the function $t_2$ is $[1, \exp(B/\lambda\epsilon)]$ and $\phi_{m,Q}^{(1)}(\boldsymbol{\theta}) \ge 1$. 
    \begin{equation} \label{Eq-batch-bound}
        \begin{aligned}
            \| & \nabla_{\boldsymbol{\theta}} F_Q(\boldsymbol{\theta}) -\nabla_{\boldsymbol{\theta}} F(\boldsymbol{\theta})\| = \lambda\epsilon \cdot \Big \| \sum_{m=1}^M \widehat{p}_m \cdot \Big [ \nabla_{\boldsymbol{\theta}} t_{1,m}(\phi_{m,Q}^{(1)}(\boldsymbol{\theta})) - \nabla_{\boldsymbol{\theta}} t_{1,m}(\phi_{m}^{(1)}(\boldsymbol{\theta})) \Big ] \Big \| \\
            & = \lambda\epsilon \cdot \Big \| \sum_{m=1}^M \widehat{p}_m \Big ( \frac{\nabla \phi_{m,Q}^{(1)}(\boldsymbol{\theta})}{\phi_{m,Q}^{(1)}(\boldsymbol{\theta})} - \frac{\nabla \phi_{m}^{(1)}(\boldsymbol{\theta})}{\phi_{m}^{(1)}(\boldsymbol{\theta})}\Big ) \Big \|  \le \lambda\epsilon \cdot \sum_{m=1}^M \widehat{p}_m \Big \|  \frac{\nabla \phi_{m,Q}^{(1)}(\boldsymbol{\theta})}{\phi_{m,Q}^{(1)}(\boldsymbol{\theta})} - \frac{\nabla \phi_{m}^{(1)}(\boldsymbol{\theta})}{\phi_{m}^{(1)}(\boldsymbol{\theta})} \Big \| \\
            & \le \lambda\epsilon \cdot \sum_{m=1}^M \widehat{p}_m \cdot \Big ( \frac{\Big\|\nabla \phi_{m,Q}^{(1)}(\boldsymbol{\theta}) - \nabla \phi_{m}^{(1)}(\boldsymbol{\theta})\Big\|}{\phi_{m,Q}^{(1)}(\boldsymbol{\theta})} + \frac{\Big\|\nabla \phi_{m}^{(1)}(\boldsymbol{\theta})\Big\| \cdot |\phi_{m,Q}^{(1)}(\boldsymbol{\theta}) -  \phi_{m}^{(1)}(\boldsymbol{\theta})|}{\phi_{m,Q}^{(1)}(\boldsymbol{\theta})\cdot \phi_{m}^{(1)}(\boldsymbol{\theta})} \Big ) \\
            & \le \lambda\epsilon \cdot \sum_{m=1}^M \widehat{p}_m \cdot \Big ( \Big\|\nabla \phi_{m,Q}^{(1)}(\boldsymbol{\theta}) - \nabla \phi_{m}^{(1)}(\boldsymbol{\theta})\Big\| + C_2C_3|\phi_{m,Q}^{(1)}(\boldsymbol{\theta}) -  \phi_{m}^{(1)}(\boldsymbol{\theta})| \Big ), 
        \end{aligned}
    \end{equation}
    where the first and second inequalities follow from the triangle inequality, the last one is because $\phi_{m}^{(1)}(\boldsymbol{\theta}), \phi_{m, Q}^{(1)}(\boldsymbol{\theta}) \ge 1$ and the fact that $\phi_{m}^{(1)}(\boldsymbol{\theta})$ is $C_2C_3$-Lipschitz continuous (see Lemma~\ref{lemma-smoothness-F}). }
    
    { By Eq.~\eqref{Eq-batch-bound}, it suffices to show the bounds of the value and gradient approximation errors. Recall that $C_2=\exp(B/(\lambda\epsilon))$. 
    We now analyze the properties of the function $t_3$. 
    For fixed $m$ and $\boldsymbol{\theta}$, the inner samples are independent conditional on $\boldsymbol{\xi}_1$, and the sample-average errors have zero conditional mean, i.e., 
    \begin{equation} \label{Eq-phi-2-expectation}
        \mathbb{E}_{\boldsymbol{\xi}_2^1,\ldots,\boldsymbol{\xi}_2^Q}\left[ \phi_{m,Q}^{(2)}(\boldsymbol{\theta};\boldsymbol{\xi}_1) - \phi_m^{(2)} (\boldsymbol{\theta};\boldsymbol{\xi}_1) \,\middle|\,\boldsymbol{\xi}_1 \right]=\boldsymbol{0},
    \end{equation}
    \begin{equation} \label{Eq-j-phi-2-expectation}        \mathbb{E}_{\boldsymbol{\xi}_2^1,\ldots,\boldsymbol{\xi}_2^Q}\left[ \frac{1}{Q}\sum_{q=1}^Q J_{\boldsymbol{\theta}}t_3(\boldsymbol{\theta};\widehat{\boldsymbol{x}}_m,\boldsymbol{\xi}_1,\boldsymbol{\xi}_2^q) -J_{\boldsymbol{\theta}}\phi_m^{(2)}(\boldsymbol{\theta};\boldsymbol{\xi}_1) \,\middle|\,\boldsymbol{\xi}_1 \right]=\boldsymbol{0},
    \end{equation}
    and it follows that
    \begin{equation}\label{Eq-phi2-expectation-bound}
    \begin{aligned}
    &\mathbb{E}_{\boldsymbol{\xi}_1,\boldsymbol{\xi}_2^1,\ldots,\boldsymbol{\xi}_2^Q}\Big\|\phi_{m,Q}^{(2)}(\boldsymbol{\theta};\boldsymbol{\xi}_1)-\phi_m^{(2)}(\boldsymbol{\theta};\boldsymbol{\xi}_1)\Big\|^2 \\
    &\quad=\frac{1}{Q}\cdot \mathbb{E}_{\boldsymbol{\xi}_1,\boldsymbol{\xi}_2^1}\Big\|t_3(\boldsymbol{\theta};\widehat{\boldsymbol{x}}_m,\boldsymbol{\xi}_1,\boldsymbol{\xi}_2^1)-\phi_m^{(2)}(\boldsymbol{\theta};\boldsymbol{\xi}_1)\Big\|^2
    \le\frac{N(C_2-1)^2}{Q},
    \end{aligned}
    \end{equation}
    \begin{equation}\label{Eq-Jt-expectation-bound}
        \mathbb{E}_{\boldsymbol{\xi}_1,\boldsymbol{\xi}_2^1,\ldots,\boldsymbol{\xi}_2^Q}\Big\|\frac{1}{Q}\sum_{q=1}^Q J_{\boldsymbol{\theta}}t_3(\boldsymbol{\theta};\widehat{\boldsymbol{x}}_m,\boldsymbol{\xi}_1,\boldsymbol{\xi}_2^q)-J_{\boldsymbol{\theta}}\phi_m^{(2)}(\boldsymbol{\theta};\boldsymbol{\xi}_1)\Big\|_{\mathrm{F}}^2 \le\frac{C_3^2}{Q}
    \end{equation}
    By the Lipschitz smoothness property of the function $t_2$ (Proposition~\ref{propos-lip-con}) and Eqs.~\eqref{Eq-phi-2-expectation} and~\eqref{Eq-phi2-expectation-bound}, we have 
    \[
        \begin{aligned}
            |\phi_{m,Q}^{(1)}(\boldsymbol{\theta})-\phi_m^{(1)}(\boldsymbol{\theta})| & = \Big | \mathbb{E}_{\boldsymbol{\xi}_1,\boldsymbol{\xi}_2^1,\ldots,\boldsymbol{\xi}_2^Q} \Big [ t_2(\phi_{m,Q}^{(2)}(\boldsymbol{\theta};\boldsymbol{\xi}_1)) - t_2(\phi_{m}^{(2)}(\boldsymbol{\theta};\boldsymbol{\xi}_1))\Big ] \Big | \\
            & \le \frac{S_2}{2} \cdot \mathbb{E}_{\boldsymbol{\xi}_1,\boldsymbol{\xi}_2^1,\ldots,\boldsymbol{\xi}_2^Q}\Big\|\phi_{m,Q}^{(2)}(\boldsymbol{\theta};\boldsymbol{\xi}_1)-\phi_m^{(2)}(\boldsymbol{\theta};\boldsymbol{\xi}_1)\Big\|^2 \le \frac{S_2N(C_2-1)^2}{2Q}. 
        \end{aligned}
    \] }
    
    {For the gradient approximation error bound of $\nabla_{\boldsymbol{\theta}}\phi_m^{(1)}(\boldsymbol{\theta})$, we have 
    \[
        \begin{aligned}
            \|\nabla_{\boldsymbol{\theta}} & \phi_{m,Q}^{(1)}(\boldsymbol{\theta})-\nabla_{\boldsymbol{\theta}}\phi_m^{(1)}(\boldsymbol{\theta})\| \\
            & = \Big \| \mathbb{E}_{\boldsymbol{\xi}_1,\boldsymbol{\xi}_2^1,\ldots,\boldsymbol{\xi}_2^Q}\Big [ ( J_{\boldsymbol{\theta}}\phi_{m,Q}^{(2)}(\boldsymbol{\theta};\boldsymbol{\xi}_1) )^\top \nabla t_{2,m}(\phi_{m,Q}^{(2)}(\boldsymbol{\theta};\boldsymbol{\xi}_1)) \Big ] \\
            & \quad - \mathbb{E}_{\boldsymbol{\xi}_1,\boldsymbol{\xi}_2^1,\ldots,\boldsymbol{\xi}_2^Q}\Big [ ( J_{\boldsymbol{\theta}}\phi_{m}^{(2)}(\boldsymbol{\theta};\boldsymbol{\xi}_1) )^\top \nabla t_{2,m}(\phi_{m}^{(2)}(\boldsymbol{\theta};\boldsymbol{\xi}_1)) \Big ]\Big \| \\
            & \le \Big\|\mathbb{E}\Big[ (J_{\boldsymbol{\theta}}\phi_{m,Q}^{(2)} -J_{\boldsymbol{\theta}}\phi_m^{(2)})^\top \cdot \big(\nabla t_{2,m}(\phi_{m,Q}^{(2)}) -\nabla t_{2,m}(\phi_m^{(2)})\big) \Big]\Big\| \\
            &\quad+ \Big\|\mathbb{E}\Big[ (J_{\boldsymbol{\theta}}\phi_{m,Q}^{(2)} -J_{\boldsymbol{\theta}}\phi_m^{(2)})^\top \cdot \big(\nabla t_{2,m}(\phi_m^{(2)})\big) \Big]\Big\| \\
            &\quad + \Big\|\mathbb{E}\Big[ (J_{\boldsymbol{\theta}}\phi_m^{(2)})^\top \big(\nabla t_{2,m}(\phi_{m,Q}^{(2)}) -\nabla t_{2,m}(\phi_m^{(2)})\big) \Big]\Big\| \\
            & = \Big\|\mathbb{E}\Big[ (J_{\boldsymbol{\theta}}\phi_{m,Q}^{(2)} -J_{\boldsymbol{\theta}}\phi_m^{(2)})^\top \cdot \big(\nabla t_{2,m}(\phi_{m,Q}^{(2)}) -\nabla t_{2,m}(\phi_m^{(2)})\big) \Big]\Big\| \\
            & \quad + \Big\|\mathbb{E}\Big[ (J_{\boldsymbol{\theta}}\phi_m^{(2)})^\top \big(\nabla t_{2,m}(\phi_{m,Q}^{(2)}) -\nabla t_{2,m}(\phi_m^{(2)})\big) \Big]\Big\| , 
        \end{aligned}
    \]
    where the first inequality follows from the triangle inequality, and the second equality follows from Eq.~\eqref{Eq-j-phi-2-expectation}. 
    Since the function $t_2$ is smooth and has a bounded domain, $\nabla t_{2}$ is also Lipschitz smooth with a finite Lipschitz constant $S_{\nabla t_2}$ on the domain. 
    By the Lipschitz smoothness of the function $t_2$, the fact that $\|J_{\boldsymbol\theta}\phi_m^{(2)}\|_{\mathrm F}\le C_3$ derived in Proposition~\ref{propos-lip-con}, we further have 
    \[
        \begin{aligned}
            \| & \nabla_{\boldsymbol{\theta}} \phi_{m,Q}^{(1)}(\boldsymbol{\theta})-\nabla_{\boldsymbol{\theta}}\phi_m^{(1)}(\boldsymbol{\theta})\| \\
            & \le S_2 \sqrt{\mathbb{E}\|J_{\boldsymbol{\theta}} \phi_{m,Q}^{(2)} - J_{\boldsymbol{\theta}}\phi_m^{(2)} \|_{\mathrm{F}}^2} \sqrt{\mathbb{E}\|\phi_{m,Q}^{(2)} -\phi_m^{(2)}\|^2} \\
            &\quad+\Big\|\mathbb{E}\Big[(J_{\boldsymbol{\theta}}\phi_m^{(2)})^\top\big(\nabla t_{2,m}(\phi_{m,Q}^{(2)})-\nabla t_{2,m}(\phi_m^{(2)})\big)\Big]\Big\|\\ 
            & \le \frac{S_2C_3\sqrt{N}(C_2-1)}{Q} + C_3 \mathbb{E}\|\nabla t_{2,m}(\phi_{m,Q}^{(2)}) -\nabla t_{2,m}(\phi_m^{(2)})-\nabla^2t_{2,m}(\phi_m^{(2)})(\phi_{m,Q}^{(2)}-\phi_m^{(2)})\| \\
            & \le \frac{S_2C_3\sqrt{N}(C_2-1)}{Q} + \frac{C_3 S_{\nabla t_2}}{2} \mathbb{E}\Big\|\phi_{m,Q}^{(2)}-\phi_m^{(2)}\Big\|^2 \\
            & \le \frac{2S_2C_3\sqrt{N}(C_2-1)+ C_3S_{\nabla t_2}N(C_2-1)^2}{2Q},
        \end{aligned} 
    \]
    where the first inequality follows from Jensen's and Cauchy-Schwarz inequalities and the Lipschitz smoothness of $t_{2,m}$, the second follows from Eqs.~\eqref{Eq-phi2-expectation-bound} and~\eqref{Eq-Jt-expectation-bound}, the third follows from the smoothness of $\nabla t_{2}$, and the last inequality follows from Eq.~\eqref{Eq-phi2-expectation-bound}. }

    {Hence, we have 
    \begin{equation} \nonumber
        \begin{aligned}
            \| & \nabla_{\boldsymbol{\theta}} F_Q(\boldsymbol{\theta}) -\nabla_{\boldsymbol{\theta}} F(\boldsymbol{\theta})\| \\
            & \le \lambda\epsilon C_3 \Big(\frac{2S_2\sqrt{N}(C_2-1)+ (S_{\nabla t_2}+C_2S_2)N(C_2-1)^2}{2Q}\Big) = \frac{C_{\mathrm{bias}}}{Q},
        \end{aligned}
    \end{equation}
    where the constant is independent of $Q$: 
    \[
        C_{\mathrm{bias}} :=\frac{\lambda\epsilon C_3}{2}\Big( 2S_2\sqrt{N}(C_2-1)+ (S_{\nabla t_2}+C_2S_2)N(C_2-1)^2\Big). 
    \]
    This completes the proof. }
\end{proof}

{We next prove the convergence property of the proposed CM-SCSC algorithm, following the notation in Section~\ref{subsec-algo-sco}. }

\begin{proof} [Proof of Theorem~\ref{theo-scsc-conver}]
    { For brevity, we use constant step sizes $\alpha = \alpha_k$ and $\beta = \beta_k$ for any $k \in [K]$. We set the step length $0 < \beta \le 1$ for any $k \in [K]$, such that $(1-\beta)^2 \le 1$. Let $\sigma$-algebra $\mathcal{F}_k$ contain the initialization and all random samples before iteration $k$, and $\mathcal{F}_k^+=\mathcal{F}_k\vee\sigma(\zeta_k)$. 
    We prove Theorem~\ref{theo-scsc-conver}\ref{theo-scsc-conver-1} in three steps, following the similar steps in~\cite{chen2021solving}.  }

{ \textbf{Step 1. Tracking variance.} 
    Define the mean squared error in iteration $k$: 
    \begin{equation}
        E_k :=\mathbb{E}_{m_{0:k-1},\boldsymbol{\zeta}_{0:k-1}}\Big[\sum_{m=1}^M\widehat{p}_m\cdot|y_{1,m}^k-\phi_{m,Q}^{(1)}(\boldsymbol{\theta}^{k-1})|^2\Big], 
    \end{equation}
    and 
    \begin{equation}\nonumber
        \delta_m^k(\boldsymbol{\theta}):=H_{m,Q}(\boldsymbol{\theta};\boldsymbol{\zeta}_k)-\phi_{m,Q}^{(1)}(\boldsymbol{\theta}).
    \end{equation} 
    For an error term in iteration $k$ and any $m \in [M]$, let $\widetilde{y}_{1,m}^{k+1}$ be the expression inside the projection in Step 8 of Algorithm~\ref{algo-SCSC-ML}, we have 
    \begin{equation} \label{Eq-single-error-term}
        \begin{aligned}
            & \widetilde{y}_{1,m}^{k+1}-\phi_{m,Q}^{(1)}(\boldsymbol{\theta}^{k}) \\
            & = (1-\beta) \Big(y_{1,m}^k + H_{m,Q}(\boldsymbol{\theta}^{k};\boldsymbol{\zeta}_k) - H_{m,Q}(\boldsymbol{\theta}^{k-1};\boldsymbol{\zeta}_k)\Big) + \beta H_{m,Q}(\boldsymbol{\theta}^{k};\boldsymbol{\zeta}_k) - \phi_{m,Q}^{(1)}(\boldsymbol{\theta}^{k}) \\
            & = (1-\beta) \Big(y_{1,m}^k-\phi_{m,Q}^{(1)}(\boldsymbol{\theta}^{k-1})\Big) + (1-\beta) \Big(\phi_{m,Q}^{(1)}(\boldsymbol{\theta}^{k-1})-\phi_{m,Q}^{(1)}(\boldsymbol{\theta}^{k})\Big) \\
            & \qquad +\beta \Big( H_{m,Q}(\boldsymbol{\theta}^{k};\boldsymbol{\zeta}_k) - \phi_{m,Q}^{(1)}(\boldsymbol{\theta}^{k}) \Big) + (1-\beta)\Big(H_{m,Q}(\boldsymbol{\theta}^{k};\boldsymbol{\zeta}_k) - H_{m,Q}(\boldsymbol{\theta}^{k-1};\boldsymbol{\zeta}_k)\Big) \\
            & = (1-\beta) \Big(y_{1,m}^k-\phi_{m,Q}^{(1)}(\boldsymbol{\theta}^{k-1})\Big) + (1-\beta) \Big( \delta_m^k(\boldsymbol{\theta}^k) - \delta_m^{k}(\boldsymbol{\theta}^{k-1})\Big) + \beta \delta_m^k(\boldsymbol{\theta}^k). 
        \end{aligned}
    \end{equation}
    For the last term in Eq.~\eqref{Eq-single-error-term}, as $\zeta_k$ is independent of $\mathcal{F}_k$ and the range of the function $t_2$ is $[1, \exp(B/\lambda\epsilon)]$, we have 
    \[
        \mathbb{E}_{\boldsymbol{\zeta}_k}\Big[\delta_m^k(\boldsymbol{\theta}^k)\mid\mathcal F_k\Big]=\mathbb{E}_{\boldsymbol{\zeta}_k}\Big[\delta_m^k(\boldsymbol{\theta}^{k-1})\mid\mathcal F_k\Big]=0,
    \]
    \[
        \mathbb{E}_{\boldsymbol{\zeta}_k}\Big[ | \delta_m^k(\boldsymbol{\theta}^k)|^2\mid\mathcal F_k\Big] \le \Big(\exp(B/\lambda\epsilon)-1\Big)^2 = (C_2-1)^2,
    \]
    and \[
    \mathbb{E}_{\boldsymbol{\zeta}_k}\Big[ H_{m,Q}(\boldsymbol{\theta}^k;\boldsymbol{\zeta}_k) - H_{m,Q}(\boldsymbol{\theta}^{k-1};\boldsymbol{\zeta}_k) \mid\mathcal F_k\Big] = \phi_{m,Q}^{(1)}(\boldsymbol{\theta}^k) - \phi_{m,Q}^{(1)}(\boldsymbol{\theta}^{k-1}). 
    \]
    Thus, for the second term in Eq.~\eqref{Eq-single-error-term}, we have 
    \[
    \begin{aligned}
        \mathbb{E}_{\boldsymbol{\zeta}_k}& \Big[ \Big( \delta_m^k(\boldsymbol{\theta}^k) - \delta_m^{k}(\boldsymbol{\theta}^{k-1})\Big)^2\mid\mathcal F_k\Big] \\
        & = \mathbb{E}_{\boldsymbol{\zeta}_k}\Big[ \Big( H_{m,Q}(\boldsymbol{\theta}^k;\boldsymbol{\zeta}_k) - H_{m,Q}(\boldsymbol{\theta}^{k-1};\boldsymbol{\zeta}_k)-\phi_{m,Q}^{(1)}(\boldsymbol{\theta}^k) + \phi_{m,Q}^{(1)}(\boldsymbol{\theta}^{k-1}))\Big)^2\mid\mathcal F_k\Big] \\
        & = \mathbb{E}_{\boldsymbol{\zeta}_k}\Big[ \Big( H_{m,Q}(\boldsymbol{\theta}^k;\boldsymbol{\zeta}_k) - H_{m,Q}(\boldsymbol{\theta}^{k-1};\boldsymbol{\zeta}_k)\Big)^2 \mid\mathcal F_k\Big] - \Big(\phi_{m,Q}^{(1)}(\boldsymbol{\theta}^k) - \phi_{m,Q}^{(1)}(\boldsymbol{\theta}^{k-1}))\Big)^2 \\
        & \le \mathbb{E}_{\boldsymbol{\zeta}_k}\Big[ \Big( H_{m,Q}(\boldsymbol{\theta}^k;\boldsymbol{\zeta}_k) - H_{m,Q}(\boldsymbol{\theta}^{k-1};\boldsymbol{\zeta}_k)\Big)^2 \mid\mathcal F_k\Big] \le  C_2^2C_3^2\cdot\|\boldsymbol{\theta}^k-\boldsymbol{\theta}^{k-1}\|^2, 
    \end{aligned}
    \]
    where the last inequality follows from the Lipschitz continuity of the functions $t_2$ and $t_3$ in Proposition~\ref{propos-lip-con}. 
    Conditioned on $\mathcal{F}_k$, taking expectation on the both sides of Eq.~\eqref{Eq-single-error-term}, and projecting $\widetilde{y}_{1,m}^{k+1}$ to $y_{1,m}^{k+1}$, we have 
    \begin{equation}
        \begin{aligned}
            & \mathbb{E}_{\boldsymbol{\zeta}_k}\Big[|y_{1,m}^{k+1}-\phi_{m,Q}^{(1)}(\boldsymbol{\theta}^k)|^2\mid\mathcal F_k\Big] \le \mathbb{E}_{\boldsymbol{\zeta}_k}\Big[|\widetilde{y}_{1,m}^{k+1}-\phi_{m,Q}^{(1)}(\boldsymbol{\theta}^k)|^2\mid\mathcal F_k\Big]\\
            & \le(1-\beta)^2\cdot\mathbb{E}_{\boldsymbol{\zeta}_k}\Big[|y_{1,m}^k-\phi_{m,Q}^{(1)}(\boldsymbol{\theta}^{k-1})|^2\mid\mathcal F_k\Big] + \beta^2 \mathbb{E}_{\boldsymbol{\zeta}_k}\Big[ | \delta_m^k(\boldsymbol{\theta}^k)|^2\mid\mathcal F_k\Big] \\
            & \qquad +  (1-\beta)^2\mathbb{E}_{\boldsymbol{\zeta}_k} \Big[ \Big( \delta_m^k(\boldsymbol{\theta}^k) - \delta_m^{k}(\boldsymbol{\theta}^{k-1})\Big)^2\mid\mathcal F_k\Big] \\ 
            & \qquad+ 2\beta(1-\beta)\mathbb{E}_{\boldsymbol{\zeta}_k}\Big[ \Big( \delta_m^k(\boldsymbol{\theta}^k) - \delta_m^{k}(\boldsymbol{\theta}^{k-1})\Big)\delta_m^k(\boldsymbol{\theta}^k)\mid\mathcal F_k\Big] \\
            & \le(1-\beta)^2\cdot\mathbb{E}_{\boldsymbol{\zeta}_k}\Big[|y_{1,m}^k-\phi_{m,Q}^{(1)}(\boldsymbol{\theta}^{k-1})|^2\mid\mathcal F_k\Big] + 2 \beta^2 \mathbb{E}_{\boldsymbol{\zeta}_k}\Big[ | \delta_m^k(\boldsymbol{\theta}^k)|^2\mid\mathcal F_k\Big] \\
            & \qquad + 2 (1-\beta)^2\mathbb{E}_{\boldsymbol{\zeta}_k} \Big[ \Big( \delta_m^k(\boldsymbol{\theta}^k) - \delta_m^{k}(\boldsymbol{\theta}^{k-1})\Big)^2\mid\mathcal F_k\Big] \\
            & \le(1-\beta)^2\cdot\mathbb{E}_{\boldsymbol{\zeta}_k}\Big[|y_{1,m}^k-\phi_{m,Q}^{(1)}(\boldsymbol{\theta}^{k-1})|^2\mid\mathcal F_k\Big] \\
            & \qquad + 2 C_2^2C_3^2\cdot\|\boldsymbol{\theta}^k-\boldsymbol{\theta}^{k-1}\|^2 + 2\beta^2(C_2-1)^2. 
            \end{aligned}
    \end{equation}
    By Eq.~\eqref{theta-update}, we define the stochastic compositional gradient as 
    \[
        \boldsymbol \nabla^k = \lambda \epsilon \cdot \nabla t_{1,m_k}^k(y_{1,m_k}^{k+1}) \cdot \Big [ \frac{1}{Q} \sum_{q=1}^Q J_{\boldsymbol{\theta}}t_{3,m_k}^{k,q} (\boldsymbol{\theta}^{k}) \Big ]^\top \cdot \nabla t_{2, m_k}^k (\boldsymbol{y}_{2, m_k}^{k+1}). 
    \]   
    By Eq.~\eqref{theta-update}, Proposition~\ref{propos-lip-con}, and Lemma~\ref{lemma-smoothness-F}, we define $C_d:=\lambda\epsilon L_1C_2C_3$ and then we have 
    \[
        \mathbb{E}_{m_k, \zeta_k} \Big [ \|\boldsymbol \nabla^k\|_2^2\mid \mathcal{F}_k \Big ]\le C_d^2,\qquad \mathbb E\|\boldsymbol{\theta}^k-\boldsymbol{\theta}^{k-1}\|_2^2  \le C_d^2\alpha^2. 
    \]
    Define $c_\alpha:=2(C_2C_3)^2C_d^2$ and $c_\beta:=2(C_2-1)^2$
    \begin{equation} \label{Eq-T6-part1}
        \begin{aligned}
        E_{k+1} & =\mathbb{E}_{m_{0:k-1},\boldsymbol{\zeta}_{0:k-1}}\Big[ \mathbb{E}_{\zeta_k}\Big [\sum_{m=1}^M\widehat{p}_m\cdot|y_{1,m}^{k+1}-\phi_{m,Q}^{(1)}(\boldsymbol{\theta}^{k})|^2 \mid \mathcal{F}_k \Big ] \Big] \\
        & \le(1-\beta)^2 E_k + 2 C_2^2C_3^2\cdot\mathbb{E}\|\boldsymbol{\theta}^k-\boldsymbol{\theta}^{k-1}\|^2 + 2\beta^2(C_2-1)^2 \\
        & = (1-\beta)^2 E_k +c_\alpha\alpha^2+c_\beta\beta^2. 
    \end{aligned}
    \end{equation} }

    { \textbf{Step 2. Gradient bias.} We now derive the relationship between $\mathbb{E}[F_Q(\boldsymbol{\theta}^{k+1})]$ and $\mathbb{E}[F_Q(\boldsymbol{\theta}^{k})]$. 
    By the Lipschitz smoothness of $F_Q$ and the update rule in our algorithm, we have
    \[
        \begin{aligned}
        F_Q(\boldsymbol{\theta}^{k+1}) & \le F_Q(\boldsymbol{\theta}^{k}) + \langle \nabla F_Q(\boldsymbol{\theta}^{k}), \boldsymbol{\theta}^{k+1}-\boldsymbol{\theta}^{k}\rangle + \frac{S_F}{2} \| \boldsymbol{\theta}^{k+1}-\boldsymbol{\theta}^{k}\|_2^2 \\
        & \le F_Q(\boldsymbol{\theta}^{k}) - \alpha \langle \nabla F_Q(\boldsymbol{\theta}^{k}), \nabla_k \rangle + \frac{S_F \alpha^2}{2} \|\nabla_k\|^2 . 
        \end{aligned}
    \]
    Define a stochastic gradient estimator as: 
    \[
        \bar{\nabla}_k := \lambda \epsilon \cdot \nabla t_{1,m_k}^k(\phi_{m_k,Q}^{(1)}(\boldsymbol{\theta}^{k}))\cdot \Big [ \frac{1}{Q} \sum_{q=1}^Q J_{\boldsymbol{\theta}}t_{3,m_k}^{k,q} (\boldsymbol{\theta}^{k}) \Big ]^\top \nabla t_{2, m_k}^k (\boldsymbol{y}_{2, m_k}^{k+1}).
    \]
    such that 
    \[  
        \mathbb{E}_{m_k, \zeta_k} \Big [ \bar{\nabla}_k \mid \mathcal{F}_k \Big] = \nabla_{\boldsymbol{\theta}} F_Q(\boldsymbol{\theta}^{k}). 
    \]
    Taking expectations over the algorithmic randomness on both sides, we have
    \begin{equation} \label{Eq-expectation-FQ}
        \begin{aligned}
            \mathbb{E}[F_Q(\boldsymbol{\theta}^{k+1})] & \le \mathbb{E}[F_Q(\boldsymbol{\theta}^{k})] - \alpha \mathbb{E}[ \langle \nabla F_Q(\boldsymbol{\theta}^{k}), \nabla_k \rangle] + \frac{S_F \alpha^2 C_d^2}{2} \\ 
            & = \mathbb{E}[F_Q(\boldsymbol{\theta}^{k})]+ \frac{S_F \alpha^2 C_d^2}{2} -\alpha\mathbb{E}[ \| \nabla F_Q(\boldsymbol{\theta}^{k}) \|^2] \\
            & \qquad - \alpha \mathbb{E}[ \langle \nabla F_Q(\boldsymbol{\theta}^{k}), \nabla_k - \nabla F_Q(\boldsymbol{\theta}^{k}) \rangle] \\
            & \le \mathbb{E}[F_Q(\boldsymbol{\theta}^{k})]+ \frac{S_F \alpha^2 C_d^2}{2} -\frac{\alpha}{2} \mathbb{E}[ \| \nabla F_Q(\boldsymbol{\theta}^{k}) \|^2] \\
            & \qquad + \frac{\alpha}{2} \mathbb{E} \| \mathbb{E} [\nabla_k | \mathcal{F}_k] - \nabla F_Q(\boldsymbol{\theta}^{k})\|^2 , 
        \end{aligned}
    \end{equation}
    where the last inequality is due to $-2 \langle \boldsymbol{a}, \boldsymbol{b}\rangle \le \|\boldsymbol{a}\|^2 + \|\boldsymbol{b}\|^2$. 
    We next analyze the last term in Eq.~\eqref{Eq-expectation-FQ}. By the definition of $\bar{\nabla}_k$, we have 
    \[
    \begin{aligned}
        \mathbb{E}\Big\| &\mathbb{E}[\nabla_k\mid\mathcal{F}_k]-\nabla F_Q(\boldsymbol{\theta}^k)\Big\|^2
        =\mathbb{E}\Big\|\mathbb{E}_{m_k,\zeta_k}[\nabla_k-\bar{\nabla}_k\mid\mathcal{F}_k]\Big\|^2\\
        &\le\mathbb{E}\Big[\sum_{m=1}^M\widehat p_m\Big\|\mathbb{E}_{\zeta_k}\Big[\lambda\epsilon\cdot\Big(\nabla t_{1,m}^k(\phi_{m,Q}^{(1)}(\boldsymbol{\theta}^k))-\nabla t_{1,m}^k(y_{1,m}^{k+1})\Big)\cdot\nabla H_{m,Q}(\boldsymbol{\theta}^k,\zeta_k)\mid\mathcal{F}_k\Big]\Big\|^2\Big]\\
        &\le (\lambda\epsilon)^2\cdot\mathbb{E}\Big[\sum_{m=1}^M\widehat p_m\,\mathbb{E}_{\zeta_k}\Big[\Big|\nabla t_{1,m}^k(\phi_{m,Q}^{(1)}(\boldsymbol{\theta}^k))-\nabla t_{1,m}^k(y_{1,m}^{k+1})\Big|^2\mid\mathcal{F}_k\Big]\cdot\mathbb{E}_{\zeta_k}\Big[\|\nabla H_{m,Q}(\boldsymbol{\theta}^k,\zeta_k)\|^2\mid\mathcal{F}_k\Big]\Big]\\
        &\le C_d^2\cdot\mathbb{E}\Big[\sum_{m=1}^M\widehat p_m\,\mathbb{E}_{\zeta_k}\Big[|y_{1,m}^{k+1}-\phi_{m,Q}^{(1)}(\boldsymbol{\theta}^k)|^2\mid\mathcal{F}_k\Big]\Big]
        =C_d^2\cdot E_{k+1},
    \end{aligned}
    \] where $C_d$ is previously defined as $C_d = \lambda\epsilon L_1 C_2 C_3$, last inequality follows from the Lipschitz smoothness of $t_{1,m}$ and the conditional second-moment bound of $\nabla_{\boldsymbol\theta}H_{m,Q}$. 
    Therefore, we have 
    \begin{equation} \label{Eq-T6-part2}
        \begin{aligned}
            \mathbb{E}[F_Q(\boldsymbol{\theta}^{k+1})] &\le \mathbb{E}[F_Q(\boldsymbol{\theta}^{k})]+ \frac{S_F \alpha^2 C_d^2}{2} -\frac{\alpha}{2} \mathbb{E}[ \| \nabla F_Q(\boldsymbol{\theta}^{k}) \|^2] + \frac{\alpha C_d^2 }{2} \cdot E_{k+1}. 
        \end{aligned}
    \end{equation}} 

    {\textbf{Step 3. Lyapunov descent. } 
    Let the finite maximum value of $F_Q$ be $F_{\max}$. 
    Since both functions $F_Q$ and $E_k$ are non-negative, we define the Lyapunov function 
    \[
        V_k := \mathbb{E}_{m_{0:k-1},\boldsymbol{\zeta}_{0:k-1}}[F_Q(\boldsymbol{\theta}^{k})] + E_k \ge 0, \quad \forall k \in [K-1], 
    \]
    where $V_0$ is a constant and is bounded by the constant $V$
    \[
        V: =  F_{\max} + (C_2-1)^2 \ge F_{\max} + E_0 \ge V_0. 
    \]
    Then, combining Eqs.~\eqref{Eq-T6-part1} and~\eqref{Eq-T6-part2} and selecting $0 < \alpha C_d^2/2 \le \beta \le 1$, $K \ge C_d^4/4$, we have 
    \begin{equation}\label{Eq-vk-iteration}
        \begin{aligned}
            V_{k+1} & = V_k + \mathbb{E}[F_Q(\boldsymbol{\theta}^{k+1})] -\mathbb{E}[F_Q(\boldsymbol{\theta}^{k})] + E_{k+1} - E_k \\
            & \le V_k + \frac{S_F \alpha^2 C_d^2}{2} -\frac{\alpha}{2} \mathbb{E}[ \| \nabla F_Q(\boldsymbol{\theta}^{k}) \|^2] \\
            & \qquad + (\beta +1) ((1-\beta)^2 E_k +c_\alpha\alpha^2+c_\beta\beta^2) - E_k \\
            & \le V_k -\frac{\alpha}{2} \mathbb{E}[ \| \nabla F_Q(\boldsymbol{\theta}^{k}) \|^2] +c_\alpha^\prime\alpha^2+2c_\beta\beta^2
        \end{aligned}
    \end{equation}
    where in the last inequality, $c_\alpha^\prime: = 2c_\alpha+ \frac{S_F C_d^2}{2}$ and the upper bound of the coefficient of $E_k$ is negative when $\beta \in (0, 1]$, i.e., $(\beta +1) (1-\beta)^2-1 < 0$. 
    By Eq.~\eqref{Eq-vk-iteration}, and $V_K \ge 0$, we obtain 
    \begin{equation}
        \frac{1}{K}\sum_{k=0}^{K-1}\mathbb{E}[ \| \nabla F_Q(\boldsymbol{\theta}^{k}) \|^2] \le \frac{2V}{K\alpha}+2c_\alpha^\prime\alpha + \frac{4c_\beta\beta^2}{\alpha}. 
    \end{equation} 
    We choose a proper constant $A_1>0$ independent of $K$ and $Q$, $\alpha = \frac{A_1}{\sqrt{K}}$, and $\beta = \frac{A_1 C_d^2}{2\sqrt{K}}$, and thus
    \[
        \frac{1}{K}\sum_{k=0}^{K-1}\mathbb{E}[ \| \nabla F_Q(\boldsymbol{\theta}^{k}) \|^2] \le \frac{C_{\text{con}}}{\sqrt{K}}, 
    \]
    where $C_{\text{con}} = 2V/A_1 + 2A_1c_\alpha^\prime + A_1c_\beta c_d^4$. 
    Using the triangle inequality, we have 
    \begin{equation}\nonumber
    \begin{aligned}
    \|\nabla_{\boldsymbol{\theta}}F(\boldsymbol{\theta}^k)\|^2
    &\le2\|\nabla_{\boldsymbol{\theta}}F_Q(\boldsymbol{\theta}^k)\|^2+2\|\nabla_{\boldsymbol{\theta}}F(\boldsymbol{\theta}^k)-\nabla_{\boldsymbol{\theta}}F_Q(\boldsymbol{\theta}^k)\|^2\\
    &\le2\|\nabla_{\boldsymbol{\theta}}F_Q(\boldsymbol{\theta}^k)\|^2+\frac{2C_{\mathrm{bias}}^2}{Q^2}, 
    \end{aligned}
    \end{equation} and it follows that 
    \begin{equation} \label{Eq-T6-KQ}
        \frac{1}{K}\sum_{k=0}^{K-1}\mathbb{E} \|\nabla_{\boldsymbol{\theta}}F(\boldsymbol{\theta}^k)\|^2 \le \frac{2C_{\text{con}}}{\sqrt{K}}+\frac{2C_{\mathrm{bias}}^2}{Q^2}. 
    \end{equation} 
    When we set the batch size $Q = \left \lceil  K^{1/4} \right \rceil$, and define $C_{\text{const}} = 2C_{\text{con}} + 2C_{\mathrm{bias}}^2$, we further have 
    \begin{equation} 
        \frac{1}{K}\sum_{k=0}^{K-1}\mathbb{E} \|\nabla_{\boldsymbol{\theta}}F(\boldsymbol{\theta}^k)\|^2 \le \frac{C_{\text{const}}}{\sqrt{K}}. 
    \end{equation} 
    This completes the proof of Theorem~\ref{theo-scsc-conver}\ref{theo-scsc-conver-1}. }

    For Theorem~\ref{theo-scsc-conver}\ref{theo-scsc-conver-2}, to ensure that $\mathbb{E}\Big[\|\nabla F(\widehat{\boldsymbol{\theta}})\|^2\Big] \le \varepsilon^2$, according to~\citet{chen2021solving}, it follows that
    \begin{equation}\nonumber
        \mathbb{E} \|\nabla_{\boldsymbol{\theta}}F(\widehat{\boldsymbol{\theta}})\|^2 \le \frac{\sum_{k=0}^{K-1} \mathbb{E}\Big[\|\nabla F(\boldsymbol{\theta}^k)\|^2\Big]}{K} \le \frac{C_{\textnormal{const}}}{\sqrt{K}} \le \varepsilon^2,
    \end{equation}
    This implies that the number of iterations required satisfies 
    \begin{equation}\nonumber
        K \ge \frac{C_{\textnormal{const}}^2}{\varepsilon^4} = \mathcal{O}(\varepsilon^{-4}).
    \end{equation}
    {For the sample complexity, in each iteration, we draw one sample from $\widehat{\mathbb{P}}_{\widehat{\boldsymbol{X}}}$ and $U_{\epsilon}$, respectively, and draw $Q = \left \lceil  K^{1/4} \right \rceil$ samples from $W_\epsilon$. Therefore, to obtain an $\varepsilon$-stationary point, the number of samples required is of order
    \[
        \mathcal{O}(2KQ) = \mathcal{O}(\varepsilon^{-4} \cdot \varepsilon^{-1}) =  \mathcal{O}(\varepsilon^{-5}). 
    \] 
    The gradient complexity, i.e., the number of gradient computations to obtain an $\varepsilon$-stationary point, is given by 
    \[
        \mathcal{O}(K(Q+2)) = \mathcal{O}(KQ)=  \mathcal{O}(\varepsilon^{-5}). 
    \]}
    This completes the proof. 
\end{proof}

\subsection{Analysis for Remark~\ref{remark-limit} in Appendix~\ref{app-sec-disc-csdro}} \label{ecsec-remark-limit}

For regularization parameter $\epsilon > 0$, define $\tau := 1/\epsilon > 0$. 
We also define 
\[
    H_{\widehat{\boldsymbol{x}},\widehat{\boldsymbol{y}}} (\boldsymbol{x}, \boldsymbol{y}) := \Psi(f(\boldsymbol{x}),\boldsymbol{y}) - \lambda c_p((\widehat{\boldsymbol{x}},\widehat{\boldsymbol{y}}), (\boldsymbol{x},\boldsymbol{y})). 
\]
By the log-Laplace principle, for any fixed $(\widehat{\boldsymbol{x}},\widehat{\boldsymbol{y}}, \boldsymbol{x})$, when $\epsilon \to 0$, we have 
\begin{equation}\nonumber
\begin{aligned}
     & \lim_{\epsilon \to 0}\, g(\widehat{\boldsymbol{x}}, \boldsymbol{x}, \lambda, \epsilon) = \lim_{\tau \to \infty}\, g(\widehat{\boldsymbol{x}}, \boldsymbol{x}, \lambda, \frac{1}{\tau}) \\&= \lim_{\tau \to \infty}\, \mathbb{E}_{\widehat{\boldsymbol{y}} \sim  \widehat{\mathbb{P}}_{\widehat{\boldsymbol{Y}}|\widehat{\boldsymbol{X}}=\widehat{\boldsymbol{x}}}}\left[ \frac{\lambda}{\tau} \log\,  \int _{\mathcal{Y}} \exp\left( \frac{\tau \cdot H_{\widehat{\boldsymbol{x}},\widehat{\boldsymbol{y}}} (\boldsymbol{x}, \boldsymbol{y})}{\lambda} \right) \mathrm{d} \nu_{\mathcal{Y}}(\boldsymbol{y}) \right] \\
     & = \mathbb{E}_{\widehat{\boldsymbol{y}} \sim  \widehat{\mathbb{P}}_{\widehat{\boldsymbol{Y}}|\widehat{\boldsymbol{X}}=\widehat{\boldsymbol{x}}}}\left[ \lambda \lim_{\tau \to \infty}\,  \frac{1} {\tau} \cdot \log\,  \int _{\mathcal{Y}} \exp\left( \frac{\tau \cdot H_{\widehat{\boldsymbol{x}},\widehat{\boldsymbol{y}}} (\boldsymbol{x}, \boldsymbol{y})}{\lambda} \right) \mathrm{d} \nu_{\mathcal{Y}}(\boldsymbol{y}) \right] \\
     & = \mathbb{E}_{\widehat{\boldsymbol{y}} \sim  \widehat{\mathbb{P}}_{\widehat{\boldsymbol{Y}}|\widehat{\boldsymbol{X}}=\widehat{\boldsymbol{x}}}}\left[ \operatorname*{ess\,sup}_{\boldsymbol{y}\sim\nu_{\mathcal{Y}}}
     \,  \Big\{H_{\widehat{\boldsymbol{x}},\widehat{\boldsymbol{y}}} (\boldsymbol{x}, \boldsymbol{y})\Big\} \right] \\
     & = \mathbb{E}_{\widehat{\boldsymbol{y}} \sim  \widehat{\mathbb{P}}_{\widehat{\boldsymbol{Y}}|\widehat{\boldsymbol{X}}=\widehat{\boldsymbol{x}}}}\left[ \sup_{\boldsymbol{y} \in \text{supp}\, \nu_{\mathcal{Y}}}
     \,  \Big\{H_{\widehat{\boldsymbol{x}},\widehat{\boldsymbol{y}}} (\boldsymbol{x}, \boldsymbol{y})\Big\} \right], 
\end{aligned}
\end{equation}
where the interchange of the limit and the conditional expectation in the third equality is because $\widehat{\mathbb P}_{\widehat{\boldsymbol{y}}\mid \widehat{\boldsymbol{x}}=\widehat{\boldsymbol{x}}}$ has finite support, and the final equality is due to our lower semicontinuous assumption on $H_{\widehat{\boldsymbol{x}},\widehat{\boldsymbol{y}}} (\boldsymbol{x}, \boldsymbol{y})$. 
Then, for Eq.~\eqref{causal-sdro-dual}, similarly, applying the Laplace principle to the log-partition function yields
\begin{equation}\nonumber
\begin{aligned}
       & \lim_{\epsilon \to 0}\,\mathbb{E}_{\widehat{\boldsymbol{x}} \sim \widehat{\mathbb{P}}_{\widehat{\boldsymbol{X}}}} \left[ \lambda\epsilon \log\, 
        \int_{\mathcal{X}} \exp\left( \frac{g(\widehat{\boldsymbol{x}}, \boldsymbol{x}, \lambda, \epsilon)}{\lambda\epsilon} \right) \mathrm{d} \nu_{\mathcal{X}}(\boldsymbol{x}) \right] \\
        = & \mathbb{E}_{\widehat{\boldsymbol{x}} \sim \widehat{\mathbb{P}}_{\widehat{\boldsymbol{X}}}} \left[ \lim_{\tau \to \infty}\, \frac{\lambda}{\tau} \log\, 
        \int_{\mathcal{X}} \exp\left( \frac{\tau \cdot g(\widehat{\boldsymbol{x}}, \boldsymbol{x}, \lambda, \frac{1}{\tau})}{\lambda} \right) \mathrm{d} \nu_{\mathcal{X}}(\boldsymbol{x}) \right] \\
        = &  \mathbb{E}_{\widehat{\boldsymbol{x}} \sim\widehat{\mathbb{P}}_{\widehat{\boldsymbol{X}}}}\left[ \operatorname*{ess\,sup}_{\boldsymbol{x}\sim\nu_{\mathcal{X}}}\Big\{  \lim_{\tau \to \infty}\, g(\widehat{\boldsymbol{x}}, \boldsymbol{x}, \lambda, \frac{1}{\tau})\Big\}  \right]\\
        = &  \mathbb{E}_{\widehat{\boldsymbol{x}} \sim\widehat{\mathbb{P}}_{\widehat{\boldsymbol{X}}}}\left[\sup_{\boldsymbol{x} \in \text{supp}\,\nu_{\mathcal{X}}}\Big\{  \lim_{\tau \to \infty}\, g(\widehat{\boldsymbol{x}}, \boldsymbol{x}, \lambda, \frac{1}{\tau})\Big\}  \right]\\
        = &  \mathbb{E}_{\widehat{\boldsymbol{x}} \sim\widehat{\mathbb{P}}_{\widehat{\boldsymbol{X}}}}\left[\sup_{\boldsymbol{x} \in \text{supp}\,\nu_{\mathcal{X}}}\Big\{\mathbb{E}_{\widehat{\boldsymbol{y}} \sim  \widehat{\mathbb{P}}_{\widehat{\boldsymbol{Y}}|\widehat{\boldsymbol{X}}=\widehat{\boldsymbol{x}}}}\left[ \sup_{\boldsymbol{y} \in \text{supp}\, \nu_{\mathcal{Y}}}\,  \Big\{H_{\widehat{\boldsymbol{x}},\widehat{\boldsymbol{y}}} (\boldsymbol{x}, \boldsymbol{y})\Big\} \right] \Big\}  \right]. 
\end{aligned}
\end{equation}
When $\text{supp}\,\nu_{\mathcal{X}} = \mathcal{X}$ and $ \text{supp}\,\nu_{\mathcal{Y}} = \mathcal{Y}$, the dual objective function of the Causal-SDRO converges to that of Causal-WDRO in \citet{yang2026decision}.

\subsection{Proof of Proposition~\ref{prop-srt-derivation} in Appendix~\ref{subsec-srf-interpret-1}} \label{ecsec-prop-srt-derivation}

\begin{proof} [Proof of Proposition~\ref{prop-srt-derivation}]
For brevity, let $z_{i,t} := \boldsymbol{w}_{i,t}^{\top}\boldsymbol{x} + b_{i,t}$ and let $s_{i,t} := \text{S}(z_{i,t})$ be the Sigmoid activation at node $i$. 
For an internal node $i$ on route $l$ in tree $t$, the routing probability $\Omega_{i,t}(\boldsymbol{x})$ is defined as:
\begin{equation} \nonumber
    \Omega_{i,t}(\boldsymbol{x}) := \begin{cases} 
        s_{i,t}, & \text{if } l_t \text{ goes left at } i; \\ 
        1 - s_{i,t},  & \text{if } l_t \text{ goes right at } i;
        \end{cases} \quad \quad \forall i \in \Lambda(l_t).
\end{equation} 
Then the route probability is given by $p_{l_t, t} (\boldsymbol{x}) = \prod_{i \in \Lambda (l_t)} \Omega_{i,t}(\boldsymbol{x})$. Taking the logarithm of both sides yields:
\begin{equation} \label{ln-transformation}
    \ln p_{l_t, t} (\boldsymbol{x}) =  \sum_{i \in \Lambda (l_t)} \ln \Omega_{i,t}(\boldsymbol{x}).
\end{equation}
Eq.~\eqref{ln-transformation} is well-defined since $p_{l_t, t} (\boldsymbol{x}) > 0$ and $\Omega_{i,t}(\boldsymbol{x}) > 0$ strictly hold for any $\boldsymbol{x} \in \mathcal{X}$. 

Taking the partial derivative of both sides of Eq.~\eqref{ln-transformation} with respect to feature $x_j$: 
\begin{equation}\nonumber
    \frac{\partial \ln p_{l_t, t} (\boldsymbol{x})}{\partial x_j} = \frac{1}{p_{l_t, t} (\boldsymbol{x})} \frac{\partial p_{l_t, t} (\boldsymbol{x})}{\partial x_j} = \sum_{i \in \Lambda (l_t)} \frac{1}{\Omega_{i,t}(\boldsymbol{x})}  \frac{\partial \Omega_{i,t}(\boldsymbol{x})}{\partial x_j}.
\end{equation}
Recall that the derivative of the Sigmoid function is $s_{i,t}' = s_{i,t}(1-s_{i,t})$, for each node $i$ on route $l_t$ in tree $t$, we define: 
\begin{equation}\nonumber
    \psi_{i,t} := \frac{1}{\Omega_{i,t}(\boldsymbol{x})} \frac{\partial \Omega_{i,t}(\boldsymbol{x})}{\partial z_{i, t}} = \begin{cases} 
        \frac{s_{i,t}(1-s_{i,t})}{s_{i,t}} = 1 - s_{i,t}, & \text{if route } l_t \text{ goes left at } i; \\ 
        \frac{-s_{i,t}(1-s_{i,t})}{1-s_{i,t}} = -s_{i,t}, & \text{if route } l_t \text{ goes right at } i. 
        \end{cases}
\end{equation}
Since $\frac{\partial \Omega_{i,t}}{\partial x_j} = \frac{\partial \Omega_{i,t}}{\partial z_{i,t}} \cdot [\boldsymbol{w}_{i,t}]_j$, we obtain the first-order derivative:
\begin{equation}\label{eq:first-order}
    \frac{\partial p_{l_t, t} (\boldsymbol{x})}{\partial x_j} = p_{l_t, t} (\boldsymbol{x}) \cdot \sum_{i \in \Lambda (l_t)} \psi_{i,t} [\boldsymbol{w}_{i, t}]_j.
\end{equation}

We next compute the second-order derivative $\frac{\partial^2 p_{l_t, t}(\boldsymbol{x})}{\partial x_j\partial x_k}$ by differentiating Eq.~\eqref{eq:first-order} with respect to $x_k$. Applying the product rule yields two terms:
\begin{equation}\nonumber
    \frac{\partial^2 p_{l_t, t}}{\partial x_j\partial x_k} = \underbrace{\frac{\partial p_{l_t, t}}{\partial x_k} \left( \sum_{i \in \Lambda (l_t)} \psi_{i,t} [\boldsymbol{w}_{i, t}]_j \right)}_{\mathbf{A}_1} + \underbrace{p_{l_t, t} \frac{\partial}{\partial x_k} \left( \sum_{i \in \Lambda (l_t)} \psi_{i,t} [\boldsymbol{w}_{i, t}]_j \right)}_{\mathbf{A}_2}, 
\end{equation}
where $\mathbf{A}_1$ is equivalent to 
\begin{equation}\nonumber
    \mathbf{A}_1 = p_{l_t, t} (\boldsymbol{x}) \left( \sum_{i \in \Lambda (l_t)} \psi_{i,t} [\boldsymbol{w}_{i, t}]_k \right) \left( \sum_{i \in \Lambda (l_t)} \psi_{i,t} [\boldsymbol{w}_{i, t}]_j \right).
\end{equation}
For $\mathbf{A}_2$, we note that $\frac{\partial \psi_{i,t}}{\partial x_k} = \frac{\partial \psi_{i,t}}{\partial z_{i,t}} \cdot [\boldsymbol{w}_{i,t}]_k$. For both directions (left or right), the derivative of $\psi_{i,t}$ with respect to $z_{i,t}$ is identical: 
\begin{equation}\nonumber
    \frac{\partial \psi_{i,t}}{\partial z_{i,t}} = \begin{cases} \frac{\partial (1-s_{i,t})}{\partial z_{i,t}} = -s_{i,t}(1-s_{i,t}), & \text{if route } l_t \text{ goes left at } i; \\ \frac{\partial (-s_{i,t})}{\partial z_{i,t}} = -s_{i,t}(1-s_{i,t}), & \text{if route } l_t \text{ goes right at } i. \end{cases}
\end{equation}
Thus we have $\frac{\partial \psi_{i,t}}{\partial x_k} = -s_{i,t}(1-s_{i,t}) [\boldsymbol{w}_{i,t}]_k$, and it follows that 
\begin{equation}\nonumber
    \mathbf{A}_2 = - p_{l_t, t} (\boldsymbol{x}) \sum_{i \in \Lambda (l_t)} s_{i,t}(1-s_{i,t})  [\boldsymbol{w}_{i, t}]_j [\boldsymbol{w}_{i, t}]_k.
\end{equation}
Combining both terms, the second-order derivative is explicitly characterized by:
\begin{equation}\nonumber
\begin{aligned}
    \frac{\partial^2 p_{l_t, t}(\boldsymbol{x})}{\partial x_j\partial x_k} &= p_{l_t, t} (\boldsymbol{x}) \Bigg[ \left( \sum_{i \in \Lambda (l_t)} \psi_{i,t} [\boldsymbol{w}_{i, t}]_j \right) \left( \sum_{i \in \Lambda (l_t)} \psi_{i,t} [\boldsymbol{w}_{i, t}]_k \right) \\
    & \quad \quad \quad \quad - \sum_{i \in \Lambda (l_t)} s_{i,t}(1-s_{i,t}) [\boldsymbol{w}_{i, t}]_j [\boldsymbol{w}_{i, t}]_k \Bigg].
\end{aligned}
\end{equation}
\end{proof}

\subsection{Proof of Proposition~\ref{prop-srf-lip} in Appendix~\ref{subsec-srf-interpret-1}} \label{ecsec-prop-srt-lip}

\begin{proof}[Proof of Proposition~\ref{prop-srf-lip}] 
Here we establish a detailed proof of the Lipschitz properties of SRF with respect to $\boldsymbol{x}$ for fixed $\boldsymbol{\theta}$. 
When the input to SRF is upper-bounded by $R_x$ and the leaf decisions are uniformly bounded, the proof of the Lipschitz properties on $\Theta$ follows from analogous calculations. 
Since the SRF consists of smooth sigmoid compositions, $f_{\boldsymbol{\theta}}^{\textnormal{SRF}}$ is infinitely differentiable. Therefore, we next prove the SRF is $L_{\mathrm{SRF}}$-Lipschitz continuous and $S_{\mathrm{SRF}}$-Lipschitz smooth. 

For Lipschitz continuity, the upper bound of $\| \nabla_{\boldsymbol{x}} f_{\boldsymbol{\theta}}^{\textnormal{SRF}}(\boldsymbol{x})\|$ is given by:
\begin{equation}\label{lip-con-norm}
    \begin{aligned}
        \| \nabla_{\boldsymbol{x}} f_{\boldsymbol{\theta}}^{\textnormal{SRF}}(\boldsymbol{x})\| 
        &= \Big \|  \frac{1}{T} \sum_{t=1}^{T} \sum_{l_t=1}^{2^{D(t)}} \nabla_{\boldsymbol{x}} p_{l_t, t} (\boldsymbol{x}) \cdot \boldsymbol{\pi}_{l_t, t}^{\top}  \Big \|  \le \frac{1}{T} \sum_{t=1}^{T} \sum_{l_t=1}^{2^{D(t)}} \Big \|  \nabla_{\boldsymbol{x}} p_{l_t, t} (\boldsymbol{x}) \cdot \boldsymbol{\pi}_{l_t, t}^{\top}  \Big \| \\ 
        & = \frac{1}{T} \sum_{t=1}^{T} \sum_{l_t=1}^{2^{D(t)}} \Big \| \nabla_{\boldsymbol{x}} p_{l_t, t} (\boldsymbol{x}) \Big \| \cdot \Big \| \boldsymbol{\pi}_{l_t, t} \Big \|,
    \end{aligned}
\end{equation}
where the inequality follows from the triangle inequality, and the second equality follows from the spectral-norm property for outer products. 

We next analyze the gradient of the route probability $p_{l_t, t}(\boldsymbol{x})$. 
Let $W_{\max} := \max_{i,t} \| \boldsymbol{w}_{i,t} \|$ be the maximum norm of gating weights, $\Pi_{\max} := \max_{l_t, t} \|\boldsymbol{\pi}_{l_t, t}\|$ be the maximum norm of leaf vectors, and $D_{\max}$ be the maximum tree depth. 
According to Proposition~\ref{prop-srt-derivation}, we have 
\begin{equation}\label{p-gradient-norm}
    \begin{aligned}
        \Big \| \nabla_{\boldsymbol{x}} p_{l_t, t} (\boldsymbol{x}) \Big \|  &=   \Big \| p_{l_t, t} (\boldsymbol{x}) \cdot \sum_{i \in \Lambda (l_t)} \psi_{i,t} \boldsymbol{w}_{i, t} \Big \| \le p_{l_t, t} (\boldsymbol{x}) \cdot \sum_{i \in \Lambda (l_t)} \Big \| \psi_{i,t} \boldsymbol{w}_{i, t} \Big \|  \le p_{l_t, t} (\boldsymbol{x}) \cdot \sum_{i \in \Lambda (l_t)} \Big \| \boldsymbol{w}_{i, t} \Big \| \\
        & \le p_{l_t, t} (\boldsymbol{x}) \cdot D_{\max} \cdot W_{\max}.
    \end{aligned}
\end{equation}
Here, the first inequality follows from the triangle inequality, and the second follows from $\psi_{i,t} \in (-1, 1)$. The final inequality is because $\Big| \Lambda (l_t)\Big| \le D_{\max}$ for all $l_t \in [2^{D(t)-1]}]$ and $t \in T$. 

Based on Eq.~\eqref{p-gradient-norm}, Eq.~\eqref{lip-con-norm} can be bounded by 
\begin{equation} \nonumber
    \begin{aligned}
        \| \nabla_{\boldsymbol{x}} f_{\boldsymbol{\theta}}^{\textnormal{SRF}}(\boldsymbol{x})\| 
        & \le \frac{1}{T} \sum_{t=1}^{T} \sum_{l_t=1}^{2^{D(t)}} p_{l_t, t} (\boldsymbol{x}) \cdot D_{\max} \cdot W_{\max} \cdot \Pi_{\max} \\
        & = D_{\max} \cdot W_{\max} \cdot \Pi_{\max} \\
        & =L^{\text{SRF}},
    \end{aligned}
\end{equation}
where the equality is due to $\sum_{l_t=1}^{2^{D(t)}} p_{l_t, t} (\boldsymbol{x})  = 1$. 
Thus, $f_{\boldsymbol{\theta}}^{\textnormal{SRF}}(\boldsymbol{x})$ is Lipschitz continuous in $\boldsymbol{x}$. 

For Lipschitz smoothness, we bound the spectral norm of the Hessian of the route probability: 
\begin{equation}\label{Hessian-p-2norm}
    \begin{aligned}
        \Big \| \nabla_{\boldsymbol{x}}^2 p_{l_t, t} (\boldsymbol{x}) \Big \|  
        & = \Bigg \| \, p_{l_t, t} (\boldsymbol{x}) \cdot \Bigg [ \Big( \sum_{i \in \Lambda (l_t)}\psi_{i,t} \boldsymbol{w}_{i, t} \Big)\Big( \sum_{i \in \Lambda (l_t)}\psi_{i,t} \boldsymbol{w}_{i, t} \Big)^{\top} \, \, - \\
            & \quad \quad \quad \quad \quad \sum_{i \in \Lambda (l_t)} \text{S}\Big( (\boldsymbol{w}_{i,t})^{\top}\boldsymbol{x} + b_{i,t}\Big) \Big( 1- \text{S}\Big((\boldsymbol{w}_{i,t})^{\top}\boldsymbol{x} + b_{i,t}\Big)  \Big) \boldsymbol{w}_{i, t} \boldsymbol{w}_{i, t}^{\top}  \Bigg ] \Bigg \|  \\
        & \le p_{l_t, t}(\boldsymbol{x}) \cdot \Bigg [  \Big \| \Big( \sum_{i \in \Lambda (l_t)} \psi_{i,t}  \boldsymbol{w}_{i, t} \Big) \Big( \sum_{i \in \Lambda (l_t)} \psi_{i,t} \boldsymbol{w}_{i, t} \Big)^{\top}  \Big \| + \\
            & \quad \quad \quad \quad \quad \quad  \Big \| \sum_{i \in \Lambda (l_t)} \text{S}\Big( (\boldsymbol{w}_{i,t})^{\top}\boldsymbol{x} + b_{i,t}\Big) \Big( 1- \text{S}\Big((\boldsymbol{w}_{i,t})^{\top}\boldsymbol{x} + b_{i,t}\Big)  \Big) \boldsymbol{w}_{i, t} \boldsymbol{w}_{i, t}^{\top}  \Big \| \Bigg ] \\
        & \le p_{l_t, t}(\boldsymbol{x}) \cdot \Bigg [  \Big \| \Big( \sum_{i \in \Lambda (l_t)}  \boldsymbol{w}_{i, t} \Big) \Big( \sum_{i \in \Lambda (l_t)} \boldsymbol{w}_{i, t} \Big)^{\top}  \Big \| + \frac{1}{4} \Big \| \sum_{i \in \Lambda (l_t)} \boldsymbol{w}_{i, t} \boldsymbol{w}_{i, t}^{\top}  \Big \| \Bigg ] \\
        & \le p_{l_t, t}(\boldsymbol{x}) \cdot \Bigg [  \sum_{i \in \Lambda (l_t)} \Big \| \boldsymbol{w}_{i, t}  \Big \| \cdot \sum_{i \in \Lambda (l_t)} \Big \| \boldsymbol{w}_{i, t}  \Big \| + \frac{1}{4} \sum_{i \in \Lambda (l_t)} \Big \|  \boldsymbol{w}_{i, t} \Big \| \Big \|  \boldsymbol{w}_{i, t} \Big \| \Bigg ] \\
        & \le p_{l_t, t}(\boldsymbol{x}) \cdot \Bigg [ D_{\max}^2\cdot W_{\max}^2 + \frac{1}{4} D_{\max} \cdot W_{\max}^2 \Bigg ] , 
    \end{aligned}
\end{equation}
where the first inequality is due to the triangle inequality, the second inequality is due to the range of $\psi_{i,t}$ and the fundamental inequality, and the third from both the triangle inequality and the property of the spectral norm for outer products. 
Thus, similar to Eq.~\eqref{lip-con-norm}, the upper bound of the gradient of $\nabla_{\boldsymbol{x}} f_{\boldsymbol{\theta}}^{\textnormal{SRF}}(\boldsymbol{x})$ is given by
\begin{equation}\nonumber
    \begin{aligned}
        \| \nabla_{\boldsymbol{x}}^2 f_{\boldsymbol{\theta}}^{\textnormal{SRF}}(\boldsymbol{x})\| 
        &= \Big \|  \frac{1}{T} \sum_{t=1}^{T} \sum_{l_t=1}^{2^{D(t)}} \nabla_{\boldsymbol{x}}^2 p_{l_t, t} (\boldsymbol{x}) \cdot \boldsymbol{\pi}_{l_t, t}^{\top}  \Big \| \\
        & \le \frac{1}{T} \sum_{t=1}^{T} \sum_{l_t=1}^{2^{D(t)}} \Big \| \nabla_{\boldsymbol{x}}^2 p_{l_t, t} (\boldsymbol{x}) \Big \| \cdot \Big \| \boldsymbol{\pi}_{l_t, t} \Big \| \\
        & \le \frac{1}{T} \sum_{t=1}^{T} \sum_{l_t=1}^{2^{D(t)}}  p_{l_t, t}(\boldsymbol{x}) \cdot \Bigg [ D_{\max}^2\cdot W_{\max}^2 + \frac{1}{4} D_{\max} \cdot W_{\max}^2 \Bigg ]  \cdot \Pi_{\max} \\
        & = D_{\max}(D_{\max}+\frac{1}{4}) \cdot W_{\max}^2 \cdot \Pi_{\max} \\
        & = S^{\text{SRF}}, 
    \end{aligned}
\end{equation}
where the second inequality is due to the definition of $\Pi_{\max}$ and Eq.~\eqref{Hessian-p-2norm}. 
Therefore, $f_{\boldsymbol{\theta}}^{\textnormal{SRF}}(\boldsymbol{x})$ is $S^{\text{SRF}}$-Lipschitz smoothness. 
\end{proof}

\subsection{Proof of Theorem~\ref{theo-saa-sample} in Appendix~\ref{app-subsec-algo-saa}}  \label{ecsec-theo-saa}

As an essential part of sample complexity analysis, we first introduce the following lemmas. Based on Cram$\acute{\textnormal{e}}$r's large deviations theorem, we introduce the following Lemma~\ref{lem-cramer}. 
\begin{lemma}\label{lem-cramer}
\textnormal{\textbf{\citep[Cram$\acute{\textnormal{e}}$r's Inequality,][]{kleywegt2002sample}.}}
Let $X_1, \dots, X_n$ be i.i.d. samples of a zero-mean random variable $X$ with finite variance $\sigma^2$. For any $\delta > 0$, it holds
\begin{equation}\nonumber
    \mathbb{P}\left(\frac{1}{n}\sum_{i=1}^{n}X_i \ge \delta\right) \le \exp(-nI(\delta)),
\end{equation}
where $I(\delta) := \sup_{t \in \mathbb{R}}\{t\delta - \log\,  M(t)\}$ is the rate function of random variable $X$, and $M(t) := \mathbb{E}e^{tX}$ is the moment generating function of $X$. For any $\kappa > 0$, there exists $\delta_1 > 0$, for any $\delta \in (0, \delta_1)$, $I(\delta) \ge \frac{\delta^2}{(2+\kappa)\sigma^2}$. 
\end{lemma}
\citet{hu2020sample} extend Cram$\acute{\textnormal{e}}$r's Inequality from random variables to random vectors, as shown in Lemma~\ref{lem-concentration}.
\begin{lemma} \label{lem-concentration}
    \textnormal{\textbf{\citep[Concentration Inequality,][]{hu2020sample}.}} Let $\boldsymbol{X}_1, \dots, \boldsymbol{X}_N$ be i.i.d. samples of a zero-mean random vector $\boldsymbol{X} \in \mathbb{R}^k$ with finite variance $\mathbb{E}\|\boldsymbol{X}\|^2 = \sigma^2 < \infty$. Then for any $\kappa > 0$, there exists $\delta_1 > 0$ such that for any $\delta \in (0, \delta_1)$, it holds that
    \begin{equation}\nonumber 
        \mathrm{Pr}\left(\left\|\frac{1}{N}\sum_{i=1}^{N}\boldsymbol{X}_i\right\| \ge \delta\right) \le 2k \exp\left(-\frac{N\delta^2}{(2+\kappa)\sigma^2}\right).
    \end{equation}
\end{lemma}

Using Lemma~\ref{lem-concentration}, we present the following Lemma~\ref{eclem-prsup}. 

\begin{lemma}\label{eclem-prsup}
    Under Assumption~\ref{assumption-2}, for any $\kappa > 0 $, there exists an $\delta_1>0$ such that for any $\delta \in \left(0, \delta_1 \right) $, it holds that 
    \begin{equation}\nonumber
    \begin{aligned}
        & \mathrm{Pr} \left ( \sup_{\boldsymbol{\theta} \in \Theta} \left| \widehat{F}_{N_1,N_2,N_3} \left ( \boldsymbol{\theta}  \right ) - F \left ( \boldsymbol{\theta} \right )\right|  > \delta\right )  \\
        \le & \mathcal{O}\left ( 1 \right )  \left ( \frac{4L_1L_2L_3D_{\Theta}}{\delta} \right )^{d_{\theta}} \Big( N_1^2 N_2 \exp\Big( -\frac{N_3 \delta^2}{36(2+\kappa)\lambda^2\epsilon ^2L_1^2L_2^2\sigma_3^2} \Big) \\
        & \quad \quad \quad \quad + N_1\exp\Big( -\frac{N_2 \delta^2}{36(2+\kappa)\lambda^2\epsilon ^2L_1^2\sigma_2^2}\Big) + \exp\Big(-\frac{N_1 \delta^2}{36(2+\kappa)\lambda^2\epsilon ^2\sigma_1^2}\Big) \Big).
    \end{aligned}
    \end{equation}
\end{lemma}
\begin{proof}[Proof of Lemma~\ref{eclem-prsup}]
    For $\upsilon \in (0, 1)$, the set $\{\boldsymbol{x}_l\}_{l=1}^Q$ is said to be a $\upsilon$-net of $\mathcal{X}$, if $\boldsymbol{x}_l \in \mathcal{X}$, $\forall l = 1, \dots, Q$, and the following holds: $\forall \boldsymbol{x} \in \mathcal{X}, \exists l(\boldsymbol{x}) \in \{1, \dots, Q\}$ such that $\|\boldsymbol{x} - \boldsymbol{x}_{l(\boldsymbol{x})}\| \le \upsilon$. 
    We construct a $\upsilon$-net to eliminate the supremum over $\boldsymbol{\theta}$ and apply a concentration inequality to bound the probability. First, we pick a $\upsilon$-net $\left \{ \boldsymbol{\theta}_{l} \right \}_{l=1}^{Q}$ on the decision set $\Theta \subseteq \mathbb{R}^{d_{\theta}}$, such that $L_1L_2L_3\upsilon = \delta/4$. Under Assumption~\ref{assumption-2-theta}, $\Theta$ has a finite diameter $D_{\Theta}$, for any $\upsilon \in (0, 1)$, there exists a $\upsilon$-net of $\Theta$, and the size of the $\upsilon$-net is bounded, $Q \le \mathcal{O}((D_{\mathcal{\Theta}}/\upsilon)^{d_{\theta}})$~\citep{shapiro2021lectures}. 
    By definition of $\upsilon$-net, we have $\forall \boldsymbol{\theta} \in \Theta, \exists l(\boldsymbol{\theta}) \in \{1, 2, \dots, Q\}$, s.t.
    \begin{equation*}
        \|\boldsymbol{\theta}  - \boldsymbol{\theta} _{l(\boldsymbol{\theta} )}\| \le v = \frac{\delta}{4L_1 L_2 L_3}.
    \end{equation*}

    Based on Proposition~\ref{propos-lip-con}, we have
    \begin{equation}\nonumber
        \left| \widehat{F}_{N_1,N_2,N_3}\left ( \boldsymbol{\theta}  \right ) - \widehat{F}_{N_1,N_2,N_3} \left ( \boldsymbol{\theta} _{l(\boldsymbol{\theta} )} \right ) \right| \le L_1L_2L_3 \|\boldsymbol{\theta}  - \boldsymbol{\theta} _{l(\boldsymbol{\theta} )}\| \le \frac{\delta}{4}, 
    \end{equation}
    and
    \begin{equation}\nonumber
        \left| F\left ( \boldsymbol{\theta} _{l(\boldsymbol{\theta} )}  \right ) -F \left ( \boldsymbol{\theta} \right ) \right| \le L_1L_2L_3 \|\boldsymbol{\theta}  - \boldsymbol{\theta} _{l(\boldsymbol{\theta} )}\| \le \frac{\delta}{4}.
    \end{equation}

    Thus, for any $\boldsymbol{\theta} \in \Theta$, we have
    \begin{equation}\nonumber
    \begin{aligned}
       &  \left| \widehat{F}_{N_1,N_2,N_3}\left ( \boldsymbol{\theta}  \right ) -F \left ( \boldsymbol{\theta} \right ) \right| \\
       \le & \left| \widehat{F}_{N_1,N_2,N_3}\left ( \boldsymbol{\theta}  \right ) - \widehat{F}_{N_1,N_2,N_3} \left ( \boldsymbol{\theta} _{l(\boldsymbol{\theta} )} \right ) \right| + \left| \widehat{F}_{N_1,N_2,N_3}\left ( \boldsymbol{\theta} _{l(\boldsymbol{\theta} )}  \right ) -F \left ( \boldsymbol{\theta} _{l(\boldsymbol{\theta} )} \right ) \right| + \left| F\left ( \boldsymbol{\theta} _{l(\boldsymbol{\theta} )}  \right ) -F \left ( \boldsymbol{\theta} \right ) \right| \\
       \le & \frac{\delta}{2} + \left| \widehat{F}_{N_1,N_2,N_3}\left ( \boldsymbol{\theta} _{l(\boldsymbol{\theta} )}  \right ) -F \left ( \boldsymbol{\theta} _{l(\boldsymbol{\theta} )} \right ) \right| 
       \le \frac{\delta}{2} + \max_{l=1, \cdots, Q} \left| \widehat{F}_{N_1,N_2,N_3}\left ( \boldsymbol{\theta} _{l}  \right ) -F \left ( \boldsymbol{\theta} _{l} \right ) \right|  \\
       \le & \frac{\delta}{2} + \sum_{l=1}^{Q} \left| \widehat{F}_{N_1,N_2,N_3}\left ( \boldsymbol{\theta} _{l}  \right ) -F \left ( \boldsymbol{\theta} _{l} \right ) \right| 
    \end{aligned}
    \end{equation}

    It follows that 
    \begin{equation}\label{prove-Delta-3terms}
        \begin{aligned}
            &\mathrm{Pr} \left ( \sup_{\boldsymbol{\theta} \in \Theta} \left| \widehat{F}_{N_1,N_2,N_3} \left ( \boldsymbol{\theta}  \right ) -F \left ( \boldsymbol{\theta} \right )\right|  > \delta\right ) \\ 
            \le &  \mathrm{Pr} \left ( \sum_{l=1}^{Q} \left| \widehat{F}_{N_1,N_2,N_3}\left ( \boldsymbol{\theta} _{l}  \right ) -F \left ( \boldsymbol{\theta} _{l} \right ) \right|   > \frac{\delta}{2}\right ) 
            \le \sum_{l=1}^{Q}  \mathrm{Pr} \left ( \left| \widehat{F}_{N_1,N_2,N_3}\left ( \boldsymbol{\theta} _{l}  \right ) -F \left ( \boldsymbol{\theta} _{l} \right ) \right|   > \frac{\delta}{2}\right ) \\
            \le & \sum_{l=1}^{Q}  \mathrm{Pr} \left ( \left| \widehat{F}_{N_1,N_2,N_3}\left ( \boldsymbol{\theta} _{l}  \right ) -\widehat{F}_{N_1,N_2} \left ( \boldsymbol{\theta} _{l} \right ) \right| + \left| \widehat{F}_{N_1,N_2} \left ( \boldsymbol{\theta} _{l}  \right ) - \widehat{F}_{N_1} \left ( \boldsymbol{\theta} _{l} \right ) \right| + \left| \widehat{F}_{N_1}\left ( \boldsymbol{\theta} _{l}  \right ) -F \left ( \boldsymbol{\theta} _{l} \right ) \right|  > \frac{\delta}{2}\right ) \\
            \le & \underbrace{\sum_{l=1}^{Q} \mathrm{Pr} \left ( \left| \widehat{F}_{N_1,N_2,N_3}\left ( \boldsymbol{\theta} _{l}  \right ) -\widehat{F}_{N_1,N_2} \left ( \boldsymbol{\theta} _{l} \right ) \right|  > \frac{\delta}{6}\right )}_{\Delta_1} + \underbrace{\sum_{l=1}^{Q} \mathrm{Pr} \left ( \left| \widehat{F}_{N_1,N_2} \left ( \boldsymbol{\theta} _{l}  \right ) - \widehat{F}_{N_1} \left ( \boldsymbol{\theta} _{l} \right ) \right|  > \frac{\delta}{6}\right )}_{\Delta_2} \\ & \quad + \underbrace{\sum_{l=1}^{Q} \mathrm{Pr} \left ( \left| \widehat{F}_{N_1}\left ( \boldsymbol{\theta} _{l}  \right ) -F \left ( \boldsymbol{\theta} _{l} \right ) \right|  > \frac{\delta}{6}\right )}_{\Delta_3}. 
        \end{aligned}
    \end{equation}
    For the term $\Delta_1$ in Eq.~\eqref{prove-Delta-3terms}, we have
    \begin{equation}
        \begin{aligned}
            \Delta_1 &= \sum_{l=1}^{Q} \mathrm{Pr} \left ( \left| \widehat{F}_{N_1,N_2,N_3}\left ( \boldsymbol{\theta} _{l}  \right ) -\widehat{F}_{N_1,N_2} \left ( \boldsymbol{\theta} _{l} \right ) \right|  > \frac{\delta}{6}\right ) \\
            & = \sum_{l=1}^{Q} \mathrm{Pr} \Big ( \Big|  \frac{\lambda\epsilon}{N_1}  \sum_{i=1}^{N_1} t_1\Big( \frac{1}{N_2} \sum_{j=1}^{N_2} t_2\Big( \frac{1}{N_3} \sum_{k=1}^{N_3}t_3\Big( \boldsymbol{\theta}; \widehat{\boldsymbol{x}}^i, \boldsymbol{\xi}_1^j, \boldsymbol{\xi}_2^k  \Big); \widehat{\boldsymbol{x}}^i, \boldsymbol{\xi}_1^j \Big)\Big) \\
            & \quad \quad \quad \quad \quad -\frac{\lambda\epsilon}{N_1}  \sum_{i=1}^{N_1} t_1\Big( \frac{1}{N_2} \sum_{j=1}^{N_2} t_2\Big( \mathbb{E}_{\boldsymbol{\xi}_2}\Big[t_3\Big( \boldsymbol{\theta}; \widehat{\boldsymbol{x}}^i, \boldsymbol{\xi}_1^j, \boldsymbol{\xi}_2^k  \Big)\Big]; \widehat{\boldsymbol{x}}^i, \boldsymbol{\xi}_1^j \Big)\Big) \left ( \boldsymbol{\theta} _{l} \right ) \Big|  > \frac{\delta}{6}\Big ) . 
        \end{aligned}
    \end{equation}
    By Proposition~\ref{propos-lip-con}, 
    \begin{equation}
        \begin{aligned}
            \Delta_1 &\le \sum_{l=1}^{Q} \ \mathrm{Pr} \Bigg ( \max_{i=1,\cdots, N_1; j = 1, \cdots, N_2} \Big| L_1L_2 \| \frac{1}{N_3}\sum_{k=1}^{N_3}t_3\Big( \boldsymbol{\theta}; \widehat{\boldsymbol{x}}^i, \boldsymbol{\xi}_1^j, \boldsymbol{\xi}_2^k  \Big) - \mathbb{E}_{\boldsymbol{\xi}_2}\Big[t_3\Big( \boldsymbol{\theta}; \widehat{\boldsymbol{x}}^i, \boldsymbol{\xi}_1^j, \boldsymbol{\xi}_2^k  \Big)\Big] \| \Big| > \frac{\delta}{6\lambda\epsilon}\Bigg ) \\
            & \le \sum_{l=1}^{Q} \sum_{i=1}^{N_1} \sum_{j=1}^{N_2} \mathrm{Pr} \left (  \| \frac{1}{N_3}\sum_{k=1}^{N_3}t_3\Big( \boldsymbol{\theta}; \widehat{\boldsymbol{x}}^i, \boldsymbol{\xi}_1^j, \boldsymbol{\xi}_2^k  \Big) - \mathbb{E}_{\boldsymbol{\xi}_2}\Big[t_3\Big( \boldsymbol{\theta}; \widehat{\boldsymbol{x}}^i, \boldsymbol{\xi}_1^j, \boldsymbol{\xi}_2^k  \Big)\Big] \| > \frac{\delta}{6\lambda\epsilon L_1L_2}\right ) \\
            & \le Q N_1 N_2 \cdot 2N_1 \exp\Big( -\frac{N_3 \delta^2}{36(2+\kappa)\lambda^2\epsilon ^2L_1^2L_2^2\sigma_3^2}   \Big), \\
        \end{aligned}
    \end{equation}
    where the first inequality follows from Lipschitz continuity, and the last inequality follows from Lemma~\ref{lem-concentration}. 
    Similarly, we obtain
    \begin{equation}
        \Delta_2 \le QN_1 \cdot 2\exp\Big( -\frac{N_2 \delta^2}{36(2+\kappa)\lambda^2\epsilon ^2L_1^2\sigma_2^2}   \Big),
    \end{equation}
    and 
    \begin{equation}\label{Delta-3}
        \Delta_3 \le Q\cdot 2\exp\Big(-\frac{N_1 \delta^2}{36(2+\kappa)\lambda^2\epsilon ^2\sigma_1^2}  \Big).
    \end{equation}
    Combining with Eqs.~\eqref{prove-Delta-3terms}-~\eqref{Delta-3}, and the fact that $Q \le \mathcal{O}\left ( 1 \right )  \left (4L_1L_2L_3D_{\Theta}/\delta \right )^{d_{\theta}} $, we can obtain the desired result of Lemma~\ref{eclem-prsup}. 
\end{proof}

In the following, we prove the results in Theorem~\ref{theo-saa-sample}. 
\begin{proof}[Proof of Theorem~\ref{theo-saa-sample}]
\begin{enumerate}
    \item For $\mathrm{Pr} \left ( F\left ( \widehat{\boldsymbol{\theta}}_{N_1,N_2,N_3}  \right ) -F \left ( \boldsymbol{\theta}^* \right ) \le \delta \right )$ in Theorem~\ref{theo-saa-sample}\ref{theorem-saa-2}, we have 
    \begin{equation}\label{prove-theo-saa-2}
        \begin{aligned}
            & \mathrm{Pr} \left ( F\left ( \widehat{\boldsymbol{\theta}}_{N_1,N_2,N_3}  \right ) -F \left ( \boldsymbol{\theta}^* \right ) > \delta \right ) \\
            = & \mathrm{Pr} \Big ( \Big[F\left ( \widehat{\boldsymbol{\theta}}_{N_1,N_2,N_3}  \right ) - \widehat{F}_{N_1,N_2,N_3}\left ( \widehat{\boldsymbol{\theta}}_{N_1,N_2,N_3}  \right ) \Big] \\ 
            & \quad + \Big[\widehat{F}_{N_1,N_2,N_3}\left ( \widehat{\boldsymbol{\theta}}_{N_1,N_2,N_3}  \right ) - \widehat{F}_{N_1,N_2,N_3}\left ( \boldsymbol{\theta}^*  \right ) \Big] + \Big[  \widehat{F}_{N_1,N_2,N_3}\left ( \boldsymbol{\theta}^*  \right ) - F \left ( \boldsymbol{\theta}^* \right ) \Big] > \delta \Big ) \\
            \le & \mathrm{Pr} \left ( F\left ( \widehat{\boldsymbol{\theta}}_{N_1,N_2,N_3}  \right ) - \widehat{F}_{N_1,N_2,N_3}\left ( \widehat{\boldsymbol{\theta}}_{N_1,N_2,N_3}  \right ) > \frac{\delta}{2} \right ) + \mathrm{Pr} \left ( \widehat{F}_{N_1,N_2,N_3}\left ( \boldsymbol{\theta}^*  \right ) - F \left ( \boldsymbol{\theta}^* \right ) > \frac{\delta}{2} \right ) \\
            \le & \mathrm{Pr} \left ( \Big| F\left ( \widehat{\boldsymbol{\theta}}_{N_1,N_2,N_3}  \right ) - \widehat{F}_{N_1,N_2,N_3}\left ( \widehat{\boldsymbol{\theta}}_{N_1,N_2,N_3} \right ) \Big| > \frac{\delta}{2} \right ) + \mathrm{Pr} \left ( \Big|\widehat{F}_{N_1,N_2,N_3}\left ( \boldsymbol{\theta}^*  \right ) - F \left ( \boldsymbol{\theta}^* \right )\Big| > \frac{\delta}{2} \right ) \\
        \end{aligned}
    \end{equation}
    where the first inequality is due to $\widehat{F}_{N_1,N_2,N_3}\left ( \widehat{\boldsymbol{\theta}}_{N_1,N_2,N_3}  \right ) - \widehat{F}_{N_1,N_2,N_3}\left ( \boldsymbol{\theta}^*  \right ) \le 0$. Using the result of Lemma~\ref{eclem-prsup}, we can obtain the desired result of Theorem~\ref{theo-saa-sample}\ref{theorem-saa-2} using Eq.~\eqref{prove-theo-saa-2}. 

    \item For Theorem~\ref{theo-saa-sample}\ref{theorem-saa-3}, to analyze 
    \begin{equation}\nonumber
         \mathrm{Pr} \left ( F\left ( \widehat{\boldsymbol{\theta}}_{N_1,N_2,N_3}  \right ) -F \left ( \boldsymbol{\theta}^* \right ) \le \delta \right ) \ge 1-\alpha,
    \end{equation}
    it suffices to study 
    \begin{equation}\nonumber
        \mathrm{Pr} \left ( F\left ( \widehat{\boldsymbol{\theta}}_{N_1,N_2,N_3}  \right ) -F \left ( \boldsymbol{\theta}^* \right ) > \delta \right ) < \alpha.
    \end{equation}
    Let each of the three terms on the right-hand side (RHS) of the inequality in Theorem~\ref{theo-saa-sample}\ref{theorem-saa-2} be no more than $\alpha/3$. This leads to 
    \begin{equation}\nonumber
        \mathcal{O}\left ( 1 \right )  \left ( \frac{8L_1L_2L_3D_{\Theta}}{\delta} \right )^{d_{\theta}}\exp\Big(-\frac{N_1 \delta^2}{144(2+\kappa)\lambda^2\epsilon ^2\sigma_1^2}\Big) \Big) < \frac{\alpha}{3},
    \end{equation}
    and then we obtain the necessary sample size from distribution $\widehat{\mathbb{P}}_{\widehat{\boldsymbol{X}}}$
    \begin{equation}\nonumber
        N_1 > \frac{\mathcal{O}\left ( 1 \right ) \sigma_1^2}{\delta^2} \Big[ d_{\theta}\log\, \left ( \frac{8L_1L_2L_3D_{\Theta}}{\delta}\right ) +\log\, \left ( \frac{1}{\alpha}  \right )    \Big ] .
    \end{equation}
    Similarly, we can obtain the desired result of $N_2$ and $N_3$. 

    Ignoring the log factors, the required sample sizes $N_1$, $N_2$, and $N_3$  are all of order $\mathcal{O}\Big( d_{\theta} / \delta^2 \Big)$. Therefore, the total sample complexity of the problem~\eqref{f-saa} for achieving a $\delta$-optimal solution is $T=N_1+N_2+N_3 = \mathcal{O}\Big( d_{\theta} / \delta^2 \Big)$. 
\end{enumerate} 
This completes the proof. 
\end{proof}

\end{document}